\documentclass[12pt]{article}
\usepackage{flushend}
\usepackage{wrapfig,url}
\usepackage[colorlinks,bookmarksopen,bookmarksnumbered,citecolor=red,urlcolor=red]{hyperref}
\usepackage{wrapfig}
\usepackage[export]{adjustbox}
\usepackage{setspace}
\usepackage{times}
\usepackage{amsmath}
\allowdisplaybreaks[4]

\usepackage[usenames]{color}
\usepackage{enumerate}
\usepackage{url}
\usepackage{subfig}
\usepackage{amsthm}
\usepackage{amsfonts,mathrsfs}
\usepackage{bbm}
\usepackage{amssymb,amsmath}
\usepackage{verbatim}
\usepackage{acronym}
\usepackage{mathtools}
\usepackage{cite}
\usepackage{graphicx}
\usepackage{algorithm}
\usepackage[noend]{algpseudocode}

\newtheorem{theorem}{Theorem}

\newtheorem{assumption}{Assumption}

\newtheorem{corollary}{Corollary}

\newtheorem{definition}{Definition}

\newtheorem{lemma}{Lemma}

\newtheorem{remark}{Remark}

\usepackage[inline,shortlabels]{enumitem} 
\makeatletter
\newcommand{\inlineitem}[1][]{%
\ifnum\enit@type=\tw@
    {\descriptionlabel{#1}}
  \hspace{\labelsep}%
\else
  \ifnum\enit@type=\z@
       \refstepcounter{\@listctr}\fi
    \quad\@itemlabel\hspace{\labelsep}%
\fi}
\makeatother


\usepackage[margin=0.9in]{geometry}
\newlength{\noteWidth}
\long\def\notes#1{\ifinner
           {\footnotesize #1}
           \else
           \marginpar{\parbox[t]{\noteWidth}{\raggedright\footnotesize #1}}
       \fi\typeout{#1}}

\def\argmin{\mathop{\rm argmin}}
 
\newcommand{\E}{\mathbb E}
\newcommand{\one}{\mathbbm 1}
\newcommand{\TV}{\mathrm{TV}}

\title{High-Probability Nash Regret for Decentralized Learning in Markov $\alpha$-Potential Games: Episodic and Fully Online Asynchronous Algorithms with Applications to Markov Congestion Games}

\author{
S. Rasoul Etesami\\
\small Department of Industrial and Systems Engineering\\
\small Coordinated Science Laboratory\\
\small University of Illinois Urbana-Champaign, Urbana, IL, USA\\ 
\small \texttt{etesami1@illinois.edu}
}

\date{}

\begin{document}

\maketitle

\begin{abstract}
We study decentralized learning of Nash equilibria (NE) in infinite-horizon discounted Markov games under bandit feedback, focusing on Markov $\alpha$-potential games. We develop KL-projected natural policy gradient (NPG) algorithms in two settings: an episodic setting with frozen policies during sampling, and a fully online setting in which each player receives a single realized cost sample per time step and policies are updated asynchronously along a continuing trajectory. For both settings, we establish finite-time high-probability bounds on the time-averaged NE gap (NE regret). Crucially, our bounds do not involve a distribution-mismatch coefficient, which can scale prohibitively with the size of the state space. In the episodic setting, we obtain an NE regret bound of order $\widetilde O(T^{-1/4})$, plus separate terms accounting for the potential approximation error $\alpha$, fixed estimation-oracle bias $L_{\widehat{A}}$, and transition-kernel sensitivity to a unilateral action change $\delta_P$. Thus, our framework accommodates multiple sources of approximation error within a unified finite-time analysis. The fully online setting introduces additional challenges from asynchronous state visitation, drifting state occupancies, simultaneous policy adaptation, and one-sample importance-weighted estimation. We address these through stopping-time, coupling, charging, and dynamic-tracking arguments, obtaining an NE regret bound of order $\widetilde O(T^{-2/15})$ plus analogous fixed approximation terms. We further identify a state-wise potential structure yielding sharper guarantees in which the potential approximation error $\alpha$ enters additively, avoiding amplification through online tracking. We then specialize our framework to independent-resource Markov congestion games (IMCGs), a class of dynamic congestion games with independently evolving resource states. We establish their approximate-potential and transition-sensitivity properties, construct decentralized estimation oracles from realized costs, and derive episodic and fully online guarantees. Finally, as an application, we introduce a novel strategic online job-scheduling problem for stochastic machines and show that our framework yields a scalable decentralized algorithm for learning stable dispatching policies. Overall, our results address several gaps in the literature by providing the first finite-time high-probability NE regret guarantees for fully online asynchronous decentralized learning in Markov $\alpha$-potential games, eliminating distribution-mismatch coefficients from the regret bounds, accommodating fixed estimation-oracle bias, and developing scalable decentralized learning algorithms with finite-time NE regret guarantees for IMCGs.

\end{abstract}

\section{Introduction}
\label{sec:introduction}

Learning equilibrium behavior in strategic multi-agent systems is a fundamental problem at the intersection of game theory and machine learning. In a noncooperative system, each player seeks to optimize its own objective while facing an environment that evolves as the other players learn simultaneously. In this regard, Nash equilibrium (NE) provides a natural stability benchmark~\cite{introNash1951}: at an NE, no player can improve its objective through a unilateral deviation. In large-scale systems, practical equilibrium learning requires three properties that are often in tension: \emph{convergence}, \emph{scalability}, and \emph{independence}. Ideally, a learning rule should converge within a controlled time, avoid exponential dependence on the number of players when the game structure permits, and allow each player to learn from local observations without exchanging policy or payoff information with others. These requirements become particularly challenging under bandit feedback, where each player observes only its own realized cost rather than the full cost table.

Markov (or stochastic) games, introduced by Shapley~\cite{introShapley1953}, extend normal-form games to sequential environments in which the state evolves stochastically as a function of the players' joint actions. They model a wide range of dynamic strategic interactions, including wireless communication~\cite{altman2008constrained,qin2023scalable}, transportation and routing systems~\cite{fox2022independent}, energy markets~\cite{etesami2018stochastic,etesami2024learning}, and cyber-physical security systems~\cite{etesami2019dynamic}. Relative to repeated static games, a central difficulty is that current actions affect not only instantaneous costs but also the distribution of future states, coupling strategic interaction with long-run state occupancies. This dynamic coupling, together with the general computational intractability of equilibrium computation~\cite{introDaskalakis2009}, motivates structured game classes in which learning can exploit additional geometry. Potential games~\cite{introMondererShapley1996} provide a canonical example in the static setting, where unilateral changes in individual players' objectives are aligned with changes in a common potential function. Their dynamic counterparts, Markov potential games (MPGs), were introduced in~\cite{leonardos2022global}, where global convergence of independent policy-gradient methods to NE policies was established. Since then, MPGs have become an important model class for provably convergent multi-agent reinforcement learning.

While potential structure facilitates convergence, establishing finite-time performance guarantees under genuinely \emph{independent} learning remains considerably more challenging. In our setting, players exchange neither policies, gradients, nor cost functions, and need not observe the actions of others. Instead, they independently randomize their actions, observe only their own realized costs, and learn from these local samples while interacting within the same evolving Markov environment. From each player's perspective, the resulting environment is inherently nonstationary because all other players adapt simultaneously. This difficulty is amplified in the fully online setting, where the system is not reset between updates, states are visited asynchronously, and policies evolve along a single continuing trajectory. Our goal is therefore to develop finite-time, high-probability guarantees on the \emph{time-averaged NE gap}, which we refer to interchangeably as \emph{NE regret}, under precisely this decentralized, fully online, bandit-feedback regime.

\subsection{Related Literature}
\label{subsec:intro-related-work}
Here, we review the literature most closely related to our work and discuss how our contributions address several gaps in the existing literature.

\smallskip
{\bf Linear-quadratic games.} Linear-quadratic (LQ) control provides an important benchmark for understanding the guarantees and limitations of policy-gradient methods in dynamic games. In the single-agent setting, Fazel et al.~\cite{introFazelLQR2018} established global convergence and polynomial sample and computational complexity for policy-gradient methods in the LQ setting, demonstrating that a nonconvex policy-optimization landscape can nevertheless possess favorable global geometry. The multi-agent setting is substantially more delicate: Mazumdar et al.~\cite{introMazumdar2020} constructed general-sum LQ games in which policy-gradient dynamics need not even converge locally to an NE. Under additional stochastic excitation, Hambly et al.~\cite{introHambly2023} proved global convergence of natural policy gradient to an NE in finite-horizon $N$-player general-sum LQ games. More recently, Hosseinirad et al.~\cite{introHosseinirad2026} characterized when LQ games admit an exact potential structure, showing that exact potentiality can be restrictive under full-state feedback but becomes richer under decoupled dynamics and information structures. Plank and Zhang~\cite{introPlankZhang2026} study distributed LQ stochastic differential games through an $\alpha$-potential framework and establish linear convergence to exact or approximate equilibria depending on the symmetry of interactions. These results demonstrate both the power and limitations of structural assumptions in continuous-state LQ models. Our focus is complementary: we study finite-state, finite-action Markov games under bandit feedback, allow a general approximate Markov-potential structure, and establish high-probability finite-time guarantees for decentralized episodic and fully online learning.

\smallskip
{\bf Markov potential and Markov $\alpha$-potential games.} A substantial recent literature has established increasingly sharp convergence guarantees for policy-gradient-type methods in exact MPGs. Fox et al.~\cite{introFox2022} proved last-iterate asymptotic convergence of independent natural policy gradient (NPG) to an NE policy in MPGs with constant learning rates. Ding et al.~\cite{introDing2022} obtained global nonasymptotic rates for independent policy gradient, including an $O(1/\epsilon^2)$ exact-gradient iteration bound and sample-based guarantees with function approximation. Zhang et al.~\cite{introZhangNeurIPS2022} analyzed decentralized softmax gradient and natural-gradient play, with and without log-barrier regularization, and quantified the role of trajectory-dependent constants. Zhang et al.~\cite{introZhangRenLi2024} further studied stationary points, local stability, and sample complexity for gradient play in stochastic games, obtaining global guarantees for MPGs. Sun et al.~\cite{introSun2023} sharpened the exact-policy-evaluation complexity of independent NPG to $O(1/\epsilon)$ under a suboptimality-gap condition. Alatur et al.~\cite{introAlatur2024} showed that Kullback–Leibler (KL) regularized policy mirror descent can improve the dependence on the number of players. Zhang et al.~\cite{introZhangDecoupled2023} considered finite-horizon Markov games with decoupled agent dynamics, deriving price-of-anarchy results and a distributed soft policy-iteration method with sample-complexity guarantees for the MPG subclass. Dong et al.~\cite{introDong2024} treated stochastic-cost and bandit-feedback potential games and MPGs using a Frank--Wolfe method with exploration and episodic gradient estimation, obtaining $O(T^{4/5})$ Nash regret for MPGs.

Several complementary extensions further illustrate the breadth of independent policy-gradient methods for structured Markov games. Cheng et al.~\cite{introCheng2024} move beyond the discounted setting and establish global convergence of independent policy gradient and NPG for infinite-horizon average-reward MPGs. Given access to a gradient or differential-$Q$ oracle, their policy-gradient, proximal-$Q$, and NPG methods reach an $\epsilon$-NE in $O(1/\epsilon^2)$ iterations; they also provide a sample-based policy-gradient guarantee with $\widetilde O(1/\epsilon^5)$ sample complexity. Jordan et al.~\cite{introJordan2024} study constrained MPGs and develop an independent policy-gradient method for computing approximate constrained NEs based on proximal-point-type updates. Thus, although these works substantially broaden the scope of independent learning in MPGs, their algorithmic and sampling models remain different from the fully online regime considered here, in which each player receives a single realized cost sample per time step and updates \emph{asynchronously}, i.e., updates its policy only at the visited state, along a single evolving trajectory.

The works closest to our fully online information structure are nevertheless materially different. Maheshwari et al.~\cite{introMaheshwari2025} study independent and decentralized learning in infinite-horizon discounted MPGs, where players observe the realized states and their own realized payoffs without knowing the underlying model. Their actor--critic dynamics operate on two timescales, with the $Q$-estimates evolving faster than the policies, and their main guarantee is asymptotic convergence based on asynchronous stochastic approximation. In contrast, our fully online algorithm applies to the more general class of \emph{Markov $\alpha$-potential games}, introduced in~\cite{introGuo2026}, where the change in a player's objective under a unilateral deviation is approximated, up to an error $\alpha$, by the change in a common potential function. Moreover, our algorithm accommodates additional sources of error in the one-sample estimation oracle and provides an explicit high-probability finite-time bound on the NE regret.

Guo et al.~\cite{introGuo2026} showed that Markov $\alpha$-potential games encompass important structured models, including Markov congestion games and perturbed team games, and derived Nash-regret guarantees for projected gradient ascent and sequential maximum-improvement schemes under exact $Q$-function oracle access. More broadly, Guo et al.~\cite{introGuoLiZhang2025} develop an $\alpha$-potential framework for general $N$-player dynamic games, characterize approximate potentiality through asymmetries of second-order derivatives, and specialize it to stochastic differential games with distributed and mean-field interactions. While these works provide complementary perspectives on approximate potentiality, our focus is on finite-state, finite-action Markov games under substantially weaker feedback: players learn independently from bandit samples, including in a fully online asynchronous regime. Moreover, under the \emph{state-wise potential} structure of Section~\ref{sec:statewise-potential-sharp}, i.e., when the potential function itself can be expressed as the value function of a base function, our guarantees isolate the approximation error as a single additive $\alpha$ term, avoiding its amplification by extra factors and thereby remaining informative over a broader approximate-potential regime.

A closely related structural direction considers stochastic games with independent or decoupled dynamics. Etesami~\cite{etesami2024learning} studies $n$-player stochastic games in which players have independently evolving state-action processes but remain coupled through their payoffs, importantly without imposing an MPG assumption. Using only their own realized payoffs, players employ online mirror descent in the dual space of occupancy measures to update their policies. For general reward functions, their episodic algorithm approaches the set of $\epsilon$-NE policies with high probability and polynomial time and sample complexity, under the weaker notion of \emph{averaged} NE regret. Jordan et al.~\cite{introJordanKamgarpour2026} consider partially observable Markov potential games (POMGs) with decoupled transition and observation dynamics. Under filter stability, they approximate the partially observable game by a finite-history superstate Markov game, show that this surrogate inherits a near-potential structure, and obtain communication-free approximate-NE learning with quasi-polynomial sample and computational complexity. Their work addresses the additional difficulty of partial observability through finite-window model estimation, whereas our setting assumes full state observability but targets a more primitive online interaction model without requiring decoupled transition dynamics.

\smallskip
{\bf Congestion games and Markov congestion games.} As an important special case of our general results for Markov $\alpha$-potential games, in the last section we consider \emph{independent-resource Markov congestion games} (IMCGs)~\cite{introCui2022}, which generalize static congestion games to dynamic stochastic environments with independent evolving resource states. In the static setting, congestion games are classical examples of exact potential games~\cite{rosenthal1973class,introMondererShapley1996}, and their structure has long made them attractive models for routing, scheduling, and resource sharing. For instance, Altman et al.~\cite{introAltman2011} study decentralized load balancing in processor-sharing systems and exploit a potential formulation to characterize equilibria and compare decentralized and centralized performance. Li et al.~\cite{introLiLiuLi2014} formulate selfish scheduling of deteriorating jobs on parallel machines as a game and develop a game-theoretic approximation algorithm that converges to a pure NE. More recently, Fardno and Etesami~\cite{introFardnoEtesami2026} develop static and dynamic game models for distributed load balancing: the static scheduling game admits a potential function with a low price of anarchy, while dynamic best-response updates yield polynomial-time convergence to balanced behavior under their queueing dynamics. Yao and Ding~\cite{introYaoDing2022} formulate data-center load balancing as an MPG and propose a fully distributed MARL method, supported primarily by simulations and real-system experiments.

Finite-sample learning in congestion games has also received growing attention. Cui et al.~\cite{introCui2022} establish polynomial sample complexity for static congestion games under bandit and semi-bandit feedback and introduce IMCGs to capture stochastic resource dynamics. Their learning algorithm for the IMCG setting achieves sublinear NE regret; however, it is episodic and reset-based, as well as centralized and computationally demanding, requiring the computation of an $\epsilon$-NE of a large matrix game to update the policies at each iteration. In contrast, by exploiting the resource-local stochastic structure of IMCGs, we show that they fall within the scope of our more general Markov $\alpha$-potential framework. This allows us to develop a fully decentralized one-sample online algorithm in which players update simultaneously using only their own realized costs, together with an explicit high-probability sublinear NE regret. Thus, our results provide the first decentralized finite-time learning theory for IMCGs without episodic resets or costly centralized equilibrium computation.

Table~\ref{tab:intro-comparison} summarizes the distinctions most relevant to our work. It compares information structures and guarantees rather than ranking the results: several existing methods achieve faster rates under stronger oracle access, for more restrictive game classes, or in expectation rather than with high probability, whereas our emphasis is on combining approximate-potential structure, bandit feedback, decentralization, and fully online asynchronous learning. As we will see later, another notable distinction between our bounds and existing NE-regret bounds is that our guarantees completely eliminate the mismatch coefficient, which can be prohibitively large and may scale with the size of the state space.

\begin{table}[t]
\centering
\caption{Representative comparison with closely related equilibrium-learning results. ``Finite-time'' refers to an explicit nonasymptotic equilibrium/NE-gap or Nash-regret guarantee.}
\label{tab:intro-comparison}
\renewcommand{\arraystretch}{1.35}
\setlength{\tabcolsep}{4pt}
\resizebox{\textwidth}{!}{%
\begin{tabular}{lcccccc}
\hline
Work & Game class & Feedback/oracle & Independent & Fully online & Finite-time & Main distinction \\
\hline
Fox et al.~\cite{introFox2022}
& MPG
& \shortstack{exact policy\\quantities}
& Yes & No & No
& \shortstack{last-iterate NPG\\convergence} \\
\hline

Ding et al.~\cite{introDing2022}
& MPG
& \shortstack{exact / sampled\\gradients}
& Yes & No & Yes
& \shortstack{global PG rates;\\function approximation} \\
\hline

Cheng et al.~\cite{introCheng2024}
& \shortstack{average-reward\\MPG}
& \shortstack{oracle / sampled\\gradients}
& Yes & No & Yes
& \shortstack{independent PG/NPG\\for average reward} \\
\hline

Jordan et al.~\cite{introJordan2024}
& \shortstack{constrained\\MPG}
& \shortstack{stochastic\\gradients}
& Yes & No & Yes
& \shortstack{independent learning\\with constraints} \\
\hline

Zhang et al.~\cite{introZhangNeurIPS2022}
& MPG
& gradient access
& Yes & No & Yes
& \shortstack{softmax PG/NPG\\rates} \\
\hline

Sun et al.~\cite{introSun2023}
& MPG
& \shortstack{exact policy\\evaluation}
& Yes & No & Yes
& \shortstack{$O(1/\epsilon)$ under\\suboptimality gap} \\
\hline

Dong et al.~\cite{introDong2024}
& MPG
& stochastic bandit
& Yes & No & Yes
& \shortstack{episodic sampling;\\$O(T^{4/5})$ Nash regret} \\
\hline

Etesami~\cite{etesami2024learning}
& \shortstack{independent-chain\\stochastic game}
& \shortstack{realized own\\payoffs}
& Yes & No & Yes$^\dagger$
& \shortstack{dual averaging/mirror descent;\\decentralized observations} \\
\hline

Jordan--Kamgarpour~\cite{introJordanKamgarpour2026}
& \shortstack{decoupled POMG /\\near-potential}
& \shortstack{local trajectories;\\model estimation}
& Yes & No & Yes
& \shortstack{partial observability;\\finite-history approximation} \\
\hline

Maheshwari et al.~\cite{introMaheshwari2025}
& MPG
& \shortstack{one-stage payoff\\samples}
& Yes & Yes & No
& \shortstack{two-timescale\\asynchronous convergence} \\
\hline

Guo et al.~\cite{introGuo2026}
& \shortstack{Markov\\$\alpha$-potential}
& \shortstack{policy-improvement\\oracle}
& Partly & No & Yes
& \shortstack{approximate-potential\\regret analysis} \\
\hline

Cui et al.~\cite{introCui2022}
& \shortstack{congestion /\\Markov congestion}
& \shortstack{bandit /\\semi-bandit}
& \shortstack{static: Yes;\\Markov: No} & No & Yes
& \shortstack{centralized Markov-\\congestion learning} \\
\hline

\textbf{This work}
& \shortstack{Markov $\alpha$-potential /\\state-wise / IMCG}
& \shortstack{bandit; one\\sample online}
& \textbf{Yes} & \textbf{Yes} & \textbf{Yes}
& \shortstack{high-probability average NE-gap;\\additive $\alpha$ in state-wise case} \\
\hline
\end{tabular}%
}

\vspace{1mm}
{\footnotesize
$^\dagger$For general reward functions, polynomial-time, high-probability sublinear regret is established in terms of the weaker \emph{weighted} NE regret criterion, which becomes a sublinear NE regret guarantee under additional reward structure.
}
\end{table}

\subsection{Contributions}
\label{subsec:intro-contributions}

The main objective of this paper is to develop a finite-time theory for independent learning of stationary Nash policies in structured Markov $\alpha$-potential games under information constraints that are substantially weaker than exact policy-evaluation or full-gradient access. The analysis is organized around a progression from a general approximate-potential model to sharper state-wise structure and, finally, to a concrete dynamic Markov congestion class. The main contributions are as follows.

\begin{enumerate}

\vspace{-0.1cm}
\item \textbf{A general decentralized KL-projected NPG framework for Markov $\alpha$-potential games.} We consider infinite-horizon discounted $n$-player Markov $\alpha$-potential games under bandit feedback and introduce a KL-projected NPG update that can be executed independently by each player. To this end, we develop an extensive analysis that explicitly tracks the sensitivity of discounted occupancies, values, and marginalized advantages to unilateral policy changes, enabling the potential-based analysis to accommodate both imperfect potential alignment ($\alpha>0$) and imperfect advantage estimation $(L_{\widehat{A}}>0)$.

\vspace{-0.3cm}
\item \textbf{High-probability NE regret under episodic bandit estimation.} Section~\ref{sec:Episodic} studies an episodic, reset-free implementation in which the policy is held fixed while samples are collected. We derive a high-probability bound on the NE regret for general Markov $\alpha$-potential games, explicitly separating the statistical term, the estimation-oracle bias term $L_{\widehat A}$, the potential approximation error $\alpha$, and the transition-sensitivity term $\delta_P$. The proof combines KL geometry, potential improvement, concentration bounds for importance-weighted estimators, and a conversion from local policy certificates to the global NE gap.

\vspace{-0.3cm}
\item \textbf{High-probability NE regret under fully online asynchronous bandit estimation.} Section~\ref{sec:fully-online-general-occupancy} removes the episodic frozen-policy structure. At every primitive time step, each player receives only a single realized cost sample and updates its policy only at the currently visited state. Because different states are updated at random, asynchronous times, and because the state distribution itself drifts as the policies evolve, standard episodic arguments no longer apply. We introduce stopping-time local indices, occupancy tracking, coupling and delayed-window charging arguments, and martingale concentration to control these interactions. The resulting theorem provides an explicit high-probability bound on the NE regret. This fills the gap between earlier fully decentralized asymptotic actor--critic results and finite-time oracle-based or episodic analyses.

\vspace{-0.3cm}
\item \textbf{Sharper NE regret bounds under state-wise potential structure.} Section~\ref{sec:statewise-potential-sharp} identifies a stronger state-wise potential structure in which the underlying Markov $\alpha$-potential function can be expressed as the discounted value function of a base potential function and develops matching episodic and fully online potential-advantage oracles. The key benefit is both conceptual and quantitative: the potential approximation parameter $\alpha$ enters the final guarantees as a \emph{single additive $\alpha$ term}. Unlike in the general online $\alpha$-potential analysis, $\alpha$ is not amplified by a tracking factor. Consequently, the resulting bounds remain informative over a substantially broader approximate-potential regime.

\vspace{-0.3cm}
\item \textbf{Decentralized learning for independent-resource Markov congestion games (IMCGs).} Section~\ref{sec:IMCG-model} specializes the theory to IMCGs, where resources have local stochastic states and their transition kernels factor across resources conditional on congestion. We establish the required transition-sensitivity and state-wise approximate-potential properties, construct episodic and one-sample estimation oracles from realized player costs, and obtain an explicit fully decentralized online NE regret guarantee. Theorem~\ref{thm:imcg-online-ne-regret} shows, in particular, that the NE regret decays at rate $\widetilde O(T^{-2/15})$, up to terms controlled by the resource-coupling parameter $\delta\ll 1$; as $\delta\to0$, the approximation floor vanishes. This provides a concrete interpretation of weak dynamic coupling as a quantitative source of approximate potentiality. As a concrete application, Section~\ref{subsec:imcg-job-scheduling} maps IMCGs to decentralized strategic scheduling on stochastic machines, where players are strategic job owners, machines are congestible resources, and machine states represent stochastic processing conditions or queue/load states.\footnote{To the best of our knowledge, dynamic job scheduling in queueing models has not been studied under this strategic Markov-game formulation, offering a fresh perspective on strategic queueing systems that is of independent interest.} The resulting guarantee shows that strategic users can learn stable approximate dispatching policies from their own realized delays through simultaneous one-sample updates, without requiring a centralized scheduler or best-response computation. This connects the abstract Markov-game theory to a real-world application.
\end{enumerate}


\subsection{Organization}

The remainder of the paper is organized as follows. Section~\ref{sec:model} introduces the general discounted Markov-game model, stationary policies, the NE-gap criterion, and the approximate-potential assumptions used throughout. Section~\ref{sec:preliminaries} develops the value, $Q$-function, marginalized-advantage, KL-projected NPG, and sensitivity tools needed for the analysis. Section~\ref{sec:Episodic} presents the episodic online bandit oracle, the decentralized episodic online algorithm, and its high-probability NE-regret analysis. Section~\ref{sec:fully-online-general-occupancy} develops the global/local stopping-time notation, the one-sample online bandit oracle, the asynchronous online algorithm, and the coverage and tracking analyses leading to the fully online NE regret guarantee. Section~\ref{sec:statewise-potential-sharp} exploits the state-wise potential structure to obtain sharper episodic and online bounds with the additive $\alpha$ dependence described above. Section~\ref{sec:IMCG-model} then verifies these assumptions for independent-resource congestion dynamics, derives the corresponding decentralized online guarantee, and concludes with the strategic job-scheduling application. Conclusions are provided in Section~\ref{sec:conclusions}, while omitted technical proofs are collected in the appendices.

\subsection{Notation}
Throughout the paper, we adopt the following notation and conventions. For any positive
integer $n$, we write $[n]:=\{1,\ldots,n\}$. For a finite set $\mathcal A$,
we denote its cardinality by $|\mathcal A|$ and the probability simplex over
$\mathcal A$ by $\Delta(\mathcal A)$. Random variables are denoted by capital
letters and their realizations by the corresponding lowercase letters. For
instance, we denote the random state at time $t$ by $S^t$ and its realization
by $s^t$. For a vector $x=(x_1,\ldots,x_n)$, we write $x_{-i}$ for the vector
obtained by removing its $i$th component; the same convention is used for
action and policy profiles. For a collection of scalars $\{x(s,a)\}_{s,a}$,
we use $x(s,\cdot)$, or simply $x(s)$ when unambiguous, for the vector indexed
by $a$ at state $s$, and $x$ for the full collection across states and actions.
We write $\langle x,y\rangle$ for the Euclidean inner product, $\|x\|_p$ for
the standard $\ell_p$ norm, and $\|\cdot\|_{\mathrm{TV}}$ for total variation
distance. For a state-indexed collection $x=\{x(s)\}_{s\in\mathcal S}$, we use
the mixed-norm notation
$\|x\|_{p,\infty}:=\max_{s\in\mathcal S}\|x(s)\|_p$ and
$\|x\|_{\mathrm{TV},\infty}:=
\max_{s\in\mathcal S}\|x(s)\|_{\mathrm{TV}}$.
For multi-player policy profiles
$\pi=\{\pi_i(\cdot\mid s):i\in[n],\,s\in\mathcal S\}$ and
$\pi'=\{\pi'_i(\cdot\mid s):i\in[n],\,s\in\mathcal S\}$, the corresponding
mixed norms sum over players; in particular,
$\|\pi-\pi'\|_{\mathrm{TV},\infty}:=
\max_{s\in\mathcal S}\sum_{i=1}^n
\|\pi_i(\cdot\mid s)-\pi_i'(\cdot\mid s)\|_{\mathrm{TV}}$,
with the sum restricted to $j\ne i$ for profiles indexed by $-i$.
The indicator of an event $E$ is denoted by $\mathbf{1}\{E\}$. For $B>0$,
we denote the projection onto the interval $[-B/2,B/2]$ by
$\operatorname{clip}_B(\cdot)$.
Unless otherwise specified, $c>0$ denotes a universal numerical constant
whose value may change from line to line, and $\widetilde O(\cdot)$ suppresses
logarithmic factors.

\section{General Model and Problem Formulation}
\label{sec:model}

We consider an infinite horizon discrete-time finite-state discounted stochastic game (a.k.a. Markov game) with $[n]=\{1,\ldots, n\}$ players, where there is a common finite state space and each player $i\in[n]$ has its own finite set of actions $\mathcal{A}_i$. We use $\mathcal{S}$ and $\mathcal{A}=\mathcal{A}_1\times \cdots \times \mathcal{A}_n$ to denote the state space and the joint action set of players, respectively. At any discrete time $t=0,1,2,\ldots$, we use $s^t\in \mathcal{S}$ and $a_i^t\in \mathcal{A}_i$, respectively, to denote the global state and the action of player $i$ at time $t$. We also use $a^t=(a_1^t,\ldots,a_n^t)\in \mathcal{A}$ to denote the joint action profile of all players at time $t$. Moreover, we let $c_i(s^t,a^t)$ be the cost received by player $i$ at time $t$, where, without loss of generality, we assume that the costs are normalized such that $c_i:\mathcal{S}\times \mathcal{A}\to  [0, 1], \forall i\in [n]$. 

At any time $t$, the information available to player $i$ is given by the history of \emph{realized} states, its actions, and its realized costs, i.e., $\mathcal{H}^t_i=\{s^{\ell},a_i^{\ell},c_i(s^{\ell},a^{\ell}): \ell=0,1,\ldots,{t-1}\}\cup\{s^{t}\}$. In particular, we note that player $i$ cannot observe other players' actions, nor can it access the structure of its own cost function $c_i(\cdot)$.  At any time $t$, player $i$ takes an action $a_i^t$ based on its information set $\mathcal{H}^t_i$ and receives an instantaneous cost $c_i(s^t,a^t)$, which depends on the state and all other players' actions. After that, the state of the system changes from $s^t$ to a new state $s^{t+1}$ with probability $P(s^{t+1}\mid s^{t},a^t)$. We refer to $P$ as the (primitive) transition probability kernel of the game.  

A general policy for player $i$ is a sequence of probability measures $\pi_i=\{\pi^t_i, t=0,1,2,\ldots\}$ over $\mathcal{A}_i$ that at each time $t$ selects an action $a_i\in \mathcal{A}_i$ based on past observations $\mathcal{H}^t_i$ with probability $\pi^t_i(a_i\mid \mathcal{H}_i^t)$. Use of general policies is often computationally expensive, and in practical applications, players are interested in easily implementable policies. In that regard, the class of \emph{stationary} policies constitutes the most well-known class of simple policies, as defined next.

\begin{definition}
A policy $\pi_i:\mathcal S\rightarrow \Delta(\mathcal A_i)$ for player $i$ is called \emph{stationary} if the probability of choosing action $a_i$ at time $t$, denoted by $\pi_i^t(a_i\mid \mathcal{H}^t)$, depends only on the current state $s^t=s$ and is independent of $t$. For a stationary policy, we use $\pi_i(a_i\mid s)$ to denote this time-independent probability. 
\end{definition}

For a stationary policy profile $\pi=(\pi_i,\pi_{-i})$ and initial state distribution $\mu\in \Delta(\mathcal{S})$, the infinite-horizon discounted value function for player $i$ is given by
\begin{align}\label{eq:value-function}
V_i^\pi(\mu)
:=\mathbb E_\pi\left[
\sum_{t=0}^{\infty}\gamma^t c_i(S^t,A^t)
\,\middle|\,
S^0\sim \mu
\right],
\end{align}
where $\gamma\in (0, 1)$ is the discount factor, and the expectation is with respect to the randomness of the state transitions and the policy profile $\pi$. Finally, a discounted Markov game is defined by the tuple $\mathcal G
=
\big([n], \mathcal S, P, \{\mathcal A_i\}_{i=1}^n, \{c_i\}_{i=1}^{m}, \gamma, \mu \big).$

\begin{definition}
Given an initial state distribution $\mu$, a stationary policy profile
$\pi^*=(\pi_i^*,\pi_{-i}^*)$ is called a stationary Nash equilibrium, or simply
a Nash equilibrium (NE), if
$V_i^{(\pi_i^*,\pi_{-i}^*)}(\mu)
\leq V_i^{(\pi_i,\pi_{-i}^*)}(\mu)$
for every player $i$ and every stationary policy $\pi_i$. Moreover, for
$\epsilon>0$, $\pi^*$ is called an $\epsilon$-Nash equilibrium
($\epsilon$-NE) if
$V_i^{(\pi_i^*,\pi_{-i}^*)}(\mu)
\leq V_i^{(\pi_i,\pi_{-i}^*)}(\mu)+\epsilon$
for every player $i$ and every stationary policy $\pi_i$. 
\end{definition}

It is known that Markov games always admit a NE within the class of stationary policies \cite{shapley1953stochastic}, and our main objective in this work is to develop online learning algorithms that can compute an $\epsilon$-NE with provable convergence-rate guarantees. However, computing or approximating a stationary NE in general Markov games is computationally intractable \cite{daskalakis2009complexity}, even under more relaxed equilibrium notions such as stationary coarse correlated equilibria \cite{daskalakis2023complexity}. Thus, to devise a scalable independent learning algorithm for computing a stationary NE, one must impose additional assumptions. To describe the assumptions adopted in this work, we first introduce the following definitions.

\begin{definition}\label{def:transition-sensitivity}
The (unilateral) transition sensitivity of the Markov game is defined as
\[
\delta_P
:=
\max_{i\in[n]}
\max_{\substack{s\in\mathcal S,\;a_i,a_i'\in\mathcal A_i\\
a_{-i}\in\mathcal A_{-i}}}
\left\|
P(\cdot\mid s,a_i,a_{-i})
-
P(\cdot\mid s,a_i',a_{-i})
\right\|_{\mathrm{TV}}.
\]
\end{definition}

Since the total variation distance between
any two probability distributions is at most one, we always have
$\delta_P\in[0,1]$. The parameter $\delta_P$ quantifies the maximum influence that a unilateral
change in one player's action can have on the distribution of the next
state, while holding the state and the other players' actions fixed. Thus,
smaller values of $\delta_P$ correspond to games in which the state dynamics
are less sensitive to individual players' actions, while $\delta_P=0$ means that no unilateral action change affects the transition distribution
(e.g., in a mean-field regime).

\begin{definition}\label{def:discounted-state-occupancy}
For any stationary policy profile $\pi$ and initial distribution
$\mu\in\Delta(\mathcal S)$, let $\mathbb P_\pi$ denote the probability law
induced by $\pi$ and the transition kernel $P$. The discounted state occupancy measure induced by $\pi$ from $\mu$ is
defined as
\[
d_\mu^\pi(s)
:=
(1-\gamma)\sum_{t=0}^\infty
\gamma^t\mathbb P_\pi\!\left(S^t=s\mid S^0\sim\mu\right),
\qquad s\in\mathcal S.
\]
Thus, $d_\mu^\pi\in\Delta(\mathcal S)$, with $d_\mu^\pi(s)$ representing
the normalized discounted frequency of visits to state $s$.
\end{definition}

Next, we consider the following definition of Markov $\alpha$-potential game from \cite{guo2025markov}. 

\begin{definition}
\label{def:alpha-markov-potential-game-uniform}
Given $\alpha\ge 0$, a Markov game is called an $\alpha$-potential game if there exists a
potential function $\Psi$ on the space of stationary policy profiles such that, for
every player $i\in[n]$, every pair of stationary policy profiles
$\pi=(\pi_i,\pi_{-i})$ and $\pi'=(\pi_i',\pi_{-i})$ that differ only in
player $i$'s policy, and every initial distribution
$\mu\in\Delta(\mathcal S)$,\footnote{Whenever $\mu$ is fixed and
understood from the context, we suppress the dependence of $\Psi$ on $\mu$ and simply write
$\Psi(\pi)$.}
\begin{align}\nonumber
\left|
\bigl(V_i^{\pi'}(\mu)-V_i^\pi(\mu)\bigr)
-
\bigl(\Psi^{\pi'}(\mu)-\Psi^\pi(\mu)\bigr)
\right|
\le
\alpha.
\end{align}
\end{definition}

\begin{remark}
Without loss of generality, we restrict $\alpha\le 1/(1-\gamma)$.
Indeed, since the stage costs are bounded in $[0,1]$, any unilateral
value difference is at most $1/(1-\gamma)$ in absolute value.
Thus, the constant function $\Psi(\pi)\equiv 1/(1-\gamma)$
trivially satisfies the $\alpha$-potential property with
$\alpha=1/(1-\gamma)$.
\end{remark}

For a policy profile $\pi=(\pi_1,\ldots,\pi_n)$, we define the state transition kernel induced by $\pi$ as
\[
\bar{P}^\pi(s' \mid s)
:=\mathbb E_{A\sim\pi(\cdot\mid s)}\left[
P(s'\mid s,A)
\right]=\sum_{a\in \mathcal{A}}P(s'\mid s,a)\prod_{j=1}^n \pi_j(a_j\mid s),
\]
where by abuse of notation, $\pi(a\mid s)=\prod_{j=1}^n \pi_j(a_j\mid s)$. Similarly, for every player $i\in[n]$, state $s\in\mathcal S$, and stationary
opponents' policy $\pi_{-i}$, we define the marginalized transition kernel as
\[
\bar P_i^{\pi_{-i}}(s'\mid s,a_i)
:=\mathbb E_{A_{-i}\sim\pi_{-i}(\cdot\mid s)}
\bigl[P(s'\mid s,a_i,A_{-i})\bigr]=\sum_{a_{-i}\in \mathcal{A}_{-i}}P(s'\mid s,a)\prod_{j\neq i} \pi_j(a_j\mid s).
\]

\begin{remark}\label{rem:marginalized-transition-sensitivity}
Definition~\ref{def:transition-sensitivity} immediately implies $\left\|
\bar P_i^{\pi_{-i}}(\cdot\mid s,a_i)
-
\bar P_i^{\pi_{-i}}(\cdot\mid s,a_i')
\right\|_{\mathrm{TV}}
\le \delta_P$ for every $i\in[n]$, $s\in\mathcal S$, stationary $\pi_{-i}$, and
$a_i,a_i'\in\mathcal A_i$. Indeed, this follows by averaging over
$A_{-i}\sim\pi_{-i}(\cdot\mid s)$ and using the convexity of total
variation distance. Conversely, since deterministic policies
$\pi_{-i}$ are admissible, taking $\pi_{-i}$ concentrated on any fixed
$a_{-i}$ recovers Definition~\ref{def:transition-sensitivity}.
\end{remark}

Throughout this work, we impose the following assumption.

\begin{assumption}\label{ass:alpha-potential-ergodicity}
There exists $\alpha\in (0, \frac{1}{1-\gamma})$ such that the Markov game $\mathcal{G}$ is a Markov $\alpha$-potential game. Moreover, for every stationary policy profile $\pi$, the induced Markov chain over $\mathcal{S}$ with state transition kernel $\bar{P}^\pi$ is irreducible and aperiodic and admits a unique stationary distribution.
\end{assumption}

Finally, to evaluate the convergence rate of our proposed algorithms to an
$\epsilon$-NE, we use the following Nash equilibrium gap function (also known
as the Nikaido--Isoda function), which has become a standard measure of the
distance of a policy profile from a NE
\cite{sun2023provably,zhang2022global,etesami2024learning}. 

\begin{definition}
Given a policy profile $\pi=(\pi_i,\pi_{-i})$, the Nash equilibrium gap function is given by:
\[
\operatorname{Gap}(\pi):=
\max_i \sup_{\pi_i'}
\left[
V_i^{(\pi_i,\pi_{-i})}(\mu)-V_i^{(\pi_i',\pi_{-i})}(\mu)
\right].\]  
\end{definition}

From this definition, it is easy to see that if $\operatorname{Gap}(\pi) \le \epsilon$, then $\pi$ must be an $\epsilon$-NE.

Finally, we consider the following coverage assumption, which imposes a mild exploration condition needed to establish meaningful finite-time NE-regret bounds.

\begin{assumption}\label{ass:coverage}
Given $\zeta>0$, assume that there exist $H_{\mathrm{cov}}\ge1$ and
$p_{\min}>0$ such that, for every stationary policy profile $\pi$ satisfying
$\pi_i(a_i\mid s)>\zeta$ for all $i$, $s$, and $a_i$, and every initial state
$s^0\in\mathcal S$, each prescribed state $s\in\mathcal S$ is visited within
$H_{\mathrm{cov}}$ steps with probability at least $p_{\min}$; that is,
\[
\mathbb P_{\pi}\left(
S^h=s
\text{ for some }1\le h\le H_{\mathrm{cov}} \mid S^0=s^0 \right)
\ge p_{\min}.
\]
\end{assumption}

Assumption~\ref{ass:coverage} is a natural coverage requirement ensuring that each state is visited with positive probability within a finite-length time window. Otherwise, if certain states are never visited or are visited with arbitrarily small probability, then regardless of how effective the algorithm is, it cannot reliably evaluate the costs associated with such states.

\section{Preliminaries}\label{sec:preliminaries}

In this section, we provide some general definitions and preliminary results that will be used throughout the paper. The proofs of all these results are provided in Appendix~\ref{app:3}.

\subsection{Value, Q-Function, Advantage Function, and Bellman Optimality}

Fix the stationary policies $\pi_{-i}$ of all players other than player $i$.
Then player $i$ faces a single-agent discounted Markov decision process (MDP) with
state space $\mathcal S$, action space $\mathcal{A}_i$, discount factor
$\gamma\in(0,1)$, \emph{marginalized stage cost}
\[
\bar{c}^{\pi_{-i}}_i(s,a_i)
:=
\mathbb E_{A_{-i}\sim \pi_{-i}(\cdot\mid s)}
\big[c_i(s,a_i,A_{-i})\big]\ \ \ \forall s\in \mathcal{S}, a_i\in \mathcal{A}_i,
\]
and \emph{marginalized transition kernel} $\bar{P}_i^{\pi_{-i}}(s'\mid s,a_i)$. Then, it is easy to see that the value function of player $i$ must satisfy the recursive \emph{Bellman equation}, given by
\begin{align}\label{eq:bellman}
V_i^\pi(s)
=
\sum_{a_i\in \mathcal{A}_i} \pi_i(a_i\mid s)
\bigg[
 \bar{c}^{\pi_{-i}}_i(s,a_i)
+
\gamma \sum_{s'\in\mathcal S}
\bar{P}^{\pi_{-i}}_i(s'\mid s,a_i)V_i^\pi(s')
\bigg]  \ \ \forall s\in \mathcal{S}.
\end{align}
Moreover, one can define the state--action Q-function and the advantage function for player $i$ as
\[
Q_i^\pi(s,a)
:=
\mathbb{E}_\pi
\bigg[
\sum_{t=0}^{\infty}
\gamma^t c_i(S^t,A^t)
\,\bigg|\,
S^0=s,\,
A^0=a
\bigg], \qquad A_i^\pi(s,a)
:=
Q_i^\pi(s,a)
-
V_i^\pi(s).
\]
Accordingly, for fixed policies of other players, $\pi_{-i}$, one can define the \emph{marginalized Q-function} and the \emph{marginalized advantage function} for a unilateral choice of player $i$'s action $a_i$ as
\begin{align}\label{eq:bellman-Q}
\!\!\!\!\!\!\bar{Q}_i^\pi(s,a_i)
&\!:=\!
\mathbb E_{A_{-i}\sim\pi_{-i}(\cdot\mid s)}
\left[
Q_i^{\pi}(s,a_i,A_{-i})
\right]
=
\bar{c}^{\pi_{-i}}_i(s,a_i)
\!+\!
\gamma \sum_{s'\in\mathcal S}
\bar{P}^{\pi_{-i}}_i(s'\mid s,a_i)V_i^\pi(s'),
\cr
\!\!\!\!\!\!\bar{A}_i^\pi(s,a_i)&\!:=\!
\mathbb E_{A_{-i}\sim\pi_{-i}(\cdot\mid s)}
\left[
A_i^{\pi}(s,a_i,A_{-i})
\right]=
\bar{c}^{\pi_{-i}}_i(s,a_i)
\!+\!
\gamma \sum_{s'\in\mathcal S}
\bar{P}^{\pi_{-i}}_i(s'\mid s,a_i)V_i^\pi(s')
\!-\!
V_i^\pi(s).
\end{align}
Specifically, for fixed policies of other players, $\pi_{-i}$, the marginalized Q-function $\bar{Q}_i^\pi(s,a_i)$ measures the expected total future cost of player $i$ at the current state $s$, if it forces itself to choose action $a_i$ once at the current time, and afterwards goes back to following its policy $\pi_i$. Similarly, for fixed $\pi_{-i}$, the marginalized advantage function measures how much better ($\bar{A}_i^\pi(s,a_i)\leq 0$) or worse ($\bar{A}_i^\pi(s,a_i)>0$) it is for player $i$ to choose action $a_i$ right now, compared to following its current policy $\pi_i$ at state $s$. In fact, the marginalized advantage function is one of the key components of our algorithmic design, as it quantifies the marginal improvement in the expected total cost obtained by selecting a particular action at a given state relative to following the current policy. 

\begin{remark}
The key difference between the full \(Q\)-function (or advantage function) and its marginalized counterpart is that, in the latter, the policies of the other players, \(\pi_{-i}\), are held fixed, and the randomness induced by their actions is treated as part of the Markov environment. Consequently, by averaging over the actions generated by the other players' policies, the marginalized quantities characterize the effective single-agent MDP faced by player \(i\).
\end{remark}

As a first step toward developing a decentralized policy-update algorithm for computing an $\epsilon$-NE, we consider the best-response problem faced by each player. In particular, given a fixed policy profile $\pi_{-i}$ of the other players, player $i$ should be able to compute, or approximately compute, a best response. Since $\pi_i$ is the only decision variable once $\pi_{-i}$ is fixed, the (best-response) 
 \emph{optimal value function} of player $i$ is defined by
\[
V_i^*(s)
:=
\inf_{\pi_i} V_i^{(\pi_i,\pi_{-i})}(s),
\]
which by principle of optimality satisfies the \emph{Bellman optimality equation}
\begin{equation}
V_i^*(s)
=
\min_{a_i\in \mathcal{A}_i}
\left\{
 \bar{c}^{\pi_{-i}}_i(s,a_i)
+
\gamma \sum_{s'\in\mathcal S}
\bar{P}^{\pi_{-i}}_i(s'\mid s,a_i)V_i^*(s')
\right\}.
\end{equation}
Similarly, the \emph{optimal marginalized Q-function} is given by
\begin{equation}
\bar{Q}_i^*(s,a_i)
=
\bar{c}^{\pi_{-i}}_i(s,a_i)
+
\gamma \sum_{s'\in\mathcal S}
\bar{P}^{\pi_{-i}}_i(s'\mid s,a_i)V_i^*(s'),
\end{equation}
and therefore the Bellman optimality equation can be written as  
\begin{align}\label{eq:Bellman-opt}
V_i^*(s)=\min_{a_i\in \mathcal{A}_i} \bar{Q}_i^*(s,a_i)=\sum_{a_i\in \mathcal{A}_i}
\pi_i^*(a_i\mid s)\,
\bar{Q}_i^*(s,a_i),
\end{align}
where $\pi^*_i\in \argmin_{\pi_i} V_i^{(\pi_i,\pi_{-i})}$ denotes the optimal stationary policy for player $i$, and the last equality follows from \eqref{eq:bellman} for the choice of $\pi_i=\pi_i^*$. 

The following lemma characterizes the best response for player $i$, which follows directly from Bellman's optimality condition, but we include a short proof in Appendix \ref{app:belman-proof} for completeness. 

\begin{lemma}\label{lemm:best}
Fix the policies $\pi_{-i}$ of all other players and consider the induced
marginalized MDP faced by player $i$. Then a stationary policy $\pi_i$ is a best response to $\pi_{-i}$ if and only if
\begin{align}
\bar{A}_i^\pi(s,a_i) &\ge 0,
\qquad \forall s,a_i,
\label{eq:adv_nonnegative}
\\
\pi_i(a_i\mid s)>0
&\Longrightarrow
\bar{A}_i^\pi(s,a_i)=0,
\qquad \forall s,a_i.
\label{eq:adv_support}
\end{align}
Thus every action used with positive probability has zero marginalized advantage, while
every action with strictly positive marginalized advantage receives zero probability.\footnote{The best-response conditions in Lemma \ref{lemm:best} can also be written as the complementarity system $\pi_i(a_i\mid s)\ge 0$, $\bar{A}_i^{\pi}(s,a_i)\ge 0$, and $\pi_i(a_i\mid s)\,\bar{A}_i^{\pi}(s,a_i)=0$.}
\end{lemma}

The following lemma is a straightforward adaptation of the performance difference lemma for single-agent MDPs \cite[Lemma 2]{agarwal2021theory}, applied to the marginalized MDP faced by player $i$. It evaluates the difference between the value functions induced by two policies of player $i$, while holding the policies of all other players fixed, in terms of the corresponding marginalized advantage function.

\begin{lemma}[Performance Difference \cite{agarwal2021theory}]\label{lemm:performance-difference}
Fix the policies of all players except player $i$, namely
$\pi_{-i}$, and let $\mu$ be any initial state distribution.
Then, for any two policies $\pi_i$ and $\pi_i'$ of player $i$,
\[
V_i^{(\pi_i,\pi_{-i})}(\mu)-V_i^{(\pi_i',\pi_{-i})}(\mu)=
\frac{1}{1-\gamma}
\sum_{s}
d_\mu^{(\pi_i',\pi_{-i})}(s)
\left\langle \pi_i(\cdot \mid s)-\pi'_i(\cdot \mid s), \bar {A}_i^{(\pi_i,\pi_{-i})}(s,\cdot)\right\rangle,
\]
where
$\bar A_i^{(\pi_i,\pi_{-i})}$
is the marginalized advantage function induced by fixing
$\pi_{-i}$.
\end{lemma}

\subsection{Natural Policy Gradient (NPG) and KL-Projected NPG}

Fix the stationary policies $\pi_{-i}$ of all players other than player $i$. Then player $i$ faces a marginalized MDP, whose transition kernel and one-stage cost are induced by $\pi_{-i}$ as defined above. Standard policy-gradient methods can therefore be applied to this marginalized MDP. In particular, Natural Policy Gradient (NPG) preconditions the policy gradient by the inverse Fisher information matrix, thereby accounting for the geometry of the policy space \cite{agarwal2021theory}. Under the softmax parameterization, the resulting NPG update admits a particularly simple closed-form expression, as introduced next.
 
\begin{lemma}[\cite{agarwal2021theory}]
Consider the marginalized MDP induced by fixing $\pi_{-i}$.
Suppose that player $i$ employs a tabular softmax policy
parameterization with exact advantage evaluations.
Then, one iteration of the NPG policy update is given by
\[
\pi_i^{t+1}(a_i\mid s)
=
\frac{
\pi_i^{t}(a_i\mid s)
\exp\!\big(
-\frac{\eta}{1-\gamma}
\bar{A}_i^{\pi^t}(s,a_i)
\big)}
{\sum_{a'_i\in\mathcal A_i}
\pi_i^{t}(a'_i\mid s)
\exp\!\big(
-\frac{\eta}{1-\gamma}
\bar{A}_i^{\pi^t}(s,a'_i)
\big)}.
\]
Furthermore, if $\eta\le (1-\gamma)^2$, then the sequence of policies generated by NPG converges globally to an optimal stationary policy. Moreover, after $O\big(1/((1-\gamma)^2\epsilon)\big)$ iterations, the generated policy is $\epsilon$-optimal in terms of the discounted value function.
\end{lemma}

The lemma shows that, after fixing the policies of the other players, NPG asymptotically drives the marginalized advantage function toward the Bellman optimality conditions characterizing an optimal stationary best response of player $i$ to $\pi_{-i}$. One drawback of NPG, however, is that some action probabilities may become arbitrarily small. This can lead to high-variance estimates of the marginalized advantage function and, in turn, make it difficult to obtain high-probability regret bounds. To address this issue, we instead consider a KL-projected variant of NPG. This variant retains the simple, easily implementable structure of NPG while ensuring a uniform lower bound on all action probabilities. More precisely, for each player $i$ and exploration parameter $\zeta\in(0,1)$, define the \emph{truncated simplex}
\[
\Delta_{i,\zeta}
:=
\left\{
p\in\Delta(\mathcal A_i):
p(a_i)\ge
\frac{\zeta}{|\mathcal A_i|},
\quad
\forall a_i\in\mathcal A_i
\right\}.
\]
At iteration $t$, assume player $i$  has a vector $r^t_i(s)\in \mathbb{R}^{|\mathcal{A}_i|}$ (e.g., $r^t_i(s)=\bar{A}^{\pi^t}_i(s,\cdot)$), and denote its current policy by $\pi_i^t$. Then, player $i$ updates its policy to $\pi_i^{t+1}$ according to the update rule
\begin{align}\label{eq:kl-projected-npg-update}
\pi_i^{t+1}(\cdot\mid s)
\in
\argmin_{p\in\Delta_{i,\zeta}}
\left\{
\frac{\eta}{1-\gamma}
\left\langle
p,
r_i^t(s)
\right\rangle
+
D_{\mathrm{KL}}
\left(
p
\middle\|
\pi_i^t(\cdot\mid s)
\right)
\right\} \ \ \ \ \ \ s\in \mathcal{S},
\end{align}
where $D_{\mathrm{KL}}(p\| q)=\sum_{a_i\in\mathcal A_i}
p(a_i)\log(p(a_i)/q(a_i))$ denotes the KL-divergence between two probability distributions $p$ and $q$. In fact, using KKT optimality conditions, the update rule \eqref{eq:kl-projected-npg-update} can also be written in the water-filling closed-form as
\begin{align}\label{eq:water-filling}
\pi_i^{t+1}(a_i\mid s)
=\max\left\{
\frac{\zeta}{|\mathcal A_i|},\ 
\lambda(s)\pi^t_i(a_i|s) \exp\Big(-\frac{\eta}{1-\gamma}r_i^t(s,a_i)\Big)
\right\}\ \ \ \forall a_i\in \mathcal{A}_i,
\end{align}
where $\lambda(s)>0$ is the normalization multiplier induced by the simplex constraint and is chosen so that the probabilities in \eqref{eq:water-filling} sum to one. Thus, the algorithm retains the multiplicative structure of NPG while
guaranteeing $\pi_i^{t+1}(a_i\mid s)
\ge
\frac{\zeta}{|\mathcal A_i|}$.
The following lemma provides a geometric characterization of the policy update under the KL-projected NPG rule \eqref{eq:kl-projected-npg-update}, with its proof given in Appendix~\ref{app:KL-geometry}. 
\begin{lemma}
\label{lem:kl-projected-local-geometry}
Consider the the KL-projected NPG \eqref{eq:kl-projected-npg-update} with
$\|r_i^t\|_\infty\le B$, and let 
\begin{align}\label{eq:KL_symm}
    &\mathcal D_i^t(s)
:=
D_{\mathrm{KL}}
\left(
\pi_i^{t+1}(\cdot\mid s)
\middle\|
\pi_i^t(\cdot\mid s)
\right)
+
D_{\mathrm{KL}}
\left(
\pi_i^t(\cdot\mid s)
\middle\|
\pi_i^{t+1}(\cdot\mid s)
\right)\cr 
&\Delta_i^t(s):=\pi_i^t(\cdot\mid s)-\pi_i^{t+1}(\cdot\mid s).
\end{align}
Then, for any state $s\in \mathcal{S}$, the followings hold:
\begin{itemize}
\item[(i)] $\displaystyle
\frac{\pi_i^{t+1}(a_i\mid s)}
{\pi_i^t(a_i\mid s)}\le e^{\frac{2\eta B}{1-\gamma}}, \quad \forall a_i\in\mathcal A_i.$

\item[(ii)] $\displaystyle
\sum_{a_i\in\mathcal A_i}\frac{\left(
\Delta_i^{t}(s,a_i)
\right)^2}{\pi_i^t(a_i\mid s)}\le e^{\frac{2\eta B}{1-\gamma}}\mathcal D_i^t(s).$

\item[(iii)] $\displaystyle
\left\langle
r_i^t(s), \Delta_i^t(s)
\right\rangle
\ge
\frac{1-\gamma}{\eta}
\mathcal D_i^t(s).$

\item[(iv)] $\displaystyle \left\|
\Delta_i^t(s)
\right\|_1
\le
\frac{\eta B}{1-\gamma}
\quad\Rightarrow\quad \left\|
\Delta_i^t
\right\|_{1,\infty}
\le
\frac{\eta B}{1-\gamma}, \quad
\left\|
\Delta_i^t
\right\|_{\mathrm{TV},\infty}
\le
\frac{\eta B}{2(1-\gamma)}.$
\end{itemize}
\end{lemma}

\subsection{Transition, Advantage, and Value Sensitivity}

Here, we provide several results that quantify the sensitivity of different quantities in the Markov game to changes in the policy profile. We begin with the following transition-sensitivity lemma, whose proof is given in Appendix~\ref{app:sensitivity-distribution}.

\begin{lemma}\label{lemm:occupancy-sensitivity}
Let $\delta_P$ be the unilateral transition sensitivity as in Definition \ref{def:transition-sensitivity} and $L_d:=\frac{\gamma\delta_P}{1-\gamma}$. Then 

\begin{itemize}\setlength{\itemsep}{0pt}\setlength{\parskip}{0pt}
\item[(i)] For two policy profiles $\pi\!=\!(\pi_i, \pi_{-i})$ and $\pi'\!=\!(\pi'_i, \pi_{-i})$ that only differ in player $i$'s policy, 
\begin{align*}
&\big\|\bar{P}^\pi(\cdot\mid s)-\bar{P}^{\pi'}(\cdot\mid s)\big\|_{\TV}
\le \delta_P \|\pi_i(\cdot\mid s)-\pi_i'(\cdot\mid s)\|_{\TV} \ \ \ \forall s\in \mathcal{S},\cr 
&\|d_{\mu}^\pi-d_{\mu}^{\pi'}\|_{\TV}
\le
L_d
\sum_{x\in\mathcal S}
d_\mu^{\pi'}(x)
\left\|
\pi_i'(\cdot\mid x)
-
\pi_i(\cdot\mid x)
\right\|_{\mathrm{TV}}.
\end{align*}
\item[(ii)] For any two policy profiles $\pi$ and $\pi'$ (not necessarily differing only in player $i$'s policy), we have
\begin{align}
\label{eq:online-localized-transition-sensitivity}
&\big\|
\bar{P}^{\pi'}(\cdot\mid s)-\bar{P}^\pi(\cdot\mid s)
\big\|_{\mathrm{TV}}
\le
\delta_P
\sum_{j=1}^n
\|\pi_j'(\cdot\mid s)-\pi_j(\cdot\mid s)\|_{\mathrm{TV}} \ \ \ \forall s\in \mathcal{S}.
\end{align}
Moreover, if $\pi$ and $\pi'$ differ only at a single state $x\in\mathcal S$, then
\begin{align}
\label{eq:online-localized-occupancy-sensitivity}
&\|d_\mu^{\pi'}-d_\mu^\pi\|_{\mathrm{TV}}
\le
L_d\ d_\mu^\pi(x)
\sum_{j=1}^n
\|\pi_j'(\cdot\mid x)-\pi_j(\cdot\mid x)\|_{\mathrm{TV}}.
\end{align}  
\end{itemize}
\end{lemma}

Next, we state the following sensitivity lemma for the marginalized advantage function. This is a standard sensitivity property of finite discounted Markov games, closely related to the standard policy-difference framework; see, e.g., \cite{KakadeLangford2002}. For completeness, we provide a short proof specialized to our setting in Appendix \ref{app:advantage-sensitivity}.

\begin{lemma}\label{lem:online-sensitivity-consequences}
Let $L_{\bar A}:=\frac{4}{1-\gamma}+\frac{4\gamma\delta_P}{(1-\gamma)^2}$. Then for any two stationary policy profiles
$\pi,\pi'$, we have
\begin{equation}
\label{eq:online-advantage-sensitivity}
\max_{i,s}
\big\|
\bar A_{i}^{\pi'}(s,\cdot)
-
\bar A_{i}^{\pi}(s,\cdot)
\big\|_{\infty}
\le
L_{\bar A}
\max_{x\in\mathcal S}
\sum_{j=1}^n
\|\pi_j'(\cdot\mid x)-\pi_j(\cdot\mid x)\|_{\mathrm{TV}}.
\end{equation}
\end{lemma}

In our analysis, we also use the following second-order value-function sensitivity lemma to control the potential improvement in our analysis. The proof of this lemma is given in Appendix \ref{app:sensitivity-second-order-advantage}. 

\begin{lemma}
\label{lem:mixed-policy-sensitivity}
For every pair of policies
$\pi_i,\pi_i'$, and every pair of opponent policies
$\sigma_{-i},\tau_{-i}$, we have
\begin{align}\nonumber
\left|
V_i^{(\pi_i',\sigma_{-i})}(\mu)
-
V_i^{(\pi_i,\sigma_{-i})}(\mu)
-
V_i^{(\pi_i',\tau_{-i})}(\mu)
+
V_i^{(\pi_i,\tau_{-i})}(\mu)
\right|
\le L_V
\left\|\pi_i'-\pi_i\right\|_{\mathrm{TV},\infty}
\left\|\sigma_{-i}-\tau_{-i}\right\|_{\mathrm{TV},\infty},
\end{align}
where $L_V:=
\frac{4}{1-\gamma}
+
\frac{12\gamma\delta_P}{(1-\gamma)^2}
+
\frac{8\gamma^2\delta_P^2}{(1-\gamma)^3}$. Here, $\left\|\sigma_{-i}-\tau_{-i}\right\|_{\mathrm{TV},\infty}
\!:=
\max_{s\in\mathcal S}
\sum_{j\ne i}
\left\|
\sigma_j(\cdot\mid s)-\tau_j(\cdot\mid s)
\right\|_{\mathrm{TV}}$ and $\|\pi_i'-\pi_i\|_{\mathrm{TV},\infty}
:=
\max_{s\in\mathcal S}
\left\|
\pi_i'(\cdot\mid s)-\pi_i(\cdot\mid s)
\right\|_{\mathrm{TV}}$.
\end{lemma}

\begin{remark}
Lemma~\ref{lem:mixed-policy-sensitivity} is closely related to the
joint-policy decomposition used in \cite{ding2022independent}. Their analysis algebraically decomposes a simultaneous policy update into
unilateral changes and mixed interaction terms. In contrast,
Lemma~\ref{lem:mixed-policy-sensitivity} directly bounds such mixed
interaction terms for discounted value functions, showing that they are
second order in the corresponding policy changes.
\end{remark}

Finally, we conclude this section with the following auxiliary lemma, which we frequently use to obtain sharper bounds on the sup-norm of the marginalized advantage and its inner product with the difference of two probability vectors. The proof is given in Appendix~\ref{app:span-lemma}.

\begin{lemma}\label{lem:span-inner-product}
For any vector $v\in\mathbb R^m$, define its span by $\operatorname{span}(v):=\max_{a\in[m]}v(a)-\min_{a\in[m]}v(a)$. For every player $i$, state
$s$, policy profile $\pi$, and two arbitrary policies $\pi'_i$ and $\pi_i''$,
\begin{equation}\label{eq:advantage-span-policy-difference}
\left|
\left\langle
\pi_i'(\cdot\mid s)-\pi_i''(\cdot\mid s),
\bar A_i^\pi(s)
\right\rangle
\right|
\le \operatorname{span}\big(\bar A_i^\pi(s,\cdot)\big)\left\|
\pi'_i(\cdot\mid s)-\pi''_i(\cdot\mid s)
\right\|_{\mathrm{TV}}.
\end{equation}
Moreover, $\|\bar A_i^\pi(s,\cdot)\|_\infty\leq \operatorname{span}\big(\bar A_i^\pi(s,\cdot)\big)\leq 1+\frac{\gamma\delta_P}{1-\gamma}= 1+L_d$.
\end{lemma}

\section{Episodic Decentralized Learning with Bandit Feedback}\label{sec:Episodic}

In this section, we first present an episodic decentralized KL-projected NPG algorithm with bandit feedback and full state-table updates, and then analyze its convergence to an $\epsilon$-Nash equilibrium policy. To this end, we first formally state our episodic estimation-oracle assumption and then describe the proposed algorithm.

\subsection{Episodic Estimation Oracle}\label{sec:filteration-oracle-episode}

In this section, we work with episodes, each consisting of a certain number of global time steps. In each episode $t$, the policy profile $\pi^t$ remains fixed until every state has been visited a prescribed number of times. During the episode, each player observes raw samples obtained from an estimation oracle and uses them to update its policy at the end of the episode. 

More precisely, let $\{\mathcal F_t\}_{t\ge0}$ denote the episodic filtration, where $\mathcal F_t$ contains all information generated before episode $t$, including the current policy profile $\pi^t$. Throughout, we write
$\mathbb E_t[\cdot]:=\mathbb E[\cdot\mid\mathcal F_t]$ and $\mathbb{P}_t(\cdot):=\mathbb{P}(\cdot\mid \mathcal{F}_t)$. Conditional on
$\mathcal F_t$, the policy $\pi^t$ remains \emph{fixed} throughout episode $t$,
while the trajectory and oracle samples generated during the episode remain random. For each state $s$, let $\tau_{t,r}(s)$ denote the (random) time of the $r$th visit
to $s$ during episode $t$, and let $\mathcal F_{\tau_{t,r}(s)}$ contain all
information available immediately before the action and oracle sample at
time $\tau_{t,r}(s)$ are generated. Thus,
$A^{\tau_{t,r}(s)}$ and the corresponding oracle sample are not measurable
with respect to $\mathcal F_{\tau_{t,r}(s)}$, but are included in the
filtration thereafter. In particular, for $r<r'$, the action and oracle
sample generated at $\tau_{t,r}(s)$ are measurable with respect to
$\mathcal F_{\tau_{t,r'}(s)}$. We write
$\mathbb E_{\tau_{t,r}(s)}[\cdot]
:=\mathbb E[\cdot\mid\mathcal F_{\tau_{t,r}(s)}]$. The episode terminates once every state has been
visited at least $L_t$ times. At the first $L_t$ visits to $s$, player $i$
receives scalar raw samples $g_i^{\tau_{t,r}(s)}\!\left(s,A^{\tau_{t,r}(s)}\right)$, $r=1,\ldots,L_t$. Here, the raw oracle is allowed to vary with time and depend on the history available up to the corresponding sampling time, including the current global policy profile. Thus, this formulation accommodates a general adaptive and policy-dependent oracle, provided that the resulting estimator satisfies the oracle assumption below. At each such visit, define the corresponding one-sample importance-weighted, action-centered estimator by
\begin{align}\label{eq:hat-A-episodic}
\widehat A_i^{\tau_{t,r}(s)}(s,a_i)
:&=
g_i^{\tau_{t,r}(s)}\!\left(s,A^{\tau_{t,r}(s)}\right)
\left(
\frac{\mathbf 1\{A_i^{\tau_{t,r}(s)}=a_i\}}
{\pi_i^t(a_i\mid s)}-1
\right),  \ \ \ \ \ a_i\in \mathcal{A}_i.
\end{align}
Thus, $\widehat A_i^{\tau_{t,r}(s)}(s,a_i)$ is the one-sample importance-weighted estimator associated with action $a_i$, centered by the realized raw sample. In particular, it is exactly action-centered under $\pi_i^t(\cdot\mid s)$, i.e.,
\[
\sum_{a_i\in\mathcal A_i}
\pi_i^t(a_i\mid s)\widehat A_i^{\tau_{t,r}(s)}(s,a_i)=0.
\]

\begin{assumption}\label{ass:episodic-oracle}
For every episode $t$, state $s$, and $r=1,\ldots,L_t$, the advantage estimator at sampling time $\tau_{t,r}(s)$ satisfies
\begin{equation}\nonumber
\left\|
\mathbb E_{\tau_{t,r}(s)}\big[
\widehat A_i^{\tau_{t,r}(s)}(s,\cdot)\big]
-\bar A_i^{\pi^t}(s,\cdot)
\right\|_\infty
\le
L_{\widehat A},
\end{equation}
for some deterministic $L_{\widehat A}$. Moreover, the raw sample satisfies $|g_i^\tau\!\left(s,A\right)|\!\leq\! 1$ a.s. for all $i,\tau,s,A$.
\end{assumption}

Assumption~\ref{ass:episodic-oracle} allows the raw oracle to vary across
sampling times and depend adaptively on the history generated within the
current episode. In particular, neither the conditional distribution nor the
conditional mean of the raw sample is required to remain fixed throughout
the episode. The only requirement is that, at every sampling time, the
conditional mean of the resulting action-centered estimator remains within
$L_{\widehat A}$ of the marginalized advantage under the fixed episode policy
$\pi^t$. This sample-time conditional formulation ensures that the centered
(pre-clipping) estimation errors form a martingale-difference sequence within
the episode, so averaging the $L_t$ samples reduces their variance as $L_t$
increases. The fixed-oracle setting is included as a special case.

\subsection{An Episodic Decentralized Algorithm with Bandit Estimation Oracle}

We are now ready to describe our online episodic algorithm. Algorithm~\ref{alg:episodic-kl-projected-npg} proceeds in episodes. At the beginning of episode $t$, each player $i$ fixes its policy $\pi_i^t$ and uses it throughout the episode. During the episode, whenever state $s$ is visited at global time $\tau$, player $i$ samples an action from $\pi_i^t(\cdot\mid s)$, observes the corresponding raw sample $g_i^\tau(s,A^\tau)$, forms the action-centered estimator $\widehat A_i^\tau(s,\cdot)$, and accumulates its contribution. The episode terminates once every state has been visited at least $L_t$ times. At the end of the episode, each player forms the averaged action-centered estimator $\widetilde r_i^t(s,a_i)$ and its clipped version $r^t_i(s,a_i)$ as defined in \eqref{eq:g-random-count-advantage-estimator}. Each player then simultaneously updates its policy at every state using the KL-projected NPG update \eqref{eq:algorithm-kl-projected-npg}. The clipping step prevents potentially large importance-weighted samples from producing unbounded estimation errors and provides the uniform boundedness needed for the policy-update geometry and subsequent high-probability analysis. The resulting policy $\pi_i^{t+1}$ is then kept fixed throughout the next episode, and the procedure is repeated.

\begin{algorithm}[t]
\caption{Episodic KL-Projected NPG with Bandit Estimation Oracle}
\label{alg:episodic-kl-projected-npg}
\begin{algorithmic}[1]
\Require Initial policy
$\pi_i^0(\cdot\mid s)\in\Delta_{i,\zeta}$, where $\Delta_{i,\zeta}:=\big\{
p\in\Delta(\mathcal A_i): p(a_i)\ge\frac{\zeta}{|\mathcal A_i|}, \forall a_i\in\mathcal A_i
\big\}$, step size
$\eta>0$, truncation parameter $\zeta\in(0,1)$, clip parameter $B= 4$, and thresholds
$L_t$.

\For{$t=0,1,2,\ldots$}
    \State Player $i$ keeps its policy $\pi_i^t$ fixed throughout
    episode $t$.

    \State Initialize $N_i^{t,0}(s)=0$, $\widetilde r_i^t(s,a_i)=0\ \forall (s,a_i)$, and set the within-episode time $k\gets0$.

    \While{$\min_{s\in\mathcal S}N_i^{t,k}(s)<L_t$}
        \State Player $i$ observes the current state $S^{t,k}$ and
        independently plays $A_i^{t,k}\sim
        \pi_i^t(\cdot\mid S^{t,k})$.

        \State Player $i$ receives a raw sample $g_i^{t,k}\big(S^{t,k},A^{t,k}\big)$.\footnotemark
        
        \If{$N_i^{t,k}(S^{t,k})<L_t$}
        \begin{align}\nonumber
            &\widehat A_i^{t,k}(S^{t,k},a_i)
            :=
            g_i^{t,k}\big(S^{t,k},A^{t,k}\big)
            \bigg(
            \frac{
            \mathbf 1\{A_i^{t,k}=a_i\}
            }{
            \pi_i^t(a_i\mid S^{t,k})
            }
            -1
            \bigg),\cr
            &\widetilde r_i^t(S^{t,k},a_i)
            \leftarrow
            \widetilde r_i^t(S^{t,k},a_i)
            +
            \frac{1}{L_t}
            \widehat A_i^{t,k}(S^{t,k},a_i),\cr
            &N_i^{t,k+1}(S^{t,k})
            \leftarrow
            N_i^{t,k}(S^{t,k})+1.
            \end{align}
            \ \ \ \ \ \ \ \ \ \ \ \ with all other entries unchanged.
        \Else \ \ Set $N_i^{t,k+1}\leftarrow N_i^{t,k}$.
        \EndIf
        \State $k\gets k+1$.
    \EndWhile

    \State Define the episode length $H_t:=k$.
    \State Let $\operatorname{clip}_{B}(\cdot)$ be  the projection onto $[-\frac{B}{2}, \frac{B}{2}]$. Define the clipped estimated advantage by
    \begin{equation*}
    r_i^t(s,a_i)
    :=
    \operatorname{clip}_{B}\!\left(\widetilde r_i^t(s,a_i)\right)
    \ \ \ \ \forall s,a_i.
    \end{equation*}

    \State For every $s\in\mathcal S$, update
    \begin{equation}\label{eq:algorithm-kl-projected-npg}
    \pi_i^{t+1}(\cdot\mid s)
    :=
    \argmin_{p\in\Delta_{i,\zeta}}
    \left\{
    \frac{\eta}{1-\gamma}\left\langle
    p,
    r_i^t(s)
    \right\rangle
    +
    D_{\mathrm{KL}}
    \left(
    p\,\middle\|\,\pi_i^t(\cdot\mid s)
    \right)
    \right\}.
    \end{equation}
\EndFor
\end{algorithmic}
\end{algorithm}
\footnotetext{Only in Algorithm~\ref{alg:episodic-kl-projected-npg}, for simplicity and with a slight abuse of notation, we write $\theta_i^{t,k}$ for the variable $\theta_i$ at the global time corresponding to within-episode time $k$ of episode $t$.}


Before proceeding with the analysis of the algorithm, we first establish that the bias and second moment of the clipped episodic estimator can be controlled. In particular, the following lemma shows that the clipped estimator admits a decomposition into a predictable component that remains within $L_{\widehat A}$ of the marginalized advantage and an estimation-error component whose occupancy-weighted conditional second moment decays as $1/L_t$. Controlling these two quantities is essential for deriving a high-probability NE-regret bound. The proof of the lemma is given in Appendix~\ref{app:episodic-estimator-noise}.

\begin{lemma}\label{lem:random-count-weighted-variance}
Fix $t$ and suppose that the policy $\pi^t$ remains fixed throughout
episode $t$. Let $\tau_{t,r}(s)$ denote the time of the $r$th visit to $s$
during episode $t$. Consider the averaged estimator and its clipped version
\begin{align}\label{eq:g-random-count-advantage-estimator}
\widetilde r_i^t(s,a_i)
:&=
\frac{1}{L_t}
\sum_{r=1}^{L_t}
\widehat A_i^{\tau_{t,r}(s)}(s,a_i),
\qquad \ \  a_i\in\mathcal A_i,\cr
r_i^t(s,a_i)
:&=
\operatorname{clip}_B\!\left(\widetilde r_i^t(s,a_i)\right),
\end{align}
where we note that $\|r_i^t\|_\infty\le \frac{B}{2}$ almost surely. Moreover, suppose that
$\pi_i^t(a_i\mid s) \ge \zeta/|\mathcal{A}_i|
\ \forall i,s,a_i$, $B\ge 4$, and the episodic estimation oracle Assumption \ref{ass:episodic-oracle} holds. Define
\begin{align}
\xi_i^t(s,a_i)
:=r_i^t(s,a_i)&-
\frac{1}{L_t}\sum_{r=1}^{L_t}
\mathbb E_{\tau_{t,r}(s)}
\left[\widehat A_i^{\tau_{t,r}(s)}(s,a_i)\right].
\label{eq:random-count-noise}
\end{align}
Then, almost surely,
\begin{align}
&\left|
\bar A_i^{\pi^t}(s,a_i)
-
r_i^t(s,a_i)+\xi_i^t(s,a_i)
\right|
\le
L_{\widehat A}
\qquad \forall i,s,a_i,
\label{eq:global-table-bias}\\
&\mathcal{V}_i^t
:=
\mathbb E_t
\bigg[
\sum_{s\in\mathcal S}\sum_{a_i\in\mathcal A_i}
d_\mu^{\pi^t}(s)
\pi_i^t(a_i\mid s)
\left(
\xi_i^t(s,a_i)
\right)^2
\bigg]
\le
\frac{|\mathcal A_i|}{L_t}.
\label{eq:random-count-weighted-variance}
\end{align}
\end{lemma}

\subsection{Analysis of the Episodic Decentralized Algorithm}

In this subsection, we show that if every player follows Algorithm~\ref{alg:episodic-kl-projected-npg}, then the collective behavior of the players approaches an $\epsilon$-NE in terms of the time-averaged NE gap (NE regret), with polynomial time and sample complexity. The analysis is organized into three parts. First, in Lemma~\ref{lem:kl-projected-potential-improvement}, we establish a descent bound for the potential function in terms of the occupancy-weighted variance of the estimation noise and the KL movement of the policy profiles. Next, in Lemma~\ref{lem:kl-projected-ne-gap}, we derive an upper bound on the NE gap in terms of the same quantities. While variants of the potential-improvement and NE-gap lemmas exist in the literature \cite{guo2025markov,introDing2022,introSun2023,introAlatur2024,introDong2024}, they typically involve distribution-mismatch coefficients, provide guarantees in expectation rather than with high probability, and assume exact-oracle access or Euclidean-projection-based algorithms. We therefore redesign and extend these arguments to the KL-projected NPG setting, where multiplicative updates and bandit estimation errors require more delicate treatment to obtain high-probability regret bounds. Finally, in Theorem~\ref{thm:high-prob-ne-regret}, we combine these results to obtain a high-probability NE-regret bound and tune the free parameters to optimize the resulting rate.

\begin{lemma}\label{lem:kl-projected-potential-improvement}
Let Assumptions \ref{ass:alpha-potential-ergodicity} and \ref{ass:episodic-oracle} hold, and choose a clipping parameter $B\ge 4$. Then, 
\begin{align}
\mathbb E_t\left[
\Psi(\pi^t)-\Psi(\pi^{t+1})\right]&\ge
\frac{1}{4\eta}
\sum_{i=1}^n
\mathbb E_t \bigg[
\sum_{s\in\mathcal S}
d_\mu^{\pi^t}(s)
\mathcal D_i^t(s)\bigg]-
\frac{\eta e^{\frac{2\eta B}{1-\gamma}}}{2(1-\gamma)^2}
\sum_{i=1}^n\mathcal{V}_i^t
\notag\\
&\quad-
\frac{n\eta L_{\widehat A}^2}{(1-\gamma)^2}
-
\frac{L_V n^2\eta^2B^2}
{8(1-\gamma)^2}-
\frac{n\eta^2B^2L_d\left(1+L_d\right)}
{(1-\gamma)^3}
-n\alpha,
\nonumber
\end{align}
where $L_V=
\frac{4}{1-\gamma}
+
\frac{12\gamma\delta_P}{(1-\gamma)^2}
+
\frac{8\gamma^2\delta_P^2}{(1-\gamma)^3}$, $L_d=\frac{\gamma\delta_P}{1-\gamma}$, and
$\mathcal{V}^t_{i}$ is the conditional second moment 
\eqref{eq:random-count-weighted-variance}.
\end{lemma}

\begin{proof}
For $i=1,\ldots,n$, define
\[
\widetilde\pi^{t,i}
:=
\left(
\pi_1^{t+1},\ldots,\pi_i^{t+1},
\pi_{i+1}^t,\ldots,\pi_n^t
\right),
\qquad
\widetilde\pi^{t,0}:=\pi^t.
\]
Then
\begin{equation}\nonumber
\Psi(\pi^t)-\Psi(\pi^{t+1})
=
\sum_{i=1}^n
\left[
\Psi(\widetilde\pi^{t,i-1})
-
\Psi(\widetilde\pi^{t,i})
\right].
\end{equation}

For every player $i$, let
$\Delta_i^t(s):=\pi_i^t(\cdot\mid s)-\pi_i^{t+1}(\cdot\mid s)$,
and define
\begin{equation}
\label{eq:G_i-def}
G_i^t
:=
V_i^{(\pi_i^t,\pi_{-i}^t)}(\mu)
-
V_i^{(\pi_i^{t+1},\pi_{-i}^t)}(\mu)
=
\frac{1}{1-\gamma}
\sum_{s\in\mathcal S}
d_\mu^{(\pi_i^{t+1},\pi_{-i}^t)}(s)
\left\langle
\Delta_i^t(s),
\bar A_i^{\pi^t}(s)
\right\rangle,
\end{equation}
where the second equality follows from the performance-difference lemma (Lemma \ref{lemm:performance-difference}) applied to player $i$'s discounted value function. Therefore, by the
$\alpha$-potential property and the mixed-policy sensitivity of player $i$'s discounted value function (Lemma \ref{lem:mixed-policy-sensitivity}), we have
\begin{align}\nonumber
\Psi(\widetilde\pi^{t,i-1})
&-
\Psi(\widetilde\pi^{t,i})
\ge
V_i^{(\pi_i^t,\widetilde\pi_{-i}^{t,i-1})}(\mu)
-
V_i^{(\pi_i^{t+1},\widetilde\pi_{-i}^{t,i-1})}(\mu)
-\alpha
\cr
&=
G_i^t
-\alpha
+
\left[
V_i^{(\pi_i^t,\widetilde\pi_{-i}^{t,i-1})}(\mu)
-
V_i^{(\pi_i^{t+1},\widetilde\pi_{-i}^{t,i-1})}(\mu)
-
V_i^{(\pi_i^t,\pi_{-i}^t)}(\mu)
+
V_i^{(\pi_i^{t+1},\pi_{-i}^t)}(\mu)
\right]
\cr
&\ge
G_i^t
-\alpha
-
L_V
\left\|
\pi_i^{t+1}-\pi_i^t
\right\|_{\mathrm{TV},\infty}
\left\|
\widetilde\pi_{-i}^{t,i-1}
-
\pi_{-i}^t
\right\|_{\mathrm{TV},\infty}.
\end{align}
Moreover,
\[
\left\|
\widetilde\pi_{-i}^{t,i-1}
-
\pi_{-i}^t
\right\|_{\mathrm{TV},\infty}
\le
\sum_{j<i}
\left\|
\pi_j^{t+1}-\pi_j^t
\right\|_{\mathrm{TV},\infty},
\]
and hence
\begin{align}\label{eq:projected-counterfactual-bound}
\Psi(\pi^t)-\Psi(\pi^{t+1})
&\ge
\sum_{i=1}^n G_i^t
-
L_V
\sum_{i=1}^n\sum_{j<i}
\left\|
\pi_i^{t+1}-\pi_i^t
\right\|_{\mathrm{TV},\infty}
\left\|
\pi_j^{t+1}-\pi_j^t
\right\|_{\mathrm{TV},\infty}
-n\alpha\cr 
&\ge
\sum_{i=1}^n G_i^t
-\frac{L_V n^2\eta^2B^2}{8(1-\gamma)^2}
-n\alpha,
\end{align}
where the second inequality follows by using Lemma \ref{lem:kl-projected-local-geometry} (part iv). Using the conditional bias bound \eqref{eq:global-table-bias} and the definition of the mean-error \eqref{eq:random-count-noise}, we have
\begin{align}
\left\langle
\bar A_i^{\pi^t}(s),
\Delta_i^t(s)
\right\rangle
&=
\left\langle
r_i^t(s),
\Delta_i^t(s)
\right\rangle
-
\left\langle
\xi_i^t(s),
\Delta_i^t(s)
\right\rangle+
\left\langle
\bar A_i^{\pi^t}(s)-r_i^t(s)+\xi_i^t(s),
\Delta_i^t(s)
\right\rangle
\notag\\
&\ge
\frac{1-\gamma}{\eta}
\mathcal D_i^t(s)
-
\left|
\left\langle
\xi_i^t(s),
\Delta_i^t(s)
\right\rangle
\right|
-
L_{\widehat A}
\left\|
\Delta_i^t(s)
\right\|_1.
\label{eq:projected-decomposition}
\end{align}
where the inequality follows from Lemma
\ref{lem:kl-projected-local-geometry} (part iii),
\eqref{eq:global-table-bias}, and H\"older's inequality.

We next bound the stochastic inner product in \eqref{eq:projected-decomposition}. Inserting
$\sqrt{\pi_i^t(a_i\mid s)}$ and its reciprocal gives
\begin{align}
\left|
\left\langle
\xi_i^t(s),
\Delta_i^t(s)
\right\rangle
\right|
&=
\left|
\sum_{a_i\in\mathcal A_i}
\left[
\sqrt{\pi_i^t(a_i\mid s)}
\,\xi_i^t(s,a_i)
\right]
\bigg[
\frac{
\Delta_i^t(s,a_i)
}{
\sqrt{\pi_i^t(a_i\mid s)}
}
\bigg]
\right|
\notag\\
&\le
\bigg(
\sum_{a_i\in\mathcal A_i}
\pi_i^t(a_i\mid s)
\left(
\xi_i^t(s,a_i)
\right)^2
\bigg)^{1/2}
\bigg(
\sum_{a_i\in\mathcal A_i}
\frac{
\left(
\Delta_i^t(s,a_i)
\right)^2
}{
\pi_i^t(a_i\mid s)
}
\bigg)^{1/2}
\notag\\
&\le
e^{\eta B/(1-\gamma)}
\bigg(
\sum_{a_i\in\mathcal A_i}
\pi_i^t(a_i\mid s)
\left(
\xi_i^t(s,a_i)
\right)^2
\bigg)^{1/2}
\sqrt{\mathcal D_i^t(s)}
\notag\\
&\le
\frac{1-\gamma}{2\eta}
\mathcal D_i^t(s)
+
\frac{
\eta e^{2\eta B/(1-\gamma)}
}{
2(1-\gamma)
}
\sum_{a_i\in\mathcal A_i}
\pi_i^t(a_i\mid s)
\left(
\xi_i^t(s,a_i)
\right)^2,
\label{eq:projected-weighted-noise}
\end{align}
where the first inequality is by Cauchy--Schwarz, the second inequality uses
Lemma~\ref{lem:kl-projected-local-geometry} (part ii), and the last inequality
follows from Young's inequality. Substituting
\eqref{eq:projected-weighted-noise} into
\eqref{eq:projected-decomposition} yields
\begin{align}\label{eq:projected-Gi-local}
\left\langle
\bar A_i^{\pi^t}(s),
\Delta_i^t(s)
\right\rangle
&\ge
\frac{1-\gamma}{2\eta}
\mathcal D_i^t(s)
-
\frac{
\eta e^{\frac{2\eta B}{1-\gamma}}
}{
2(1-\gamma)
}
\sum_{a_i\in\mathcal A_i}
\pi_i^t(a_i\mid s)
\left(
\xi_i^t(s,a_i)
\right)^2-L_{\widehat A}
\left\|
\Delta_i^t(s)
\right\|_1.
\end{align}

Adding and subtracting $d_\mu^{\pi^t}$ in \eqref{eq:G_i-def},
substituting \eqref{eq:projected-Gi-local}, and taking conditional
expectation, we obtain
\begin{align}
\mathbb E_t[G_i^t]
&\ge
\frac{1}{2\eta}
\mathbb E_t
\bigg[
\sum_{s\in\mathcal S}
d_\mu^{\pi^t}(s)
\mathcal D_i^t(s)\bigg]-
\frac{L_{\widehat A}}{1-\gamma}
\mathbb E_t
\bigg[
\sum_{s\in\mathcal S}
d_\mu^{\pi^t}(s)
\left\|
\Delta_i^t(s)
\right\|_1\bigg]
\notag\\
&\quad-
\frac{
\eta e^{2\eta B/(1-\gamma)}
}{
2(1-\gamma)^2
}
\mathbb E_t
\bigg[
\sum_{s\in\mathcal S}\sum_{a_i}
d_\mu^{\pi^t}(s)
\pi_i^t(a_i\mid s)
\left(
\xi_i^t(s,a_i)
\right)^2\bigg]
-
\frac{\eta^2B^2L_d\left(1+L_d\right)}
{(1-\gamma)^3},
\label{eq:projected-Gi-final}
\end{align}
where the last term in \eqref{eq:projected-Gi-final} is obtained using Lemma~\ref{lemm:occupancy-sensitivity} (part i), Lemma \ref{lem:kl-projected-local-geometry} (part iv), and Lemma \ref{lem:span-inner-product}:
\begin{align}
&\frac{1}{1-\gamma}
\left|
\sum_{s\in\mathcal S}
\left(
d_\mu^{(\pi_i^{t+1},\pi_{-i}^t)}(s)
-
d_\mu^{\pi^t}(s)
\right)
\left\langle
\Delta_i^t(s),
\bar A_i^{\pi^t}(s)
\right\rangle
\right|
\notag\\
&\qquad\le
\frac{2}{1-\gamma}
\left(1+L_d\right)
\left\|
d_\mu^{(\pi_i^{t+1},\pi_{-i}^t)}
-
d_\mu^{\pi^t}
\right\|_{\mathrm{TV}}
\left\|
\Delta_i^t
\right\|_{1,\infty}
\notag\\
&\qquad\le
\frac{L_d(1+L_d)}{1-\gamma}
\left(
\sum_{s\in\mathcal S}
d_\mu^{\pi^t}(s)
\left\|
\Delta_i^t(s)
\right\|_1
\right)
\left\|
\Delta_i^t
\right\|_{1,\infty}
\notag\\
&\qquad\le
\frac{\eta^2B^2L_d(1+L_d)}
{(1-\gamma)^3}.
\nonumber
\end{align}

Finally, the second error term in
\eqref{eq:projected-Gi-final} satisfies
\begin{align}
\frac{L_{\widehat A}}{1-\gamma}
\sum_{s\in\mathcal S}
d_\mu^{\pi^t}(s)
\left\|
\Delta_i^t(s)
\right\|_1
&\le
\frac{L_{\widehat A}}{1-\gamma}
\sqrt{
\sum_{s\in\mathcal S}
d_\mu^{\pi^t}(s)
\mathcal D_i^t(s)
}\le
\frac{1}{4\eta}
\sum_{s\in\mathcal S}
d_\mu^{\pi^t}(s)
\mathcal D_i^t(s)
+
\frac{\eta L_{\widehat A}^2}
{(1-\gamma)^2},
\label{eq:projected-alpha-bias-absorption}
\end{align}
where the first inequality uses Pinsker's inequality $
\left\|
\Delta_i^t(s)
\right\|_1^2\leq \mathcal D_i^t(s)$ together with Jensen
(note that
$\sum_s d_\mu^{\pi^t}(s)=1$),
and the second inequality follows from Young's inequality.
Finally, taking conditional expectation in
\eqref{eq:projected-counterfactual-bound}, and substituting
\eqref{eq:projected-Gi-final} and
\eqref{eq:projected-alpha-bias-absorption}, completes the proof.
\end{proof}

Next, we proceed to derive an upper bound on the NE gap at a generic time $t$ in terms of the same quantities: the weighted conditional second moment $\mathcal{V}_i^t$ and the KL policy movement $\mathcal{D}_i^t(s)$.

\begin{lemma}
\label{lem:kl-projected-ne-gap}
Let $\mathcal{V}_i^t$ and $\mathcal{D}_i^t(s)$ be defined as \eqref{eq:random-count-weighted-variance} and \eqref{eq:KL_symm}, respectively. For the KL-projected NPG,
\begin{align}\label{eq:NE-gap-bound}
\operatorname{Gap}(\pi^t)
&\le\bigg[
\frac{\left(1+L_d\right)}{2(1-\gamma)}
+
\frac{1}{\eta}\sqrt{\frac{2A_{\max}}{\zeta}}
\bigg]
\left\{
\sum_{i=1}^n
\mathbb E_t\bigg[
\sum_s d_\mu^{\pi^t}(s)\mathcal D_i^t(s)
\bigg]
\right\}^{1/2}
\notag\\
&\qquad+
\frac{2L_{\widehat A}}{1-\gamma}
+
\frac{1}{1-\gamma}\sqrt{\frac{2A_{\max}}{\zeta}}\sqrt{\max_{i\in[n]}\mathcal V_i^t}
+
\frac{\zeta+2L_d}{1-\gamma}
\left(1+L_d\right).
\end{align}
\end{lemma}

\begin{proof}
Fix a player $i$ and let $\pi_i^{t,*}$ be its best response to $\pi_{-i}^t$. By the performance-difference lemma 
\begin{align}\label{eq:ne-gap-first-occupancy-change}
\!\!\!\!\!\!\! V_i(\pi^t)&-V_i(\pi_i^{t,*},\pi_{-i}^t)
=
\frac{1}{1-\gamma}
\sum_{s\in\mathcal S}
d_\mu^{(\pi_i^{t,*},\pi_{-i}^t)}(s)
\left\langle
\pi_i^t(\cdot\mid s)-\pi_i^{t,*}(\cdot\mid s),
\bar A_i^{\pi^t}(s)
\right\rangle
\cr
&=
\frac{1}{1-\gamma}
\mathbb E_t\bigg[
\sum_{s\in\mathcal S}
d_\mu^{(\pi_i^{t,*},\pi_{-i}^t)}(s)
\left\langle
\pi_i^t(\cdot\mid s)-\pi_i^{t,*}(\cdot\mid s),
\bar A_i^{\pi^t}(s)
\right\rangle\bigg]
\cr
&\le
\frac{1}{1-\gamma}
\mathbb E_t\bigg[
\sum_{s\in\mathcal S}
d_\mu^{\pi^t}(s)
\left\langle
\pi_i^t(\cdot\mid s)-\pi_i^{t,*}(\cdot\mid s),
\bar A_i^{\pi^t}(s)
\right\rangle\bigg]+
\frac{2L_d(1+L_d)}{1-\gamma},
\end{align}
where the second equality follows because all the variables $\pi^t, \pi^{t,*}$, and $\bar A_i^{\pi^t}(s)$ are $\mathcal F_t$-measurable. The last inequality is obtained by changing the occupancy measure to $d_\mu^{\pi^t}$ and bounding its discrepancy using part (i) of
Lemma~\ref{lemm:occupancy-sensitivity} by $\big\|
d_\mu^{(\pi_i^{t,*},\pi_{-i}^t)}-d_\mu^{\pi^t}\big\|_1
\le 2L_d$, while using Lemma~\ref{lem:span-inner-product} and
$\|\pi_i^t(\cdot\mid s)-\pi_i^{t,*}(\cdot\mid s)\|_{\mathrm{TV}}\le 1$.

The first-order optimality condition for the KL-projected NPG gives,
for every $q(\cdot\mid s)\in\Delta_\zeta(\mathcal A_i)$,
\begin{align}\label{eq:D-A-A}
\left\langle
r_i^t(s),
\pi_i^{t+1}(\cdot\mid s)-q(\cdot\mid s)
\right\rangle
&\le
\frac{1-\gamma}{\eta}
\left\langle
\log\frac{\pi_i^{t+1}(\cdot\mid s)}
{\pi_i^t(\cdot\mid s)},
q(\cdot\mid s)-\pi_i^{t+1}(\cdot\mid s)
\right\rangle \cr
&\le
\frac{1-\gamma}{\eta}
\left\|
\log\frac{\pi_i^{t+1}(\cdot\mid s)}
{\pi_i^t(\cdot\mid s)}
\right\|_2
\left\|
q(\cdot\mid s)-\pi_i^{t+1}(\cdot\mid s)
\right\|_2 \cr
&\le
\frac{1-\gamma}{\eta}
\sqrt{2}
\left\|
\log\frac{\pi_i^{t+1}(\cdot\mid s)}
{\pi_i^t(\cdot\mid s)}
\right\|_2 \cr
&\le
\frac{1-\gamma}{\eta}
\sqrt{\frac{2|\mathcal A_i|}{\zeta}}
\sqrt{\mathcal D_i^t(s)},
\end{align}
where the second inequality uses Cauchy--Schwarz, while the third uses the
fact that both $q(\cdot\mid s)$ and $\pi_i^{t+1}(\cdot\mid s)$ are
probability vectors. Finally, since
$\pi_i^t(\cdot\mid s),\pi_i^{t+1}(\cdot\mid s)
\in\Delta_\zeta(\mathcal A_i)$, every coordinate is at least
$\zeta/|\mathcal A_i|$, and the coordinatewise inequality $\log^2\frac{x}{y}\le\frac{x-y}{\min\{x,y\}}
\log\frac{x}{y}$ implies
\begin{align}\nonumber
\left\|
\log
\frac{\pi_i^{t+1}(\cdot\mid s)}
{\pi_i^t(\cdot\mid s)}
\right\|_2^2
&\le
\frac{|\mathcal A_i|}{\zeta}
\sum_{a_i\in\mathcal A_i}
\left(
\pi_i^{t+1}(a_i\mid s)-\pi_i^t(a_i\mid s)
\right)
\log
\frac{\pi_i^{t+1}(a_i\mid s)}
{\pi_i^t(a_i\mid s)}=
\frac{|\mathcal A_i|}{\zeta}
\mathcal D_i^t(s).
\end{align}

Next, consider the feasible truncated policy $$q_i^{t,*}(\cdot\mid s)
=
(1-\zeta)\pi_i^{t,*}(\cdot\mid s)
+
\frac{\zeta}{|\mathcal A_i|}\mathbf 1,$$
and let
$\Delta_i^t(s)=
\pi_i^t(\cdot\mid s)-\pi_i^{t+1}(\cdot\mid s)$. Then, we can bound the first term in \eqref{eq:ne-gap-first-occupancy-change} as
\begin{align}
\mathbb E_t\bigg[&
\sum_{s\in\mathcal S}
d_\mu^{\pi^t}(s)
\left\langle
\pi_i^t(\cdot\mid s)-\pi_i^{t,*}(\cdot\mid s),
\bar A_i^{\pi^t}(s)
\right\rangle\bigg]\cr 
&=
\mathbb E_t\bigg[
\sum_{s\in\mathcal S}
d_\mu^{\pi^t}(s)\left\langle
\Delta_i^t(s),
\bar A_i^{\pi^t}(s)
\right\rangle\bigg]+\mathbb E_t\bigg[
\sum_{s\in\mathcal S}
d_\mu^{\pi^t}(s)
\left\langle
q_i^{t,*}(\cdot\mid s)-\pi_i^{t,*}(\cdot\mid s),
\bar A_i^{\pi^t}(s)
\right\rangle
\bigg]
\cr
&+\mathbb E_t\bigg[
\sum_{s\in\mathcal S}d_\mu^{\pi^t}(s)\left\langle
\pi_i^{t+1}(\cdot\mid s)-q_i^{t,*}(\cdot\mid s),
\bar A_i^{\pi^t}(s)
\right\rangle\bigg]\cr
&\le
\frac12
\left(1+L_d\right)
\mathbb E_t\bigg[
\sum_{s\in\mathcal S}
d_\mu^{\pi^t}(s)
\left\|\Delta_i^t(s)\right\|_1
\bigg]+
\zeta
\left(1+L_d\right)
\cr
&+
\mathbb E_t\bigg[
\sum_{s\in\mathcal S}
d_\mu^{\pi^t}(s)
\left\langle
\pi_i^{t+1}(\cdot\mid s)-q_i^{t,*}(\cdot\mid s),
r_i^t(s)-\xi_i^t(s)
\right\rangle
\bigg]+2L_{\widehat A}
\cr
&=
\frac12
\left(1+L_d\right)
\mathbb E_t\bigg[
\sum_{s\in\mathcal S}
d_\mu^{\pi^t}(s)
\left\|\Delta_i^t(s)\right\|_1\bigg]+2L_{\widehat A}
+
\zeta
\left(1+L_d\right)
\cr
&+
\mathbb E_t\bigg[
\sum_{s\in\mathcal S}
d_\mu^{\pi^t}(s)
\left\langle
\pi_i^{t+1}(\cdot\mid s)-q_i^{t,*}(\cdot\mid s),
r_i^t(s)
\right\rangle
\bigg]\cr 
&+
\mathbb E_t\bigg[
\sum_{s\in\mathcal S}
d_\mu^{\pi^t}(s)
\left\langle
q_i^{t,*}(\cdot\mid s)-\pi_i^{t+1}(\cdot\mid s),
\xi_i^t(s)
\right\rangle\bigg]
\cr
&\le
\frac12
\left(1+L_d\right)
\mathbb E_t\bigg[
\sum_{s\in\mathcal S}
d_\mu^{\pi^t}(s)
\left\|\Delta_i^t(s)\right\|_1
\bigg]+2L_{\widehat A}
+
\zeta
\left(1+L_d\right)
\cr
&+\frac{1-\gamma}{\eta}
\sqrt{\frac{2|\mathcal A_i|}{\zeta}}\,
\mathbb E_t\bigg[
\sum_{s\in\mathcal S}
d_\mu^{\pi^t}(s)
\sqrt{\mathcal D_i^t(s)}\bigg]\cr 
&+
\mathbb E_t\bigg[
\sum_{s\in\mathcal S}
d_\mu^{\pi^t}(s)
\left\langle
q_i^{t,*}(\cdot\mid s)-\pi_i^{t+1}(\cdot\mid s),
\xi_i^t(s)
\right\rangle\bigg].
\label{eq:ne-gap-key}
\end{align}
Here, the first inequality uses Lemma~\ref{lem:span-inner-product}, the truncation error bound
$\|q_i^{t,*}(\cdot\mid s)-\pi_i^{t,*}(\cdot\mid s)\|_1\leq 2\zeta$,
the a.s.\ conditional bias bound \eqref{eq:global-table-bias}, and
$\|\pi_i^{t+1}(\cdot\mid s)-q_i^{t,*}(\cdot\mid s)\|_1\le 2$. The next equality follows by separating $r_i^t(s)-\xi_i^t(s)$ into its two terms, and the last inequality follows from \eqref{eq:D-A-A} with
$q=q_i^{t,*}$. Now, using Pinsker's inequality
$\|\Delta_i^t(s)\|_1^2\le \mathcal D_i^t(s)$ in \eqref{eq:ne-gap-key} gives
\begin{align}
\mathbb E_t\bigg[
&\sum_{s\in\mathcal S}
d_\mu^{\pi^t}(s)
\left\langle
\pi_i^t(\cdot\mid s)-\pi_i^{t,*}(\cdot\mid s),
\bar A_i^{\pi^t}(s)
\right\rangle\bigg]
\nonumber\\
&\le
\bigg(
\frac{1+L_d}{2}+
\frac{1-\gamma}{\eta}
\sqrt{\frac{2|\mathcal A_i|}{\zeta}}
\bigg)
\mathbb E_t\bigg[
\sum_{s\in\mathcal S}
d_\mu^{\pi^t}(s)
\sqrt{\mathcal D_i^t(s)}\bigg]+
2L_{\widehat A}
+
\zeta
\left(1+L_d\right)
\nonumber\\
&\qquad+
\mathbb E_t\bigg[
\sum_{s\in\mathcal S}
d_\mu^{\pi^t}(s)
\left\langle
q_i^{t,*}(\cdot\mid s)-\pi_i^{t+1}(\cdot\mid s),
\xi_i^t(s)
\right\rangle\bigg]
\nonumber\\
&\le
\bigg(
\frac{1+L_d}{2}
+
\frac{1-\gamma}{\eta}
\sqrt{\frac{2|\mathcal A_i|}{\zeta}}
\bigg)
\bigg\{
\mathbb E_t\bigg[
\sum_s
d_\mu^{\pi^t}(s)
\mathcal D_i^t(s)\bigg]\bigg\}^{1/2}\!\!\!\!+
2L_{\widehat A}
+
\zeta
\left(1+L_d\right)
\nonumber\\
&\qquad+
\mathbb E_t\bigg[
\sum_{s\in\mathcal S}
d_\mu^{\pi^t}(s)
\left\langle
q_i^{t,*}(\cdot\mid s)-\pi_i^{t+1}(\cdot\mid s),
\xi_i^t(s)
\right\rangle\bigg].
\label{eq:ne-gap-random-occupancy-CS}
\end{align}
where the second inequality holds by Jensen's inequality for the concave function $x\mapsto\sqrt{x}$,
applied first over the state probability distribution $d_\mu^{\pi^t}$ and then to the conditional
expectation. Now, it remains to control the stochastic noise term in
\eqref{eq:ne-gap-random-occupancy-CS}. Applying weighted Cauchy--Schwarz, we get
\begin{align}
&\mathbb E_t\bigg[
\sum_s
d_\mu^{\pi^t}(s)
\left\langle
q_i^{t,*}(\cdot\mid s)-\pi_i^{t+1}(\cdot\mid s),
\xi_i^t(s)
\right\rangle\bigg]
\cr
&\le
\mathbb E_t\bigg[
\bigg(
\!\sum_s d_\mu^{\pi^t}(s)
\sum_{a_i}
\pi_i^t(a_i\mid s)
\bigl(\xi_i^t(s,a_i)\bigr)^2
\bigg)^{\frac{1}{2}}
\!\bigg(\!
\sum_s d_\mu^{\pi^t}(s)
\sum_{a_i}
\frac{
\bigl(q_i^{t,*}(a_i\mid s)\!-\!\pi_i^{t+1}(a_i\mid s)\bigr)^2
}{
\pi_i^t(a_i\mid s)
}
\bigg)^{1/2}
\bigg]
\cr
&\le
\sqrt{\frac{2|\mathcal A_i|}{\zeta}}\,
\mathbb E_t\bigg[
\bigg(
\sum_s d_\mu^{\pi^t}(s)
\sum_{a_i}
\pi_i^t(a_i\mid s)
\bigl(\xi_i^t(s,a_i)\bigr)^2
\bigg)^{1/2}
\bigg]\le
\sqrt{\frac{2|\mathcal A_i|}{\zeta}}\,
\sqrt{\mathcal V_i^t},
\label{eq:ne-gap-noise-final}
\end{align}
where for the first inequality, we write
\[
\begin{aligned}
&d_\mu^{\pi^t}(s)\xi_i^t(s,a_i)
\bigl(q_i^{t,*}(a_i\mid s)-\pi_i^{t+1}(a_i\mid s)\bigr)
\\
&\qquad=
\Big(
\sqrt{d_\mu^{\pi^t}(s)\pi_i^t(a_i\mid s)}\,
\xi_i^t(s,a_i)
\Big)
\Bigg(
\sqrt{d_\mu^{\pi^t}(s)}
\ \frac{
q_i^{t,*}(a_i\mid s)-\pi_i^{t+1}(a_i\mid s)
}{
\sqrt{\pi_i^t(a_i\mid s)}
}
\Bigg),
\end{aligned}
\]
and apply Cauchy--Schwarz over $(s,a_i)$. Moreover, the second inequality in \eqref{eq:ne-gap-noise-final} uses the exploration bound $\pi_i^t(a_i\mid s)\ge\zeta/|\mathcal A_i|$, $\sum_s d_\mu^{\pi^t}(s)=1$, and the fact that $\sum_{a_i}(p(a_i)-q(a_i))^2\le2$ for any two probability distributions $p$ and $q$. Finally, the last inequality in \eqref{eq:ne-gap-noise-final} follows from conditional Jensen's inequality
and the definition of $\mathcal V_i^t$.

By combining \eqref{eq:ne-gap-first-occupancy-change},
\eqref{eq:ne-gap-random-occupancy-CS}, and
\eqref{eq:ne-gap-noise-final}, and taking the maximum over
$i=1,\ldots,n$, we obtain
\begin{align}\nonumber
\operatorname{Gap}(\pi^t)
&\le
\max_{i\in[n]}
\left\{
\bigg[
\frac{1+L_d}{2(1-\gamma)}+
\frac{1}{\eta}\sqrt{\frac{2|\mathcal A_i|}{\zeta}}
\bigg]
\bigg\{
\mathbb E_t \bigg[
\sum_s d_\mu^{\pi^t}(s)\mathcal D_i^t(s)
\bigg]
\bigg\}^{1/2}
\right\}\cr
&\qquad\qquad+
\frac{2L_{\widehat A}}{1-\gamma}
+
\frac{1}{1-\gamma}\sqrt{\frac{2A_{\max}}{\zeta}}\sqrt{\max_{i\in[n]}\mathcal V_i^t}
+
\frac{\zeta+2L_d}{1-\gamma}
\left(1+L_d\right).
\end{align}
This relation together with $\mathbb E_t[
\sum_s d_\mu^{\pi^t}(s)\mathcal D_i^t(s)]
\!\le\!
\sum_{j=1}^n
\mathbb E_t[
\sum_s d_\mu^{\pi^t}(s)\mathcal D_j^t(s)]$ and $A_{\max}=\max |\mathcal{A}_i|$, 
proves the desired bound \eqref{eq:NE-gap-bound}.
\end{proof}

\begin{theorem}\label{thm:high-prob-ne-regret}
Fix $T\ge1$ and $\rho\in(0,1)$. Assume
$T\ge\log(4/\rho)$, $\alpha+L_{\widehat{A}}\!>\!0$, and choose
\begin{align}
\eta
&=
\frac{1-\gamma}{8}
\max\bigg\{
T^{-1/2},\,
\frac{\sqrt{(1-\gamma)\alpha}}
{\sqrt{10n}(1+L_d)}
\bigg\},
\label{eq:parameters}\\
\zeta
&=
\min\bigg\{
\frac12,\,
2\max\bigg\{
\frac{n^{1/3}A_{\max}^{1/3}L_{\widehat A}^{2/3}}
{(1+L_d)^{2/3}},\,
\frac{\sqrt n\,A_{\max}^{1/3}(1-\gamma)^{1/6}\alpha^{1/6}}
{(1+L_d)^{1/3}}
\bigg\}
\bigg\}.
\label{eq:parameters-zeta}
\end{align}
Suppose that each player follows Algorithm
\ref{alg:episodic-kl-projected-npg} with $L_t\ge
\frac{A_{\max}}{\eta}$ and $B=4$. Then, with probability at least $1-\rho$, the NE regret is at most 
\begin{align}\label{eq:final-NE-regret-episodic}
\frac1T
\sum_{t=0}^{T-1}
\operatorname{Gap}(\pi^t)\lesssim{} \mathfrak C
\left(\!
\frac{\log(4/\rho)}{T}
\!\right)^{\!1/4}
\!\!\!\!\!+\frac{\big[nA_{\max}(1\!+\!L_d)\big]^{1/3}}{1-\gamma}
L_{\widehat A}^{2/3}\!+\!
\frac{\sqrt n\,A_{\max}^{1/3}(1\!+\!L_d)^{2/3}}
{1-\gamma}
\alpha^{1/6}
\!+\!
\frac{L_d(1\!+\!L_d)}{1-\gamma}.
\end{align}
Here, $\lesssim$ denotes an inequality up to a universal numerical constant,
$L_d=\frac{\gamma\delta_P}{1-\gamma}$, and
\begin{equation}\label{eq:C-stat}
\mathfrak C
:=\frac{1}{1-\gamma}\bigg(
1+L_d
+
\frac{n^{5/6}A_{\max}^{1/3}(1+L_d)^{4/3}}
{
L_{\widehat A}^{1/3}+
\big(n(1-\gamma)(1+L_d)^2\alpha\big)^{1/12}}\bigg).
\end{equation}
\end{theorem}

\begin{proof}
By the choice of $L_t$ and Lemma~\ref{lem:random-count-weighted-variance}, simultaneously for all $t$,
\begin{equation}
\max_{i\in[n]}\mathcal V_i^t\le\eta.
\label{eq:variance-bound-episodic}
\end{equation}
Define
\begin{equation}\nonumber
K_t
:=
\sum_{i=1}^n
\mathbb E_t
\bigg[
\sum_{s\in\mathcal S}
d_\mu^{\pi^t}(s)
\mathcal D_i^t(s)
\bigg].
\end{equation}
By Lemma~\ref{lem:kl-projected-potential-improvement} and
\eqref{eq:variance-bound-episodic}, we have
\begin{align}
K_t
\le{}&
4\eta\,
\mathbb E_t
\left[
\Psi(\pi^t)-\Psi(\pi^{t+1})
\right]
+
\frac{2n\eta^3
e^{\frac{2\eta B}{1-\gamma}}}
{(1-\gamma)^2}
\notag\\
&+
\frac{4n\eta^2
L_{\widehat A}^2}
{(1-\gamma)^2}
+
\frac{L_Vn^2\eta^3B^2}
{2(1-\gamma)^2}
+
\frac{4n\eta^3B^2L_d(1+L_d)}
{(1-\gamma)^3}
+
4n\alpha\eta.
\label{eq:hp-K-one-step}
\end{align}

We next control the fluctuation of the potential improvement. Using the
intermediate profiles $\widetilde\pi^{t,i}$ defined in the proof of
Lemma~\ref{lem:kl-projected-potential-improvement}, the $\alpha$-potential property, and the performance-difference identity (Lemma \ref{lemm:performance-difference}) applied sequentially to each player, we have
\begin{align}
\left|
\Psi(\pi^{t+1})-\Psi(\pi^t)
\right|
&\le
\sum_{i=1}^n
\left|
\Psi(\widetilde\pi^{t,i})
-
\Psi(\widetilde\pi^{t,i-1})
\right|
\notag\\
&\le
\sum_{i=1}^n
\left|
V_i^{\widetilde\pi^{t,i}}(\mu)
-
V_i^{\widetilde\pi^{t,i-1}}(\mu)
\right|
+n\alpha
\notag\\
&=
\frac{1}{1-\gamma}
\sum_{i=1}^n
\bigg|
\sum_{s\in\mathcal S}
d_\mu^{\widetilde\pi^{t,i}}(s)
\left\langle
\pi_i^{t+1}(\cdot\mid s)-\pi_i^t(\cdot\mid s),
\bar A_i^{\widetilde\pi^{t,i-1}}(s)
\right\rangle
\bigg|
+n\alpha
\notag\\
&\le
\frac{1+L_d}{2(1-\gamma)}
\sum_{i=1}^n
\left\|
\pi_i^{t+1}-\pi_i^t
\right\|_{1,\infty}
+n\alpha\le
\frac{n\eta B(1+L_d)}{2(1-\gamma)^2}+n\alpha.
\label{eq:hp-potential-increment}
\end{align}
Here, the first inequality follows from the triangle inequality, while
the second follows from the $\alpha$-potential property. The
equality follows from the performance-difference identity applied
sequentially to each player. The next inequality uses Lemma~\ref{lem:span-inner-product} and the fact that
$d_\mu^{\widetilde\pi^{t,i}}$ is a probability distribution. The
last inequality follows from the policy-increment
bound of Lemma~\ref{lem:kl-projected-local-geometry} (part iv). Now let
\[
Z_{t+1}
:=
\Psi(\pi^t)-\Psi(\pi^{t+1})
-
\mathbb E_t
\left[
\Psi(\pi^t)-\Psi(\pi^{t+1})
\right].
\]
Then $\{Z_{t+1}\}_{t\ge0}$ is a martingale-difference sequence with
respect to $\{\mathcal F_t\}$ and, by
\eqref{eq:hp-potential-increment}, we have\footnote{Here and below, $\lesssim$ hides only universal numerical constants.} $|Z_{t+1}|
\lesssim\frac{n\eta B(1+L_d)}{(1-\gamma)^2}
+n\alpha$. Hence, by Azuma--Hoeffding,\footnote{If $\{Z_t\}_{t=1}^T$ is a martingale-difference sequence satisfying $|Z_t|\le b$ almost surely, then, for any $\varepsilon\in(0,1)$, with probability at least $1-\varepsilon$, we have $|\sum_{t=1}^T Z_t| \le b\sqrt{2T\log(2/\varepsilon)}$. Here we take $\varepsilon=\rho/2$.} with probability at least $1-\rho/2$, 
\begin{equation}\nonumber 
\Big| \sum_{t=0}^{T-1}Z_{t+1} \Big| \lesssim \Big( \frac{n\eta B(1+L_d)}{(1-\gamma)^2}+n\alpha \Big) \sqrt{ T\log(4/\rho)}.
\end{equation}
Moreover, connecting any two policy profiles through at most $n$
unilateral changes and using the $\alpha$-potential property together with $|V_i^\pi(\mu)|\le1/(1-\gamma)$ gives $\left|
\Psi(\pi)-\Psi(\pi')
\right| \le n\big(
\frac{2}{1-\gamma}+\alpha\big)$. Therefore, on this event, which has
probability at least $1-\rho/2$,
\begin{align}
\sum_{t=0}^{T-1}
\mathbb E_t
\left[
\Psi(\pi^t)-\Psi(\pi^{t+1})
\right]
&=
\Psi(\pi^0)-\Psi(\pi^T)
-
\sum_{t=0}^{T-1}Z_{t+1}
\cr
&\lesssim
n\Big(
\frac{1}{1-\gamma}+\alpha\Big)
+\Big(
\frac{n\eta B(1+L_d)}{(1-\gamma)^2}
+n\alpha
\Big)
\sqrt{
T\log(4/\rho)}.
\label{eq:hp-potential-telescope}
\end{align}

Summing \eqref{eq:hp-K-one-step}, dividing by $T$, and using
\eqref{eq:hp-potential-telescope}, together with
$T\ge\log(4/\rho)$ and $B=4$, yields
\begin{align}
\frac1T\sum_{t=0}^{T-1}K_t
\lesssim{}&
\frac{n\eta}{(1-\gamma)T}
+
\frac{n\eta^2(1+L_d)}{(1-\gamma)^2}
\sqrt{\frac{\log(4/\rho)}{T}}
+
n\alpha\eta\cr 
&+
\frac{n\eta^2L_{\widehat A}^2}{(1-\gamma)^2}
+
\frac{n\Big(1+nL_V+L_d(1+L_d)/(1-\gamma)\Big)\eta^3}{(1-\gamma)^2}.
\label{eq:hp-average-K}
\end{align}

Using the definitions $L_V=\frac{4}{1-\gamma}
+\frac{12\gamma\delta_P}{(1-\gamma)^2}
+\frac{8\gamma^2\delta_P^2}{(1-\gamma)^3}$ and 
$L_d=\frac{\gamma\delta_P}{1-\gamma}$, we have
\[1+nL_V+\frac{L_d(1+L_d)}{1-\gamma}
\leq
\frac{10 n(1+L_d)^2}{1-\gamma}.
\]
Hence \eqref{eq:hp-average-K} implies
\begin{align}
\frac1T\sum_{t=0}^{T-1}K_t
\lesssim{}&
\frac{n\eta}{(1-\gamma)T}
+
\frac{n\eta^2(1+L_d)}{(1-\gamma)^2}
\sqrt{\frac{\log(4/\rho)}{T}}
+
n\alpha\eta+
\frac{n\eta^2L_{\widehat A}^2}{(1-\gamma)^2}
+\frac{n^2(1+L_d)^2\eta^3}{(1-\gamma)^3}.
\label{eq:hp-average-K-simple}
\end{align}

Since $\alpha\le1/(1-\gamma)$, $n\ge1$, $L_d\ge0$, and $T\ge 1$, by \eqref{eq:parameters}, we get $\eta\leq \frac{1-\gamma}{8}$. Thus all exponential factors appearing in
Lemma~\ref{lem:kl-projected-potential-improvement} are bounded by universal constants.
By Lemma~\ref{lem:kl-projected-ne-gap},
\eqref{eq:variance-bound-episodic}, and Jensen's inequality,
\begin{align}
\frac1T\sum_{t=0}^{T-1}\operatorname{Gap}(\pi^t)
\lesssim{}&
\bigg(
\frac{1+L_d}{1-\gamma}
+\frac1\eta\sqrt{\frac{A_{\max}}{\zeta}}
\bigg)
\sqrt{\frac1T\sum_{t=0}^{T-1}K_t}
+\frac{L_{\widehat A}}{1-\gamma}
\notag\\
&+\frac1{1-\gamma}\sqrt{\frac{2A_{\max}\eta}{\zeta}}
+\frac{(\zeta+L_d)(1+L_d)}{1-\gamma}.
\label{eq:hp-gap-average}
\end{align}

Substituting \eqref{eq:hp-average-K-simple} into
\eqref{eq:hp-gap-average} and applying
$\sqrt{\sum_j x_j}\le\sum_j\sqrt{x_j}$ to the resulting square
root gives
\begin{align}
\frac1T\sum_{t=0}^{T-1}\operatorname{Gap}(\pi^t)
\lesssim{}&
\bigg(
\frac{1+L_d}{1-\gamma}
+\frac1\eta\sqrt{\frac{A_{\max}}{\zeta}}
\bigg)
\bigg[
\sqrt{\frac{n\eta}{(1-\gamma)T}}
+\sqrt{n\alpha\eta}
+\frac{\sqrt n\,\eta L_{\widehat A}}{1-\gamma}
\notag\\
&\qquad
+\frac{\sqrt{n(1+L_d)}\,\eta}{1-\gamma}
\bigg(\frac{\log(4/\rho)}{T}\bigg)^{1/4}
+\frac{n(1+L_d)\eta^{3/2}}{(1-\gamma)^{3/2}}
\bigg]
\notag\\
&+\frac{L_{\widehat A}}{1-\gamma}
+\frac1{1-\gamma}\sqrt{\frac{2A_{\max}\eta}{\zeta}}
+\frac{(\zeta+L_d)(1+L_d)}{1-\gamma}.
\label{eq:hp-gap-expanded}
\end{align}
We now use the choices of $\eta$ and $\zeta$ in
\eqref{eq:parameters}--\eqref{eq:parameters-zeta}, together
with $T\ge\log(4/\rho)$. Suppose the cap in \eqref{eq:parameters-zeta} is inactive.\footnote{If the cap in \eqref{eq:parameters-zeta} is active, then either $\frac{[nA_{\max}(1+L_d)]^{1/3}}{1-\gamma}L_{\widehat A}^{2/3}
\gtrsim(1-\gamma)^{-1}$ or 
$\frac{\sqrt n\,A_{\max}^{1/3}(1+L_d)^{2/3}}{1-\gamma}\alpha^{1/6}
\gtrsim (1-\gamma)^{-1}$. Thus, in either case, the claimed regret bound \eqref{eq:final-NE-regret-episodic} already dominates the trivial uniform
bound $\operatorname{Gap}(\pi^t)\le (1-\gamma)^{-1}$.} Then $\zeta
\gtrsim
\frac{n^{1/3}A_{\max}^{1/3}L_{\widehat A}^{2/3}}{(1+L_d)^{2/3}}$ and $\zeta
\gtrsim
\frac{\sqrt n\,A_{\max}^{1/3}(1-\gamma)^{1/6}\alpha^{1/6}}{(1+L_d)^{1/3}},$ while $
\eta
\lesssim
\frac{1-\gamma}{\sqrt T}
+ \frac{(1-\gamma)^{3/2}}{1+L_d}\sqrt{\frac{\alpha}{n}}$. Substituting these bounds into the nonvanishing terms of
\eqref{eq:hp-gap-expanded} gives
\begin{align}
\frac{\sqrt{nA_{\max}}L_{\widehat A}}
{(1-\gamma)\sqrt{\zeta}}
&\lesssim
\frac{[nA_{\max}(1+L_d)]^{1/3}}
{1-\gamma}
L_{\widehat A}^{2/3},
\notag\\
\sqrt{\frac{nA_{\max}\alpha}{\eta\zeta}}
+
\frac{n(1+L_d)\sqrt{A_{\max}\eta}}
{(1-\gamma)^{3/2}\sqrt{\zeta}}
&\lesssim
\frac{\sqrt n\,A_{\max}^{1/3}(1+L_d)^{2/3}}
{(1-\gamma)^{5/6}}
\alpha^{1/6}
+
\frac{n(1+L_d)\sqrt{A_{\max}}}
{(1-\gamma)\sqrt{\zeta}}
\left(\frac{\log(4/\rho)}{T}\right)^{1/4}.
\label{eq:hp-offset-calculation}
\end{align}
The additional term introduced by the adaptive-oracle decomposition satisfies
\[
\frac{1}{1-\gamma}\sqrt{\frac{2A_{\max}\eta}{\zeta}}
\le
\frac{\sqrt{2}n(1+L_d)\sqrt{A_{\max}\eta}}{(1-\gamma)^{3/2}\sqrt{\zeta}},
\]
and is therefore already bounded by the second estimate in
\eqref{eq:hp-offset-calculation}. The remaining nonvanishing terms in
\eqref{eq:hp-gap-expanded}, including
$L_{\widehat A}/(1-\gamma)$ and
$\zeta(1+L_d)/(1-\gamma)$, are bounded
by the same quantities. Hence, using $(1-\gamma)^{-5/6}\le(1-\gamma)^{-1}$, the total nonvanishing contribution is
\begin{equation}
\frac{[nA_{\max}(1+L_d)]^{1/3}}
{1-\gamma}
L_{\widehat A}^{2/3}
+
\frac{\sqrt n\,A_{\max}^{1/3}(1+L_d)^{2/3}}
{1-\gamma}
\alpha^{1/6}
+
\frac{L_d(1+L_d)}{1-\gamma}.
\label{eq:hp-offset-bound}
\end{equation}

For the $T$-dependent terms, the same lower bounds on $\zeta$ give
\begin{align*}
\frac{n(1+L_d)\sqrt{A_{\max}}}
{(1-\gamma)\sqrt{\zeta}}
\lesssim
\frac{n^{5/6}A_{\max}^{1/3}(1+L_d)^{4/3}}
{(1-\gamma)
\max\left\{
L_{\widehat A}^{1/3},\,
n^{1/12}(1-\gamma)^{1/12}(1+L_d)^{1/6}\alpha^{1/12}
\right\}}\lesssim  \mathfrak C.
\end{align*}
Moreover, $\frac{1+L_d}{1-\gamma}\lesssim\mathfrak C$. Hence, using also
$T\ge\log(4/\rho)$, all the
$T$-dependent terms in \eqref{eq:hp-gap-expanded} are bounded by
$\mathfrak C
\big(
\frac{\log(4/\rho)}{T}
\big)^{1/4}$, with terms decaying faster than $T^{-1/4}$ absorbed into this bound.
Combining this estimate with \eqref{eq:hp-offset-bound} yields the bound \eqref{eq:final-NE-regret-episodic}.
\end{proof}

Theorem~\ref{thm:high-prob-ne-regret} provides a high-probability bound on the NE regret and hence guarantees the existence of an iterate among the first $T$ episodes with an NE gap of the same order. The bound consists of a statistical error decaying as $\widetilde O(T^{-1/4})$ and an approximation floor governed by the game and oracle parameters: smaller $L_{\widehat A}$ corresponds to more accurate marginalized-advantage estimation, smaller $\alpha$ means that the game is closer to an exact Markov potential game, and smaller $\delta_P$ means that individual actions have less influence on the transition dynamics. The dependence on $n$, $A_{\max}$, and $(1-\gamma)^{-1}$ reflects the additional difficulty arising from more players, larger action spaces, and longer horizons.

\begin{remark}
We note that the $\alpha$-dependent term in \eqref{eq:final-NE-regret-episodic} remains below the trivial upper bound $1/(1-\gamma)$ on the NE regret, provided $\alpha\lesssim 1/[n^3A_{\max}^2(1+L_d)^4]$. In particular, when $L_d=O(1)$, this reduces to $\alpha\lesssim 1/(n^3A_{\max}^2)$, with no additional polynomial dependence on $1-\gamma$.
\end{remark}

\begin{remark}
The bound \eqref{eq:hp-gap-expanded} also permits alternative tunings that make the decaying term independent of $\alpha$ and $L_{\widehat A}$, at the cost of the slower rate $\widetilde O(T^{-1/6})$. For instance, relative to the parameter choice in Theorem~\ref{thm:high-prob-ne-regret}, choosing $\eta$ by balancing the statistical scale $T^{-1/2}$ with the potential-error scale $\sqrt{\alpha/[n(1+L_d)^2]}$, and $\zeta$ as the maximum of the corresponding statistical, oracle-bias, and potential-error scales, yields a high-probability NE regret bound with a decaying term of order $\widetilde O(n^{2/3}A_{\max}^{1/3}(1+L_d)T^{-1/6}/(1-\gamma))$, independent of $\alpha$ and $L_{\widehat A}$. The fixed oracle-bias, potential-error, and transition-sensitivity terms scale respectively as $L_{\widehat A}^{2/3}$, $\alpha^{1/6}$, and $L_d(1+L_d)/(1-\gamma)$, with nearly the same coefficients as in Theorem~\ref{thm:high-prob-ne-regret}.
\end{remark}

\begin{remark}
\label{rem:primitive-time-ne-regret}
Under Assumption~\ref{ass:coverage}, the episodic guarantee in Theorem~\ref{thm:high-prob-ne-regret} can be translated from episode count into primitive time. Indeed, a union bound over states, samples, and episodes shows that, with high probability, collecting $L$ samples per state in each episode requires $\widetilde O(H_{\mathrm{cov}}L/p_{\min})$ primitive steps. Since $L\asymp A_{\max}/\eta$ and $1/\eta\lesssim \sqrt{T}/(1-\gamma)$ under \eqref{eq:parameters}, completing $T$ episodes requires $N_T=\widetilde O(H_{\mathrm{cov}}A_{\max}T^{3/2}/[p_{\min}(1-\gamma)])$ primitive steps. Hence, after $N$ primitive steps, $T_N=\widetilde\Omega([p_{\min}(1-\gamma)N/(H_{\mathrm{cov}}A_{\max})]^{2/3})$. Therefore, the bound in Theorem~\ref{thm:high-prob-ne-regret} remains valid with its $\mathfrak C T^{-1/4}$ term replaced by $\widetilde O(\mathfrak C[H_{\mathrm{cov}}A_{\max}/(p_{\min}(1-\gamma)N)]^{1/6})$, while the fixed approximation terms remain unchanged. Thus, the episodic $\widetilde O(T^{-1/4})$ rate becomes $\widetilde O(N^{-1/6})$ in primitive time.
\end{remark}

The episodic setting studied in this section provides a useful benchmark for the more challenging fully online problem. By freezing the policy within each episode, we create a locally stationary environment in which players can collect samples and estimate their advantage and value functions under a fixed policy. In the fully online setting, this separation between learning and policy updates is no longer available: policies evolve continuously along a single trajectory, so the data-generating process changes as learning proceeds. A natural way to bridge the two settings is to view the episodic dynamics as a stationary reference process and design the fully online algorithm so that its trajectory can be suitably coupled with this reference process. If the online trajectory can track the episodic one while keeping the discrepancy between them controlled, the online process should approximately inherit the favorable behavior established for its episodic counterpart. This observation motivates the approach of the next section. At a high level, we design the online algorithm with an aggregation-tracking mechanism that tracks the evolving marginalized advantages while controlling policy drift. This allows the fully online trajectory to remain close to a corresponding episodic frozen-policy trajectory over suitable time windows. Making this intuition rigorous is substantially more delicate, as policy drift, changing state occupancies, asynchronous state visits, bandit estimation errors, and delayed local information all interact. We address these challenges through stopping-time, coupling, charging, and dynamic-tracking arguments, which allow us to transfer the essential stability of the episodic analysis to the fully online setting without resets or stationary sampling windows.

\section{Fully Online Asynchronous Decentralized Learning with One-Sample Bandit Feedback}
\label{sec:fully-online-general-occupancy}

In this section, we extend our episodic online algorithm to a fully online asynchronous algorithm that runs along a single realized trajectory of the Markov game. No policy is frozen for an episode, and no batch of samples is collected. Whenever a state is visited, each player draws one action, receives one raw scalar stochastic oracle observation, forms an importance-weighted action vector and its action-centered one-sample marginalized-advantage estimator, recursively aggregates this estimator, and immediately updates the policy stored at the visited state. Thus, the algorithm operates asynchronously across states, with each player tracking its marginalized advantage.

\subsection{Global vs. Local Notation and Filtrations}\label{sec:online-notations}

Throughout, we reserve $t$ for global time and $k$
for the local visit index. For every $s\in\mathcal S$, we use
\begin{align*}
N_t(s):=\sum_{\tau=0}^{t-1}\mathbf 1\{S^\tau=s\}
\end{align*}
to denote the number of visits to $s$ \emph{before} global time $t$. Let $\tau_k(s)$
denote the global time of the $(k+1)$st visit to $s$, so that
$N_{\tau_k(s)}(s)=k$ and $S^{\tau_k(s)}=s$. In particular, at the state $S^t$ visited at global time $t$, we have $t=\tau_{N_t(S^t)}(S^t)$. At global time $t$, the current policy profile $\pi^t$ is defined
statewise by $\pi^t=(\pi_i^t(\cdot\mid s)\ \forall i,s)$. In our fully online algorithm, since players'
policies at state $s$ are updated only when $s$ is visited, exactly
$N_t(s)$ local updates have occurred at $s$ before global time $t$.
Consequently, the policy stored at state $s$ at the beginning of
any global time $t$ is the locally indexed policy
$\pi_i^{N_t(s)}(\cdot\mid s)$, and therefore we write
$$
\pi_i^t(\cdot\mid s)
:=
\pi_i^{N_t(s)}(\cdot\mid s),
\qquad i\in[n],\ s\in\mathcal S.
$$
In other words, if $N_t(s)=k$, then
$\pi_i^t(\cdot\mid s)=\pi_i^k(\cdot\mid s)$. Moreover, for notational simplicity, for global-policy quantities evaluated
at the $(k+1)$st visit to state $s$, we use the shorthands

$$
d_\mu^k(s)
:=
d_\mu^{\pi^{\tau_k(s)}}(s),
\qquad
\bar A_i^k(s,\cdot)
:=
\bar A_i^{\pi^{\tau_k(s)}}(s,\cdot).
$$

Thus, $d_\mu^k(s)$ and $\bar A_i^k(s,\cdot)$ denote, respectively, the
occupancy of state $s$ and player $i$'s marginalized advantage at
state $s$, both evaluated under the global policy profile at the $(k+1)$st
visit to $s$.

\subsubsection*{Global and Stopped Filtrations} 

Let $\{\mathcal F_t\}_{t\ge0}$ denote the global filtration, where
$\mathcal F_t$ contains the complete history up to the beginning of time
$t$, including the current state $S^t$, but before the actions $A^t$ and
the oracle randomness at time $t$ are generated. Throughout, we write
$\mathbb E_t[\cdot]:=\mathbb E[\cdot\mid\mathcal F_t]$. Thus, conditioning on $\mathcal F_t$ fixes $S^t$, the visit counts $N_t(\cdot)$, the stored
policy profile $\pi^t$, and all recursive estimates constructed before time
$t$, whereas $A^t$ and the new raw oracle samples, denoted by $g^t$, remain
random.

Since our proposed algorithm is also indexed by the local visit count $k$
at each state $s$, we will frequently condition on the global history
available at the corresponding visit time $\tau_k(s)$. For this purpose,
note that $\tau_k(s)$ is an $\{\mathcal F_t\}$-stopping time, since
$\{\tau_k(s)\le t\}\in\mathcal F_t$. We therefore use the corresponding
stopped sigma-field $\mathcal F_{\tau_k(s)}$ and write $
\mathbb E_{\tau_k(s)}[\cdot]
:=
\mathbb E[\cdot\mid\mathcal F_{\tau_k(s)}].$\footnote{For a stopping time $\tau$,
$\mathcal F_\tau:=\{E:E\cap\{\tau\le t\}\in\mathcal F_t
\text{ for every }t\}$.}
Thus, $\mathcal F_{\tau_k(s)}$ represents the global history available at
the beginning of the $(k+1)$st visit to state $s$, before the fresh actions
and oracle randomness at that visit are generated, and
$\mathcal F_{\tau_k(s)}\subseteq\mathcal F_{\tau_{k+1}(s)}$. In particular, if $t=\tau_k(s)$, then $S^t=s$, $N_t(s)=k$, and
$\pi_i^k(\cdot\mid s)=\pi_i^t(\cdot\mid s)$ are
$\mathcal F_t=\mathcal F_{\tau_k(s)}$-measurable, while $A^t$, $g^t$, and
the resulting one-sample estimators are not. We use $\mathbb E_t, \mathbb P_t$ when working in global time $t$ and $\mathbb E_{\tau_k(s)}, \mathbb P_{\tau_k(s)}$ when working with
the local visit index $k$.

\subsection{Online One-Sample Estimation Oracle}

At time $t=\tau_k(s)$, player $i$ observes state $S^t=s$ and draws
$A_i^t\sim\pi_i^k(\cdot\mid s)=\pi_i^t(\cdot\mid s)$ independently of the others and receives a scalar raw sample $g_i^t(s,A^t)$. Here, the conditional law of the raw sample $g_i^t$ is allowed to vary with time and depend on the entire history available up to time $t$, including the current global policy profile $\pi^t$, with its fresh randomness generated after conditioning on $\mathcal F_t$. Thus, this formulation allows for a general adaptive and policy-dependent oracle. It then forms the corresponding one-sample importance-weighted and action-centered estimator by
\begin{align}\nonumber
\widehat A_i^k(s,a_i)
:=
g_i^t(s,A^t)
\left(
\frac{\mathbf 1\{A_i^t=a_i\}}
{\pi_i^k(a_i\mid s)}-1
\right),
\qquad a_i\in\mathcal A_i.
\end{align}
Thus, $\widehat A_i^k(s,\cdot)$ is the one-sample importance-weighted estimator centered by the realized raw sample such that $\sum_{a_i\in\mathcal A_i}
\pi_i^k(a_i\mid s)
\widehat A_i^k(s,a_i)=0.$

\begin{assumption}
\label{ass:online-one-sample-oracle}
At every global time $t=\tau_k(s)$, the one-sample advantage estimator satisfies
\begin{equation}\nonumber
\left\|
\mathbb E_{\tau_k(s)}[\widehat A_i^k(s,\cdot)]
-
\bar A_i^k(s,\cdot)
\right\|_\infty
\le
L_{\widehat{A}},
\end{equation}
for some deterministic $L_{\widehat{A}}$. Moreover, the raw sample satisfies $|g_i^t(s,A)|\leq 1$ a.s. for every $i,t,s,A$.
\end{assumption}

Next, we state the following lemma, which provides almost-sure bounds on the weighted centered estimation noise and its conditional second moment. These bounds will be used later to derive our high-probability NE-regret bounds. The proof is given in Appendix~\ref{app:online-importance-sampling-bounds}. 

\begin{lemma}\label{lemm:conditional-variance-online}
At each actual visit $t=\tau_k(s)$ to the state $s$, define\footnote{By a slight abuse of notation from the previous section, we use $\xi_i^k(s,a_i)$, which also satisfies $\mathbb E_{\tau_k(s)}[\xi_i^k(s,a_i)]=0$.}
\begin{align*}
\xi_i^k(s,a_i)
:=
\widehat A_i^k(s,a_i)
-
\mathbb E_{\tau_k(s)}[\widehat A_i^k(s,a_i)],
\end{align*}
which denotes the centered estimation noise generated at the $(k+1)$st visit
to state $s$. Assume $\pi_i^k(\cdot\mid s)\in \Delta_{i,\zeta}$, and define the weighted conditional second moment at the $(k+1)$st visit to $s$ by
\begin{equation}\label{eq:online-occupancy-weighted-variance}
\mathcal V_i^k(s)
:=d_\mu^k(s)\mathbb E_{\tau_k(s)}\!\bigg[
\sum_{a_i\in\mathcal A_i}
\pi_i^k(a_i\mid s)
\bigl(\xi_i^k(s,a_i)\bigr)^2
\bigg].
\end{equation}
Then, almost surely, $\mathcal V_i^k(s)\leq A_{\max}$ and $\sum_{a_i}\pi_i^k(a_i\mid s)
\bigl(\xi_i^k(s,a_i)\bigr)^2\leq 10 A_{\max}/\zeta$. 
\end{lemma}

Finally, for later use, let $M:=1+L_d$, where by Lemma~\ref{lem:span-inner-product}, it follows that
$\|\bar A_i^\pi(s,\cdot)\|_\infty\le M$ for every $i,\pi,s$,
and let $\operatorname{clip}_{M}(\cdot)$ denote coordinatewise projection onto
$[-M,M]$. Moreover, assuming that $\pi_i^k(\cdot\mid s)\in \Delta_{i,\zeta}$, we have the following almost sure bound:
\[
\|\widehat A_i^k(s)\|_\infty
\le\frac{A_{\max}}{\zeta}+1
\le\frac{2A_{\max}}{\zeta}.
\]

\begin{algorithm}[t]
\caption{Fully Online One-Sample KL-Projected NPG}
\label{alg:fully-online-general-occupancy}
\begin{algorithmic}[1]
\State Initialize $\pi_i^0(\cdot\mid s)$ uniformly and
$r_i^{-1}(s)=0$ for every player $i$ and state $s$, and set
$N_0(s)=0$ for every $s$. Also, choose $M=1+L_d$ and $\omega \in (0, \frac{1}{2})$. 
\For{$t=0,1,2,\ldots$}
    \State Observe $S^t=s$ and set $k=N_t(s)$ (so that $\tau_k(s)=t$ and $\pi_i^t(\cdot\mid s)=\pi_i^k(\cdot\mid s)$).
    \State Player $i$ independently draws an action according to its current policy
    $A_i^t\sim\pi_i^k(\cdot\mid s)$.
    \State Player $i$ receives the scalar oracle sample $g_i^t(s,A^t)$ and forms the one-sample estimator
    \[
    \widehat A_i^k(s,a_i)
    =
    g_i^t(s,A^t)
    \left(
    \frac{\mathbf 1\{A_i^t=a_i\}}
    {\pi_i^k(a_i\mid s)}-1
    \right),
    \qquad a_i\in\mathcal A_i.
    \]
    \State Player $i$ recursively aggregates and clips the new estimator:
    \begin{align*}
    \widetilde r_i^k(s)
    &=
    (1-\omega)r_i^{k-1}(s)
    +
    \omega\widehat A_i^k(s),\\
    r_i^k(s)
    &=
    \operatorname{clip}_{M}
    \bigl(\widetilde r_i^k(s)\bigr).
    \end{align*}
    \State Simultaneously, every player updates
    \begin{align}\nonumber
    \pi_i^{k+1}(\cdot\mid s)
    =
    \argmin_{p_i\in\Delta_{i,\zeta}}
    \left\{
    \frac{\eta}{1-\gamma}
    \langle p_i,r_i^k(s)\rangle
    +
    D_{\mathrm{KL}}
    \bigl(p_i\|\pi_i^k(\cdot\mid s)\bigr)
    \right\}.
    \end{align}
    \State Set $\pi_i^{t+1}(\cdot\mid s')=\pi_i^t(\cdot\mid s')$ for every $s'\neq s$ (at state $s$, note that $\pi_i^{t+1}(\cdot\mid s)=\pi_i^{k+1}(\cdot\mid s)$).
    \State Set $N_{t+1}(s)=N_t(s)+1$ and leave all other counters unchanged.
    \State The state transits according to
    $S^{t+1}\sim P(\cdot\mid S^t,A^t)$.
\EndFor
\end{algorithmic}
\end{algorithm}

\subsection{An Online Decentralized Algorithm with One-Sample Bandit Oracle}

With these notations and definitions, we are now ready to describe our
algorithm. Algorithm~\ref{alg:fully-online-general-occupancy} is a fully
online implementation of Algorithm~\ref{alg:episodic-kl-projected-npg} that
uses only one oracle sample at each global time step. At time $t$, after
observing the current state $S^t=s$, the players set $k=N_t(s)$ and use the
policy component $\pi^t(\cdot\mid s)=\pi^k(\cdot\mid s)$ associated with the current local visit
count to sample their actions and construct importance-weighted one-sample
estimates of their advantages. These noisy estimates are recursively averaged
with weight $\omega$ and clipped at level $M$ to obtain the surrogate
advantage vectors $r_i^k(s)$. Each player then performs a KL-projected NPG update over the
truncated simplex $\Delta_{i,\zeta}$ using $r_i^k(s)$, thereby changing the
policy component at the visited state from $\pi_i^k(\cdot\mid s)$ to
$\pi_i^{k+1}(\cdot\mid s)$, while the policy components at all states
$s'\neq s$ remain unchanged.

To make this local-time interpretation precise, fix a state $s$ and let $\tau_0(s)\!<\!\tau_1(s)\!<\!\tau_2(s)\!<\!\cdots$ denote its successive visit times. At global time $t=\tau_k(s)$, we have
$N_t(s)=k$, so the policy component currently stored at $s$ is
$\pi^t(\cdot\mid s)=\pi^k(\cdot\mid s)$. The players draw their actions according to this
state-local policy, receive the new scalar oracle sample $g_i^t(s,A^t)$, use it
together with the previously stored estimator $r_i^{k-1}(s)$ to form
$r_i^k(s)$, and then perform the KL-projected NPG update $\bigl(r_i^k(s),\pi_i^k(\cdot\mid s)\bigr)\rightarrow
\pi_i^{k+1}(\cdot\mid s)$. Importantly, during the intervisit interval $\tau_k(s)<t<\tau_{k+1}(s)$, state $s$ is not visited, so neither its stored estimator nor its stored
policy component is updated. Other states may nevertheless be visited and updated during this interval,
so the full global policy profile
$\pi^t=\{\pi_i^t(\cdot\mid x):x\in\mathcal S,\ i\in[n]\}$ continues
to evolve. The resulting local timing at state $s$ can be illustrated in the following table:
\begin{center}
\resizebox{\textwidth}{!}{$
\begin{array}{c|ccccc}
\text{global time}
&
\tau_k(s)^-
&
\tau_k(s)
&
\tau_k(s)<t<\tau_{k+1}(s)
&
\tau_{k+1}(s)
&
\tau_{k+1}(s)^+
\\
\hline
\text{state visited}
&
-
&
s
&
s'\neq s
&
s
&
-
\\
\text{policy component at }s
&
\pi^k(\cdot\mid s)
&
\pi^k(\cdot\mid s)\to\pi^{k+1}(\cdot\mid s)
&
\pi^{k+1}(\cdot\mid s)
&
\pi^{k+1}(\cdot\mid s)\to\pi^{k+2}(\cdot\mid s)
&
\pi^{k+2}(\cdot\mid s)
\\
\text{estimator at }s
&
r^{k-1}(s)
&
r^{k-1}(s)\to r^k(s)
&
r^k(s)
&
r^k(s)\to r^{k+1}(s)
&
r^{k+1}(s).
\end{array}
$}
\end{center}

We note that although both $r_i^k(s)$ and $\pi_i^{k+1}(\cdot\mid s)$ remain stale between consecutive visits to $s$, the evolving global policy profile affects the next estimator update. Indeed, since the new oracle sample $g_i^{\tau_{k+1}(s)}(s,A^{\tau_{k+1}(s)})$ may have a conditional mean depending on the current global policy $\pi^{\tau_{k+1}(s)}$, policy updates at states $x\neq s$ during the intervisit interval can affect the update $r_i^k(s)\longrightarrow r_i^{k+1}(s)$. In other words, while $r_i^k(s)$ is updated only at the local visit times $k,k+1,\ldots$, each such change can be influenced, through the new oracle sample, by the history of the globally evolving policy profile. These global policy changes also modify the nonstationary state-transition law and hence the return time $\tau_{k+1}(s)-\tau_k(s)$. Thus, each state admits its own local KL-projected NPG update sequence, while the global state process and policy profile evolve continuously without episodes, global policy-freezing periods, or multiple samples per update.

\begin{remark}
The time-varying oracle $g_i^t$ allows for a broad class of adaptive and policy-dependent estimators, including rollout-based and value-based constructions, provided that the required conditional bias assumption (Assumption \ref{ass:online-one-sample-oracle}) holds. An important consequence of this generality is that, although $r_i^k(s)$ is updated only at visits to state $s$, its next fresh innovation $\widehat A_i^{k+1}(s,\cdot)$, constructed from the raw oracle sample $g_i^{\tau_{k+1}(s)}(s,A^{\tau_{k+1}(s)})$, need not be statewise local: its conditional distribution may depend on the entire history and current global policy profile, and hence on policy updates made at other states. Thus, the locally indexed estimator processes are generally coupled through their successive innovations. Our analysis accommodates this coupling through the global filtration, moving-target tracking, sensitivity, coupling, and charging arguments.\footnote{A fixed state-action oracle $g_i(s,A)$ is a simpler special case, for which updates at other states do not affect the conditional distribution of the next innovation at $s$, and the recursive estimator dynamics can be analyzed statewise, with their discrepancy from the globally coupled marginalized advantage already controlled by the oracle assumption.}
\end{remark}

As we noted earlier in the description of the algorithm, when attention is restricted to a fixed state $s$, its
successive visits generate a local KL-projected NPG sequence indexed by $k=0,1,2,\ldots$, and the KL optimality condition can be applied pathwise to
each such local update. This does not require $r_i^k(s)$ to be independent of
the past: although $r_i^k(s)$ is a random function of the global history up
to $\tau_k(s)$, once that history and the new oracle sample are realized,
$r_i^k(s)$ is simply the vector appearing in the corresponding KL
optimization problem. Therefore, we have the following corollary.  

\begin{corollary}\label{cor:local-kl-projected-geometry}
For any fixed state $s$, the local policy updates at that state follow the
standard KL-projected NPG update rule. For
the update performed at the $(k+1)$st visit to state $s$, define
\begin{align}\nonumber
\mathcal D_i^k(s)
&:=
D_{\mathrm{KL}}\!\left(\pi_i^{k+1}(\cdot\mid s)\middle\|\pi_i^k(\cdot\mid s)\right)
+
D_{\mathrm{KL}}\!\left(\pi_i^k(\cdot\mid s)\middle\|\pi_i^{k+1}(\cdot\mid s)\right),\\
\Delta_i^k(s)
&:=
\pi_i^k(\cdot\mid s)-\pi_i^{k+1}(\cdot\mid s).
\end{align}
Since $\|r_i^k(s)\|_\infty\le M$, all the conclusions of
Lemma~\ref{lem:kl-projected-local-geometry} hold for the local updates
$k=0,1,2,\ldots$, with $B$ replaced by $M$.
\end{corollary}

\subsection{Analysis of the Fully Online Asynchronous Decentralized Algorithm}

Equipped with these preliminary results and notation, we are ready to analyze the performance of Algorithm~\ref{alg:fully-online-general-occupancy}. We first provide a brief overview of the proof strategy. The proof of the main result is organized so that the statistical analysis and the asynchronous Markov dynamics remain within the same occupancy-weighted geometry. The one-sample recursive estimator is analyzed directly along the realized trajectory. Mild clipping of the recursively aggregated estimate keeps the tracking error uniformly bounded without changing the one-sample nature of the algorithm. This allows changes in the occupancy weights themselves to be controlled by the localized occupancy-sensitivity bound, while a single global Freedman argument controls the weighted martingale noise.

The remaining steps preserve this geometry. The potential-improvement argument retains the occupancy factor in both the leading KL term and all second-order one-state perturbation terms, so the cumulative KL energy is controlled directly by the occupancy-weighted tracking error. The Nash-gap lemma uses the most recent KL and tracking certificates stored at each state. A frozen-policy coupling is then used only to show that these certificates are at most $H_T=\widetilde O(H_{\rm cov}/p_{\min})$ steps old and to transport their occupancy weights from their creation times to the current time. Crucially, because all the tracking, potential, and delayed-certificate arguments are carried out in the occupancy-weighted geometry, the polynomial $|\mathcal S|$ dependence arising from a statewise analysis disappears; $|\mathcal S|$ appears only logarithmically through the high-probability coverage event.

\subsubsection{High-probability occupancy-weighted recursive tracking}

At the $(k+1)$st visit to $s$, define the tracking error and its policy-weighted squared error by\footnote{For initialization, since $r_i^{-1}(s)=\boldsymbol{0}$, we define $e_i^{-1}(s,\cdot):=-\bar A_i^{0}(s,\cdot), \mathcal E_i^{-1}(s):=\sum_{a_i}
\pi_i^0(a_i\mid s)\bigl(e_i^{-1}(s,a_i)\bigr)^2$, where $\pi_i^0(\cdot\mid s)$ is the initial uniform policy. These
quantities are used only when the first visit to $s$ occurs.}
\begin{equation}\label{eq:online-tracking-error}
e_i^k(s,a_i)
:=r_i^k(s,a_i)-\bar A_i^k(s,a_i),
\qquad
\mathcal E_i^k(s)
:=\sum_{a_i}\pi_i^k(a_i\mid s)
\bigl(e_i^k(s,a_i)\bigr)^2.
\end{equation}
We note that since both the clipped aggregate $r_i^k(s)$ and the true advantage
$\bar A_i^k(s,\cdot)$ belong to $[-M,M]$ coordinatewise, we have
$\|e_i^k(s)\|_\infty\le 2M$ and $\mathcal E_i^k(s)\le4M^2$. We further define
\begin{align*}
\mathfrak E_T
&:=
\sum_{t=0}^{T-1}d_\mu^{\pi^t}(S^t)
\sum_{i=1}^n\mathcal E_i^{N_t(S^t)}(S^t),\\
\mathfrak D_T
&:=
\sum_{t=0}^{T-1}d_\mu^{\pi^t}(S^t)
\sum_{i=1}^n\mathcal D_i^{N_t(S^t)}(S^t),
\end{align*}
which are the cumulative occupancy-weighted squared tracking error and the
cumulative occupancy-weighted symmetric KL movement, respectively.  

\begin{definition}
\label{def:bounded-intervisit-event}
For any horizon $T\!\ge\!1$ and deterministic $H_T\!\ge\! 1$, define the
bounded intervisit event
\[
\mathcal H_T(H_T)
:=
\left\{
\max_{s\in\mathcal S}\tau_0(s)\le H_T,\quad
\max_{s,k:\tau_k(s)<T}
\bigl(\tau_{k+1}(s)-\tau_k(s)\bigr)\le H_T
\right\}.
\]
\end{definition}

Thus, on the event $\mathcal H_T(H_T)$, every state is first visited within $H_T$
steps, and the time between any two consecutive visits occurring before
horizon $T$ is at most $H_T$. The following lemma shows that, when the
intervisit times are uniformly bounded, the cumulative squared tracking error
remains controlled despite the continuously evolving global policy and the
resulting drift in the advantage targets. In particular, for a sufficiently
fast-decaying stepsize $\eta$, the time-averaged occupancy-weighted squared
tracking error vanishes, so the moving advantage targets are tracked
asymptotically on average. The proof of the following dynamic tracking lemma is given in Appendix \ref{app:dynamic-traking}.

\begin{lemma}\label{lem:online-tracking}
Fix $T\ge1$ and $\rho\in(0,1)$, and let $\Lambda_T:=\log\big(\frac{16nA_{\max}(|\mathcal S|+T)}{\rho}\big)$. Suppose that $\omega, \zeta \in (0, \frac{1}{2}]$, and $\frac{\eta M}{1-\gamma}\le c_0\omega$, where $c_0>0$ is a sufficiently small constant. If $\mathcal V_i^k(s)\le\sigma^2$ for every $i,s,k$, then there exists an
event $\mathcal T_T$ with $\mathbb P(\mathcal T_T)\ge1-\rho/2$ such that,
on $\mathcal T_T\cap\mathcal H_T(H_T)$, we have
\begin{align}
\frac{\mathfrak E_T}{T}
\le c\Bigg[&
\frac{nM^2}{\omega T}
+n\sigma^2\omega
+nL_{\widehat A}^2
+nM\sigma\sqrt{\frac{\Lambda_T}{T}}
+\frac{nM\sqrt{A_{\max}}\,\Lambda_T}{T\sqrt\zeta}+n\omega \sigma\sqrt{\frac{A_{\max}\Lambda_T}{T\zeta}}
\notag\\
&
+\frac{n\omega A_{\max}\Lambda_T}{T\zeta}
+\frac{n^3L_{\bar A}^2H_T\eta^2M^2}
{(1-\gamma)^2\omega^2}
+\Big(\frac{n^{3/2}M^2L_d}{\omega}+\frac{n^{7/2}L_{\bar A}^2 H_T^2\eta^2M^2L_d}
{(1-\gamma)^2\omega^2}\Big)
\sqrt{\frac{\mathfrak D_T}{T}}
\Bigg].
\label{eq:online-E-closed}
\end{align}
\end{lemma}

It is worth noting that \cite[Lemma B.1]{introDong2024} also employs a moving-target tracking argument for episodic bandit learning in MPGs. There, however, the policy is frozen throughout each sampled trajectory, and the tracking error arises from the change of the policy gradient across episodes. In our fully online setting, the policy is updated after every primitive interaction using a single realized cost sample, so each state evolves on a random local clock $\tau_k(s)$ while policies at other states may change between consecutive visits. Consequently, both the marginalized advantage $\bar A_i^{\pi^t}(s,\cdot)$ and the occupancy measure $d_\mu^{\pi^t}$ drift over random revisit intervals. Hence, Lemma~\ref{lem:online-tracking} extends the moving-target tracking principle to this substantially more challenging coupled-drift setting, using stopping-time localization and occupancy sensitivity to handle the interaction of asynchronous state-wise learning, policy drift, and occupancy drift along a single continuously evolving Markov trajectory.

\subsubsection{Potential improvement under asynchronous online updates}

Here, we extend the potential improvement lemma from the episodic setting to the fully online setting, where, at each global time $t$, only the players' policies at the visited state $S^t$ are updated.

\begin{lemma}\label{lem:online-potential-improvement}
Fix an arbitrary global time $t$, and assume that for a sufficiently small constant $c_0,$
\begin{equation}\label{eq:online-potential-absorption}
\frac{\eta M}{1-\gamma}\le c_0\omega,
\qquad
\frac{\eta n\left(L_{\bar A}+L_d(1+L_d)\right)}{1-\gamma}\le c_0,
\end{equation}
where $L_d=\frac{\gamma\delta_P}{1-\gamma}$. Then for every realization of
the trajectory of Algorithm~\ref{alg:fully-online-general-occupancy},
\begin{align*}
\Psi(\pi^t)-\Psi(\pi^{t+1})
\ge{}&
\frac{1}{4\eta}
d_\mu^{\pi^t}(S^t)
\sum_{i=1}^n
\mathcal D_i^{N_t(S^t)}(S^t)-
\frac{c\eta}{(1-\gamma)^2}
d_\mu^{\pi^t}(S^t)
\sum_{i=1}^n
\mathcal E_i^{N_t(S^t)}(S^t)
-
n\alpha,
\end{align*}
where $c$ is a universal constant. Consequently,
\begin{equation}\label{eq:online-D-from-potential}
\mathfrak D_T
\le
c\left[
\frac{n\eta}{1-\gamma}
+
n\eta\alpha T
+
\frac{\eta^2}{(1-\gamma)^2}\mathfrak E_T
\right].
\end{equation}
\end{lemma}

\begin{proof}
Fix an arbitrary global time $t$, and write $S^t=s$ and $N_t(s)=k$. By the definition of the visit time $\tau_k(s)$, we then have
$t=\tau_k(s)$. Moreover, by the global/local policy convention introduced in subsection \ref{sec:online-notations}, $\pi_j^t(\cdot\mid x)
=\pi_j^{N_t(x)}(\cdot\mid x)\ \forall  j, x\in\mathcal S$, and hence, at the currently visited state $s$,
\begin{equation}\label{eq:online-local-global-identification}
\pi_j^t(\cdot\mid s)=\pi_j^k(\cdot\mid s) \ \ \forall j,
\qquad
d_\mu^{\pi^t}(s)=d_\mu^k(s),
\qquad
\bar A_j^{\pi^t}(s,\cdot)
=
\bar A_j^k(s,\cdot)\ \ \forall j.
\end{equation}
Algorithm~\ref{alg:fully-online-general-occupancy} updates only the policy
components at $s$. Thus, for any $j\in [n]$,
\begin{equation}\label{eq:online-global-update-identification}
\pi_j^{t+1}(\cdot\mid s)
=
\pi_j^{k+1}(\cdot\mid s),
\qquad
\pi_j^{t+1}(\cdot\mid x)
=
\pi_j^t(\cdot\mid x),
\quad x\neq s.
\end{equation}

For each $i=0,\ldots,n$, let the intermediate global policy profile $\widetilde\pi^{t,i}$ be obtained from $\pi^t$ by replacing, \emph{only} at state $s$, the policies of players $1,\ldots,i$ by their updated policies, i.e.,
\begin{align}\label{eq:tile-pi-online}
\widetilde\pi_j^{t,i}(\cdot\mid x)
:=
\begin{cases}
\pi_j^{k+1}(\cdot\mid s),
& x=s,\ j\le i,\\
\pi_j^k(\cdot\mid s),
& x=s,\ j>i,\\
\pi_j^t(\cdot\mid x),
& x\neq s,
\end{cases}
\end{align}
where in view of \eqref{eq:online-local-global-identification}--%
\eqref{eq:online-global-update-identification}, we have $\widetilde\pi^{t,0}=\pi^t$ and  $\widetilde\pi^{t,n}=\pi^{t+1}$. Consequently,
\begin{equation}\label{eq:online-potential-telescoping}
\Psi(\pi^t)-\Psi(\pi^{t+1})
=
\sum_{i=1}^n
\left[
\Psi(\widetilde\pi^{t,i-1})
-
\Psi(\widetilde\pi^{t,i})
\right]
\ge \sum_{i=1}^n
\left[ V_i^{\widetilde\pi^{t,i-1}}(\mu)
-
V_i^{\widetilde\pi^{t,i}}(\mu)\right]
- n\alpha,
\end{equation}
where the inequality follows by the $\alpha$-potential property. Thus, we only need to lower bound each unilateral value improvement on the right-hand side of \eqref{eq:online-potential-telescoping}. 

With an abuse of notation,
let $(\pi_i^{k+1}(s),\pi_{-i}^t)$ denote the policy profile obtained from $\pi^t$ by replacing only player $i$'s policy at state $s$ by
$\pi_i^{k+1}(\cdot\mid s)$. Let $\Delta_i^k(s)=\pi_i^k(\cdot\mid s)-\pi_i^{k+1}(\cdot\mid s)$ and define\footnote{Since $\pi_i^k(\cdot\mid s)=\pi_i^t(\cdot\mid s)$, we have $\pi^t=(\pi_i^k(s),\pi_{-i}^t)$.}
\begin{align}
\!\!\!G_i^t
&:=
V_i^{(\pi_i^k(s),\pi_{-i}^t)}(\mu)
-
V_i^{(\pi_i^{k+1}(s),\pi_{-i}^t)}(\mu)=\frac{1}{1-\gamma}
d_\mu^{(\pi_i^{k+1}(s),\pi_{-i}^t)}(s)
\left\langle
\Delta_i^k(s),\bar A_i^k(s,\cdot)
\right\rangle
\label{eq:intermediate-online}\\
&\ge \frac{1}{1-\gamma}
d_\mu^{k}(s)
\left\langle
\Delta_i^k(s),\bar A_i^k(s,\cdot)
\right\rangle-\frac{1}{1-\gamma}
\bigg|\Big(d_\mu^{(\pi_i^{k+1}(s),\pi_{-i}^t)}(s)-d_\mu^{k}(s)\Big)
\left\langle
\Delta_i^k(s),\bar A_i^k(s,\cdot)
\right\rangle\bigg|
\label{eq:online-player-pdl},
\end{align}
where the equality follows from \eqref{eq:online-local-global-identification} and the performance-difference Lemma~\ref{lemm:performance-difference}, as the two policy profiles differ only in player $i$'s policy at state $s$.

Since $(\pi_i^{k+1}(s),\pi_{-i}^t)$ and $\pi^t$ differ only in player
$i$'s policy at the single state $s$,
\eqref{eq:online-localized-occupancy-sensitivity} yields
\begin{equation}\label{eq:online-one-state-occupancy-change}
\Big|
d_\mu^{(\pi_i^{k+1}(s),\pi_{-i}^t)}(s)-d_\mu^k(s)
\Big|
\le
L_d\,d_\mu^k(s)
\|
\Delta_i^k(s)\|_{\mathrm{TV}}.
\end{equation}
Therefore, using Lemma~\ref{lem:span-inner-product}, relation \eqref{eq:online-one-state-occupancy-change}, and Pinsker's inequality $\|\Delta^k_i(s)\|^2_1\leq \mathcal D_i^k(s)$, we get
\begin{align}
&\frac{1}{1-\gamma}
\bigg|
\Big(
d_\mu^{(\pi_i^{k+1}(s),\pi_{-i}^t)}(s)-d_\mu^k(s)
\Big)
\left\langle
\Delta_i^k(s),\bar A_i^k(s,\cdot)
\right\rangle
\bigg|\le\frac{cL_d(1+L_d)}{1-\gamma}
d_\mu^k(s)\mathcal D_i^k(s).
\label{eq:online-occupancy-remainder}
\end{align}

By part~(iii) of
Lemma~\ref{lem:kl-projected-local-geometry}, applied to the local update
through Corollary~\ref{cor:local-kl-projected-geometry}, we have $\left\langle
r_i^k(s),\Delta_i^k(s)
\right\rangle
\ge
\frac{1-\gamma}{\eta}\mathcal D_i^k(s)$, which together with $r_i^k(s)
=
\bar A_i^k(s,\cdot)+e_i^k(s)$ from \eqref{eq:online-tracking-error} implies
\begin{equation}\label{eq:online-advantage-inner-product}
\left\langle
\bar A_i^k(s,\cdot),\Delta_i^k(s)
\right\rangle
\ge
\frac{1-\gamma}{\eta}\mathcal D_i^k(s)
-
\left|
\left\langle
e_i^k(s),\Delta_i^k(s)
\right\rangle
\right|.
\end{equation}
To bound the last term in \eqref{eq:online-advantage-inner-product}, weighted Cauchy--Schwarz gives
\begin{align}
\left|
\left\langle
e_i^k(s),\Delta_i^k(s)
\right\rangle
\right|
&\le
\bigg(
\sum_{a_i}
\pi_i^k(a_i\mid s)
\bigl(e_i^k(s,a_i)\bigr)^2
\bigg)^{1/2}
\bigg(
\sum_{a_i}
\frac{\bigl(\Delta_i^k(s,a_i)\bigr)^2}
{\pi_i^k(a_i\mid s)}
\bigg)^{1/2}\cr 
&\leq e^{\eta M/(1-\gamma)}
\sqrt{\mathcal E_i^k(s)\mathcal D_i^k(s)},
\label{eq:online-tracking-inner-product}
\end{align}
where the first factor is exactly $\sqrt{\mathcal E_i^k(s)}$ by
\eqref{eq:online-tracking-error}, while the second inequality follows from part~(ii) of Lemma~\ref{lem:kl-projected-local-geometry} through
Corollary~\ref{cor:local-kl-projected-geometry}. Since by condition \eqref{eq:online-potential-absorption}
$\eta M/(1-\gamma)\le c_0\omega\le c_0$, the exponential factor is
bounded by a universal constant. Applying Young's inequality to
\eqref{eq:online-tracking-inner-product} in
\eqref{eq:online-advantage-inner-product} thus gives
\begin{equation}\label{eq:online-local-advantage-descent}
\frac{1}{1-\gamma}d_\mu^k(s)
\left\langle
\bar A_i^k(s,\cdot),\Delta_i^k(s)
\right\rangle
\ge
\frac{1}{2\eta}d_\mu^k(s)\mathcal D_i^k(s)
-
\frac{c\eta}{(1-\gamma)^2}
d_\mu^k(s)\mathcal E_i^k(s).
\end{equation}
Combining
\eqref{eq:online-player-pdl},
\eqref{eq:online-occupancy-remainder}, and
\eqref{eq:online-local-advantage-descent}, we conclude that
\begin{align}
G_i^t
\ge{}& 
\frac{1}{2\eta}d_\mu^k(s)\mathcal D_i^k(s)
-
\frac{c\eta}{(1-\gamma)^2}
d_\mu^k(s)\mathcal E_i^k(s)-\frac{cL_d(1+L_d)}{1-\gamma}
d_\mu^k(s)\mathcal D_i^k(s).
\label{eq:online-weighted-unilateral-descent}
\end{align}

It remains to compare $G_i^t$ with the $i$th term in the simultaneous
telescoping decomposition \eqref{eq:online-potential-telescoping}. By
construction, $\widetilde\pi^{t,i-1}$ and $\widetilde\pi^{t,i}$ differ only
in player $i$'s policy at state $s$, from
$\pi_i^k(\cdot\mid s)$ to $\pi_i^{k+1}(\cdot\mid s)$. Hence, by performance difference lemma (Lemma~\ref{lemm:performance-difference}),
\begin{align}
V_i^{\widetilde\pi^{t,i-1}}(\mu)
-
V_i^{\widetilde\pi^{t,i}}(\mu)
=
\frac{1}{1-\gamma}
d_\mu^{\widetilde\pi^{t,i}}(s)
\left\langle
\Delta_i^k(s),
\bar A_i^{\widetilde\pi^{t,i-1}}(s,\cdot)
\right\rangle.
\nonumber
\end{align}
Subtracting \eqref{eq:intermediate-online} from this equation and applying the triangle inequality after adding and subtracting the corresponding terms with the common reference occupancy $d_\mu^k(s)$, we obtain
\begin{align}
\!\!\!\left|
\bigl[
V_i^{\widetilde\pi^{t,i-1}}(\mu)
-
V_i^{\widetilde\pi^{t,i}}(\mu)
\bigr]
-
G_i^t
\right|&\le
\frac{1}{1-\gamma}
\Big|
d_\mu^{\widetilde\pi^{t,i}}(s)
-
d_\mu^{k}(s)
\Big|
\left|
\left\langle
\Delta_i^k(s),
\bar A_i^{\widetilde\pi^{t,i-1}}(s,\cdot)
\right\rangle
\right|
\notag\\
&+
\frac{1}{1-\gamma}
\Big|
d_\mu^{(\pi_i^{k+1}(s),\pi_{-i}^t)}(s)
-
d_\mu^{k}(s)
\Big|
\left|
\left\langle
\Delta_i^k(s),
\bar A_i^k(s,\cdot)
\right\rangle
\right|
\notag\\
&+
\frac{1}{1-\gamma}d_\mu^k(s)
\left|
\left\langle
\Delta_i^k(s),
\bar A_i^{\widetilde\pi^{t,i-1}}(s,\cdot)
-
\bar A_i^k(s,\cdot)
\right\rangle
\right|.
\label{eq:online-simultaneous-decomposition}
\end{align}
By definition \eqref{eq:tile-pi-online}, the policy profiles $\widetilde\pi^{t,i}$ and
$\pi^t$ agree at every state $x\neq s$, while at state $s$ they differ only in the
policies of players $j\le i$, which are $\pi_j^{k+1}(\cdot\mid s)$ under
$\widetilde\pi^{t,i}$ and $\pi_j^k(\cdot\mid s)$ under
$\pi^t$. Therefore,
\eqref{eq:online-localized-occupancy-sensitivity} gives
\begin{align}
\Big|
d_\mu^{\widetilde\pi^{t,i}}(s)
-
d_\mu^k(s)
\Big|\le
L_d\,
d_\mu^k(s)
\sum_{j\le i}\|\Delta_j^k(s)\|_{\mathrm{TV}}.
\label{eq:online-intermediate-occupancy-difference-final}
\end{align}
Similarly, by definition \eqref{eq:tile-pi-online} $\widetilde\pi^{t,i-1}$ and $\pi^t=(\pi_i^k(s),\pi_{-i}^t)$ agree everywhere except
that, at state $s$, the policies of players $j<i$ are
$\pi_j^{k+1}(\cdot\mid s)$ under $\widetilde\pi^{t,i-1}$ and
$\pi_j^k(\cdot\mid s)$ under $\pi^t$. Hence
\eqref{eq:online-advantage-sensitivity} gives
\begin{equation}\label{eq:online-intermediate-advantage-difference}
\left\|
\bar A_i^{\widetilde\pi^{t,i-1}}(s,\cdot)
-
\bar A_i^k(s,\cdot)
\right\|_\infty
\le
L_{\bar A}
\sum_{j<i}\|\Delta_j^k(s)\|_{\mathrm{TV}}.
\end{equation}
Now, applying Lemma~\ref{lem:span-inner-product} to bound the first two inner products in \eqref{eq:online-simultaneous-decomposition}, and then using the bounds \eqref{eq:online-one-state-occupancy-change}, \eqref{eq:online-intermediate-occupancy-difference-final}, and
\eqref{eq:online-intermediate-advantage-difference} in
\eqref{eq:online-simultaneous-decomposition}, we get
\begin{align}\nonumber
\left|
\bigl[
V_i^{\widetilde\pi^{t,i-1}}(\mu)
-
V_i^{\widetilde\pi^{t,i}}(\mu)
\bigr]
-
G_i^t
\right|\le \frac{c\big(L_{\bar A}+L_d(1+L_d)\big)}{1-\gamma}
d_\mu^k(s)
\|\Delta_i^k(s)\|_{\mathrm{TV}}
\sum_{j\le i}\|\Delta_j^k(s)\|_{\mathrm{TV}}.
\end{align}
Summing the above relation over $i$ and using Pinsker's inequality $\|\Delta_i^k(s)\|_{\mathrm{TV}}^2
\le
\frac14\mathcal D_i^k(s)$, we obtain
\begin{align}
\sum_{i=1}^n
\left|
\bigl[
V_i^{\widetilde\pi^{t,i-1}}(\mu)
-
V_i^{\widetilde\pi^{t,i}}(\mu)
\bigr]
-
G_i^t
\right|&\leq 
\frac{c\big(L_{\bar A}+L_d(1+L_d)\big)}{1-\gamma}
d_\mu^k(s)\Big(\sum_{i=1}^n\|\Delta_i^k(s)\|_{\mathrm{TV}}\Big)^2\cr 
&\le \frac{cn\big(L_{\bar A}+L_d(1+L_d)\big)}{1-\gamma}
d_\mu^k(s)
\sum_{i=1}^n\mathcal D_i^k(s).
\label{eq:online-simultaneous-error-sum}
\end{align}

Combining
\eqref{eq:online-potential-telescoping}, \eqref{eq:online-weighted-unilateral-descent}, and
\eqref{eq:online-simultaneous-error-sum} gives
\begin{align*}
\Psi(\pi^t)-\Psi(\pi^{t+1})
&\ge
\left[
\frac{1}{2\eta}
-\frac{cL_d(1+L_d)}{1-\gamma}
-\frac{cn\big(L_{\bar A}+L_d(1+L_d)\big)}{1-\gamma}
\right]
d_\mu^k(s)\sum_{i=1}^n\mathcal D_i^k(s)\cr 
&\qquad-
\frac{c\eta}{(1-\gamma)^2}
d_\mu^k(s)\sum_{i=1}^n\mathcal E_i^k(s)
-
n\alpha.
\end{align*}
Since $n\ge1$, the two coefficients are jointly bounded
by $\frac{cn}{1-\gamma}\big(L_{\bar A}+L_d(1+\frac{\gamma\delta_P}{1-\gamma})\big)$. Thus, choosing the constant $c_0>0$ in
\eqref{eq:online-potential-absorption} sufficiently small allows this term
to be absorbed into the leading $1/(2\eta)$ coefficient. Recalling that
$s=S^t$, $k=N_t(S^t)$, and, by
\eqref{eq:online-local-global-identification},
$d_\mu^k(s)=d_\mu^{\pi^t}(S^t)$, this yields
\begin{align*}
\Psi(\pi^t)-\Psi(\pi^{t+1})
\ge{}&
\frac{1}{4\eta}
d_\mu^{\pi^t}(S^t)
\sum_{i=1}^n\mathcal D_i^{N_t(S^t)}(S^t)
-
\frac{c\eta}{(1-\gamma)^2}
d_\mu^{\pi^t}(S^t)
\sum_{i=1}^n\mathcal E_i^{N_t(S^t)}(S^t)
-
n\alpha.
\end{align*}

Finally, summing the above relation over $t$ and using the definitions of $\mathfrak D_T$ and
$\mathfrak E_T$ gives
\[
\frac{1}{4\eta}\mathfrak D_T
\le
\Psi(\pi^0)-\Psi(\pi^T)
+
\frac{c\eta}{(1-\gamma)^2}\mathfrak E_T
+
n\alpha T.
\]
By the $\alpha$-potential property, changing the players' policies
one at a time and using $|V_i^\pi(\mu)|\le \frac{1}{1-\gamma}$ gives $\Psi(\pi^0)-\Psi(\pi^T)
\le n\big(\frac{2}{1-\gamma}+\alpha\big)\leq \frac{3n}{1-\gamma}$. Multiplying the above relation by $4\eta$, using this relation, and absorbing numerical constants into the universal constant $c$ yields \eqref{eq:online-D-from-potential}.
\end{proof}

\subsubsection{A delayed occupancy-weighted NE gap}

Recall that at the beginning of a global time $t$, the current policy stored at a state $s$ is
$\pi^{N_t(s)}(\cdot\mid s)$, whereas the most recent update that produced
this policy was performed at the preceding visit, namely at global time
$\tau_{N_t(s)-1}(s)$. Recall also that $\bar A_i^{N_t(s)-1}(s,\cdot)$ denotes the marginalized advantage under the global policy at that preceding visit. Thus, once $s$ has been visited at least once, the
quantities $\mathcal D_i^{N_t(s)-1}(s)$ and $\mathcal E_i^{N_t(s)-1}(s)$, are respectively the policy movement and tracking error associated with
the most recent update at $s$ strictly before time $t$. These quantities
are generally delayed relative to the current global policy $\pi^t$:
although the policy stored at $s$ has not changed since that visit,
policies at other states may have changed in the meantime, and hence the
global-policy quantities $d_\mu^{\pi^t}$ and
$\bar A_i^{\pi^t}$ need not coincide with their values at the time
of that update. This delay is inherent in the fully online statewise
scheme, since only the currently visited state is updated at each global
time. The next lemma uses these most recent statewise quantities to
control the current NE gap, while the subsequent coupling lemma
shows that, with high probability, the delay is uniformly bounded and
allows the corresponding quantities to be related to
those at their original update times.

\begin{lemma}
\label{lem:online-ne-gap}
Suppose that every state has been visited strictly before global time $t$,
i.e., $N_t(s)\ge 1$ for every $s\in\mathcal S$. Then
\begin{align}\label{eq:online-ne-gap-bound}
\operatorname{Gap}(\pi^t)
\le{}&\frac{c(\zeta+L_d)(1+L_d)}{1-\gamma}
+\frac{c}{\eta}\sqrt{\frac{A_{\max}}{\zeta}}
\left(
\sum_s d_\mu^{\pi^t}(s)
\sum_{i=1}^n
\mathcal D_i^{N_t(s)-1}(s)
\right)^{1/2}
\notag\\
&+
\frac{c}{1-\gamma}\sqrt{\frac{A_{\max}}{\zeta}}
\left(
\sum_s d_\mu^{\pi^t}(s)
\sum_{i=1}^n
\mathcal E_i^{N_t(s)-1}(s)
\right)^{1/2}
\notag\\
&+
\frac{c}{1-\gamma}
\left(
\sum_s d_\mu^{\pi^t}(s)
\sum_{i=1}^n
\left\|
\bar A_i^{\pi^t}(s,\cdot)
-
\bar A_i^{N_t(s)-1}(s,\cdot)
\right\|_\infty^2
\right)^{1/2}.
\end{align}
\end{lemma}

\begin{proof}
Fix a player $i$ and let $\pi_i^{t,*}$ be a best response to
$\pi_{-i}^t$. By the
performance-difference lemma,
\begin{align}\nonumber
V_i(\pi^t)-V_i(\pi_i^{t,*},\pi_{-i}^t)
={}&
\frac{1}{1-\gamma}
\sum_{s\in\mathcal S}
d_\mu^{(\pi_i^{t,*},\pi_{-i}^t)}(s)
\left\langle
\pi_i^t(\cdot\mid s)-\pi_i^{t,*}(\cdot\mid s),
\bar A_i^{\pi^t}(s,\cdot)
\right\rangle.
\end{align}
As the policies $(\pi_i^{t,*},\pi_{-i}^t)$ and $\pi^t$ differ only in player $i$,
Lemma~\ref{lemm:occupancy-sensitivity} gives $\|
d_\mu^{(\pi_i^{t,*},\pi_{-i}^t)}
-
d_\mu^{\pi^t}\|_{\mathrm{TV}}
\le L_d$. Consequently, by adding and subtracting $d_\mu^{\pi^t}$ in the above
relation and using Lemma~\ref{lem:span-inner-product}, we obtain
\begin{align}\label{eq:online-ne-current-occupancy}
V_i(\pi^t)-V_i(\pi_i^{t,*},\pi_{-i}^t)
\le{}&
\frac{1}{1-\gamma}
\sum_{s\in\mathcal S}
d_\mu^{\pi^t}(s)
\left\langle
\pi_i^t(\cdot\mid s)-\pi_i^{t,*}(\cdot\mid s),
\bar A_i^{\pi^t}(s,\cdot)
\right\rangle
+\frac{cL_d(1+L_d)}{1-\gamma}.
\end{align}

As in the proof of Lemma~\ref{lem:kl-projected-ne-gap}, introduce the
feasible truncated best response
\[
q_i^{t,*}(\cdot\mid s)
:=
(1-\zeta)\pi_i^{t,*}(\cdot\mid s)
+
\frac{\zeta}{|\mathcal A_i|}\mathbf 1
\in\Delta_{i,\zeta}.
\]
Since $\left\|
q_i^{t,*}(\cdot\mid s)-\pi_i^{t,*}(\cdot\mid s)
\right\|_1
\le 2\zeta$, using Lemma~\ref{lem:span-inner-product}, we have
\begin{align}\label{eq:online-ne-truncation}
\left\langle
\pi_i^t(\cdot\mid s)-\pi_i^{t,*}(\cdot\mid s),
\bar A_i^{\pi^t}(s,\cdot)
\right\rangle\le
\left\langle
\pi_i^t(\cdot\mid s)-q_i^{t,*}(\cdot\mid s),
\bar A_i^{\pi^t}(s,\cdot)
\right\rangle
+\zeta\left(1+L_d\right).
\end{align}

We now relate the first term on the right-hand side of \eqref{eq:online-ne-truncation}, which is
defined at the current global time $t$, to the most recent local update at
state $s$. Fix $s\in\mathcal S$ and write $k=N_t(s)\ge1$. The most recent visit to
$s$ before time $t$ is $\tau_{k-1}(s)$, at which the local update changes
$\pi_i^{k-1}(\cdot\mid s)$ to $\pi_i^k(\cdot\mid s)$ for every player $i$. Since there is no
subsequent visit to $s$ before time $t$, the policy stored at $s$ at the
beginning of time $t$ is therefore $\pi^t(\cdot\mid s)=\pi^k(\cdot\mid s)$.

We next invoke the KL optimality condition for the local update
performed at $\tau_{k-1}(s)$. By
Corollary~\ref{cor:local-kl-projected-geometry}, this update is exactly an
instance of the KL-projected NPG update in
Lemma~\ref{lem:kl-projected-local-geometry}, with the update vector
$r_i^{k-1}(s)$. Thus, applying the same variational-inequality and
Cauchy--Schwarz argument used in \eqref{eq:D-A-A}, with the global time
there replaced by the local index $k-1$, yields,
\begin{align}\label{eq:online-stale-KL-optimality}
\left\langle
r_i^{k-1}(s),
\pi_i^k(\cdot\mid s)-q(\cdot\mid s)
\right\rangle
\le
\frac{1-\gamma}{\eta}
\sqrt{\frac{2A_{\max}}{\zeta}}
\sqrt{\mathcal D_i^{k-1}(s)} \ \ \ \ \ \ \ \ \ \forall q(\cdot\mid s)\in\Delta_{i,\zeta}.
\end{align}
In particular, although $q_i^{t,*}$ depends on the global policy at the
later time $t$, it may be substituted for $q$ in
\eqref{eq:online-stale-KL-optimality}.\footnote{Indeed, the variational inequality
associated with the realized update at $\tau_{k-1}(s)$ is a pathwise
inequality holding for every element of $\Delta_{i,\zeta}$; no adaptedness or measurability of the subsequently chosen comparator is required.} By the definition of the local tracking error in
\eqref{eq:online-tracking-error},
\begin{equation}\label{eq:online-stale-r-decomposition}
r_i^{k-1}(s)
=
\bar A_i^{k-1}(s,\cdot)
+
e_i^{k-1}(s).
\end{equation}
Using \eqref{eq:online-stale-r-decomposition} and the fact that $\pi_i^t(\cdot\mid s)=\pi_i^k(\cdot\mid s)$, we can write
\begin{align}
&\!\!\!\!\!\!\!\!\!\left\langle
\pi_i^t(\cdot\mid s)-q_i^{t,*}(\cdot\mid s),
\bar A_i^{\pi^t}(s,\cdot)
\right\rangle=
\left\langle
\pi_i^k(\cdot\mid s)-q_i^{t,*}(\cdot\mid s),
r_i^{k-1}(s)
\right\rangle\cr 
&\!\!\!\!\!\quad-
\left\langle
\pi_i^k(\cdot\mid s)-q_i^{t,*}(\cdot\mid s),
e_i^{k-1}(s)
\right\rangle+
\left\langle
\pi_i^k(\cdot\mid s)-q_i^{t,*}(\cdot\mid s),
\bar A_i^{\pi^t}(s,\cdot)
-
\bar A_i^{k-1}(s,\cdot)
\right\rangle\!.
\nonumber
\end{align}
The first term on the right-hand side is controlled by
\eqref{eq:online-stale-KL-optimality}. For the tracking term, since we have
$\pi_i^{k-1}(a_i\mid s)\ge\zeta/|\mathcal A_i|\ge\zeta/A_{\max}$,
the definition of $\mathcal E_i^{k-1}(s)$ and Cauchy--Schwarz give
\begin{align}
\left|
\left\langle
\pi_i^k(\cdot\mid s)-q_i^{t,*}(\cdot\mid s),
e_i^{k-1}(s)
\right\rangle
\right|
&\le
\left\|
\pi_i^k(\cdot\mid s)-q_i^{t,*}(\cdot\mid s)
\right\|_2
\bigg(
\sum_{a_i}
\frac{\pi_i^{k-1}(a_i\mid s)
(e_i^{k-1}(s,a_i))^2}
{\pi_i^{k-1}(a_i\mid s)}
\bigg)^{1/2}
\notag\\
&\le
\sqrt{2}\,
\sqrt{\frac{A_{\max}}{\zeta}}\,
\Big(
\sum_{a_i}
\pi_i^{k-1}(a_i\mid s)
(e_i^{k-1}(s,a_i))^2\Big)^{1/2}
\notag\\
&=
\sqrt{\frac{2A_{\max}}{\zeta}}\,
\sqrt{\mathcal E_i^{k-1}(s)}.
\nonumber
\end{align}
Finally, since $\pi_i^k(\cdot\mid s)$ and
$q_i^{t,*}(\cdot\mid s)$ are probability distributions,
$\|\pi_i^k(\cdot\mid s)-q_i^{t,*}(\cdot\mid s)\|_1\le2$.
Hence, by H\"older's inequality,
\begin{align}\label{eq:online-ne-delay-local}
\left|
\left\langle
\pi_i^k(\cdot\mid s)-q_i^{t,*}(\cdot\mid s),
\bar A_i^{\pi^t}(s,\cdot)-\bar A_i^{k-1}(s,\cdot)
\right\rangle
\right|
&\le
2\left\|\bar A_i^{\pi^t}(s,\cdot)-\bar A_i^{k-1}(s,\cdot)\right\|_\infty.
\end{align}
Combining
\eqref{eq:online-stale-KL-optimality}--\eqref{eq:online-ne-delay-local}
therefore gives
\begin{align}\label{eq:online-ne-local-certificate}
\left\langle
\pi_i^t(\cdot\mid s)-q_i^{t,*}(\cdot\mid s),
\bar A_i^{\pi^t}(s,\cdot)
\right\rangle&\le
\frac{1-\gamma}{\eta}
\sqrt{\frac{2A_{\max}}{\zeta}}
\sqrt{\mathcal D_i^{k-1}(s)}+
\sqrt{\frac{2A_{\max}}{\zeta}}
\sqrt{\mathcal E_i^{k-1}(s)}\cr 
&\qquad+2
\left\|
\bar A_i^{\pi^t}(s,\cdot)
-
\bar A_i^{k-1}(s,\cdot)
\right\|_\infty.
\end{align}

We now return from the local index $k$ to global time $t$. Since
$k=N_t(s)$, multiplying \eqref{eq:online-ne-local-certificate} by
$d_\mu^{\pi^t}(s)$ and summing over $s$, together with Cauchy--Schwarz under the
probability distribution $d_\mu^{\pi^t}$, yields
\begin{align}
&
\sum_s d_\mu^{\pi^t}(s)
\left\langle
\pi_i^t(\cdot\mid s)-q_i^{t,*}(\cdot\mid s),
\bar A_i^{\pi^t}(s,\cdot)
\right\rangle
\notag\\
&\qquad\le
\frac{1-\gamma}{\eta}
\sqrt{\frac{2A_{\max}}{\zeta}}
\left(
\sum_s d_\mu^{\pi^t}(s)
\mathcal D_i^{N_t(s)-1}(s)
\right)^{1/2}\!\!\!\!\!\!+
\sqrt{\frac{2A_{\max}}{\zeta}}
\left(
\sum_s d_\mu^{\pi^t}(s)
\mathcal E_i^{N_t(s)-1}(s)
\right)^{1/2}
\notag\\
&\qquad+
2
\left(
\sum_s d_\mu^{\pi^t}(s)
\left\|
\bar A_i^{\pi^t}(s,\cdot)
-
\bar A_i^{N_t(s)-1}(s,\cdot)
\right\|_\infty^2
\right)^{1/2}.
\label{eq:online-ne-global-certificate}
\end{align}

Substituting \eqref{eq:online-ne-truncation} and
\eqref{eq:online-ne-global-certificate} into
\eqref{eq:online-ne-current-occupancy} gives
\begin{align}
V_i(\pi^t)-V_i(\pi_i^{t,*},\pi_{-i}^t)
\le{}&\frac{c(\zeta+L_d)(1+L_d)}{1-\gamma}+
\frac{c}{\eta}\sqrt{\frac{A_{\max}}{\zeta}}
\left(
\sum_s d_\mu^{\pi^t}(s)
\mathcal D_i^{N_t(s)-1}(s)
\right)^{1/2}
\notag\\
&+
\frac{c}{1-\gamma}\sqrt{\frac{A_{\max}}{\zeta}}
\left(
\sum_s d_\mu^{\pi^t}(s)
\mathcal E_i^{N_t(s)-1}(s)
\right)^{1/2}
\notag\\
&+
\frac{c}{1-\gamma}
\left(
\sum_s d_\mu^{\pi^t}(s)
\left\|
\bar A_i^{\pi^t}(s,\cdot)
-
\bar A_i^{N_t(s)-1}(s,\cdot)
\right\|_\infty^2
\right)^{1/2}.
\label{eq:online-ne-player-bound}
\end{align}
Finally, each of the three nonnegative player-specific sums in
\eqref{eq:online-ne-player-bound} is bounded by its corresponding sum
over all players. Therefore, taking the maximum over $i\in[n]$ and
recalling the definition of $\operatorname{Gap}(\pi^t)$ yields
\eqref{eq:online-ne-gap-bound}.
\end{proof}

\subsubsection{Frozen-policy coupling and delayed transfer}

The next lemma provides the key bridge between the fixed-policy coverage assumption and the fully online, time-varying dynamics of the algorithm. The main idea is to couple the actual trajectory over each coverage window with an auxiliary episodic trajectory in which the policy is frozen at the beginning of the window. Since the policy changes only at the currently visited state and each update is small, condition~\eqref{eq:online-coupling-condition} ensures that the cumulative transition perturbation within a coverage window is sufficiently small. Consequently, the uniform state-coverage guarantee for the frozen policy transfers to the actual nonstationary trajectory, yielding, with high probability, a uniform bound $H_T$ on the time between successive visits to every state. This bounded-delay property is then used to relate quantities indexed by the most recent completed visit to a state to their corresponding contemporaneous online quantities. In particular, each delayed KL movement or estimation-error term can be charged to at most $H_T$ global time steps, up to additional errors caused by drift in the state-occupancy distribution. The same bounded-delay argument, together with the sensitivity of the marginalized advantages to policy changes, controls the discrepancy between the current advantage $\bar A_i^{\pi^t}$ and the stale advantage $\bar A_i^{N_t(s)-1}$. Thus, the lemma provides the mechanism for converting statewise, visit-indexed estimates and policy movements into global-time bounds that can be used in the subsequent NE-gap and regret analysis. The proof is given in Appendix~\ref{app:coupling-lemma}.

\begin{lemma}\label{lem:online-coupling-transfer}
Suppose the coverage Assumption 
\ref{ass:coverage} holds and \begin{equation}\label{eq:online-coupling-condition}
\eta
\le
\frac{p_{\min}(1-\gamma)}
{2\delta_P nM H_{\rm cov}^2}.
\end{equation}
Let $H_T:=H_{\rm cov}\lceil
\frac{2}{p_{\min}}
\log(\frac{2(|\mathcal S|+T)}{\rho})
\rceil$. Then, with probability at least $1-\rho/2$, the bounded intervisit event
$\mathcal H_T(H_T)$ given in Definition
\ref{def:bounded-intervisit-event} holds. In particular, on
$\mathcal H_T(H_T)$, we have
$N_t(s)\ge1$ for every $s\in\mathcal S$ and $t\ge H_T+1$, and
\begin{align}
\sum_{t=H_T+1}^{T-1}\sum_s
d_\mu^{\pi^t}(s)
\sum_{i=1}^n\mathcal D_i^{N_t(s)-1}(s)
&\le H_T\mathfrak D_T+2\frac{L_d n^2M^3}{(1-\gamma)^3}
H_TT\eta^3,
\label{eq:online-delayed-D-charging}\\
\sum_{t=H_T+1}^{T-1}\sum_s
d_\mu^{\pi^t}(s)
\sum_{i=1}^n\mathcal E_i^{N_t(s)-1}(s)
&\le H_T\mathfrak E_T+8\frac{L_d n^2M^3}{1-\gamma}
H_TT\eta,
\label{eq:online-delayed-E-charging}
\end{align}
Moreover, the occupancy-weighted stale marginalized-advantage drift is bounded by
\begin{align}\nonumber
&\sum_{t=H_T+1}^{T-1}
\left\{
\sum_s d_\mu^{\pi^t}(s)
\sum_i
\left\|
\bar A_i^{\pi^t}(s,\cdot)
-
\bar A_i^{N_t(s)-1}(s,\cdot)
\right\|_\infty^2
\right\}^{1/2}
\le \frac{L_{\bar A} n^{3/2}M}{1-\gamma}H_TT\eta.
\end{align}
\end{lemma}

\subsubsection{High-probability fully online NE-regret}

Finally, by combining the above lemmas, we can state the main result of this section, which provides a high-probability NE-regret bound for fully online asynchronous decentralized learning in Markov $\alpha$-potential games. We sketch only the main steps of the proof here and defer the parameter-feasibility checks, term-by-term tuning, and treatment of the capped regimes to Appendix~\ref{app:online-final-proof}. 
\begin{theorem}\label{thm:fully-online-general-ne-regret}
Let Assumptions~\ref{ass:alpha-potential-ergodicity},
\ref{ass:coverage}, and \ref{ass:online-one-sample-oracle} hold, and let the
players follow Algorithm~\ref{alg:fully-online-general-occupancy}. Fix $T\ge1$
and $\rho\in(0,1)$ and define $\Lambda_T:=
\log(\frac{16nA_{\max}(|\mathcal S|+T)}{\rho})$ and $H_T
:=H_{\rm cov}\lceil
\frac{2}{p_{\min}}
\log(\frac{2(|\mathcal S|+T)}{\rho})\rceil$. Let $x_T:=\frac{\Lambda_T}{T}$, $u_T:=\max\{x_T/(1-\gamma),\alpha\}$, and define $\eta_{\rm cov}:=\frac{p_{\min}(1-\gamma)}
{2\delta_P n(1+L_d)H_{\rm cov}^2}$.\footnote{Here, we adopt the convention that $\eta_{\rm cov}=+\infty$ when $\delta_P=0$.} Choose
\begin{equation}\label{eq:online-final-tuning}
\begin{aligned}
\eta
&=
\min\left\{\eta_{\rm cov}, \
c_\eta
\frac{1-\gamma}{nL_{\bar A}}
H_T^{-1/2}u_T^{3/5}
\right\},\qquad 
\omega=
\min\left\{
\frac14,\,
c_\omega H_T^{1/4}u_T^{2/5}
\right\},
\\
\zeta
&=
\min\left\{
\frac14,\,
c_\zeta
\max\left\{
H_T^{1/2}u_T^{2/15},
L_{\widehat A}^{2/3}
\right\}
\right\},
\end{aligned}
\end{equation}
for sufficiently small universal constants
$c_\eta,c_\omega,c_\zeta>0$. Then, with probability at least $1-\rho$,
\begin{align*}
\frac1T\sum_{t=0}^{T-1}\operatorname{Gap}(\pi^t)
\le{}&
\mathfrak C_1 \left(\frac{\Lambda_T}{(1-\gamma)T}+\alpha\right)^{2/15}\!\!\!\!+\mathfrak C_2 L_{\widehat A}^{2/3}+
\frac{cL_d(1+L_d)}{1-\gamma},
\end{align*}
where $L_d=\frac{\gamma\delta_P}{1-\gamma}$ and
\begin{align}\nonumber
\mathfrak C_1=
\widetilde O\bigg(\!\sqrt{\frac{H_{\rm cov}}{p_{\min}}}\Big(
\frac{\sqrt n\,A_{\max}+n(1+L_d)\sqrt{A_{\max}}}{1-\gamma}\Big)\!\bigg), \quad 
\mathfrak C_2=
\widetilde O\bigg(\!\sqrt{\frac{H_{\rm cov}}{p_{\min}}}\Big(
\frac{
1+L_d+\sqrt{nA_{\max}}
}{
1-\gamma
}
\Big)\!\bigg).
\end{align}
\end{theorem}

\begin{proof}[Proof sketch]
Recall that
$M=1+L_d$. The choices in \eqref{eq:online-final-tuning} ensure that the
conditions of Lemmas~\ref{lem:online-tracking},
\ref{lem:online-potential-improvement}, and
\ref{lem:online-coupling-transfer} hold. Thus, with probability at least
$1-\rho$, the events $\mathcal H_T$ and $\mathcal T_T$ hold
simultaneously, and we work on this event below.

By Lemma~\ref{lem:online-potential-improvement},
\begin{equation}\label{eq:online-D-proof-sketch}
\frac{\mathfrak D_T}{T}
\le
c\left[
\frac{n\eta}{(1-\gamma)T}
+n\eta\alpha
+\frac{\eta^2}{(1-\gamma)^2}
\frac{\mathfrak E_T}{T}
\right].
\end{equation}
On the other hand, Lemma~\ref{lem:online-tracking}, together with
$\mathcal V_i^k(s)\le A_{\max}$ from
Lemma~\ref{lemm:conditional-variance-online}, gives a bound on
$\mathfrak E_T/T$ in terms of $(\mathfrak D_T/T)^{1/2}$.
Substituting \eqref{eq:online-D-proof-sketch} into this bound and applying
Young's inequality closes the feedback between the tracking error and the
KL policy movement. Substitution of the parameters in
\eqref{eq:online-final-tuning} then gives
\begin{equation}\label{eq:online-E-combined-rate-sketch}
\frac{\mathfrak E_T}{T}
\le
c\left[
\left(nA_{\max}+n^{3/2}(1+L_d)^2\right)
H_T^{1/4}u_T^{2/5}
+nL_{\widehat A}^2
\right].
\end{equation}

It remains to convert these global movement and tracking bounds into
NE regret. On $\mathcal H_T$, every state has been visited by time
$H_T+1$. Hence, averaging Lemma~\ref{lem:online-ne-gap}, applying
Jensen's inequality, and then using
Lemma~\ref{lem:online-coupling-transfer} yield
\begin{align}\label{eq:online-regret-master-sketch}
\frac1T\sum_{t=0}^{T-1}&\operatorname{Gap}(\pi^t)
\le{}
\frac{c(\zeta+L_d)(1+L_d)}{1-\gamma}
+
\frac{c}{\eta}\sqrt{\frac{A_{\max}}{\zeta}}
\left(
H_T\frac{\mathfrak D_T}{T}
+
2\frac{L_d n^2M^3}{(1-\gamma)^3}H_T\eta^3
\right)^{1/2}\cr 
&+
\frac{c}{1-\gamma}\sqrt{\frac{A_{\max}}{\zeta}}
\left(
H_T\frac{\mathfrak E_T}{T}
+
8\frac{L_d n^2M^3}{1-\gamma}H_T\eta
\right)^{1/2}
+
\frac{L_{\bar A}n^{3/2}H_T\eta M}{(1-\gamma)^2}
+
\frac{H_T+1}{T}.
\end{align}
Finally, substituting
\eqref{eq:online-D-proof-sketch} and
\eqref{eq:online-E-combined-rate-sketch} into
\eqref{eq:online-regret-master-sketch}, and using the parameter choices
in \eqref{eq:online-final-tuning}, gives the claimed bound. The detailed
term-by-term verification and proof steps are given in
Appendix~\ref{app:online-final-proof}.
\end{proof}

\begin{remark}
It is instructive to compare the dependence on $\alpha$ with the finite-time bounds in \cite[Theorem~6.1]{guo2025markov}. After normalizing one-stage costs to $[0,1]$, their two bounds are nontrivial, i.e., become less than $(1-\gamma)^{-1}$, only when $\alpha\lesssim (1-\gamma)^5/(n^4\widetilde\kappa_\mu^{\,2}A_{\max}^2)$ and $\alpha\lesssim (1-\gamma)^4/(\min\{\kappa_\mu,|\mathcal S|\}^{4}nA_{\max})$, respectively, up to universal constants. In contrast, after separating the state-coverage cost specific to our online setting, our $\alpha$-dependent term scales as $\widetilde O((\sqrt n A_{\max}+n(1+L_d)\sqrt{A_{\max}})\alpha^{2/15}/(1-\gamma))$ and thus imposes no additional polynomial restriction on $\alpha$ in $1-\gamma$, apart from that through $L_d$.
\end{remark}


\section{Sharper Guarantees under State-Wise Potentials}\label{sec:statewise-potential-sharp}

The analysis in Sections~\ref{sec:Episodic} and~\ref{sec:fully-online-general-occupancy}
was developed for a general policy-level potential $\Psi$ and therefore uses the
players' marginalized advantages $\bar A_i^\pi$. In this section, we show that
when the $\alpha$-potential function $\Psi$ admits a value-function representation
in terms of an underlying state-wise potential, the same algorithms and proofs
can be adapted to work directly with the potential advantages. The main benefit
is that the approximation error $\alpha$ no longer enters the
potential-improvement arguments; instead, it is used only when potential
suboptimality is converted into the players' NE gap. Consequently, the
fractional-power dependence on $\alpha$ in
Sections~\ref{sec:Episodic} and~\ref{sec:fully-online-general-occupancy} is
replaced by an additive $\alpha$ term, leading to improved bounds. Since the
resulting analysis closely parallels that of the previous two sections and
substantially overlaps with it, we omit the detailed proofs. Instead, in
Appendix~\ref{app:state-wise-appendix}, we highlight the modifications needed
to obtain the improved bounds from the earlier results. These modifications
require only minor changes to the analysis and notation while leveraging the
additional state-wise potential structure.

\subsection{State-wise potential structure}\label{sec:statewise-potential-structure}

Throughout this section, we impose the following assumption. 

\begin{assumption}\label{eq:state-wise-potential}
The $\alpha$-potential function of the game, denoted by $\Psi^{\pi}(\mu)$ can be represented as the discounted value of a base potential $\Phi:\mathcal S\times\mathcal A\to\mathbb R$, i.e.,
\begin{equation}\label{eq:statewise-Psi-representation}
\Psi^{\pi}(\mu)
=\mathbb E_\pi\bigg[
\sum_{h=0}^{\infty}\gamma^h\Phi(S^h,A^h)
\,\bigg|\,
S^0\sim\mu
\bigg].
\end{equation} 
\end{assumption}

Similar as before, for every stationary policy profile $\pi$, one can define the \emph{marginalized potential $Q$-function} and the \emph{marginalized potential advantage function} as
\begin{align}\label{eq:marginalized-potentia-advantage-Q}
\bar Q_{\Phi,i}^\pi(s,a_i)
&:=
\mathbb E_{A_{-i}\sim\pi_{-i}(\cdot\mid s)}
\bigg[
\Phi(s,a_i,A_{-i})
+
\gamma\sum_{s'}P(s'\mid s,a_i,A_{-i})\Psi^{\pi}(s')
\bigg],\cr 
\bar A_{\Phi,i}^\pi(s,a_i)
&:=\bar Q_{\Phi,i}^\pi(s,a_i)-\Psi^{\pi}(s).
\end{align}
Since \eqref{eq:statewise-Psi-representation} is itself a discounted Markov value function, the unilateral performance-difference identity holds exactly. That is, for $\pi'=(\pi_i',\pi_{-i})$, we have\footnote{For simplicity, whenever there is no ambiguity regarding the initial distribution, we write $\Psi(\pi)$ instead of $\Psi^{\pi}(\mu)$.}
\begin{equation}\label{eq:statewise-potential-PDL}
\Psi(\pi')-\Psi(\pi)
=\frac{1}{1-\gamma}
\sum_{s\in\mathcal S}d_\mu^{\pi'}(s)
\left\langle
\pi_i'(\cdot\mid s)-\pi_i(\cdot\mid s),
\bar A_{\Phi,i}^\pi(s,\cdot)
\right\rangle.
\end{equation}
Moreover, the same sensitivity argument used for marginalized player
advantages gives a finite constant $L_{\bar A_\Phi}$ such that
\begin{equation*}
\max_{i,s,a_i}
\left|
\bar A_{\Phi,i}^{\pi'}(s,a_i)-
\bar A_{\Phi,i}^{\pi}(s,a_i)
\right|
\le
L_{\bar A_\Phi}
\max_{x\in\mathcal S}
\sum_{j=1}^n
\|\pi_j'(\cdot\mid x)-\pi_j(\cdot\mid x)\|_{\rm TV}.
\end{equation*}
For later use, define the range of the one-stage potential by
\begin{equation}\label{eq:statewise-Phi-range}
R_\Phi:=\max_{s,a}\Phi(s,a)-\min_{s,a}\Phi(s,a).
\end{equation}
Since adding a constant to $\Phi$ does not change either $\bar A_{\Phi,i}$ or any policy difference of $\Psi$, we may first shift $\Phi$ so that its range is contained in $[0,R_\Phi]$, where without loss of generality we may assume $R_\Phi>0$.\footnote{If
$R_\Phi=0$, then $\Psi$ given by \eqref{eq:statewise-Psi-representation} is policy independent and \eqref{eq:statewise-alpha-final-conversion} directly gives
$\operatorname{Gap}(\pi)\le\alpha$.} Applying the corresponding sensitivity argument to the normalized function $\Phi/R_\Phi$ and then scaling the resulting bound back by $R_\Phi$ gives
\begin{equation*}
L_{\bar A_\Phi}:=R_\Phi\left(\frac{4}{1-\gamma}+\frac{4\gamma\delta_P}{(1-\gamma)^2}\right)=\frac{4R_\Phi(1+L_d)}{1-\gamma}.
\end{equation*}

The crucial observation is that if Assumption~\ref{eq:state-wise-potential} holds, then players can aim to learn their marginalized potential advantages $\bar{A}^{\pi}_{\Phi,i}$, which share substantially more information through the common state-wise potential $\Phi$, rather than their individual marginalized advantage functions $\bar{A}^{\pi}_{i}$. In particular, \eqref{eq:statewise-potential-PDL} is exact. Hence, if the policy updates are driven by estimators of $\bar A_{\Phi,i}^\pi$, potential improvement can be analyzed without invoking the $\alpha$-potential property. The latter is needed only for a best-response deviation, for which Definition~\ref{def:alpha-markov-potential-game-uniform} gives
\begin{equation}\label{eq:statewise-alpha-final-conversion}
V_i^\pi(\mu)-V_i^{(\pi_i',\pi_{-i})}(\mu)
\le
\Psi(\pi)-\Psi(\pi_i',\pi_{-i})+\alpha.
\end{equation}
As a result, the analysis in the previous two sections can be repeated almost verbatim after replacing $\bar A_i^\pi$ with $\bar A_{\Phi,i}^\pi$ in the algorithms and analysis, and, in particular, in the oracle estimation assumptions. The main difference is that the $\alpha$-potential property is now invoked in the NE-gap lemma rather than in the potential-improvement lemma.

\subsection{Episodic learning with a potential-advantage oracle}

We first give the state-wise-potential counterpart of Section~\ref{sec:Episodic}.
The episodic filtration, stopping times, random episode lengths, clipping, and
KL projection are exactly as in that section. The only change is the target of
the estimation oracle. To make the resulting oracle explicit, at the $r$th
visit to state $s$ in episode $t$, player $i$ receives a scalar raw sample
$g_i^{\tau_{t,r}(s)}(s,A^{\tau_{t,r}(s)})$ and forms
\begin{align}\label{eq:hat-A-Phi-episodic}
\widehat A_{\Phi,i}^{\tau_{t,r}(s)}(s,a_i)
:&=
g_i^{\tau_{t,r}(s)}\!\left(s,A^{\tau_{t,r}(s)}\right)
\bigg(
\frac{\mathbf 1\{A_i^{\tau_{t,r}(s)}=a_i\}}
{\pi_i^t(a_i\mid s)}-1
\bigg),\qquad a_i\in\mathcal A_i.
\end{align}
As before, this estimator is exactly action-centered under
$\pi_i^t(\cdot\mid s)$.

\begin{assumption}\label{ass:episodic-potential-oracle}
For every episode $t$, state $s$, and $r=1,\ldots,L_t$, the estimator in
\eqref{eq:hat-A-Phi-episodic} satisfies
\begin{equation*}
\left\|
\mathbb E_{\tau_{t,r}(s)}\!\left[
\widehat A_{\Phi,i}^{\tau_{t,r}(s)}(s,\cdot)
\right]
-
\bar A_{\Phi,i}^{\pi^t}(s,\cdot)
\right\|_\infty
\le L_{\widehat A},
\end{equation*}
for some deterministic $L_{\widehat A}$. Moreover, as in Assumption~\ref{ass:episodic-oracle}, we assume $|g_i^{\tau}(s,A)|\le1$ a.s.\footnote{This is a normalization of the oracle output only and does not impose any normalization on $\Phi$.}
\end{assumption}

Next, we state the following lemma, which provides analogues of our marginalized-advantage lemmas for this special case of marginalized potential advantages. All omitted proofs in this section are given in Appendix~\ref{app:state-wise-appendix}.  

\begin{lemma}\label{lem:statewise-potential-scale-bounds}
Let $M_\Phi:=R_\Phi(1+L_d)$, where $R_\Phi$ denotes the potential range defined in \eqref{eq:statewise-Phi-range}. Then
\begin{equation}\label{eq:statewise-potential-advantage-bound}
\|\bar A_{\Phi,i}^{\pi}(s,\cdot)\|_\infty\leq \operatorname{span}(\bar A_{\Phi,i}^{\pi}(s,\cdot))\le M_\Phi,
\qquad
\forall i, \pi, s.
\end{equation}
Moreover,
\begin{equation}\label{eq:statewise-potential-advantage-sensitivity}
\|\bar A_{\Phi,i}^{\pi}(s,\cdot)-\bar A_{\Phi,i}^{\pi'}(s,\cdot)\|_\infty
\le L_{\bar A_\Phi}\|\pi-\pi'\|_{\mathrm{TV},\infty}.
\end{equation}
where $L_{\bar A_\Phi}:=\frac{4R_\Phi(1+L_d)}{1-\gamma}=\frac{4M_\Phi}{1-\gamma}$. Finally, for any $\pi_i,\pi_i',\sigma_{-i},\tau_{-i}$,
\begin{align}\nonumber
\big|\Psi(\pi_i',\sigma_{-i})-\Psi(\pi_i,\sigma_{-i})-\Psi(\pi_i',\tau_{-i})+\Psi(\pi_i,\tau_{-i})\big|\le L_{\Psi}\|\pi_i'-\pi_i\|_{\mathrm{TV},\infty}\|\sigma_{-i}-\tau_{-i}\|_{\mathrm{TV},\infty},
\end{align}
where $L_{\Psi}
:=R_\Phi\big(\frac{4}{1-\gamma}+\frac{12\gamma\delta_P}{(1-\gamma)^2}+\frac{8\gamma^2\delta_P^2}{(1-\gamma)^3}\big)
=\frac{R_\Phi}{1-\gamma}(4+12L_d+8L_d^2)$.
\end{lemma}

\subsubsection*{State-wise-potential episodic algorithm}

All purely statistical statements in Section~\ref{sec:Episodic}, including the random-count variance bound and the KL-projected local geometry, remain unchanged, since their proofs use only the oracle boundedness, importance weighting, action-centering, and the geometry of the KL update. The only substantive changes in the regret analysis occur in the potential-improvement and NE-gap arguments, as will be discussed later. The resulting algorithm is precisely the adaptation of Algorithm~\ref{alg:episodic-kl-projected-npg} to the marginalized potential advantage function and is summarized in Algorithm~\ref{alg:statewise-potential-episodic-kl-projected-npg}.

\begin{algorithm}[t]
\caption{State-wise-Potential Episodic KL-Projected NPG}
\label{alg:statewise-potential-episodic-kl-projected-npg}
\begin{algorithmic}[1]
\Require Initial policies
$\pi_i^0(\cdot\mid s)\in\Delta_{i,\zeta}$, step size $\eta>0$,
truncation parameter $\zeta\in(0,1)$, clipping parameter $B=4$, and
episode thresholds $L_t$ as in
Algorithm~\ref{alg:episodic-kl-projected-npg}.

\State Initialize the visit counters and the running potential-advantage
estimates as in Algorithm~\ref{alg:episodic-kl-projected-npg}.

\For{$t=0,1,2,\ldots$}
    \State Keep the policy profile $\pi^t$ fixed throughout episode $t$.

    \State Collect samples until every state has been visited $L_t$ times, exactly as in
    Algorithm~\ref{alg:episodic-kl-projected-npg}.

    \State At each visit, each player $i$ constructs the potential-advantage
    estimator according to \eqref{eq:hat-A-Phi-episodic}.

    \State Update the potential-advantage estimates as in Algorithm~\ref{alg:episodic-kl-projected-npg}, with $\widehat A_i^t$ replaced by $\widehat A_{\Phi,i}^t$.

    \State Apply the KL-projected NPG update using the resulting estimate
    of $\bar A_{\Phi,i}^{\pi^t}$, obtaining
    $\pi_i^{t+1}.$
\EndFor

\State \textbf{Output:} The policy sequence $\{\pi^t\}_{t\ge0}$.
\end{algorithmic}
\end{algorithm}

For clarity, throughout this section we attach a subscript $\Phi$ only to quantities whose definitions change because the estimation target is now the potential advantage. Thus, $\mathcal V_{\Phi,i}^t$, $\mathcal V_{\Phi,i}^k(s)$, $\mathcal E_{\Phi,i}^k(s)$, and $\mathfrak E_{\Phi,T}$ denote the corresponding potential-oracle variance and tracking quantities. In contrast, $\mathcal D_i^t(s)$, $\mathcal D_i^k(s)$, and $\mathfrak D_T$ retain their original notation because they are the same KL policy-movement quantities as in Sections~\ref{sec:Episodic} and~\ref{sec:fully-online-general-occupancy}. The same convention is used for occupancy, visit-count, stopping-time, and coverage quantities, whose definitions are unchanged. 

For completeness, define the episodic averaged and clipped estimators by
\begin{align*}
\widetilde r_{\Phi,i}^t(s,a_i)
&:=\frac1{L_t}\sum_{r=1}^{L_t}\widehat A_{\Phi,i}^{\tau_{t,r}(s)}(s,a_i),\cr
r_{\Phi,i}^t(s,a_i)
&:=\operatorname{clip}_{B}\!\left(\widetilde r_{\Phi,i}^t(s,a_i)\right),
\end{align*}
and let
\begin{equation*}
\xi_{\Phi,i}^t(s,a_i)
:=r_{\Phi,i}^t(s,a_i)-\frac1{L_t}\sum_{r=1}^{L_t}
\mathbb E_{\tau_{t,r}(s)}\!\left[\widehat A_{\Phi,i}^{\tau_{t,r}(s)}(s,a_i)\right].
\end{equation*}
The corresponding occupancy-weighted conditional second moment is
\begin{equation}\label{eq:statewise-episodic-variance}
\mathcal V_{\Phi,i}^t
:=\mathbb E_t\!\left[
\sum_{s\in\mathcal S}\sum_{a_i\in\mathcal A_i}
d_\mu^{\pi^t}(s)\pi_i^t(a_i\mid s)(\xi_{\Phi,i}^t(s,a_i))^2
\right].
\end{equation}
Since $0\le g_i^\tau\le1$, the proof of Lemma~\ref{lem:random-count-weighted-variance} applies without any rescaling and gives
\begin{equation}\label{eq:statewise-episodic-variance-bound}
\mathcal V_{\Phi,i}^t\le\frac{|\mathcal A_i|}{L_t}.
\end{equation}

Next, we derive analogues of the corresponding lemmas in the state-wise-potential setting. Note that, unlike Lemma~\ref{lem:kl-projected-potential-improvement}, the bound in Lemma~\ref{lem:statewise-episodic-potential-improvement} contains no term involving $\alpha$. Instead, the $\alpha$ term now enters through the NE-gap bound in Lemma~\ref{lem:statewise-episodic-ne-gap}, which is the main source of improvement over the earlier bounds.

\begin{lemma}\label{lem:statewise-episodic-potential-improvement}
Let Assumptions~\ref{ass:alpha-potential-ergodicity},
\ref{eq:state-wise-potential}, and \ref{ass:episodic-potential-oracle} hold,
and choose $B=4$. Then
\begin{align}\label{eq:statewise-episodic-potential-improvement}
\mathbb E_t\left[
\Psi(\pi^t)-\Psi(\pi^{t+1})\right]
&\ge
\frac{1}{4\eta}
\sum_{i=1}^n
\mathbb E_t\bigg[
\sum_{s\in\mathcal S}
d_\mu^{\pi^t}(s)\mathcal D_i^t(s)\bigg]
-
\frac{\eta e^{\frac{2\eta B}{1-\gamma}}}{2(1-\gamma)^2}
\sum_{i=1}^n\mathcal V_{\Phi,i}^t
\notag\\
&\quad-
\frac{n\eta L_{\widehat A}^2}{(1-\gamma)^2}
-
\frac{L_{\Psi}n^2\eta^2B^2}{8(1-\gamma)^2}
-
\frac{n\eta^2B^2L_dM_\Phi}{(1-\gamma)^3}.
\end{align}
\end{lemma}

\begin{lemma}\label{lem:statewise-episodic-ne-gap}
Let $\mathcal V_{\Phi,i}^t$ be defined as in \eqref{eq:statewise-episodic-variance}, and let $\mathcal D_i^t(s)$ be defined as in \eqref{eq:KL_symm}. For the state-wise-potential episodic KL-projected NPG,
\begin{align}\label{eq:statewise-episodic-NE-gap}
\operatorname{Gap}(\pi^t)
&\le
\alpha+
\bigg[
\frac{M_\Phi}{2(1-\gamma)}
+
\frac{1}{\eta}\sqrt{\frac{2A_{\max}}{\zeta}}
\bigg]
\left\{
\sum_{i=1}^n
\mathbb E_t\bigg[
\sum_s d_\mu^{\pi^t}(s)\mathcal D_i^t(s)
\bigg]
\right\}^{1/2}
\notag\\
&\qquad+
\frac{2L_{\widehat A}}{1-\gamma}
+
\frac{1}{1-\gamma}\sqrt{\frac{2A_{\max}}{\zeta}}
\sqrt{\max_{i\in[n]}\mathcal V_{\Phi,i}^t}
+
\frac{M_\Phi(\zeta+2L_d)}{1-\gamma}.
\end{align}
\end{lemma}

By combining these two lemmas and carefully tuning the stepsize and mixing parameters, we obtain:



\begin{theorem}\label{thm:statewise-episodic-ne-regret}
Let Assumptions~\ref{ass:alpha-potential-ergodicity},
\ref{eq:state-wise-potential}, and \ref{ass:episodic-potential-oracle} hold.
Fix $\rho\in(0,1)$ and $T\ge\log(4/\rho)$, and choose
\begin{equation}\label{eq:statewise-episodic-eta}
\eta=\frac{1-\gamma}{8\sqrt T},\qquad \zeta
=
\min\bigg\{\frac12,\,
\frac{2n^{1/3}A_{\max}^{1/3}L_{\widehat A}^{2/3}}
{[R_{\Phi}(1+L_d)]^{2/3}}
\bigg\}.
\end{equation}
If players follow Algorithm \ref{alg:statewise-potential-episodic-kl-projected-npg} with $L_t\ge A_{\max}/\eta$ and $B=4$, then, with probability at least
$1-\rho$,
\begin{align}\label{eq:statewise-episodic-final}
\frac1T\sum_{t=0}^{T-1}\operatorname{Gap}(\pi^t)
\le{}&
\mathfrak C_{\Phi}
\left(\frac{\log(4/\rho)}{T}\right)^{1/4}\!\!\!\!+
\alpha
+
\frac{\big[nA_{\max}R_{\Phi}(1+L_d)\big]^{1/3}}{1-\gamma}
L_{\widehat A}^{2/3}
+
\frac{cR_{\Phi}(1+L_d) L_d}{1-\gamma}
,
\end{align}
where $c$ is a universal constant and $\mathfrak C_{\Phi}$ is upper-bounded by 
\[
\mathfrak C_{\Phi}
\lesssim
\frac{1+\sqrt{nR_\Phi}(1+L_d)}
{1-\gamma}\left(
\sqrt n\,R_\Phi(1+L_d)(1+L_{\widehat A})+
\left[nA_{\max}R_\Phi(1+L_d)L_{\widehat A}^{-1}\right]^{1/3}\right).
\]
\end{theorem}

\subsection{Fully online one-sample learning with a potential-advantage oracle}
\label{sec:statewise-online}

In this subsection, we give the counterpart of
Section~\ref{sec:fully-online-general-occupancy} under the state-wise potential Assumption~\ref{eq:state-wise-potential}, with an improved NE regret bound. All global/local indexing, stopping times, coverage events, and state-wise policy storage remain unchanged. At global time $t=\tau_k(s)$, player $i$ receives a scalar raw sample
$g_i^t(s,A^t)$ and forms
\begin{equation}\label{eq:statewise-online-estimator}
\widehat A_{\Phi,i}^k(s,a_i)
:=
g_i^t(s,A^t)
\left(
\frac{\mathbf 1{A_i^t=a_i}}
{\pi_i^k(a_i\mid s)}-1
\right),\qquad a_i\in\mathcal A_i.
\end{equation}
In particular,
$\sum_{a_i}\pi_i^k(a_i\mid s)\widehat A_{\Phi,i}^k(s,a_i)=0$ almost surely.

\begin{assumption}\label{ass:online-potential-oracle}
At every global time $t=\tau_k(s)$, almost surely we have $| g_i^t(s,A^t)|\le1$ and 
\begin{equation*}
\left\|
\mathbb E_{\tau_k(s)}
[\widehat A_{\Phi,i}^k(s,\cdot)]
-
\bar A_{\Phi,i}^{\pi^t}(s,\cdot)
\right\|_\infty
\le L_{\widehat A}.
\end{equation*}

\end{assumption}

\begin{algorithm}[t]
\caption{State-wise-Potential Fully Online KL-Projected NPG}
\label{alg:statewise-potential-fully-online}
\begin{algorithmic}[1]
\Require Initial policies $\{\pi_i^0\}_{i=1}^n$, step size $\eta>0$,
truncation parameter $\zeta\in(0,1)$, and all other parameters as in
Algorithm~\ref{alg:fully-online-general-occupancy}, with the clipping level
$M$ replaced by $M_\Phi$.

\State Initialize the visit counters and the potential-advantage trackers
$r_{\Phi,i}^0(s,\cdot)$ as in
Algorithm~\ref{alg:fully-online-general-occupancy}.

\For{$t=0,1,2,\ldots$}
    \State Observe the current state $S^t=s$ and let
    $k=N_t(s)$ be its current visit index.

    \State Each player $i$ independently samples
    $A_i^t\sim\pi_i^k(\cdot\mid s)$.

    \State Each player $i$ constructs the potential-advantage estimator
    $\widehat A_{\Phi,i}^k(s,\cdot)$ according to
    \eqref{eq:statewise-online-estimator}.

    \State Update and clip the tracker $r_{\Phi,i}^k(s,\cdot)$ exactly as in
    Algorithm~\ref{alg:fully-online-general-occupancy}, with
    $\widehat A_i^k$ replaced by $\widehat A_{\Phi,i}^k$.

    \State Apply the same KL-projected NPG update at state $s$ using
    $r_{\Phi,i}^k(s,\cdot)$, obtaining $\pi_i^{k+1}(\cdot\mid s)$.

    \State Leave the policies at all unvisited states unchanged.
\EndFor

\State \textbf{Output:} The policy sequence $\{\pi^t\}_{t\ge0}$.
\end{algorithmic}
\end{algorithm}
 
Using these notations, the fully online state-wise-potential algorithm is summarized in Algorithm~\ref{alg:statewise-potential-fully-online}, which closely resembles Algorithm~\ref{alg:fully-online-general-occupancy}. For completeness, and to provide a formal performance guarantee for this algorithm, at $t=\tau_k(s)$ define
\begin{align*}
&\xi_{\Phi,i}^k(s,a_i):=\widehat A_{\Phi,i}^k(s,a_i)-\mathbb E_{\tau_k(s)}[\widehat A_{\Phi,i}^k(s,a_i)],\\
&\mathcal V_{\Phi,i}^k(s):=d_\mu^k(s)\mathbb E_{\tau_k(s)}\bigg[\sum_{a_i}\pi_i^k(a_i\mid s)(\xi_{\Phi,i}^k(s,a_i))^2\bigg].
\end{align*}
The same calculation as in Lemma~\ref{lemm:conditional-variance-online}, using $|g_i^t|\le 1$, gives $\mathcal V_{\Phi,i}^k(s)\le A_{\max}$ and almost surely, $\sum_{a_i}\pi_i^k(a_i\mid s)(\xi_{\Phi,i}^k(s,a_i))^2\le 10A_{\max}/\zeta$. Define
\begin{align*}
&e_{\Phi,i}^k(s,a_i):=r_{\Phi,i}^k(s,a_i)-\bar A_{\Phi,i}^{\pi^{\tau_k(s)}}(s,a_i), \qquad 
\mathcal E_{\Phi,i}^k(s):=\sum_{a_i}\pi_i^k(a_i\mid s)(e_{\Phi,i}^k(s,a_i))^2,\cr 
&\mathfrak E_{\Phi,T}:=\sum_{t=0}^{T-1}d_\mu^{\pi^t}(S^t)\sum_{i=1}^n\mathcal E_{\Phi,i}^{N_t(S^t)}(S^t), \qquad\quad\ \mathfrak D_T:=\sum_{t=0}^{T-1}d_\mu^{\pi^t}(S^t)\sum_{i=1}^n\mathcal D_i^{N_t(S^t)}(S^t).
\end{align*}
The following lemma provides counterparts of the online bounds in Section~\ref{sec:fully-online-general-occupancy}, adapted to the state-wise-potential setting.

\begin{lemma}\label{lem:statewise-online-bounds}
Fix $T\ge1$ and $\rho\in(0,1)$, and let $\Lambda_T:=\log(\frac{16nA_{\max}(|\mathcal S|+T)}{\rho})$. Suppose $\omega,\zeta\in(0,1/2]$,
$\eta M_\Phi/(1-\gamma)\le c_0\omega$, $\eta n(L_{\bar A_\Phi}+L_dM_\Phi)\le c_0(1-\gamma)$, and $\eta
\le \frac{p_{\min}(1-\gamma)}
{2\delta_PnM_\Phi H_{\rm cov}^2}$ for a sufficiently small constant $c_0>0$. Define $H_T:=H_{\rm cov}\lceil
\frac{2}{p_{\min}}
\log(\frac{2(|\mathcal S|+T)}{\rho})\rceil$. Suppose Assumptions \ref{ass:alpha-potential-ergodicity}, \ref{ass:coverage}, \ref{eq:state-wise-potential}, and
\ref{ass:online-potential-oracle} hold.
\begin{enumerate}[(i)]
\item \textbf{Tracking.}
Since $\mathcal V_{\Phi,i}^k(s)\le A_{\max}$ for all $i,s,k$, on
$\mathcal T_T\cap\mathcal H_T(H_T)$ for some
$\mathbb P(\mathcal T_T)\ge1-\rho/2$,
\begin{align}\label{eq:statewise-online-E-closed}
&\!\!\!\!\!\!\!\!\frac{\mathfrak E_{\Phi,T}}{T}
\le c\Bigg[
\frac{nM_\Phi^2}{\omega T}
+n\omega A_{\max}+nL_{\widehat A}^2
+nM_\Phi\sqrt{\frac{A_{\max}\Lambda_T}{T}}
+\frac{nM_\Phi\sqrt{A_{\max}}\,\Lambda_T}{T\sqrt\zeta}
+n\omega A_{\max}\sqrt{\frac{\Lambda_T}{T\zeta}}
\notag\\
&\!\!\!\!\!\!\!+
\frac{n\omega A_{\max}\Lambda_T}{T\zeta}
+\frac{n^3L_{\bar A_\Phi}^2H_T\eta^2M_\Phi^2}
{(1-\gamma)^2\omega^2}
+\bigg(
\frac{n^{3/2}M_\Phi^2L_d}{\omega}
+\frac{n^{7/2}L_{\bar A_\Phi}^2H_T^2\eta^2M_\Phi^2L_d}
{(1-\gamma)^2\omega^2}
\bigg)
\sqrt{\frac{\mathfrak D_T}{T}}
\Bigg].
\end{align}

\item \textbf{Potential improvement.}
\begin{align}\label{eq:statewise-online-potential-final}
\Psi(\pi^t)-\Psi(\pi^{t+1})
\ge{}&
\frac{d_\mu^{\pi^t}(S^t)}{4\eta}
\sum_{i=1}^n\mathcal D_i^{N_t(S^t)}(S^t)
-
\frac{c\eta\,d_\mu^{\pi^t}(S^t)}{(1-\gamma)^2}
\sum_{i=1}^n\mathcal E_{\Phi,i}^{N_t(S^t)}(S^t),
\end{align}
\begin{equation}\label{eq:statewise-online-D-from-potential}
\mathfrak D_T
\le
c\left[
\frac{\eta R_\Phi}{1-\gamma}
+
\frac{\eta^2}{(1-\gamma)^2}\mathfrak E_{\Phi,T}
\right].
\end{equation}

\item \textbf{NE-gap certificate.}
\begin{align*}
\operatorname{Gap}(\pi^t)
\le{}&
\alpha+\frac{cM_\Phi(\zeta+L_d)}{1-\gamma}
+\frac{c}{\eta}\sqrt{\frac{A_{\max}}{\zeta}}
\bigg(\sum_s d_\mu^{\pi^t}(s)
\sum_i\mathcal D_i^{N_t(s)-1}(s)\bigg)^{1/2}
\\
&+\frac{c}{1-\gamma}\sqrt{\frac{A_{\max}}{\zeta}}
\bigg(\sum_s d_\mu^{\pi^t}(s)
\sum_i\mathcal E_{\Phi,i}^{N_t(s)-1}(s)\bigg)^{1/2}
\\
&+\frac{c}{1-\gamma}
\bigg(\sum_s d_\mu^{\pi^t}(s)\sum_i
\big\|\bar A_{\Phi,i}^{\pi^t}(s,\cdot)
-\bar A_{\Phi,i}^{N_t(s)-1}(s,\cdot)\big\|_\infty^2
\bigg)^{1/2}.
\end{align*}

\item \textbf{Coupling transfer.}
With probability at least $1-\rho/2$,
$\mathcal H_T(H_T)$ holds and
\begin{align*}
&\sum_{t=H_T+1}^{T-1}\sum_s d_\mu^{\pi^t}(s)
\sum_i\mathcal D_i^{N_t(s)-1}(s)
\le
H_T\mathfrak D_T
+2\frac{L_dn^2M_\Phi^3}{(1-\gamma)^3}H_TT\eta^3,
\\
&\sum_{t=H_T+1}^{T-1}\sum_s d_\mu^{\pi^t}(s)
\sum_i\mathcal E_{\Phi,i}^{N_t(s)-1}(s)
\le
H_T\mathfrak E_{\Phi,T}
+8\frac{L_dn^2M_\Phi^3}{1-\gamma}H_TT\eta,
\\ 
&\sum_{t=H_T+1}^{T-1}
\bigg\{
\sum_s d_\mu^{\pi^t}(s)
\sum_i
\left\|
\bar A_{\Phi,i}^{\pi^t}(s,\cdot)
-\bar A_{\Phi,i}^{N_t(s)-1}(s,\cdot)
\right\|_\infty^2
\bigg\}^{1/2}
\le
\frac{L_{\bar A_\Phi}n^{3/2}M_\Phi}{1-\gamma}H_TT\eta.
\end{align*}
\end{enumerate}
\end{lemma}

Finally, using this lemma and following steps similar to those in the proof of Theorem~\ref{thm:fully-online-general-ne-regret}, we obtain the following improved NE regret bound for state-wise Markov $\alpha$-potential games, in which the $\alpha$ term appears without any multiplicative blow-up factor. 
   
\begin{theorem}\label{thm:statewise-online-ne-regret}
Let Assumptions~\ref{ass:alpha-potential-ergodicity},
\ref{eq:state-wise-potential}, \ref{ass:coverage}, and
\ref{ass:online-potential-oracle} hold. Fix $T\ge1$
and $\rho\in(0,1)$. Define
\begin{equation*}
\begin{gathered}
\Lambda_T:=\log\Big(\frac{16nA_{\max}(|\mathcal S|+T)}{\rho}\Big),
\qquad
H_T:=H_{\rm cov}\left\lceil
\frac{2}{p_{\min}}\log\!\frac{2(|\mathcal S|+T)}{\rho}
\right\rceil,
\\
u_T:=\frac{\Lambda_T}{(1-\gamma)T},
\qquad
\eta_{\rm cov}^{\Phi}:=
\frac{p_{\min}(1-\gamma)}
{2\delta_PnR_\Phi(1+L_d)H_{\rm cov}^2},
\end{gathered}
\end{equation*}
where $L_d=\frac{\gamma\delta_P}{1-\gamma}$. Let the players follow Algorithm \ref{alg:statewise-potential-fully-online} with
\begin{equation}\label{eq:statewise-online-final-tuning}
\begin{aligned}
\eta
&=
\min\bigg\{
\eta_{\rm cov}^{\Phi},\,
c_{\eta}
\frac{(1-\gamma)^2H_T^{-1/2}u_T^{3/5}}
{nR_\Phi(1+L_d)}
\bigg\},
\qquad
\omega
=
\min\bigg\{
\frac14,\,
c_{\omega}H_T^{1/4}u_T^{2/5}
\bigg\},
\\
\zeta
&=
\min\bigg\{
\frac14,\,
c_{\zeta}
\max\bigg\{
H_T^{1/2}u_T^{2/15},\,
\frac{n^{1/3}A_{\max}^{1/3}L_{\widehat A}^{2/3}}
{R_\Phi^{2/3}(1+L_d)^{2/3}}
\bigg\}
\bigg\},
\end{aligned}
\end{equation}
where $c_{\eta},c_{\omega},c_{\zeta}>0$ are sufficiently small constants. Then, with probability at least $1-\rho$,
\begin{align}\label{eq:statewise-online-final-regret}
\frac1T\sum_{t=0}^{T-1}\operatorname{Gap}(\pi^t)
\le{}&
\mathfrak C_{\Phi,1}
\left(\frac{\Lambda_T}{(1-\gamma)T}\right)^{2/15}\!\!\!+
\alpha
+
\mathfrak C_{\Phi,2}L_{\widehat A}^{2/3}
+
\frac{c R_\Phi(1+L_d) L_d}{1-\gamma},
\end{align}
where one may take\footnote{Here, $\widetilde O(\cdot)$ hides only universal numerical constants and
logarithmic factors.}
\begin{align}\nonumber
\mathfrak C_{\Phi,1}
\!=\!
\widetilde O\bigg(\!
\sqrt{\frac{H_{\rm cov}}{p_{\min}}}
\frac{
\sqrt n\,A_{\max}
\!+\!nR_\Phi(1\!+\!L_d)\sqrt{A_{\max}}
}{1-\gamma}
\bigg), \ \ \mathfrak C_{\Phi,2}\!=\!
\widetilde O\bigg(\!
\sqrt{\frac{H_{\rm cov}}{p_{\min}}}
\frac{\big(nA_{\max}R_\Phi(1\!+\!L_d)\big)^{1/3}}{1-\gamma}
\bigg).
\end{align}
\end{theorem}


\section{Independent-Resource Markov Congestion Games (IMCGs)}
\label{sec:IMCG-model}
In this section, we consider an important subclass of Markov games, namely, Markov congestion games, which generalize the static congestion games that have been studied extensively in the literature. This class of games is motivated by a variety of real-world applications, including dynamic resource allocation, dynamic routing, and dynamic job scheduling under stochastic dynamics. While static congestion games are known to be potential games, Markov congestion games generally do not admit an exact Markov potential function. Consequently, designing scalable learning algorithms for this class has remained a major challenge.

The work of \cite[Proposition~4.2]{guo2025markov} studies this class and shows that Markov congestion games belong to the class of Markov $\alpha$-potential games under a Lipschitz continuity assumption on the transition kernels. However, the resulting $\alpha$ can scale with the size of the state space $|\mathcal S|$ and other game parameters. Since any bounded game can, in principle, be viewed as a Markov $\alpha$-potential game for a sufficiently large $\alpha$, obtaining an $\alpha$ that grows substantially with the problem dimensions can limit the usefulness of such a characterization for deriving nontrivial equilibrium guarantees. On the other hand, \cite{cui2022learning} considers Markov congestion games with independent resource chains and develops a centralized algorithm that achieves sublinear NE regret with high probability. Their approach, however, is centralized and computationally demanding, as each iteration requires solving a large static game induced by the $Q$-function; developing a fully decentralized algorithm is left as an open problem.

In this section, we address both issues by applying the results developed earlier in this paper. We show that Markov congestion games with independent resource chains can be characterized as Markov $\alpha$-potential games with an $\alpha$ that is small and, importantly, independent of the size of the state space. At the same time, these games admit a state-wise $\alpha$-potential structure. Thus, they benefit from both properties developed in this paper, allowing us to obtain scalable, fully online, and decentralized learning algorithms with sublinear NE regret guarantees that hold with arbitrarily high probability.
   
\vspace{-0.3cm}
\subsection{Problem Formulation for IMCGs}\label{sec:formulation-IMCG}

Markov congestion games with independent resource chains (IMCGs) \cite{cui2022learning} are a subclass of Markov games, described in Section~\ref{sec:model}, where actions select subsets of resources, costs increase with resource congestion, and resource chains evolve independently.

More specifically, an IMCG is described by a finite set of players $[n]=\{1,\ldots,n\}$ and a finite set of resources $\mathcal E=\{1,\ldots,m\}$. Time evolves in discrete stages $t=0,1,2,\ldots$. At each time $t$, we assume that each resource $e\in\mathcal E$ possesses a local state $s_{e}^t\in\mathcal{S}_e$, where $\mathcal{S}_e$ is a finite state space associated with resource $e$, and the global state of the system at time $t$ is denoted by $s^t=(s^t_{1},\ldots,s^t_{m})
\in
\mathcal{S}
:=
\prod_{e=1}^{m}\mathcal{S}_e.$ As before, we write $s=(s_1,\ldots,s_m)$ for a generic global state, and denote the random variables by capital letters. At each time $t$, each player $i\in[n]$ selects an action $a_i^t\in\mathcal{A}_i$, consisting of a subset of resources, where the only restriction on the action set $\mathcal{A}_i$ is that each of its elements has cardinality at most $q$.\footnote{Note that we always have $q\leq m$. In particular, if $q=1$, each player can select at most one resource.} Let $a^t=(a^t_{1},\ldots,a^t_{n})\in \mathcal{A}_1\times\cdots\times \mathcal{A}_n$ denote the joint action profile at time $t$. Then, the congestion level (or load) on resource $e$ at time $t$ is denoted by
\[
n_e(a^t)
=
\sum_{i=1}^n \boldsymbol{1}_{\{e\in a^t_i\}},
\]
and the vector of loads is given by $n(a^t)
=
\bigl(n_1(a^t),\ldots,n_m(a^t)\bigr).$

\begin{remark}\label{rem:weighted-cngestion}
Our results naturally extend to the case where different players have different weights $w_{i,e}$ for different resources, in which case one can work with the weighted congestion measure $n_e(a^t)=\sum_{i=1}^n w_{i,e}\boldsymbol{1}_{\{e\in a_i^t\}}$. However, for simplicity of presentation and to avoid introducing additional notation, we focus on the homogeneous-weight setting.
\end{remark}

We assume resources are independent  in the sense that conditional on the current state and the action profile, each resource evolves according to its own Markov chain independently. More precisely, for every resource $e\in\mathcal E$, let $P_e(\cdot|x,k)$ denote a resource transition kernel on $\mathcal{S}_e$, where $x\in\mathcal{S}_e,$ and $k\in\{0,\ldots,n\}$
represents the current congestion level on resource $e$. Conditioned on the current state and action profile, the next-state distribution factorizes as
\begin{equation}
P(S^{t+1}=s'|S^t=s,A^t=a)
=
\prod_{e=1}^{m}
P_e
\bigl(
s'_e
\mid
s_e,
n_e(a)
\bigr).
\nonumber
\end{equation}

For every resource $e$, let $c_e:\mathcal S_e\times\{0,\ldots,n\}
\rightarrow
\mathbb [0, 1]$ denote the local congestion cost. We assume that the one-stage cost of player $i$ is additive over the resources selected
by player $i$, that is
\begin{align}\nonumber
c_i(s,a)
=
\sum_{e\in a_i}
c_e(s_e,n_e(a)),
\end{align}
where without loss of generality, we assume that players' costs are normalized so that $c_i\in [0, 1]$. Similar as before, for a stationary policy profile $\pi$ and initial state distribution $\mu\in\Delta(\mathcal{S})$, the infinite-horizon discounted value function for player $i$ is given by
\begin{align}\nonumber
V_i^\pi(\mu)
:=\mathbb E_\pi\bigg[
\sum_{t=0}^{\infty}\gamma^t c_i(S^t,A^t)
\,\bigg|\,
S^0\sim \mu
\bigg],
\end{align}
Finally, an IMCG is defined by the tuple $\mathcal G
=
\big([n], \mathcal S, \{P_e\}_{e\in \mathcal{E}}, \{\mathcal A_i\}_{i=1}^n, \{c_e\}_{e\in \mathcal{E}}, \gamma, \mu \big).$

\begin{definition}[Local load  sensitivity]\label{ass:local-load-sen}
The load transition sensitivity of the IMCG is defined as

$$
\delta:=
\max_{\substack{e\in \mathcal{E},s_e\in \mathcal{S}_e\\
k\in \{0,\ldots,n-1\}}}
\left\|P_e(\cdot\mid s_e,k+1)-P_e(\cdot\mid s_e,k)\right\|_{\TV}.
$$

\end{definition}

The parameter $\delta$ measures the maximum change in the transition law of a resource caused by increasing its load by one player. Since total variation distance is at most one, we always have $\delta\le1$. In many large-scale congestion systems, however, the effect of a single additional player on the resource dynamics is expected to be small, so that $\delta\ll1$. This is particularly natural when the system contains many players and the transition dynamics depend smoothly on the aggregate congestion level, in which case the marginal effect of any individual player becomes small.

\begin{definition}
\label{def:alpha-markov-potential-game-local}
A Markov game is called a local $\alpha$-potential game if there exists a
potential function $\Psi$ on the space of stationary policy profiles such that, for
every player $i\in[n]$, every pair of stationary policy profiles
$\pi=(\pi_i,\pi_{-i})$ and $\pi'=(\pi_i',\pi_{-i})$ that differ only in
player $i$'s policy, and every initial distribution
$\mu\in\Delta(\mathcal S)$.
\begin{align}\nonumber
\left| \bigl(V_i^{\pi'}(\mu)-V_i^\pi(\mu)\bigr) - \bigl(\Psi^{\pi'}(\mu)-\Psi^{\pi}(\mu)\bigr) \right| \le \alpha \sum_{s\in\mathcal S} d_\mu^{\pi'}(s) \left\| \pi_i'(\cdot\mid s)-\pi_i(\cdot\mid s) \right\|_{\mathrm{TV}}. \end{align}
\end{definition}

Definition~\ref{def:alpha-markov-potential-game-local} is stronger and more local than the Markov $\alpha$-potential game notion in Definition~\ref{def:alpha-markov-potential-game-uniform}, which requires the discrepancy between a player's value change and the
corresponding potential change to be uniformly bounded by $\alpha$ over
unilateral policy deviations. In contrast, the local Definition \ref{def:alpha-markov-potential-game-local} scales this discrepancy by the occupancy-weighted policy distance
$\sum_{s\in\mathcal S}d_\mu^{\pi'}(s)
\|\pi_i'(\cdot\mid s)-\pi_i(\cdot\mid s)\|_{\mathrm{TV}}$.
Since this quantity is at most one, the local condition implies the condition of
Definition \ref{def:alpha-markov-potential-game-uniform}, whereas the converse need not hold. More
importantly, the approximation error vanishes with the size of the
unilateral policy deviation, making the condition local in policy space
and allowing the analysis to exploit the geometry of the value and
potential functions along the policy updates. We will show that IMCGs benefit from this stronger local $\alpha$-potential property.

\subsection{Transition Kernel Sensitivity for IMCGs}

Consider the IMCG, where we recall that the local state of resource $e$ is $s_e\in\mathcal S_e$, and the global state is $s=(s_1,\ldots,s_m)\in\mathcal S$. Then we have the following transition kernel sensitivity lemma.

\begin{lemma}\label{lemm:transition-sensitivity}
For any fixed policies $\pi_{-i}$ and any two actions $a_i, a_i'\in \mathcal{A}_i$ of player $i$ in an IMCG,
\[
\max_{s\in \mathcal{S}}
\left\|
\bar{P}^{\pi_{-i}}_i(\cdot\mid s, a_i)
-
\bar{P}^{\pi_{-i}}_i(\cdot\mid s, a'_i)
\right\|_{\TV}
\le
|a_i\triangle a_i'|\delta\leq 2q \delta,
\]
where $\bar{P}^{\pi_{-i}}_i(\cdot\mid s,a_i)$ is the marginalized transition kernel of player $i$ and $|a_i\triangle a_i'|$ denotes the cardinality of the symmetric difference between the sets $a_i$ and $a'_i$. In particular, $\delta_P\leq 2q \delta$. 
\end{lemma}


\begin{proof}
Fix a state $s$ and a realization of the other players' actions
$a_{-i}$. For each resource $e$, let 
\begin{align*}
&\mu_e
:=
P_e
\left(
\cdot\mid s_e,
n_{e}(a_{-i})+\boldsymbol{1}_{\{e\in a_i\}}
\right),\cr 
&\nu_e
:=
P_e
\left(
\cdot\mid s_e,
n_{e}(a_{-i})+\boldsymbol{1}_{\{e\in a_i'\}}
\right),
\end{align*}
where by abuse of notation $n_{e}(a_{-i})\!=\!\sum_{j\ne i}\boldsymbol{1}_{\{e\in a_j\}}$. Then, by independence of the resource chains,
\[
P(\cdot\mid s,a_i,a_{-i})=\prod_{e=1}^m \mu_e \qquad \mbox{and} \qquad
P(\cdot\mid s,a'_i,a_{-i})=\prod_{e=1}^m \nu_e.
\]
If $e\notin a_i\triangle a_i'$, then $\boldsymbol{1}_{\{e\in a_i\}}
=
\boldsymbol{1}_{\{e\in a_i'\}}$, and therefore $\mu_e=\nu_e$. Thus the only coordinates in which $\prod_{e=1}^m \mu_e$ and $\prod_{e=1}^m \nu_e$ differ are those in the symmetric difference $a_i\triangle a_i'$. By the standard telescoping bound for product measures, 
\[
\left\|
\prod_{e=1}^m \mu_e
-
\prod_{e=1}^m \nu_e
\right\|_{\TV}
\le
\sum_{e=1}^m
\|\mu_e-\nu_e\|_{\TV}=
\sum_{e\in a_i\triangle a_i'}
\|\mu_e-\nu_e\|_{\TV}.
\]
where the second equality holds because $\mu_e=\nu_e$ whenever $e\notin a_i\triangle a_i'$. 

For each $e\in a_i\triangle a_i'$, the local load on resource $e$ differs by exactly one, so
by the definition of the local load  sensitivity (Definition \ref{ass:local-load-sen}), we have $\|\mu_e-\nu_e\|_{\TV}\le \delta$.
Hence, for fixed $s$ and $a_{-i}$,
\begin{align}\nonumber
\left\|P(\cdot\mid s,a_i,a_{-i})-P(\cdot\mid s,a'_i,a_{-i})\right\|_{\TV}=\left\|
\prod_{e=1}^m \mu_e -\prod_{e=1}^m \nu_e
\right\|_{\TV}
\le
|a_i\triangle a_i'|\delta.
\end{align}
This in view of Definition \ref{def:transition-sensitivity} and the fact that $|a_i\triangle a_i'|\leq 2q$ shows that $\delta_P\leq 2q \delta.$  

Finally, $\bar{P}^{\pi_{-i}}_i(\cdot\mid s, a_i)$ and $\bar{P}^{\pi_{-i}}_i(\cdot\mid s, a'_i)$ are obtained by averaging the primitive transition kernels $P(\cdot\mid s,a_i,a_{-i})$ and $P(\cdot\mid s,a'_i,a_{-i})$ over $a_{-i}\sim\pi_{-i}(\cdot\mid s)$. As total variation is convex, averaging
cannot increase the distance. Therefore, for every $s\in \mathcal{S}$,
\[
\left\|
\bar{P}^{\pi_{-i}}_i(\cdot\mid s, a_i)
-
\bar{P}^{\pi_{-i}}_i(\cdot\mid s, a'_i)
\right\|_{\TV}
\le
|a_i\triangle a_i'|\delta.
\]
Taking the maximum over $s\in \mathcal{S}$ completes the proof.
\end{proof}

\subsection{Local \(\alpha\)-Potential with State-wise Potential Function for IMCGs}

Static congestion games possess a particularly strong structure: every unilateral change in a player's cost is captured exactly by the corresponding change in the Rosenthal potential \cite{rosenthal1973class,milchtaich1996congestion}. More precisely, for each \emph{fixed} state $s$, define the Rosenthal potential
\begin{align}\label{eq:static-rosenthial-potential}
\Phi(s,a)
=\sum_{e\in \mathcal{E}}
\sum_{k=1}^{n_e(a)}
c_e(s_e,k).
\end{align}
Then, for every player $i$, any fixed state $s\in \mathcal{S}$, and every unilateral action deviation from $a_i$ to $a_i'$,
\[
\Phi(s,a_i',a_{-i})-\Phi(s,a_i,a_{-i})
=
c_i(s,a_i',a_{-i})-c_i(s,a_i,a_{-i}).
\]
Thus, at every fixed state, the congestion game admits an exact potential representation. This observation suggests a natural candidate for extending the static Rosenthal potential to the dynamic Markov setting, which has also been investigated in prior works \cite{fox2022independent,guo2025markov}: use $\Phi(s,a)$ as the stage cost and define its infinite-horizon discounted value as the Markov potential function. If the exact potential identity holds state by state, one might hope that accumulating these identities along the stochastic trajectory would preserve, at least approximately, the potential structure at the policy level. The main difficulty is that a unilateral policy deviation changes not only the actions selected at each state, but also the future state distribution. Consequently, the statewise exact-potential identity alone does not imply an exact Markov potential identity, and a direct comparison between unilateral value changes and changes in the discounted Rosenthal potential can lead to loose or pessimistic $\alpha$-potential bounds \cite{guo2025markov}.

The key observation in our setting is that the independent local resource dynamics make this distributional mismatch controllable. Exploiting this structure, we show that the discounted Rosenthal potential satisfies a much stronger proximity relation for IMCGs: the approximation error scales locally both in terms of the magnitude of the unilateral policy deviation and the load transition sensitivity. As a result, IMCGs enjoy the best of both worlds: they satisfy the local Markov $\alpha$-potential property of Definition~\ref{def:alpha-markov-potential-game-local}, while the same discounted Rosenthal construction provides a state-wise Markov potential function in the sense of Assumption~\ref{eq:state-wise-potential}.

\begin{lemma}\label{thm:alpha-potential}
Consider an IMCG and define the state-wise Markov potential function
\[\Psi^{\pi}(\mu):=\mathbb E_\pi
\!\bigg[
\sum_{t=0}^{\infty}
\gamma^t
\Phi(S^t,A^t)
\;\bigg|\;
S^0\sim \mu\bigg],
\]
where $\Phi(s,a)=
\sum_{e\in \mathcal{E}}
\sum_{k=1}^{n_e(a)}
c_e(s_e,k)$ denotes the static Rosential potential function at state $s$. Let $\pi=(\pi_i,\pi_{-i})$ and
$\pi'=(\pi_i',\pi_{-i})$ be two policy profiles that differ only in
player $i$'s policy. Then, for every player $i$ and any initial distribution $\mu\in\Delta(\mathcal S)$, for the choice of $\alpha:=
\frac{4nq\gamma\delta}{(1-\gamma)^2}$, 
\begin{align}\nonumber
\left|
\bigl(V_i^{\pi'}(\mu)-V_i^\pi(\mu)\bigr)
-
\bigl(\Psi^{\pi'}(\mu)-\Psi^{\pi}(\mu)\bigr)
\right|
&\le
\alpha
\sum_{s\in\mathcal S}
d_\mu^{\pi'}(s)
\left\|
\pi_i'(\cdot\mid s)-\pi_i(\cdot\mid s)
\right\|_{\mathrm{TV}}.
\end{align}
\end{lemma}

\begin{proof}
For the sake of completeness, we first verify the static potential identity by following \cite{rosenthal1973class}, with a straightforward extension to set-valued actions. Fix a state
\(s\in\mathcal S\), fix the actions \(a_{-i}\) of all players other than
player \(i\), and let \(a_i,a_i'\in\mathcal{A}_i\) be two possible actions of player \(i\).  Define
$a=(a_i,a_{-i})$ and $a'=(a_i',a_{-i})$. For each resource \(e\), 
\[
n_e(a)=n_{e}(a_{-i})+\mathbf 1\{e\in a_i\},
\qquad
n_e(a')=n_{e}(a_{-i})+\mathbf 1\{e\in a_i'\},
\]
where $n_{e}(a_{-i}):=
\sum_{j\neq i}\mathbf 1\{e\in a_j\}$. Then, the change in player \(i\)'s one-stage cost is
\begin{align*}
c_i(s,a')-c_i(s,a)
&=
\sum_{e\in a_i'} c_e(s_e,n_e(a'))
-
\sum_{e\in a_i} c_e(s_e,n_e(a))  \\
&=
\sum_{e=1}^m
\Big[
\mathbf 1\{e\in a_i'\}
c_e\big(s_e,n_{e}(a_{-i})+1\big)
-
\mathbf 1\{e\in a_i\}
c_e\big(s_e,n_{e}(a_{-i})+1\big)
\Big]  \\
&=
\sum_{e=1}^m
\Big(
\mathbf 1\{e\in a_i'\}
-
\mathbf 1\{e\in a_i\}
\Big)
c_e\big(s_e,n_{e}(a_{-i})+1\big).
\end{align*}

On the other hand, the change in the static Rosenthal potential is
\begin{align*}
\Phi(s,a')
-
\Phi(s,a)
&=
\sum_{e=1}^m
\bigg[
\sum_{k=1}^{n_e(a')} c_e(s_e,k)
-
\sum_{k=1}^{n_e(a)} c_e(s_e,k)
\bigg].
\end{align*}
As the congestion on a resource \(e\) changes by at most one depending on whether \(e\in a_i\) or \(e\in a_i'\),
\[
\sum_{k=1}^{n_e(a')} c_e(s_e,k)
-
\sum_{k=1}^{n_e(a)} c_e(s_e,k)
=
\Big(
\mathbf 1\{e\in a_i'\}
-
\mathbf 1\{e\in a_i\}
\Big)
c_e\big(s_e,n_{e}(a_{-i})+1\big).
\]
Hence, for every fixed \(s\), \(a_{-i}\), and any \(a_i,a_i'\in \mathcal{A}_i\), we have
\begin{align}\label{eq:stastic-potential-identity}
c_i(s,a_i',a_{-i})-c_i(s,a_i,a_{-i})
=\Phi(s,a_i',a_{-i})
-\Phi(s,a_i,a_{-i}).
\end{align}
Let us define
\[
\bar{c}_i^\pi(s)
:=
\E_{A\sim\pi(\cdot\mid s)}[c_i(s,A)],
\qquad
\bar{\Phi}^\pi(s)
:=
\E_{A\sim\pi(\cdot\mid s)}
[\Phi(s,A)] .
\]
Now, conditional on \(s\), take expectations in \eqref{eq:stastic-potential-identity} with respect to
\(A_{-i}\sim \pi_{-i}(\cdot\mid s)\),
\(A_i\sim \pi_i(\cdot\mid s)\), and
\(A_i'\sim \pi_i'(\cdot\mid s)\), independently. Using
linearity of expectation, we obtain
\[
\bar{c}_i^{\pi'}(s)-\bar{c}_i^\pi(s)
=
\bar{\Phi}^{\pi'}(s)-\bar{\Phi}^\pi(s)\ \ \forall s\in\mathcal S. 
\]
Moreover, using the definition of the value function and the state-wise Markov potential function,
\[
V_i^\pi(s)
=
\frac{1}{1-\gamma}\sum_{x\in\mathcal S}d_s^\pi(x)\bar{c}_i^\pi(x),
\qquad
\Psi^\pi(s)
=
\frac{1}{1-\gamma}\sum_{x\in\mathcal S}d_s^\pi(x)\bar{\Phi}^\pi(x).
\]
Therefore, if we define $h(x):=\bar{c}_i^\pi(x)-\bar{\Phi}^\pi(x)
=\bar{c}_i^{\pi'}(x)-\bar{\Phi}^{\pi'}(x)$, then
\begin{align}\label{eq:V-h-Psi}
\big(V_i^{\pi'}(s)-V_i^\pi(s)\big)
&-
\big(\Psi^{\pi'}(s)-\Psi^\pi(s)\big)\cr
&=
\frac{1}{1-\gamma}
\sum_{x\in\mathcal S}
\Big[
d_s^{\pi'}(x)
\big(\bar{c}_i^{\pi'}(x)-\bar{\Phi}^{\pi'}(x)\big)
-
d_s^\pi(x)
\big(\bar{c}_i^{\pi}(x)-\bar{\Phi}^{\pi}(x)\big)
\Big]\cr
&=
\frac{1}{1-\gamma}
\sum_{x\in\mathcal S}
\big(d_s^{\pi'}(x)-d_s^\pi(x)\big)h(x).
\end{align}
Moreover, we can uniformly bound the static potential $\Phi(s,a)$ for any $s$ and $a$ as follows:
\begin{align}\nonumber
\Phi(s,a)\!=\!
\sum_{e \in \mathcal{E}}
\sum_{k=1}^{n_e(a)}
c_e(s_e,k)
\!\le\!
\sum_{e\in \mathcal{E}} n_e(a)c_e(s_e,n_e(a))\!=\!\sum_{i=1}^n\sum_{e\in a_i}c_e(s_e,n_e(a))\!=\!
\sum_{i=1}^n c_i(s,a)
\le n,
\end{align}
where the first inequality holds by monotonicity of the resource costs, the second equality follows by double counting the player--resource incidences, using
$n_e(a)=\sum_{i=1}^n\mathbf{1}_{\{e\in a_i\}}$, and the last equality holds because $c_i\in [0, 1]$ (note that using the notation \eqref{eq:statewise-Phi-range}, this means $R_{\Phi}\leq n$). Hence \(\bar{\Phi}^\pi(s)\in [0, n]\), and since $\bar{c}^{\pi}_i\in [0,1]$, we have $\|h\|_\infty \le n-1$. Thus, using \eqref{eq:V-h-Psi}, we obtain
\begin{align*}
\left|
\big(V_i^{\pi'}(s)-V_i^\pi(s)\big)
-
\big(\Psi^{\pi'}(s)-\Psi^\pi(s)\big)
\right|
&\le
\frac{n-1}{1-\gamma}
\sum_{x\in\mathcal S}
\left|d_s^{\pi'}(x)-d_s^\pi(x)\right|\cr 
&\leq \frac{4nq\gamma\delta}{(1-\gamma)^2}
\sum_{x\in\mathcal S}
d_s^{\pi'}(x)
\left\|
\pi_i'(\cdot\mid x)
-
\pi_i(\cdot\mid x)
\right\|_{\mathrm{TV}},
\end{align*}
where the second inequality uses part (i) of Lemma \ref{lemm:occupancy-sensitivity} with $L_d=\frac{\gamma\delta_P}{1-\gamma}$ and because $\delta_P\leq 2q\delta$ by Lemma \ref{lemm:transition-sensitivity}. Taking expectation from the above inequality over $s\sim \mu$ and using the convexity of the absolute value completes the proof.
\end{proof}

\subsection{Episodic and One-Sample Estimation Oracles for IMCGs}

So far, we have shown that IMCGs are (local) Markov $\alpha$-potential games with $\alpha=\frac{4nq\gamma\delta}{(1-\gamma)^2}$ and that they also admit a state-wise potential function (see Assumption~\ref{eq:state-wise-potential}), as established in Lemma~\ref{thm:alpha-potential}. In practical settings, this choice of $\alpha$ is expected to be very small, since one would typically expect $\delta\ll 1$. Crucially, $\alpha$ does not scale with either the size of the state space $|\mathcal{S}|$ or the sizes of the action spaces $|\mathcal{A}_i|$, which can be exponentially large. Therefore, IMCGs fall naturally within the class of state-wise Markov $\alpha$-potential games studied in Section~\ref{sec:statewise-potential-sharp}. However, to apply the algorithms developed in Section~\ref{sec:statewise-potential-sharp}, we still need to construct episodic or one-sample estimation oracles for the marginalized potential advantage that satisfy Assumption~\ref{ass:episodic-potential-oracle} or Assumption~\ref{ass:online-potential-oracle}, respectively. In this subsection, we show that such estimators can be constructed directly by leveraging the Rosenthal function \eqref{eq:static-rosenthial-potential}. To this end, the following lemma provides a one-stage approximation to the marginalized potential advantage, whose proof is deferred to Appendix~\ref{app:estimate-oracle-IMCG}. 

\begin{lemma}\label{thm:one-stage}
Let $\Phi$ be the static Rosenthal potential function defined in
\eqref{eq:static-rosenthial-potential}. Given a policy profile
$\pi=(\pi_i,\pi_{-i})$, let $\bar{Q}_{\Phi,i}^{\pi}(s,a_i)$ and
$\bar{A}_{\Phi,i}^{\pi}(s,a_i)$ denote, respectively, the marginalized
potential $Q$-function and the marginalized potential advantage function
that were defined in \eqref{eq:marginalized-potentia-advantage-Q}. Let
\begin{align}\label{eq:marginalized-both-costs}
&\bar{c}_i^{\pi_{-i}}(s,a_i)
:=
\mathbb E_{A_{-i}\sim \pi_{-i}(\cdot\mid s)}
\left[
c_i(s,a_i,A_{-i})
\right]
\qquad \mbox{and} \qquad
\bar{c}_i^{\pi}(s)
:=
\mathbb E_{A\sim\pi(\cdot\mid s)}
\left[c_i(s,A)
\right].
\end{align}
Then, for every player $i$, state $s\in \mathcal{S}$, and action
$a_i\in \mathcal{A}_i$, we have
\begin{align}\label{eq:surrogate-mariginalized-potential}
\left|
\bar{A}_{\Phi,i}^\pi(s,a_i)
-
\big(
\bar{c}_i^{\pi_{-i}}(s,a_i)
- \bar{c}_i^{\pi}(s)
\big)
\right|
\le
\frac{4nq\gamma\delta}{1-\gamma}.
\end{align}
\end{lemma}

Comparing \eqref{eq:surrogate-mariginalized-potential} with
Assumption~\ref{ass:episodic-potential-oracle} or
Assumption~\ref{ass:online-potential-oracle} suggests that if the
estimation oracle invoked by player $i$ can estimate the surrogate quantity
$\bar c_i^{\pi_{-i}}(s,a_i)-\bar c_i^\pi(s)$ in an episodic or
one-sample fashion, then the resulting estimator can be used to satisfy
Assumption~\ref{ass:episodic-potential-oracle} or
Assumption~\ref{ass:online-potential-oracle}, respectively, with
$L_{\widehat{A}}:=\frac{4nq\gamma\delta}{1-\gamma}$. The following lemma
shows that this is indeed the case; its proof is given in
Appendix~\ref{app:one-sample-episodic-IMCG-sample}.

\begin{lemma}\label{lem:imcg-potential-oracles}
For IMCGs, the choice of raw sample $g_i^\tau(s,A^\tau):=c_i(s,A^\tau)$ yields
\begin{enumerate}[(i)]
\item[(i)] an episodic estimation oracle satisfying
Assumption~\ref{ass:episodic-potential-oracle} with $L_{\widehat A}=\frac{4nq\gamma\delta}{1-\gamma}$; and
\item[(ii)] a fully online one-sample estimation oracle satisfying
Assumption~\ref{ass:online-potential-oracle} with $L_{\widehat A}=\frac{4nq\gamma\delta}{1-\gamma}$.
\end{enumerate}
\end{lemma}

\subsection{Decentralized Online Learning for IMCGs with NE-Gap Guarantees}

Finally, by combining Lemmas~\ref{thm:alpha-potential} and
\ref{lem:imcg-potential-oracles} with Theorems~\ref{thm:statewise-episodic-ne-regret}
and \ref{thm:statewise-online-ne-regret}, we obtain fully decentralized
online learning algorithms for IMCGs with finite-time NE regret
performance guarantees. For brevity, we only present the fully online
version, which adapts Algorithm~\ref{alg:statewise-potential-fully-online}
to the IMCG setting and is summarized in
Algorithm~\ref{alg:fully-online-general-occupancy-IMCG}. An analogous
result can be established in the episodic setting by applying
Theorem~\ref{thm:statewise-episodic-ne-regret} and adapting
Algorithm~\ref{alg:statewise-potential-episodic-kl-projected-npg} to
IMCGs. The main result of this section is given in the following theorem,
whose proof follows immediately from the preceding results.

\begin{algorithm}[t]
\caption{Fully Online Algorithm for Player $i$ in IMCGs}
\label{alg:fully-online-general-occupancy-IMCG}
\begin{algorithmic}[1]
\State Initialize $\pi_i^0(\cdot\mid s)$ uniformly and
$r_i^{-1}(s)=0$ for every state $s$, and set
$N_0(s)=0$ for every $s$. Also, choose
$M_{\Phi}=n(1+\frac{2q\gamma\delta}{1-\gamma})$ and
$\omega \in (0, \frac{1}{2})$. 
\For{$t=0,1,2,\ldots$}
    \State Observe $S^t=s$ and set $k=N_t(s)$ to be the number of times state $s$ was visited before $t$.
    \State Player $i$ independently draws an action
    $A_i^t\sim \pi_i^k(\cdot\mid s)=\pi_i^t(\cdot\mid s)$.
    \State Player $i$ observes its realized cost $c_i(s,A^t)$ and forms the one-sample estimator
    \[
    \widehat A_i^k(s,a_i)
    =
    c_i(s,A^t)
    \left(
    \frac{\mathbf 1\{A_i^t=a_i\}}
    {\pi_i^k(a_i\mid s)}-1
    \right),
    \qquad a_i\in\mathcal A_i.
    \]
    \State Player $i$ aggregates and clips the new estimator:
    \begin{align*}
    \widetilde r_i^k(s)
    &=
    (1-\omega)r_i^{k-1}(s)
    +
    \omega\widehat A_i^k(s),\\
    r_i^k(s)
    &=
    \operatorname{clip}_{M_{\Phi}}
    \bigl(\widetilde r_i^k(s)\bigr).
    \end{align*}
    \State Player $i$ updates its policy only at the observed state $S^t=s$ as
    \begin{align}\nonumber
    \pi_i^{k+1}(\cdot\mid s)
    =
    \argmin_{p_i\in\Delta_{i,\zeta}}
    \left\{
    \frac{\eta}{1-\gamma}
    \langle p_i,r_i^k(s)\rangle
    +
    D_{\mathrm{KL}}
    \bigl(p_i\|\pi_i^k(\cdot\mid s)\bigr)
    \right\}.
    \end{align}
    \State Set $\pi_i^{t+1}(\cdot\mid s')=\pi_i^t(\cdot\mid s')$ for every $s'\neq s$.
    \State Set $N_{t+1}(s)=N_t(s)+1$ and leave all other counters unchanged.
    \State The state transits according to
    $S^{t+1}\sim P(\cdot\mid S^t,A^t)$.
\EndFor
\end{algorithmic}
\end{algorithm}

\begin{theorem}\label{thm:imcg-online-ne-regret}
Consider an IMCG and let Assumption~\ref{ass:coverage} hold. Fix $T\ge1$ and
$\rho\in(0,1)$, and suppose that each player follows
Algorithm~\ref{alg:fully-online-general-occupancy-IMCG} with the parameter
choices specified in Theorem~\ref{thm:statewise-online-ne-regret}, using
the corresponding IMCG parameter bounds. Then, with probability at least $1-\rho$,
\begin{align}\label{eq:imcg-online-final-regret}
\frac1T\sum_{t=0}^{T-1}\operatorname{Gap}(\pi^t)
\le{}&
\mathfrak C_{{\rm I},1}
\Big(\frac{\Lambda_T}{(1-\gamma)T}\Big)^{2/15}
+\mathfrak C_{{\rm I},2}
\Big(\frac{nq\gamma\delta}{1-\gamma}\Big)^{2/3}
+c\frac{nq\gamma\delta}{(1-\gamma)^2}
\Big(1+\frac{q\gamma\delta}{1-\gamma}\Big),
\end{align}
where $\Lambda_T=\log\big(\frac{16nA_{\max}(|\mathcal S|+T)}{\rho}\big)$, $c$ is a universal constant, and one may take
\begin{align}\nonumber
\mathfrak C_{{\rm I},1}
&=\widetilde O\!\bigg(\sqrt{\frac{H_{\rm cov}}{p_{\min}}}
\frac{n^2A_{\max}}{1-\gamma}
\Big(1+\frac{2q\gamma\delta}{1-\gamma}\Big)\bigg), \quad
\mathfrak C_{{\rm I},2}
=\widetilde O\!\bigg(\sqrt{\frac{H_{\rm cov}}{p_{\min}}}
\frac{n^{2/3}A_{\max}^{1/3}}{1-\gamma}
\Big(1+\frac{2q\gamma\delta}{1-\gamma}\Big)^{1/3}\bigg).
\end{align}
In particular, as $\delta\to0$, with all other problem parameters fixed,
with probability at least $1-\rho$,
\[
\frac1T\sum_{t=0}^{T-1}\operatorname{Gap}(\pi^t)
\le
\mathfrak C_{{\rm I},1}
\Big(\frac{\Lambda_T}{(1-\gamma)T}\Big)^{2/15}
+o_\delta(1).
\]
\end{theorem}

\begin{proof}
By Lemma~\ref{lemm:transition-sensitivity}, every IMCG satisfies
$\delta_P\le 2q\delta$, and hence
$L_d=\frac{\gamma\delta_P}{1-\gamma}
\le\frac{2q\gamma\delta}{1-\gamma}$.
Moreover, the proof of Lemma~\ref{thm:alpha-potential} shows that the
discounted Rosenthal potential satisfies
Assumption~\ref{eq:state-wise-potential} with $R_\Phi\le n$, and that
IMCGs are local Markov $\alpha$-potential games with
$\alpha=\frac{4nq\gamma\delta}{(1-\gamma)^2}$.
Finally, by part~(ii) of Lemma~\ref{lem:imcg-potential-oracles}, the
realized cost sample used in
Algorithm~\ref{alg:fully-online-general-occupancy-IMCG} provides a
one-sample estimation oracle satisfying
Assumption~\ref{ass:online-potential-oracle} with
$L_{\widehat A}=\frac{4nq\gamma\delta}{1-\gamma}$.
Thus, under Assumption~\ref{ass:coverage}, all the hypotheses of
Theorem~\ref{thm:statewise-online-ne-regret} hold. Substituting the
valid IMCG upper bounds
\[
R_\Phi\le n,\qquad
\delta_P\le2q\delta,\qquad
L_d\le\frac{2q\gamma\delta}{1-\gamma},\qquad
L_{\widehat A}=\frac{4nq\gamma\delta}{1-\gamma},\qquad
\alpha=\frac{4nq\gamma\delta}{(1-\gamma)^2}
\]
into \eqref{eq:statewise-online-final-tuning} and
\eqref{eq:statewise-online-final-regret}, and absorbing universal
numerical constants and dominated terms into $c$ and the
$\widetilde O(\cdot)$ factors, gives
\eqref{eq:imcg-online-final-regret}. The final claim follows immediately
by letting $\delta\to0$ with all other problem parameters fixed.
\end{proof}

\subsection{Strategic Online Job Scheduling on Independent Stochastic Machines}
\label{subsec:imcg-job-scheduling}

A particularly compelling application of IMCGs arises in online job scheduling on heterogeneous machines \cite{etesami2022optimal,fardno2026game}. Motivated by this setting, we formulate a strategic online job-scheduling problem in which $n$ self-interested users (or transmitters) receive jobs over time and share a collection of $m$ heterogeneous machines. At each time $t$, the users simultaneously decide where to dispatch their currently available jobs, while each machine processes its assigned jobs according to its current operating state and congestion level, with its processing dynamics evolving stochastically over time (see Figure~\ref{fig:imcg-job-scheduling} for an illustration). This formulation builds on our earlier work on distributed load balancing and dynamic job scheduling \cite{fardno2026game}, where self-interested users allocate jobs among heterogeneous servers to minimize their individual processing delays. In particular, the static load-balancing game in \cite{fardno2026game} is shown to be an exact potential game, while its dynamic extension captures the carryover of unfinished jobs across time. Here, we show that this strategic online formulation naturally fits the IMCG framework, thereby providing a concrete application of our earlier results and yielding a fully online decentralized learning scheme with finite-time NE-regret guarantees.

Formally, let $\mathcal E=[m]$ denote the set of heterogeneous machines, with each machine $e$ having a finite local state space $\mathcal S_e$. Its state $s_e^t\in\mathcal S_e$ represents its current operating state, which may encode, for example, its queue length, workload, or processing condition. The global system state is $s^t=(s_e^t)_{e\in\mathcal E}\in\mathcal S:=\prod_{e\in\mathcal E}\mathcal S_e$. Each user $i\in[n]$ has a finite set $\mathcal A_i$ of admissible dispatching patterns, where an action $a_i\in\mathcal A_i$ specifies the subset of machines to which user $i$ dispatches its currently available jobs. As in the IMCG model, we assume that each user can use at most $q$ machines at each time step, so that $|a_i|\le q$. Given a joint dispatching decision $a=(a_1,\ldots,a_n)$, define $n_e(a):=\sum_{i=1}^n \one\{e\in a_i\}$ as the number of users simultaneously dispatching jobs to machine $e$.

For each machine $e$, let $c_e(s_e,n_e)$ denote the delay cost associated with using that machine when its local state is $s_e$ and its congestion level is $n_e$. Accordingly, the one-stage delay incurred by user $i$ under state $s$ and joint dispatching decision $a$ is $c_i(s,a)=\sum_{e\in a_i}c_e\bigl(s_e,n_e(a)\bigr)$. Thus, the machines correspond to the resources of an IMCG, while the users' dispatching decisions determine their resource subsets and the resulting congestion levels. Finally, suppose that the stochastic processing dynamics of different machines are conditionally independent and that the evolution of each machine depends only on its current local state and congestion level. Then there exist local transition kernels $P_e$ such that $P(S^{t+1}=s'\mid S^t=s,A^t=a)=\prod_{e\in\mathcal E}P_e\bigl(s'_e\mid s_e,n_e(a)\bigr)$, which is precisely the independent-resource transition structure of an IMCG. Consequently, the strategic online job-scheduling problem formulated above is an IMCG, and the decentralized learning framework developed in this paper applies directly. For heterogeneous job sizes, the weighted-congestion extension discussed in Remark~\ref{rem:weighted-cngestion} can be used. The formulation here should therefore be viewed as a stochastic simultaneous-dispatch extension of the dynamic scheduling model in \cite{fardno2026game}.

\begin{figure}[t]
    \centering
    \includegraphics[width=.92\linewidth]{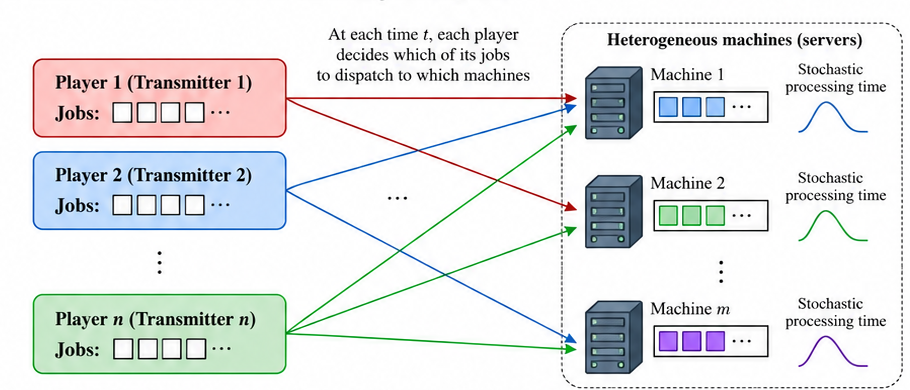}
    \caption{Strategic online job scheduling on stochastic machines. Each
    player receives jobs over time and chooses a subset of machines to which
    its jobs are dispatched. The machines are shared resources whose local
    processing conditions evolve stochastically and independently, conditional
    on their current states and congestion levels.}
    \label{fig:imcg-job-scheduling}
\end{figure}

Algorithm~\ref{alg:fully-online-general-occupancy-IMCG} has a particularly
simple interpretation in this application. At time $t$, every player observes
the current machine states/loads $S^t$, independently samples a dispatching pattern
from its current policy, and sends its jobs accordingly. After the dispatch,
player $i$ only needs to observe its own realized delay
$c_i(S^t,A^t)$; it does not need to know the other players' costs or policies.
Using this single realized delay, the player forms the importance-weighted
advantage estimate in
Algorithm~\ref{alg:fully-online-general-occupancy-IMCG}, recursively averages
and clips this estimate, and then performs a KL-projected NPG update only
at the state that was just observed. Intuitively, dispatching patterns that
repeatedly produce smaller delays at a given machine-state configuration gain
probability, while the KL regularization prevents the policy from changing too
abruptly and the truncation parameter $\zeta$ maintains sufficient exploration.
Thus, all users adapt simultaneously from their own experienced delays and the
observed machine states, without a centralized scheduler, without computing a
best response, and without observing the other users' actions.

Theorem~\ref{thm:imcg-online-ne-regret} then gives a finite-time equilibrium
guarantee for this decentralized scheduling mechanism. In particular, with
probability at least $1-\rho$, the average incentive of any user to
unilaterally change its dispatching policy satisfies
\[
\frac1T\sum_{t=0}^{T-1}\operatorname{Gap}(\pi^t)
\le
\mathfrak C_{{\rm I},1}
\Big(\frac{\Lambda_T}{(1-\gamma)T}\Big)^{2/15}
+\mathfrak C_{{\rm I},2}
\Big(\frac{nq\gamma\delta}{1-\gamma}\Big)^{2/3}
+c\frac{nq\gamma\delta}{(1-\gamma)^2}
\Big(1+\frac{q\gamma\delta}{1-\gamma}\Big).
\]
Here, $\delta$ has a transparent physical meaning: it measures how much the
future operating condition of a machine can change when one additional user
dispatches a job to it. When the effect of a single user is small
($\delta\ll1$), which occurs particularly when there are many users and
machines in the system, the last two terms are small, while the remaining term
decreases as $\widetilde O(T^{-2/15})$. Consequently, for any target accuracy
above the $\delta$-dependent approximation floor, a polynomial number of online
interactions suffices, with high probability, to make the time-averaged NE gap small. A similar phenomenon has also been proved for a sequential update
rule and observed in simulations for a simultaneous update rule in
\cite{fardno2026game}. Equivalently, over a long run, only a small fraction of
time steps can have a large unilateral improvement opportunity. In the scheduling
interpretation, the users therefore learn dispatching policies under which no
single user can substantially reduce its discounted delay by changing its
policy alone. In other words, strategic users can learn stable approximate
equilibrium behavior in polynomial time through fully decentralized,
simultaneous, one-sample updates, even when they can observe only their own
realized costs and the machine conditions evolve stochastically over time.

\section{Conclusions}\label{sec:conclusions}

This paper developed a finite-time framework for decentralized learning in structured Markov games under bandit feedback. Starting from general Markov $\alpha$-potential games, we introduced KL-projected NPG algorithms and established high-probability NE regret bounds in both episodic and fully online settings. In the fully online setting, players learn from a single continuing trajectory, receive only one realized cost sample per time step, and update asynchronously whenever the corresponding state is visited. The analysis therefore required controlling not only bandit estimation errors but also the interactions among policy drift, changing state occupancies, asynchronous state visits, and delayed local information. The stopping-time, coupling, charging, and dynamic-tracking arguments developed here provide a way to handle these effects without episodic resets or centralized coordination. We further showed that sharper guarantees are possible when the game possesses a state-wise potential structure. In this case, the approximation error $\alpha$ enters the resulting NE regret guarantees as a single additive term rather than being amplified through the online tracking analysis, enlarging the regime of approximate potentiality for which the bounds remain informative. 

Finally, we specialized the framework to IMCGs, where stochastic resource dynamics are coupled to strategic behavior through congestion. We showed that weak sensitivity of resource dynamics to an individual player's actions induces an approximate state-wise potential structure and derived a fully decentralized, one-sample online NE regret guarantee. In particular, the NE regret decays at rate $\widetilde O(T^{-2/15})$ without any distribution-mismatch coefficient, up to additive approximation terms controlled by the resource-coupling parameter that vanish as the dynamic coupling becomes weak. The strategic job-scheduling application illustrates the practical interpretation of this result: strategic users can learn approximately stable dispatching policies from realized costs while continuously interacting with stochastic machines, without requiring a centralized scheduler.

As future research directions, one could consider improving our bounds through sharper tracking or variance-reduction arguments while retaining one-sample bandit feedback and asynchronous updates. It would also be useful to extend the framework beyond discounted finite-state games, for example, to average-reward, partially observed \cite{introJordanKamgarpour2026}, or function-approximation settings. More broadly, our results suggest that weak dynamic coupling and approximate potential structure provide a promising route toward finite-time guarantees for independent learning in large stochastic multi-agent systems where centralized learning and strong oracle access are impractical.

\vspace{-0.3cm}
\paragraph{Acknowledgment.}
The core ideas and technical contributions of this work originated with the author and were developed in part in a research proposal submitted prior to this work. These include the coupling argument, charging and delay lemmas, dynamic tracking analysis, and estimation and sensitivity lemmas. The author acknowledges the use of generative AI (GPT-5.6 Sol) for brainstorming and writing assistance. Specifically, AI was used to refine proofs; fill in standard technical steps, including proofs of Lemmas~\ref{lemm:best}, \ref{lem:online-sensitivity-consequences}, and \ref{lemm:coupling}; carry out step-size and parameter tuning for the main regret bounds; and brainstorm sharper analytical techniques, including the tailored span bounds in Lemma~\ref{lem:statewise-potential-scale-bounds}. The author independently verified and revised all such material. Generative AI also assisted with exposition, literature review, grammatical editing, and generating tables and Figure~\ref{fig:imcg-job-scheduling}.

\bibliographystyle{IEEEtran}
\bibliography{thesisrefs}

\appendix
 
\section{Omitted Proofs for Section \ref{sec:preliminaries}}\label{app:3}

\subsection{Proof of Lemma \ref{lemm:best}}\label{app:belman-proof}
\begin{proof}
Suppose first that $\pi_i$ is a best response to $\pi_{-i}$.
Then $V_i^\pi = V_i^*$, $\bar{Q}_i^\pi = \bar{Q}_i^*$, and the Bellman optimality equation \eqref{eq:Bellman-opt} gives $V_i^\pi(s)
=
\min_{a_i\in \mathcal{A}_i} \bar{Q}_i^\pi(s,a_i)$. Hence
\[
\bar{A}_i^\pi(s,a_i)
=
\bar{Q}_i^\pi(s,a_i)-V_i^\pi(s)
=
\bar{Q}_i^\pi(s,a_i)-\min_{a'_i\in \mathcal{A}_i} \bar{Q}_i^\pi(s,a'_i)
\ge 0.
\]

Also, from \eqref{eq:Bellman-opt}, every action in the support of an optimal stationary policy
attains the minimum:
\[
\pi_i(a_i\mid s)>0
\quad\Longrightarrow\quad
\bar{Q}_i^\pi(s,a_i)
=
\min_{a'_i\in \mathcal{A}_i} \bar{Q}_i^\pi(s,a'_i) \quad\Longrightarrow\quad \bar{A}_i^\pi(s,a_i)=0.
\]
Conversely, assume \eqref{eq:adv_nonnegative} and \eqref{eq:adv_support} hold. The first condition \eqref{eq:adv_nonnegative} implies
\[
\bar{Q}_i^\pi(s,a_i)\ge V_i^\pi(s),
\qquad \forall a_i\in \mathcal{A}_i.
\]

Since $\pi_i(\cdot\mid s)$ is a probability distribution, there exists some action $a_i$ with $\pi_i(a_i\mid s)>0$. By the second condition \eqref{eq:adv_support}, $\bar{Q}_i^\pi(s,a'_i)=V_i^\pi(s)$. Therefore
\[
V_i^\pi(s)
=
\min_{a_i} \bar{Q}_i^\pi(s,a_i),
\qquad \forall s\in \mathcal{S}.
\]
Thus $\pi_i$ satisfies the Bellman optimality equation and hence is an
optimal stationary policy, i.e., a best response to $\pi_{-i}$.
\end{proof}

\subsection{Proof of Lemma \ref{lem:kl-projected-local-geometry}}\label{app:KL-geometry}
\begin{proof}
Fix any arbitrary but fixed state $s$. By the KKT optimality condition for the KL-projected NPG update rule \eqref{eq:kl-projected-npg-update}, there
exists $c:=\lambda(s)>0$ such that \eqref{eq:water-filling} holds. Since
\[
1
=
\sum_{a_i}\pi_i^{t+1}(a_i\mid s)
\ge
c\sum_{a_i}
\pi_i^t(a_i\mid s)\exp\Big(-\frac{\eta}{1-\gamma} r_i^t(s,a_i)\Big)
\ge
c e^{-\frac{B\eta }{1-\gamma}},
\]
we have $c\le e^{\frac{B\eta }{1-\gamma}}$. Hence, if the lower-bound constraint in \eqref{eq:water-filling} is inactive,
\[
\frac{\pi_i^{t+1}(a_i\mid s)}
{\pi_i^t(a_i\mid s)}
=
c \exp\Big(-\frac{\eta}{1-\gamma} r_i^t(s,a_i)\Big)
\le e^{\frac{2\eta B}{1-\gamma}},
\]
while if it is active, $\frac{\pi_i^{t+1}(a_i\mid s)}
{\pi_i^t(a_i\mid s)}
=
\frac{\zeta/|\mathcal A_i|}
{\pi_i^t(a_i\mid s)}
\le1.$ This proves part (i). 

To show (ii), consider the valid inequality $(x-1)^2\le\max\{1,x\}(x-1)\log x, \forall x>0$. Choosing $x
=
\frac{
\pi_i^{t+1}(a_i\mid s)
}{
\pi_i^t(a_i\mid s)
}$ and using
the first part, we get
\begin{align}\nonumber
\sum_{a_i}\!
\frac{
\left(
\pi_i^{t+1}(a_i| s)
-
\pi_i^t(a_i| s)
\right)^2
}{
\pi_i^t(a_i| s)
}\le
e^{\frac{2\eta B}{1-\gamma}}
\sum_{a_i}\!
\left(
\pi_i^{t+1}(a_i| s)
\!-\!
\pi_i^t(a_i| s)
\right)
\log
\frac{
\pi_i^{t+1}(a_i| s)
}{
\pi_i^t(a_i| s)
}=
e^{\frac{2\eta B}{1-\gamma}}\mathcal D_i^t(s).
\end{align}

Part (iii) follows directly from the first-order variational optimality
condition for the KL-projected NPG update
\eqref{eq:kl-projected-npg-update}. Indeed, optimality of
$\pi_i^{t+1}(\cdot\mid s)$ gives, for any $p\in\Delta_{i,\zeta}$,
\[
\left\langle
\frac{\eta}{1-\gamma}r_i^t(s)
+
\nabla_p D_{\mathrm{KL}}
\!\left(
p\,\middle\|\,\pi_i^t(\cdot\mid s)
\right)
\Big|_{p=\pi_i^{t+1}(\cdot\mid s)},
\,
p-\pi_i^{t+1}(\cdot\mid s)
\right\rangle
\ge 0.
\]
Choosing the feasible test point $p=\pi_i^t(\cdot\mid s)$ and using
$\nabla_p D_{\mathrm{KL}}(p\|q)=\log(p/q)+\mathbf 1$, together with the fact
that both $\pi_i^t(\cdot\mid s)$ and $\pi_i^{t+1}(\cdot\mid s)$ are
probability vectors, yields
\[
\frac{\eta}{1-\gamma}
\left\langle
r_i^t(s),
\Delta_i^t(s)
\right\rangle
+
\left\langle
\log\frac{\pi_i^{t+1}(\cdot\mid s)}
{\pi_i^t(\cdot\mid s)},
\Delta_i^t(s)
\right\rangle
\ge 0.
\]
The result now follows from $-\big\langle
\log\frac{\pi_i^{t+1}(\cdot\mid s)}
{\pi_i^t(\cdot\mid s)},
\Delta_i^t(s)
\big\rangle=\mathcal D_i^t(s)$.

We next bound the policy increment. From part (iii) and
$\|r_i^t\|_\infty\le B$,
\begin{align}
\frac{1-\gamma}{\eta}
\mathcal D_i^t(s)
&\le
\left\langle
r_i^t(s),
\Delta_i^t(s)
\right\rangle
\le
\left\|
r_i^t(s)
\right\|_\infty
\left\|
\Delta_i^t(s)
\right\|_1
\le
B
\left\|
\Delta_i^t(s)
\right\|_1.
\label{eq:projected-increment-upper}
\end{align}
On the other hand, Pinsker's inequality gives
\begin{align}
\mathcal D_i^t(s)
&=D_{\mathrm{KL}}
\left( \pi_i^t(\cdot\mid s)
\middle\|
\pi_i^{t+1}(\cdot\mid s)
\right)+
D_{\mathrm{KL}}
\left(
\pi_i^{t+1}(\cdot\mid s)
\middle\|
\pi_i^t(\cdot\mid s)
\right)
\ge
\left\|
\Delta_i^t(s)
\right\|_1^2.
\label{eq:projected-increment-pinsker}
\end{align}
Combining
\eqref{eq:projected-increment-upper} and
\eqref{eq:projected-increment-pinsker} proves (iv).
\end{proof}

\subsection{Proof of Lemma \ref{lemm:occupancy-sensitivity}}\label{app:sensitivity-distribution}

\begin{proof}
(i) For $\pi=(\pi_i,\pi_{-i})$,
$\bar P^\pi(s'\mid s)=\sum_{a_i\in\mathcal A_i}
\pi_i(a_i\mid s)\bar P_i^{\pi_{-i}}(s'\mid s,a_i)$. Hence
\[
\bar P^\pi(\cdot\mid s)-\bar P^{\pi'}(\cdot\mid s)
=
\sum_{a_i\in\mathcal A_i}
\bigl(\pi_i(a_i\mid s)-\pi_i'(a_i\mid s)\bigr)
\bar P_i^{\pi_{-i}}(\cdot\mid s,a_i).
\]
Let $\mathcal A_i^+:=\{a_i:\pi_i(a_i\mid s)\ge\pi_i'(a_i\mid s)\}$
and $\mathcal A_i^-:=\mathcal A_i\setminus\mathcal A_i^+$. Since both
policies sum to one,
\[
D:=
\sum_{a_i\in\mathcal A_i^+}(\pi_i(a_i\mid s)-\pi_i'(a_i\mid s))
=
\sum_{a_i'\in\mathcal A_i^-}(\pi_i(a_i\mid s)-\pi_i'(a_i\mid s))
=
\|\pi_i(\cdot\mid s)-\pi_i'(\cdot\mid s)\|_{\TV}.
\]
For $D>0$, set
$p_{a_i}:=(\pi_i(a_i\mid s)-\pi_i'(a_i\mid s))/D$ on $\mathcal A_i^+$ and
$q_{a_i'}:=(\pi'_i(a_i\mid s)-\pi_i(a_i\mid s))/D$ on $\mathcal A_i^-$.
Then $\sum_{a_i\in\mathcal A_i^+}p_{a_i}
=\sum_{a_i'\in\mathcal A_i^-}q_{a_i'}=1$, and hence
\begin{align*}
\|\bar P^\pi(\cdot\mid s)-\bar P^{\pi'}(\cdot\mid s)\|_{\TV}
&=
D\Big\|
\sum_{a_i\in\mathcal A_i^+}
p_{a_i}\bar P_i^{\pi_{-i}}(\cdot\mid s,a_i)
-
\sum_{a_i'\in\mathcal A_i^-}
q_{a_i'}\bar P_i^{\pi_{-i}}(\cdot\mid s,a_i')
\Big\|_{\TV}\\
&=
D\Big\|
\sum_{\substack{a_i\in\mathcal A_i^+\\a_i'\in\mathcal A_i^-}}
p_{a_i}q_{a_i'}
\bigl(
\bar P_i^{\pi_{-i}}(\cdot\mid s,a_i)
-
\bar P_i^{\pi_{-i}}(\cdot\mid s,a_i')
\bigr)
\Big\|_{\TV}\\
&\le
D
\sum_{\substack{a_i\in\mathcal A_i^+\\a_i'\in\mathcal A_i^-}}
p_{a_i}q_{a_i'}
\big\|
\bar P_i^{\pi_{-i}}(\cdot\mid s,a_i)
-
\bar P_i^{\pi_{-i}}(\cdot\mid s,a_i')
\big\|_{\TV}\\
&\le
D
\sup_{a_i,a_i'\in\mathcal A_i}
\big\|
\bar P_i^{\pi_{-i}}(\cdot\mid s,a_i)
-
\bar P_i^{\pi_{-i}}(\cdot\mid s,a_i')
\big\|_{\TV}\\
&\le
\delta_P
\|\pi_i(\cdot\mid s)-\pi_i'(\cdot\mid s)\|_{\TV}.
\end{align*}
Indeed, the second equality follows from
$\sum_{a_i}p_{a_i}=\sum_{a_i'}q_{a_i'}=1$. The first inequality then follows from the triangle
inequality, while the second uses
$\sum_{a_i,a_i'}p_{a_i}q_{a_i'}=1$.
Finally, the last inequality follows from Remark \ref{rem:marginalized-transition-sensitivity}
and $D=\|\pi_i(\cdot\mid s)-\pi_i'(\cdot\mid s)\|_{\TV}$.

To show the second inequality in (i), using the definition of the discounted occupancy measure, for any two policy profiles $\pi,\pi'$ (not necessarily differing only in the policy of player $i$), we obtain
\begin{align}\label{eq:contraction-tv-visitation}
\left\|
d_s^{\pi'}-d_s^\pi
\right\|_{\mathrm{TV}}
&=
\left\|
(1-\gamma)
\sum_{t=0}^{\infty}
\gamma^t
\left[
(\bar{P}^{\pi'})^t(\cdot\mid s)
-
(\bar{P}^\pi)^t(\cdot\mid s)
\right]
\right\|_{\mathrm{TV}}
\cr
&\le
(1-\gamma)
\sum_{t=1}^{\infty}
\gamma^t
\left\|
(\bar{P}^{\pi'})^t(\cdot\mid s)
-
(\bar{P}^\pi)^t(\cdot\mid s)
\right\|_{\mathrm{TV}}
\cr
&\le
(1-\gamma)
\sum_{t=1}^{\infty}
\gamma^t
\sum_{k=0}^{t-1}
\sum_{x\in\mathcal S}
(\bar{P}^{\pi'})^k(x\mid s)
\left\|
\bar{P}^{\pi'}(\cdot\mid x)
-
\bar{P}^\pi(\cdot\mid x)
\right\|_{\mathrm{TV}}
\cr
&=
(1-\gamma)
\sum_{k=0}^{\infty}
\left(
\sum_{t=k+1}^{\infty}\gamma^t
\right)
\sum_{x\in\mathcal S}
(\bar{P}^{\pi'})^k(x\mid s)
\left\|
\bar{P}^{\pi'}(\cdot\mid x)
-
\bar{P}^\pi(\cdot\mid x)
\right\|_{\mathrm{TV}}
\cr
&=
\gamma
\sum_{k=0}^{\infty}
\gamma^k
\sum_{x\in\mathcal S}
(\bar{P}^{\pi'})^k(x\mid s)
\left\|
\bar{P}^{\pi'}(\cdot\mid x)
-
\bar{P}^\pi(\cdot\mid x)
\right\|_{\mathrm{TV}}
\cr
&=
\frac{\gamma}{1-\gamma}
\sum_{x\in\mathcal S}
d_s^{\pi'}(x)
\left\|
\bar{P}^{\pi'}(\cdot\mid x)
-
\bar{P}^\pi(\cdot\mid x)
\right\|_{\mathrm{TV}}
\cr
&\le
\frac{\gamma\delta_P}{1-\gamma}
\sum_{x\in\mathcal S}
d_s^{\pi'}(x)
\left\|
\pi_i'(\cdot\mid x)
-
\pi_i(\cdot\mid x)
\right\|_{\mathrm{TV}},
\end{align}
where the first inequality uses the triangle inequality and the fact
that the $t=0$ term vanishes, the second inequality follows from
Lemma~\ref{lemm:coupling} with
$A=\bar{P}^{\pi'}$ and $B=\bar{P}^\pi$, and the last inequality follows from the
first part of the lemma. Finally, since
$d_\mu^\pi=\mathbb E_{s\sim\mu}[d_s^\pi]$ and
$d_\mu^{\pi'}=\mathbb E_{s\sim\mu}[d_s^{\pi'}]$, the convexity of
total variation gives $\big\|
d_\mu^{\pi'}-d_\mu^\pi
\big\|_{\mathrm{TV}}
\le
L_d
\sum_{x\in\mathcal S}
d_\mu^{\pi'}(x)
\left\|
\pi_i'(\cdot\mid x)
-
\pi_i(\cdot\mid x)
\right\|_{\mathrm{TV}}$.
\smallskip

(ii) For any two policy profiles $\pi$ and $\pi'$, introduce intermediate profiles obtained by replacing the players' policies one at a time. Applying the first
inequality of part (i) to each consecutive pair and using the triangle inequality yields the first inequality of part (ii), i.e.,
\begin{align}\label{eq:P-general-telescop}
\left\|
\bar{P}^{\pi'}(\cdot\mid x)-\bar{P}^\pi(\cdot\mid x)
\right\|_{\mathrm{TV}}
\le
\delta_P
\sum_{j=1}^n
\|\pi_j'(\cdot\mid x)-\pi_j(\cdot\mid x)\|_{\mathrm{TV}} \ \ \ \forall \pi, \pi', x\in \mathcal{S}.
\end{align}
Next, since the induced transition row $\bar{P}^\pi(\cdot\mid s)$ depends only on the policy profile $\pi(\cdot\mid s)$ at the state $s$, if $\pi$ and $\pi'$ differ only at state $x$, then $\bar{P}^\pi$ and $\bar{P}^{\pi'}$ coincide in every row except possibly the row corresponding to $x$. Using \eqref{eq:contraction-tv-visitation} that holds for any two stationary policy profiles, we obtain
\begin{align*}
\|d_\mu^{\pi'}-d_\mu^\pi\|_{\mathrm{TV}}
&\le
\frac{\gamma}{1-\gamma}
\sum_{s\in \mathcal{S}}d_\mu^\pi(s)\left\|
\bar{P}^{\pi'}(\cdot\mid s)-\bar{P}^\pi(\cdot \mid s)
\right\|_{\mathrm{TV}}\\
&=
\frac{\gamma}{1-\gamma}
d_\mu^\pi(x)
\left\|
\bar{P}^{\pi'}(\cdot\mid x)-\bar{P}^\pi(\cdot\mid x)
\right\|_{\mathrm{TV}}.
\end{align*}
Combining the preceding two displayed inequalities and using
$L_d=\gamma\delta_P/(1-\gamma)$ proves
\eqref{eq:online-localized-occupancy-sensitivity}.
\end{proof}

\begin{lemma}\label{lemm:coupling}
Let $A$ and $B$ be two transition matrices on a finite state space
$\mathcal S$. For every state $s$, let $A(\cdot\mid s)$ and $A^t(\cdot\mid s)$ denote the $s$th rows of the row-stochastic transition matrices $A$ and $A^k$, respectively (similarly for $B$ and $B^t$). Then, for every initial state $s\in\mathcal S$ and every
integer $t\ge1$,
\begin{align}
\left\|
A^t(\cdot\mid s)-B^t(\cdot\mid s)
\right\|_{\mathrm{TV}}
&\le
\sum_{k=0}^{t-1}
\sum_{x\in\mathcal S}
A^k(x\mid s)
\left\|
A(\cdot\mid x)-B(\cdot\mid x)
\right\|_{\mathrm{TV}}.
\nonumber
\end{align}
\end{lemma}

\begin{proof}
Consider the identity $A^t-B^t=\sum_{k=0}^{t-1}
A^k(A-B)B^{t-1-k}$. For every state $s$, we can write 
\begin{align}
\left\|
A^t(\cdot\mid s)-B^t(\cdot\mid s)
\right\|_{\mathrm{TV}}
&\le
\sum_{k=0}^{t-1}
\left\|
A^k(\cdot\mid s)(A-B)B^{t-1-k}
\right\|_{\mathrm{TV}}
\notag\\
&\le
\sum_{k=0}^{t-1}
\left\|
A^k(\cdot\mid s)(A-B)
\right\|_{\mathrm{TV}}
\notag\\
&=
\sum_{k=0}^{t-1}
\left\|
\sum_{x\in\mathcal S}
A^k(x\mid s)
\bigl(
A(\cdot\mid x)-B(\cdot\mid x)
\bigr)
\right\|_{\mathrm{TV}}
\notag\\
&\le
\sum_{k=0}^{t-1}
\sum_{x\in\mathcal S}
A^k(x\mid s)
\left\|
A(\cdot\mid x)-B(\cdot\mid x)
\right\|_{\mathrm{TV}},
\end{align}
where the second inequality uses the contraction of total variation
under a transition matrix,\footnote{For any row vector $v$ and row-stochastic matrix $B$, we have
$\|vB\|_{\mathrm{TV}}\le \|v\|_{\mathrm{TV}}$, because
$\|vB\|_1\le\sum_x |v(x)|\sum_y B(y\mid x)=\|v\|_1$.} and the last inequality uses
convexity of the total variation norm.
\end{proof}

\subsection{Proof of Lemma \ref{lem:online-sensitivity-consequences}}\label{app:advantage-sensitivity}
\begin{proof}
Boundedness of the stage costs $\|c_i\|_{\infty}\leq 1$, together with the Bellman equations and the
resolvent bound
$\|(I-\gamma \bar{P}^\pi)^{-1}\|_{\infty\to\infty}\le(1-\gamma)^{-1}$,
implies that the value, $Q$-, and hence marginalized advantage functions are
Lipschitz continuous with respect to the policy profile.  

Since $\bar c_i^\pi(s):=
\sum_{a\in\mathcal A}
c_i(s,a)\prod_{j=1}^n\pi_j(a_j\mid s)$ is a multilinear function in players' policies, telescoping the players one at a time and using
$|c_i(s,a)|\le 1$ gives
\begin{align}\label{eq:bar-c-telescop}
\|\bar c_i^{\pi'}-\bar c_i^\pi\|_\infty
\le
2
\max_{x\in\mathcal S}
\sum_{j=1}^n
\|\pi_j'(\cdot\mid x)-\pi_j(\cdot\mid x)\|_{\mathrm{TV}}.
\end{align}
Consider the vector value function $V_i^{\pi}:=(V_i^{\pi}(s), s\in \mathcal{S})$. Then, the Bellman equations \eqref{eq:bellman} for all $s\in \mathcal{S}$ can be written in a compact vector form as $V_i^\pi
=\bar c_i^\pi+\gamma \bar{P}^\pi V_i^\pi$ and $V_i^{\pi'}
=
\bar c_i^{\pi'}+\gamma \bar{P}^{\pi'}V_i^{\pi'}$ therefore imply
\[
V_i^{\pi'}-V_i^\pi
=
(I-\gamma \bar{P}^{\pi'})^{-1}
\left[
\bar c_i^{\pi'}-\bar c_i^\pi
+
\gamma(\bar{P}^{\pi'}-\bar{P}^\pi)V_i^\pi
\right].
\]
Using
$\|V_i^\pi\|_\infty\le (1-\gamma)^{-1}$,
$\|(I-\gamma \bar{P}^{\pi'})^{-1}\|_{\infty\to\infty}
\le(1-\gamma)^{-1}$, and part (ii) of Lemma \ref{lemm:occupancy-sensitivity},
\[
\max_x
\|\bar{P}^{\pi'}(\cdot\mid x)-\bar{P}^\pi(\cdot\mid x)\|_{\mathrm{TV}}
\le
\delta_P \max_x
\sum_{j=1}^n
\|\pi_j'(\cdot\mid x)-\pi_j(\cdot\mid x)\|_{\mathrm{TV}},
\]
together with
\[
\bigl\|(\bar{P}^{\pi'}-\bar{P}^\pi)V_i^\pi\bigr\|_\infty
\le
2\|V_i^\pi\|_\infty
\max_x
\|\bar{P}^{\pi'}(\cdot\mid x)-\bar{P}^\pi(\cdot\mid x)\|_{\mathrm{TV}},
\]
we obtain
\begin{align}\label{eq:value-sensitivity-proof}
\|V_i^{\pi'}-V_i^\pi\|_\infty
\le
\Big(\frac{2}{1-\gamma}+\frac{2\gamma\delta_P}{(1-\gamma)^2}\Big)
\max_{x\in\mathcal S}
\sum_{j=1}^n
\|\pi_j'(\cdot\mid x)-\pi_j(\cdot\mid x)\|_{\mathrm{TV}}.
\end{align}
Similarly, the marginalized $Q$-functions consist of the same
one-stage costs and continuation-value terms. Applying the preceding
mixture bounds to these terms, and then centering the marginalized
$Q$-functions to form $\bar A_i$, gives the stated bound with
$L_{\bar A}=\frac{4}{1-\gamma}+\frac{4\gamma\delta_P}{(1-\gamma)^2}$ and proves \eqref{eq:online-advantage-sensitivity}.
\end{proof}

\subsection{Proof of Lemma \ref{lem:mixed-policy-sensitivity}}\label{app:sensitivity-second-order-advantage}
\begin{proof}
By the Bellman equation in vector form, for any stationary policy profile $\pi$,
\[
V_i^\pi
=
\bar c_i^\pi
+
\gamma \bar P^\pi V_i^\pi,
\]
where
$\bar c_i^\pi(s):=\mathbb E_{A\sim\pi(\cdot\mid s)}[c_i(s,A)]$.
Therefore, subtracting the Bellman equations corresponding to the four
policy profiles gives
\begin{align*}
\left(I-\gamma\bar P^{(\pi_i',\sigma_{-i})}\right)
&\Big(
V_i^{(\pi_i',\sigma_{-i})}
-
V_i^{(\pi_i,\sigma_{-i})}
-
V_i^{(\pi_i',\tau_{-i})}
+
V_i^{(\pi_i,\tau_{-i})}
\Big)
\\
={}&
\bar c_i^{(\pi_i',\sigma_{-i})}
-
\bar c_i^{(\pi_i,\sigma_{-i})}
-
\bar c_i^{(\pi_i',\tau_{-i})}
+
\bar c_i^{(\pi_i,\tau_{-i})}
\\
&+
\gamma
\left(
\bar P^{(\pi_i',\sigma_{-i})}
-
\bar P^{(\pi_i,\sigma_{-i})}
\right)
\left(
V_i^{(\pi_i,\sigma_{-i})}
-
V_i^{(\pi_i,\tau_{-i})}
\right)
\\
&+
\gamma
\left(
\bar P^{(\pi_i',\sigma_{-i})}
-
\bar P^{(\pi_i',\tau_{-i})}
\right)
\left(
V_i^{(\pi_i',\tau_{-i})}
-
V_i^{(\pi_i,\tau_{-i})}
\right)
\\
&+
\gamma
\Big(
\bar P^{(\pi_i',\sigma_{-i})}
-
\bar P^{(\pi_i,\sigma_{-i})}
-
\bar P^{(\pi_i',\tau_{-i})}
+
\bar P^{(\pi_i,\tau_{-i})}
\Big)
V_i^{(\pi_i,\tau_{-i})}.
\end{align*}
Since the expected one-stage cost and the induced transition kernel are
multilinear in the players' policies, telescoping the opponents' policies
one at a time, as in the proofs of Lemmas~\ref{lemm:occupancy-sensitivity}
and \ref{lem:online-sensitivity-consequences} (see \eqref{eq:bar-c-telescop} and \eqref{eq:P-general-telescop}), gives
\begin{align*}
&\left\|
\bar c_i^{(\pi_i',\sigma_{-i})}
-
\bar c_i^{(\pi_i,\sigma_{-i})}
-
\bar c_i^{(\pi_i',\tau_{-i})}
+
\bar c_i^{(\pi_i,\tau_{-i})}
\right\|_\infty\le
4
\|\pi_i'-\pi_i\|_{\mathrm{TV},\infty}
\|\sigma_{-i}-\tau_{-i}\|_{\mathrm{TV},\infty},\cr 
&\left\|
\bar P^{(\pi_i',\sigma_{-i})}
-
\bar P^{(\pi_i,\sigma_{-i})}
-
\bar P^{(\pi_i',\tau_{-i})}
+
\bar P^{(\pi_i,\tau_{-i})}
\right\|_{\infty\to\infty}\le
8\delta_P
\|\pi_i'-\pi_i\|_{\mathrm{TV},\infty}
\|\sigma_{-i}-\tau_{-i}\|_{\mathrm{TV},\infty}.
\end{align*}
Moreover, part~(i) of Lemma~\ref{lemm:occupancy-sensitivity}, together
with the identity
$\|K\|_{\infty\to\infty}
=2\max_s\|K(\cdot\mid s)\|_{\mathrm{TV}}$
for differences of stochastic kernels and telescoping the opponents'
policies one at a time, gives
\begin{align*}
&\left\|
\bar P^{(\pi_i',\sigma_{-i})}
-
\bar P^{(\pi_i,\sigma_{-i})}
\right\|_{\infty\to\infty}
\le
2\delta_P\|\pi_i'-\pi_i\|_{\mathrm{TV},\infty},\cr 
&\left\|
\bar P^{(\pi_i',\sigma_{-i})}
-
\bar P^{(\pi_i',\tau_{-i})}
\right\|_{\infty\to\infty}
\le
2\delta_P\|\sigma_{-i}-\tau_{-i}\|_{\mathrm{TV},\infty}.
\end{align*}
On the other hand, the Bellman-resolvent argument leading to
\eqref{eq:value-sensitivity-proof} in the proof of
Lemma~\ref{lem:online-sensitivity-consequences} gives
\begin{align*}
&\left\|
V_i^{(\pi_i,\sigma_{-i})}
-
V_i^{(\pi_i,\tau_{-i})}
\right\|_\infty
\le
\left(
\frac{2}{1-\gamma}
+
\frac{2\gamma\delta_P}{(1-\gamma)^2}
\right)
\|\sigma_{-i}-\tau_{-i}\|_{\mathrm{TV},\infty},\cr 
&\left\|
V_i^{(\pi_i',\tau_{-i})}
-
V_i^{(\pi_i,\tau_{-i})}
\right\|_\infty
\le
\left(
\frac{2}{1-\gamma}
+
\frac{2\gamma\delta_P}{(1-\gamma)^2}
\right)
\|\pi_i'-\pi_i\|_{\mathrm{TV},\infty}.
\end{align*}
Finally, using $\|
(I-\gamma\bar P^{(\pi_i',\sigma_{-i})})^{-1}\|_{\infty\to\infty}
\le\frac{1}{1-\gamma}$, as in the proof of
Lemma~\ref{lem:online-sensitivity-consequences}, we obtain
\begin{align*}
\left\|
V_i^{(\pi_i',\sigma_{-i})}
-
V_i^{(\pi_i,\sigma_{-i})}
-
V_i^{(\pi_i',\tau_{-i})}
+
V_i^{(\pi_i,\tau_{-i})}
\right\|_\infty\le L_V \|\pi_i'-\pi_i\|_{\mathrm{TV},\infty}
\|\sigma_{-i}-\tau_{-i}\|_{\mathrm{TV},\infty},
\end{align*}
where $L_V=
\frac{4}{1-\gamma}
+
\frac{12\gamma\delta_P}{(1-\gamma)^2}
+
\frac{8\gamma^2\delta_P^2}{(1-\gamma)^3}$. Averaging over the initial distribution $\mu$ proves the result.
\end{proof}

\subsection{Proof of Lemma \ref{lem:span-inner-product}}\label{app:span-lemma}
\begin{proof}
Let $c:=\frac12\left(\max_a v(a)+\min_a v(a)\right)$. Since for any vector $x\in\mathbb R^m$ such that $\sum_a x(a)=0$, we have $\langle x,v\rangle=\langle x,v-c\mathbf 1\rangle$, by H\"older's inequality, we get
\begin{align}\label{eq:span-x}
|\langle x,v\rangle|
\le
\|x\|_1\|v-c\mathbf 1\|_\infty
=
\frac12\|x\|_1\operatorname{span}(v).
\end{align}
Since $\bar A_i^\pi(s,a_i)=\bar Q_i^\pi(s,a_i)-V_i^\pi(s)$ and
$V_i^\pi(s)$ is independent of $a_i$, we have $\operatorname{span}\big(\bar A_i^\pi(s,\cdot)\big)
=\operatorname{span}\big(\bar Q_i^\pi(s,\cdot)\big)$. For any $a_i,a_i'\in\mathcal A_i$,
\begin{align*}
\left|
\bar Q_i^\pi(s,a_i)-\bar Q_i^\pi(s,a_i')
\right|
&\le
\left|
\bar c_i^{\pi_{-i}}(s,a_i)
-
\bar c_i^{\pi_{-i}}(s,a_i')
\right|+
\gamma
\left|
\left\langle
\bar P_i^{\pi_{-i}}(\cdot\mid s,a_i)
-
\bar P_i^{\pi_{-i}}(\cdot\mid s,a_i'),
V_i^\pi
\right\rangle
\right|
\\
&\le
1+
\gamma\delta_P\operatorname{span}(V_i^\pi)
\le
1+\frac{\gamma\delta_P}{1-\gamma},
\end{align*}
where the second inequality uses $c_i(s,a_i)\in[0,1]$, Definition~\ref{def:transition-sensitivity}, and \eqref{eq:span-x}, while the last inequality follows from
$0\le V_i^\pi(s)\le1/(1-\gamma)$. Taking the maximum over
$a_i,a_i'$ proves $\operatorname{span}(\bar{A}^{\pi}_i(s,\cdot))\leq 1+\frac{\gamma\delta_P}{1-\gamma}$. Finally, choosing $v:=\bar A_i^\pi(s,\cdot)$ and $x:=\pi'_i(\cdot\mid s)-\pi_i''(\cdot\mid s)$ in \eqref{eq:span-x}, and noting that 
$\frac12\|\pi'_i(\cdot\mid s)-\pi_i''(\cdot\mid s)\|_1=\|\pi'_i(\cdot\mid s)-\pi''_i(\cdot\mid s)\|_{\mathrm{TV}}$, proves \eqref{eq:advantage-span-policy-difference}.

Moreover, by the definition of the marginalized advantage function,
\[
\sum_{a_i\in\mathcal A_i}
\pi_i(a_i\mid s)\bar A_i^\pi(s,a_i)
=
0.
\]
Hence,
\[
\min_{a_i\in\mathcal A_i}\bar A_i^\pi(s,a_i)
\le
0
\le
\max_{a_i\in\mathcal A_i}\bar A_i^\pi(s,a_i).
\]
Therefore,
\[
\left\|\bar A_i^\pi(s,\cdot)\right\|_\infty
\le
\operatorname{span}\big(\bar A_i^\pi(s,\cdot)\big)
\le
1+\frac{\gamma\delta_P}{1-\gamma}
=
1+L_d,
\]
which completes the proof.
\end{proof}

\newpage
\section{Omitted Proofs for Section \ref{sec:Episodic}}

\subsection{Proof of Lemma \ref{lem:random-count-weighted-variance}}\label{app:episodic-estimator-noise}
\begin{proof}
For every state $s$ and $r=1,\ldots,L_t$, let $\tau_{t,r}(s)$
denote the time of the $r$-th visit to state $s$ during episode $t$.
By the definition of the sample-time filtration in Section~\ref{sec:filteration-oracle-episode},
$\mathcal F_{\tau_{t,r}(s)}$ contains all information available immediately
before the action $A^{\tau_{t,r}(s)}$ is drawn and the corresponding raw sample is generated. In particular, since the policy profile $\pi^t$ remains fixed
throughout the episode, $A^{\tau_{t,r}(s)}\mid\mathcal F_{\tau_{t,r}(s)}
\sim\pi^t(\cdot\mid s),$ with the players' action draws conditionally independent. Moreover,
$\mathcal F_t\subseteq\mathcal F_{\tau_{t,r}(s)}$, and, for $r<r'$, the
action and oracle sample generated at $\tau_{t,r}(s)$ are measurable with
respect to $\mathcal F_{\tau_{t,r'}(s)}$. 

Now, let us fix $(i,s,a_i)$ and define the centered observations
\begin{align}\nonumber
X_r(s,a_i)
:=
\widehat A_i^{\tau_{t,r}(s)}(s,a_i)
-
\mathbb E_{\tau_{t,r}(s)}
\left[
\widehat A_i^{\tau_{t,r}(s)}(s,a_i)
\right],
\qquad r=1,\ldots,L_t.
\end{align}
The centered observations satisfy the martingale-difference property along the
successive visits to $s$. Indeed, the observation generated at visit $r$ has
conditional zero mean given the information available immediately before that
visit, i.e., $\mathbb E_{\tau_{t,r}(s)}
\left[X_r(s,a_i)\right]=0.$ Moreover, for $r'<r$, the earlier observation $X_{r'}(s,a_i)$ is
measurable with respect to $\mathcal F_{\tau_{t,r}(s)}$. Now, let us define
\begin{align}\nonumber
m_i^t(s,a_i)
:=
\frac{1}{L_t}
\sum_{r=1}^{L_t}
\mathbb E_{\tau_{t,r}(s)}
\left[
\widehat A_i^{\tau_{t,r}(s)}(s,a_i)
\right].
\end{align}
Then, using \eqref{eq:g-random-count-advantage-estimator} and \eqref{eq:random-count-noise}, we have
\begin{align}
&\widetilde r_i^t(s,a_i)-m_i^t(s,a_i)
=\frac{1}{L_t}\sum_{r=1}^{L_t}X_r(s,a_i),\cr 
&\xi_i^t(s,a_i)=r_i^t(s,a_i)-m_i^t(s,a_i).
\label{eq:episodic-centered-martingale-average}
\end{align}
Moreover, Assumption \ref{ass:episodic-oracle} and the triangle inequality give
\begin{align}\nonumber
\bigg|
\bar A_i^{\pi^t}(s,a_i)
-
r_i^t(s,a_i)
+
\xi_i^t(s,a_i)
\bigg|
&=
\left|
\frac{1}{L_t}
\sum_{r=1}^{L_t}
\left(
\bar A_i^{\pi^t}(s,a_i)
-
\mathbb E_{\tau_{t,r}(s)}
\left[
\widehat A_i^{\tau_{t,r}(s)}(s,a_i)
\right]
\right)
\right|
\cr
&\le
\frac{1}{L_t}
\sum_{r=1}^{L_t}
\left|
\bar A_i^{\pi^t}(s,a_i)
-
\mathbb E_{\tau_{t,r}(s)}
\left[
\widehat A_i^{\tau_{t,r}(s)}(s,a_i)
\right]
\right|
\le
L_{\widehat A},
\end{align}
where the equality holds by \eqref{eq:random-count-noise}. This proves \eqref{eq:global-table-bias}.

To bound the conditional second moment $\mathcal{V}_i^t$, using the definition of $X_r(s,a_i)$ and the action-centered estimator
$\widehat A_i^{\tau_{t,r}(s)}(s,\cdot)$ from
\eqref{eq:hat-A-episodic}, we have
\begin{align}
&\mathbb E_{\tau_{t,r}(s)}
\bigg[
\sum_{a_i\in\mathcal A_i}
\pi_i^t(a_i\mid s)
\left(
X_r(s,a_i)
\right)^2
\bigg]\le
\mathbb E_{\tau_{t,r}(s)}
\bigg[
\sum_{a_i\in\mathcal A_i}
\pi_i^t(a_i\mid s)
\left(
\widehat A_i^{\tau_{t,r}(s)}(s,a_i)
\right)^2
\bigg]
\cr
&\qquad=
\mathbb E_{\tau_{t,r}(s)}
\bigg[
\Big(g_i^{\tau_{t,r}(s)}\big(s,A^{\tau_{t,r}(s)}\big)\Big)^2
\bigg(
\frac{1}{\pi_i^t\big(A_i^{\tau_{t,r}(s)}\mid s\big)}-1
\bigg)
\bigg]
\cr
&\qquad\leq 
\sum_{a_i\in\mathcal A_i}
\mathbb E_{\tau_{t,r}(s)}
\left[
\left.
\Big(g_i^{\tau_{t,r}(s)}
\big(s, a_i,A_{-i}^{\tau_{t,r}(s)}\big)\Big)^2\right|
A_i^{\tau_{t,r}(s)}=a_i
\right]\le
|\mathcal A_i|.
\label{eq:episodic-one-sample-weighted-variance}
\end{align}
The first inequality follows from conditional centering: for each $a_i$,
subtracting
$\mathbb E_{\tau_{t,r}(s)}
[\widehat A_i^{\tau_{t,r}(s)}(s,a_i)]$
reduces the conditional second moment by the square of the conditional
mean. For the equality, by \eqref{eq:hat-A-episodic}, direct expansion of the $\pi_i^t(\cdot\mid s)$-weighted squared action-centered estimator gives
\begin{align*}
\sum_{a_i\in\mathcal A_i}
\pi_i^t(a_i\mid s)
\left(
\widehat A_i^{\tau_{t,r}(s)}(s,a_i)
\right)^2=\left(g_i^{\tau_{t,r}(s)}\big(s,A^{\tau_{t,r}(s)}\big)\right)^2
\bigg(\frac{1}{\pi_i^t\big(A_i^{\tau_{t,r}(s)}\mid s\big)}-1\bigg).
\end{align*}
The second inequality simply drops the nonpositive term $-1$, and taking conditional expectation by averaging
over $A_i^{\tau_{t,r}(s)}\sim\pi_i^t(\cdot\mid s)$, which cancels the importance-weighting denominator. Finally, the bounded-oracle assumption
$|g_i^{\tau_{t,r}(s)}(s,A)|\le1$ bounds each of the resulting
$|\mathcal A_i|$ terms by $1$.

On the other hand, since $|g_i^\tau(s,A)|\le1$, we have
\begin{align}\nonumber
&\left|
\mathbb E_{\tau_{t,r}(s)}
\left[
\widehat A_i^{\tau_{t,r}(s)}(s,a_i)
\right]
\right|\cr
&\qquad\le
\mathbb E_{\tau_{t,r}(s)}
\bigg[
\left|g_i^{\tau_{t,r}(s)}\!\left(s,A^{\tau_{t,r}(s)}\right)\right|
\frac{\mathbf 1\{A_i^{\tau_{t,r}(s)}=a_i\}}
{\pi_i^t(a_i\mid s)}
\bigg]
+
\mathbb E_{\tau_{t,r}(s)}
\left[
\left|g_i^{\tau_{t,r}(s)}\!\left(s,A^{\tau_{t,r}(s)}\right)\right|
\right]\cr
&\qquad=
\mathbb E_{\tau_{t,r}(s)}
\left[
\left.
\left|g_i^{\tau_{t,r}(s)}\!\left(s,A^{\tau_{t,r}(s)}\right)\right|\right|
A_i^{\tau_{t,r}(s)}=a_i
\right]
+
\mathbb E_{\tau_{t,r}(s)}
\left[
\left|g_i^{\tau_{t,r}(s)}\!\left(s,A^{\tau_{t,r}(s)}\right)\right|
\right]
\le2.
\end{align}
Consequently, $|m_i^t(s,a_i)|\le2\le\frac{B}{2}$, and therefore, $\operatorname{clip}_B(m_i^t(s,a_i))=m_i^t(s,a_i)$. Since $r_i^t(s,a_i)=\operatorname{clip}_B(\widetilde{r}^t_i(s,a_i))$, using the $1$-Lipschitzness\footnote{That is,
$\left|\operatorname{clip}_B(x)-\operatorname{clip}_B(y)\right|
\le |x-y|,\forall x,y\in\mathbb R.$} of the clipping and \eqref{eq:episodic-centered-martingale-average}, we have
\begin{align}
|\xi_i^t(s,a_i)|
=
\left|
\operatorname{clip}_B\!\left(\widetilde r_i^t(s,a_i)\right)
-
\operatorname{clip}_B\!\left(m_i^t(s,a_i)\right)
\right|\le
\left|
\widetilde r_i^t(s,a_i)-m_i^t(s,a_i)
\right|.
\label{eq:episodic-clipping-contraction}
\end{align}
Therefore, for every state $s$ we can write
\begin{align*}
&\mathbb E_t
\bigg[
\sum_{a_i\in\mathcal A_i}
\pi_i^t(a_i\mid s)
\left(
\xi_i^t(s,a_i)
\right)^2
\bigg]\leq \mathbb E_t
\bigg[
\sum_{a_i\in\mathcal A_i}
\pi_i^t(a_i\mid s)
\left(
\widetilde r_i^t(s,a_i)-m_i^t(s,a_i)
\right)^2
\bigg]
\cr
&=
\mathbb E_t
\bigg[
\sum_{a_i\in\mathcal A_i}
\pi_i^t(a_i\mid s)
\bigg(
\frac{1}{L_t}\sum_{r=1}^{L_t}X_r(s,a_i)
\bigg)^2
\bigg]=
\frac{1}{L_t^2}
\sum_{r=1}^{L_t}
\mathbb E_t
\bigg[
\sum_{a_i\in\mathcal A_i}
\pi_i^t(a_i\mid s)
\left(
X_r(s,a_i)
\right)^2
\bigg]
\le
\frac{|\mathcal A_i|}{L_t},
\end{align*}
where the first inequality uses \eqref{eq:episodic-clipping-contraction} and the first equality follows from \eqref{eq:episodic-centered-martingale-average}. The second equality follows from the martingale-difference property, which ensures that the cross terms in the squared average vanish by conditioning on $\mathcal F_{\tau_{t,r}(s)}$ and the tower property. Indeed, for $r'<r$, $X_{r'}(s,a_i)$ is $\mathcal F_{\tau_{t,r}(s)}$-measurable, and hence $\mathbb E_t\left[X_{r'}(s,a_i)X_r(s,a_i)\right]=\mathbb E_t\left[X_{r'}(s,a_i)\mathbb E_{\tau_{t,r}(s)}\left[X_r(s,a_i)\right]\right]=0$. Moreover, the last inequality uses \eqref{eq:episodic-one-sample-weighted-variance} and the tower property.  

Finally, using \eqref{eq:episodic-clipping-contraction} and since $\pi^t$ and $d^{\pi^t}_{\mu}$ are $\mathcal F_t$-measurable
and $\sum_s d_\mu^{\pi^t}(s)=1$, we obtain
\begin{align}\nonumber
\mathcal V_i^t
&=
\mathbb E_t
\bigg[
\sum_{s,a_i}
d_\mu^{\pi^t}(s)
\pi_i^t(a_i\mid s)
\bigl(\xi_i^t(s,a_i)\bigr)^2
\bigg]
\le
\frac{|\mathcal A_i|}{L_t}.\tag*{\hfill\qedsymbol}
\end{align}
\renewcommand{\qedsymbol}{}
\end{proof}

\vspace{-2cm}
\section{Omitted Proofs for Section \ref{sec:fully-online-general-occupancy}}

\subsection{Proof of Lemma \ref{lemm:conditional-variance-online}}\label{app:online-importance-sampling-bounds}

\begin{proof}
At an actual visit $t=\tau_k(s)$, the conditional second moment is controlled as
\begin{align*}
\mathcal{V}_i^k(s)&\le
d_\mu^k(s)\mathbb E_{\tau_k(s)}\bigg[
\sum_{a_i\in\mathcal A_i}
\pi_i^k(a_i\mid s)
\bigl(\widehat A_i^k(s,a_i)\bigr)^2
\bigg]\cr
&=
d_\mu^k(s)\mathbb E_{\tau_k(s)}\bigg[
(g_i^t(s,A^t))^2
\bigg(
\frac{1}{\pi_i^k(A_i^t\mid s)}-1
\bigg)
\bigg]\cr
&\le
d_\mu^k(s)\mathbb E_{\tau_k(s)}\!\left[
\frac{(g_i^t(s,A^t))^2}
{\pi_i^k(A_i^t\mid s)}
\right]
\le
d_\mu^k(s)\,|\mathcal{A}_i|
\le A_{\max},
\end{align*}
where the first inequality follows because conditional centering minimizes mean-square error. The equality follows directly from the definition $\widehat A_i^k(s,a_i)
=g_i^t(s,A^t)
\big(\mathbf 1\{A_i^t=a_i\}/\pi_i^k(a_i\mid s)-1
\big)$. Indeed, for the realized action $A_i^t$,
\begin{align}\nonumber
\sum_{a_i\in\mathcal A_i}
\pi_i^k(a_i\mid s)
\bigl(\widehat A_i^k(s,a_i)\bigr)^2
&=
(g_i^t(s,A^t))^2
\sum_{a_i\in\mathcal A_i}
\pi_i^k(a_i\mid s)
\left(
\frac{\mathbf 1\{A_i^t=a_i\}}
{\pi_i^k(a_i\mid s)}-1
\right)^2\cr
&=
(g_i^t(s,A^t))^2
\left(
\frac{1}{\pi_i^k(A_i^t\mid s)}-1
\right).
\end{align}
Thus, the subtraction of the realized raw sample is exactly the action-centering operation and reduces the corresponding $\pi_i^k(\cdot\mid s)$-weighted second moment by $(g_i^t(s,A^t))^2$ relative to the uncentered importance-weighted vector. The next inequality simply drops this nonnegative reduction. The following inequality holds by conditioning on $\mathcal F_{\tau_k(s)}$ and averaging over $A_i^{\tau_k(s)}\sim\pi_i^k(\cdot\mid s)$, which cancels the importance-weighting denominator, while the bounded-oracle assumption $|g_i^{\tau_k(s)}|\le 1$ gives a contribution of at most $1$ for each $a_i\in\mathcal A_i$. The last inequality uses $d_\mu^k(s):=d_\mu^{\pi^{\tau_k(s)}}(s)\le1$.

Moreover, almost surely, we have
\begin{align*}
\!\!\!\!\!\!\sum_{a_i}\pi_i^k(a_i\mid s)
\bigl(\xi_i^k(s,a_i)\bigr)^2
&\!\le\!
2\!\sum_{a_i}\pi_i^k(a_i\mid s)
\bigl(\widehat A_i^k(s,a_i)\bigr)^2
\!+\!
2\!\sum_{a_i}\pi_i^k(a_i\mid s)
\bigl(\mathbb E_{\tau_k(s)}[\widehat A_i^k(s,a_i)]\bigr)^2
\cr
&\le
2\sum_{a_i}\pi_i^k(a_i\mid s)
\bigl(\widehat A_i^k(s,a_i)\bigr)^2
+8
\cr
&=
2(g_i^t(s,A^t))^2
\left(
\frac{1}{\pi_i^k(A_i^t\mid s)}-1
\right)+8
\cr
&\le
\frac{2(g_i^t(s,A^t))^2}{\pi_i^k(A_i^t\mid s)}+8
\le
\frac{2|\mathcal{A}_i|}{\zeta}+8
\le
10\frac{A_{\max}}{\zeta}.
\end{align*}
The first inequality uses
$\xi_i^k=\widehat A_i^k-
\mathbb E_{\tau_k(s)}[\widehat A_i^k]$ while the second one uses $\|\mathbb E_{\tau_k(s)}[\widehat A_i^k(s)]\|_\infty\le2$,
which follows directly from $|g_i^t(s,A^t)|\le 1$ and the definition of
$\widehat A_i^k$. The equality follows from the direct weighted second-moment identity above, and the next inequality drops the nonnegative action-centering reduction. The following inequality uses $\pi_i^k(\cdot\mid s)\ge\frac{\zeta}{|\mathcal{A}_i|}$ and $|g_i^t(s,A^t)|\le1$. Finally, the last inequality holds because $\zeta\in (0, 1)$ and $A_{\max}=\max_i |\mathcal{A}_i|\ge 1$.
\end{proof}

\subsection{Proof of Lemma \ref{lem:online-tracking}}\label{app:dynamic-traking}
Since the proof is relatively long, we divide it into four steps and briefly describe the role of each step. In Step~1, we derive the one-step tracking-error recursion and control the stochastic fluctuations by constructing suitable global-time martingale-difference sequences. In Step~2, we sum the tracking recursion along the realized trajectory and control the resulting occupancy-weight drift using the localized occupancy-sensitivity bound. In Step~3, we bound the drift of the true marginalized advantages between consecutive visits to the same state, using the bounded-intervisit event and charging the resulting occupancy drift to the cumulative policy movement. Finally, in Step~4, we combine these estimates and pass from the lagged tracking errors appearing naturally in the recursion to the current tracking errors, yielding the desired bound. With these steps in mind, we now provide a formal proof.

\begin{proof}
\subsubsection*{Step 1: Tracking error recursion and the global martingales.}
Fix $s$ and $k\ge0$, and put $t=\tau_k(s)$. Recall that $\mathcal F_{\tau_k(s)}$ contains the complete global history, including the current state $S^t$, but is taken before the fresh actions are drawn and oracle randomness at the $(k+1)$st visit to state $s$ is generated. In particular, $\pi^{\tau_k(s)}$, $r_i^{k-1}(s)$, and $e_i^{k-1}(s)$ are
$\mathcal F_{\tau_k(s)}$-measurable. The fresh actions and oracle observations
are generated afterward. 

For $k\!\ge\! 1$, let $\widetilde e_i^k(s,a_i)=\widetilde r_i^k(s,a_i)\!-\!
\bar A_i^k(s,a_i)$ denote the tracking error before clipping, and define
\[
\Delta A_i^k(s,a_i):=
\bar A_i^{k-1}(s,a_i)
-\bar A_i^k(s,a_i), \ \ \ \Delta A_i^0(s)=0.
\]
Then using $\xi_i^k(s,a_i)
:=\widehat A_i^k(s,a_i)
-\mathbb E_{\tau_k(s)}[\widehat A_i^k(s,a_i)]$ and \eqref{eq:online-tracking-error}, we have
\begin{align}
\sqrt{\pi_i^k(a_i\mid s)}\,\widetilde e_i^k(s,a_i)
&=
\sqrt{\pi_i^k(a_i\mid s)}
\left[
(1-\omega)r_i^{k-1}(s,a_i)
+\omega\widehat A_i^k(s,a_i)
-\bar A_i^k(s,a_i)
\right]
\notag\\
&=
\sqrt{\pi_i^k(a_i\mid s)}
\Bigg\{
(1-\omega)
\left[e_i^{k-1}(s,a_i)+\Delta A_i^k(s,a_i)\right]
\notag\\
&\qquad\qquad
+\omega\left[
\mathbb E_{\tau_k(s)}[\widehat A_i^k(s,a_i)]
-\bar A_i^k(s,a_i)
\right]
+\omega\xi_i^k(s,a_i)
\Bigg\}
\notag\\
&=
u_i^k(s,a_i)
+\omega\sqrt{\pi_i^k(a_i\mid s)}\,\xi_i^k(s,a_i),
\label{eq:online-u-error}
\end{align}
where we recall that $\pi_i^k(\cdot\mid s)$ is the policy stored at state $s$
immediately before the $(k+1)$st local update, and $u_i^k(s,a_i)$ is defined to be the $\mathcal F_{\tau_k(s)}$ predictable part of \eqref{eq:online-u-error} before clipping, i.e., \footnote{Since $\bar A_i^k(s,\cdot):=\bar A_i^{\pi^{\tau_k(s)}}(s,\cdot)$ is determined by the policy profile $\pi^{\tau_k(s)}$, it is $\mathcal F_{\tau_k(s)}$-measurable.}
\begin{align}\nonumber
u_i^k(s,a_i)
\!:={}&\sqrt{\pi_i^k(a_i\mid s)}\Bigg\{\!
(1-\omega)\bigl[e_i^{k-1}(s,a_i)\!+\!\Delta A_i^k(s,a_i)\bigr]\!+\!\omega\Bigl[
\mathbb E_{\tau_k(s)}[\widehat A_i^k(s,a_i)]
\!-\!\bar A_i^k(s,a_i)
\Bigr]\!\Bigg\}.
\end{align}
Squaring this relation, summing over $a_i$, and applying Young's inequality, for some constant $c\leq 5$,
\begin{align}\label{eq:sum-u-square}
\sum_{a_i}(u_i^k(s,a_i))^2
&\le
(1-\omega)
\sum_{a_i}\pi_i^k(a_i\mid s)
\bigl(e_i^{k-1}(s,a_i)\bigr)^2
+\frac{c}{\omega}
\sum_{a_i}\pi_i^k(a_i\mid s)
\bigl(\Delta A_i^k(s,a_i)\bigr)^2\cr
&\qquad+c\omega
\sum_{a_i}\pi_i^k(a_i\mid s)
\left(
\mathbb E_{\tau_k(s)}
\bigl[\widehat A_i^k(s,a_i)\bigr]
-\bar A_i^k(s,a_i)
\right)^2\cr
&\hspace{-1cm}\le
(1-\omega)
\sum_{a_i}\pi_i^{k}(a_i\mid s)
\bigl(e_i^{k-1}(s,a_i)\bigr)^2
+\frac{c}{\omega}
\sum_{a_i}\pi_i^k(a_i\mid s)
\bigl(\Delta A_i^k(s,a_i)\bigr)^2+c\omega L_{\widehat A}^2\cr 
&\hspace{-1cm}\le
(1-\omega)e^{\frac{2\eta M}{1-\gamma}}
\sum_{a_i}\pi_i^{k-1}(a_i\mid s)
\bigl(e_i^{k-1}(s,a_i)\bigr)^2
+\frac{c}{\omega}
\sum_{a_i}\pi_i^k(a_i\mid s)
\bigl(\Delta A_i^k(s,a_i)\bigr)^2+c\omega L_{\widehat A}^2\cr 
&\hspace{-1cm}\le
(1-c_1\omega)\mathcal E_i^{k-1}(s)
+\frac{c}{\omega}
\sum_{a_i}\pi_i^k(a_i\mid s)
\bigl(\Delta A_i^k(s,a_i)\bigr)^2
+c\omega L_{\widehat A}^2,
\end{align}
where the second inequality uses the oracle Assumption~\ref{ass:online-one-sample-oracle}. The third inequality uses part~(i) of Lemma~\ref{lem:kl-projected-local-geometry} (see also Corollary~\ref{cor:local-kl-projected-geometry}). The last inequality follows from the definition in \eqref{eq:online-tracking-error} and the condition $\eta M/(1-\gamma)\le c_0\omega$, which, for sufficiently small $c_0>0$, implies $(1-\omega)e^{2\eta M/(1-\gamma)}\leq (1-\omega)e^{2c_0\omega}\leq 1-c_1\omega$ for some universal constant $c_1>0$.  

Next, note that since the true target $\bar A_i^k(s,a_i)$ belongs to $[-M,M]$ and clipping
is Euclidean projection onto this interval, by the nonexpansive property of the projection, we have
\begin{equation}\nonumber
|e_i^k(s,a_i)|=|\operatorname{clip}_{M}
    \bigl(\widetilde r_i^k(s,a_i)\bigr)-
\operatorname{clip}_{M}
    \bigl(\bar A_i^k(s,a_i)\bigr) |\leq |\widetilde r_i^k(s,a_i)-
\bar A_i^k(s,a_i)|=|\widetilde e_i^k(s,a_i)|.
\end{equation}
Therefore,
\begin{equation}\label{eq:online-clipping-nonexpansive}
|e_i^k(s,a_i)|\leq |\widetilde e_i^k(s,a_i)|, \ \ \ 
\qquad
\|e_i^k(s)\|_\infty\le2M.
\end{equation}
Squaring \eqref{eq:online-u-error}, using
\eqref{eq:online-clipping-nonexpansive}, and summing over $a_i$, we obtain 
\begin{align}\nonumber
\mathcal E_i^k(s)\leq \sum_{a_i}(u_i^k(s,a_i))^2
+\omega^2\sum_{a_i}\pi_i^k(a_i\mid s)(\xi_i^k(s,a_i))^2+2\omega\sum_{a_i}u_i^k(s,a_i)\sqrt{\pi_i^k(a_i\mid s)}\,\xi_i^k(s,a_i).
\end{align}
Using \eqref{eq:sum-u-square} into the above relation and adding and subtracting the conditional second moment gives  
\begin{align}\label{eq:online-one-step-tracking-energy}
\mathcal E_i^k(s)&\le(1-c_1\omega)\mathcal E_i^{k-1}(s)
+\frac{c}{\omega}\sum_{a_i}\pi_i^k(a_i\mid s)(\Delta A_i^k(s,a_i))^2
+c\omega L_{\widehat{A}}^2\cr 
&\qquad+Y_i^k(s)+V_i^k(s)+\omega^2\mathbb E_{\tau_k(s)}\!\bigg[
\sum_{a_i}\pi_i^k(a_i\mid s)(\xi_i^k(s,a_i))^2\bigg],
\end{align}
where $Y_i^k(s)$ and $V_i^k(s)$ are defined as
\begin{align}\nonumber
Y_i^k(s):=2\omega\sum_{a_i}u_i^k(s,a_i)
&\sqrt{\pi_i^k(a_i\mid s)}\xi_i^k(s,a_i),\cr
V_i^k(s):=\omega^2\bigg[
\sum_{a_i}\pi_i^k(a_i\mid s)(\xi_i^k(s,a_i))^2
&-\mathbb E_{\tau_k(s)}\Big[
\sum_{a_i}\pi_i^k(a_i\mid s)(\xi_i^k(s,a_i))^2\Big]
\bigg].
\end{align}
We note that $\mathbb E_{\tau_k(s)}[\xi_i^k(s,a_i)]=0$ by definition, and since $u_i^k(s,a_i)$ and $\pi_i^k(a_i\mid s)$ are $\mathcal F_{\tau_k(s)}$-measurable, it follows that $\mathbb E_{\tau_k(s)}[Y_i^k(s)]=0$, while $\mathbb E_{\tau_k(s)}[V_i^k(s)]=0$ follows directly from its conditional centering. Thus, both $Y_i^k(s)$ and $V_i^k(s)$ are conditionally mean zero given $\mathcal F_{\tau_k(s)}$.

To apply martingale concentration along the actual online trajectory, we now
reindex these local increments in global time. Indeed, at every global time $t$, setting $S^t=s$ and
$N_t(S^t)=k$ gives $t=\tau_k(s)=\tau_{N_t(S^t)}(S^t)$, so $Y_i^{N_t(S^t)}(S^t)$ and $V_i^{N_t(S^t)}(S^t)$ are precisely the local statewise
increments $Y_i^{k}(s)$ and $V_i^{k}(s)$ generated at time $t$. In particular, since
$\mathcal F_{\tau_k(s)}=\mathcal F_t$, their conditional mean-zero property
with respect to $\mathcal F_{\tau_k(s)}$ implies that they have conditional
mean zero given $\mathcal F_t$. Since $S^t$,
$\pi^t$, and hence $d_\mu^{\pi^t}(S^t)$ are $\mathcal F_t$-measurable, it
follows that the occupancy-weighted increments $d_\mu^{\pi^t}(S^t)Y_i^{N_t(S^t)}(S^t)$ and $d_\mu^{\pi^t}(S^t)V_i^{N_t(S^t)}(S^t)$ are $\mathcal F_{t+1}$-measurable and have conditional mean zero given
$\mathcal F_t$. Thus, the global-time reindexing chronologically merges the
statewise mean-zero increments while preserving their martingale-difference
property, yielding martingale-difference sequences with respect to the global
filtration $(\mathcal F_t)_{t\ge0}$.\footnote{The global-time sequence interleaves the statewise local sequences in the realized order of visits, with no local increment carried forward between visits. For example, if the increments for $s$ occur at $\tau_0(s)=1$, $\tau_1(s)=4$, $\tau_2(s)=6$ and those for $s'$ at $\tau_0(s')=2$, $\tau_1(s')=3$, $\tau_2(s')=5$, their chronological merging is $V_i^0(s)$,$V_i^0(s')$,$V_i^1(s')$,$V_i^1(s)$,$V_i^2(s')$,$V_i^2(s)$, which is exactly the global-time sequence $\{V_i^{N_t(S^t)}(S^t)\}_{t\ge 0}$.} Therefore, in the following, we compute their conditional second moments to obtain high-probability concentration results.

Since
$\pi_i^t(\cdot\mid S^t)
=\pi_i^{N_t(S^t)}(\cdot\mid S^t)$,
and $\pi_i^t(\cdot\mid S^t)$,
$d_\mu^{\pi^t}(S^t)$, and
$u_i^{N_t(S^t)}(S^t,\cdot)$ are $\mathcal F_t$-measurable, while $\xi_i^{N_t(S^t)}(S^t)
=
\widehat A_i^{N_t(S^t)}(S^t)
-\mathbb E_t\big[
\widehat A_i^{N_t(S^t)}(S^t)
\big]$, by Cauchy--Schwarz, we have
\begin{align}
&\mathbb E_t\!\left[
\bigl(
d_\mu^{\pi^t}(S^t)
Y_i^{N_t(S^t)}(S^t)
\bigr)^2
\right]
\notag\\
&\qquad\le
4\omega^2
\bigl(d_\mu^{\pi^t}(S^t)\bigr)^2
\left(
\sum_{a_i}
\bigl(
u_i^{N_t(S^t)}(S^t,a_i)
\bigr)^2
\right)
\mathbb E_t\!\left[
\sum_{a_i}
\pi_i^t(a_i\mid S^t)
\bigl(
\xi_i^{N_t(S^t)}(S^t,a_i)
\bigr)^2
\right]
\notag\\
&\qquad\le
c\omega^2 M^2
d_\mu^{\pi^t}(S^t)
\mathbb E_t\!\left[
\sum_{a_i}
\pi_i^t(a_i\mid S^t)
\bigl(
\xi_i^{N_t(S^t)}(S^t,a_i)
\bigr)^2
\right]
\notag\\
&\qquad=
c\omega^2M^2
\mathcal V_i^{N_t(S^t)}(S^t).
\label{eq:online-global-Y-variance}
\end{align}
Here, the first inequality uses Cauchy--Schwarz and measurability of
$u_i^{N_t(S^t)}(S^t,\cdot)$. The second inequality holds because $d_\mu^{\pi^t}(S^t)\le1$ and $\sum_{a_i}(u_i^k(s,a_i))^2\le cM^2$ for all $s$. To see the latter, since $|g_i^t(s,A)|\le1$, the definition of $\widehat A_i^k$ implies $\big|\mathbb E_{\tau_k(s)}[\widehat A_i^k(s,a_i)]\big|\le2$, and therefore $\big\|\mathbb E_{\tau_k(s)}[\widehat A_i^k(s,\cdot)]-\bar A_i^k(s,\cdot)\big\|_\infty\le2+M$. Combining this with $\|e_i^k(s)\|_\infty\le2M$ and $\|\Delta A_i^k(s)\|_\infty\le2M$ in the definition of $u_i^k$ yields $\sum_{a_i}(u_i^k(s,a_i))^2\le cM^2$.\footnote{Here, using the fact that $M\ge 1$, we absorb the constant terms added to $M$ into the universal constant $c$.} Finally, the equality in \eqref{eq:online-global-Y-variance} follows from $t=\tau_{N_t(S^t)}(S^t)$ and the definition of
$\mathcal V_i^{N_t(S^t)}(S^t)$ in
\eqref{eq:online-occupancy-weighted-variance}. 

On the other hand, the martingale increments can be bounded as
\begin{align*}
\left|
d_\mu^{\pi^t}(S^t)
Y_i^{N_t(S^t)}(S^t)
\right|&\le
2\omega 
\bigg(
\sum_{a_i}
\bigl(
u_i^{N_t(S^t)}(S^t,a_i)
\bigr)^2
\bigg)^{1/2}
\bigg(
\sum_{a_i}
\pi_i^t(a_i\mid S^t)
\bigl(
\xi_i^{N_t(S^t)}(S^t,a_i)
\bigr)^2
\bigg)^{1/2}\cr 
&\le
c\omega M
\left(\frac{A_{\max}}{\zeta}\right)^{1/2},
\end{align*}
where the first inequality follows directly by Cauchy--Schwarz to the definition of $Y_i^{N_t(S^t)}(S^t)$ and using $d_\mu^{\pi^t}(S^t)\leq 1$, while the second inequality uses $\sum_{a_i}(u_i^{N_t(S^t)}(S^t,a_i))^2
\le cM^2$, and the almost-sure bound
$\sum_{a_i}\pi_i^t(a_i\mid S^t)
(\xi_i^{N_t(S^t)}(S^t,a_i))^2
\le 10A_{\max}/\zeta$ from Lemma~\ref{lemm:conditional-variance-online}.

Similarly, we can bound the conditional occupancy-weighted second moment of $V_i$ as 
\begin{align}
\mathbb E_t\!\left[
\bigl(
d_\mu^{\pi^t}(S^t)
V_i^{N_t(S^t)}(S^t)
\bigr)^2
\right]&\le
\omega^4
\bigl(d_\mu^{\pi^t}(S^t)\bigr)^2
\mathbb E_t\!\left[
\left(
\sum_{a_i}
\pi_i^t(a_i\mid S^t)
\bigl(
\xi_i^{N_t(S^t)}(S^t,a_i)
\bigr)^2
\right)^2
\right]\cr 
&\le
c\frac{\omega^4A_{\max}}{\zeta}
\bigl(d_\mu^{\pi^t}(S^t)\bigr)^2
\mathbb E_t\!\left[
\sum_{a_i}
\pi_i^t(a_i\mid S^t)
\bigl(
\xi_i^{N_t(S^t)}(S^t,a_i)
\bigr)^2
\right]\cr
&\le
c\frac{\omega^4A_{\max}}{\zeta}
d_\mu^{\pi^t}(S^t)
\mathbb E_t\!\left[
\sum_{a_i}
\pi_i^t(a_i\mid S^t)
\bigl(
\xi_i^{N_t(S^t)}(S^t,a_i)
\bigr)^2
\right]\cr
&=
c\frac{\omega^4A_{\max}}{\zeta}
\mathcal V_i^{N_t(S^t)}(S^t).
\label{eq:online-global-V-variance}
\end{align}
The first inequality follows from the definition of
\(V_i^{N_t(S^t)}(S^t)\) and the fact that conditional centering can only
decrease the second moment. The second inequality uses the almost-sure bound from Lemma~\ref{lemm:conditional-variance-online}, and the third inequality
uses \(d_\mu^{\pi^t}(S^t)\le1\). Moreover, using the definition of $V_i^{N_t(S^t)}(S^t)$ and the same almost-sure bound from Lemma~\ref{lemm:conditional-variance-online}, it is easy to see that  
\[
\Big|d_\mu^{\pi^t}(S^t)V_i^{N_t(S^t)}(S^t)\Big|
\le c\frac{\omega^2A_{\max}}{\zeta}.
\]

Since $\mathcal V_i^k(s)\le\sigma^2$ for every $i,s,k$, the predictable
quadratic variations in \eqref{eq:online-global-Y-variance} and
\eqref{eq:online-global-V-variance} satisfy,
\begin{align*}
\sum_{t<T}
\mathbb E_t\!\left[
\bigl(
d_\mu^{\pi^t}(S^t)
Y_i^{N_t(S^t)}(S^t)
\bigr)^2
\right]
&\le
c\omega^2M^2\sigma^2T,\\
\sum_{t<T}
\mathbb E_t\!\left[
\bigl(
d_\mu^{\pi^t}(S^t)
V_i^{N_t(S^t)}(S^t)
\bigr)^2
\right]
&\le
c\omega^4\sigma^2\frac{A_{\max}}{\zeta}T.
\end{align*}
Therefore, applying Freedman's inequality\footnote{If $(X_t,\mathcal F_{t+1})$ is a martingale
difference sequence with $|X_t|\le b$ a.s. and
$\sum_{t<T}\mathbb E_t[X_t^2]\le v$ deterministically, then, with probability
at least $1-\rho$, $\sum_{t<T}X_t
\le c\bigl(\sqrt{v\log(1/\rho)}+b\log(1/\rho)\bigr)$, for a universal constant $c$.}
to each of the two martingales for each $i\in[n]$ with failure probability
$\rho/(4n)$, and taking a union bound, yields an event $\mathcal T_T$ with
$\mathbb P(\mathcal T_T)\ge1-\rho/2$. Since
$\log(4n/\rho)\le\Lambda_T$, the predictable quadratic-variation and
deterministic increment bounds above imply that, on $\mathcal T_T$,
simultaneously for all players,
\begin{align}\label{eq:online-global-Y-V-Freedman}
\sum_{t<T}d_\mu^{\pi^t}(S^t)
Y_i^{N_t(S^t)}(S^t)
&\le
c\omega M \sigma\sqrt{T\Lambda_T}
+
c\omega M \Lambda_T\sqrt{\frac{A_{\max}}{\zeta}},\cr 
\sum_{t<T}d_\mu^{\pi^t}(S^t)
V_i^{N_t(S^t)}(S^t)
&\le
c\omega^2\sigma
\sqrt{\frac{A_{\max}T\Lambda_T}{\zeta}}
+
c\omega^2\Lambda_T\frac{A_{\max}}{\zeta}.
\end{align}

\subsubsection*{Step 2: weighted telescoping and occupancy-weight drift}

We first control the weighted telescoping term that will arise from tracking error. For fixed
$i,s$, summation by parts and \eqref{eq:online-clipping-nonexpansive} give
\begin{align}\label{eq:telescop-k-e}
&\sum_{k:\tau_k(s)<T}\!\!d_\mu^k(s)
\bigl(\mathcal E_i^{k-1}(s)-\mathcal E_i^k(s)\bigr)\cr 
&\qquad\qquad\leq 
d_\mu^0(s)\mathcal E_i^{-1}(s)\mathbf1\{\tau_0(s)<T\}
+\!\!\!\!\!\!\sum_{k\ge1:\tau_k(s)<T}
\!\!\!\!\bigl(d_\mu^k(s)-d_\mu^{k-1}(s)\bigr)
\mathcal E_i^{k-1}(s)\cr 
&\qquad\qquad\le
4M^2d_\mu^0(s)\mathbf1\{\tau_0(s)<T\}
+4M^2\sum_{k\ge1:\tau_k(s)<T}
|d_\mu^k(s)-d_\mu^{k-1}(s)|.
\end{align}

For a fixed state $s$, consider an intervisit interval
$[\tau_{k-1}(s),\tau_k(s))$. During each such interval, the
local policy at $s$ is not updated, but the global policy profile can still
change as other states are visited and their policies are updated.
Consequently, even though $s$ itself is not visited, its occupancy
$d_\mu^{\pi^t}(s)$ can change over global time. Thus, the total change in the
occupancy of $s$ between two consecutive visits is the telescoping sum of all
global-time occupancy changes occurring in between. In particular, using the shorthand
$d_\mu^k(s):=d_\mu^{\pi^{\tau_k(s)}}(s)$ and the triangle inequality,
\begin{align*}
|d_\mu^k(s)-d_\mu^{k-1}(s)|
&=
\bigg|
\sum_{t=\tau_{k-1}(s)}^{\tau_k(s)-1}
\left(
d_\mu^{\pi^{t+1}}(s)-d_\mu^{\pi^t}(s)
\right)
\bigg|\le
\sum_{t=\tau_{k-1}(s)}^{\tau_k(s)-1}
\Big|
d_\mu^{\pi^{t+1}}(s)-d_\mu^{\pi^t}(s)
\Big|.
\end{align*}
Since the intervisit intervals are disjoint for each fixed state, by summing the above relation we get
\begin{align}\label{eq:online-weight-drift-telescope}
\sum_s\sum_{k\ge1:\tau_k(s)<T}
|d_\mu^k(s)-d_\mu^{k-1}(s)|
&\le
\sum_{t<T}\sum_s
\left|
d_\mu^{\pi^{t+1}}(s)
-d_\mu^{\pi^t}(s)
\right|=
2\sum_{t<T}
\|d_\mu^{\pi^{t+1}}-d_\mu^{\pi^t}\|_{\mathrm{TV}}.
\end{align}

At global time $t$, the two policy profiles $\pi^t$ and $\pi^{t+1}$
differ only at $S^t$. Indeed, setting $S^t=s$ and $N_t(s)=k$, we have
$N_{t+1}(s)=k+1$, and hence $\pi_j^t(\cdot\mid S^t)=\pi_j^k(\cdot\mid s), \pi_j^{t+1}(\cdot\mid S^t)=\pi_j^{k+1}(\cdot\mid s)$. Thus, the local policy-increment bound from
Corollary~\ref{cor:local-kl-projected-geometry} applies to the update at
global time $t$, i.e.,
\begin{align}\label{eq:policy-drift-tracking}
\|\pi_j^{t+1}(\cdot\mid S^t)-\pi_j^t(\cdot\mid S^t)\|_{\mathrm{TV}}
\le \frac{\eta M}{1-\gamma}.
\end{align}
Moreover, to bound the right-hand side of \eqref{eq:online-weight-drift-telescope}, we can write
\begin{align}\label{eq:online-total-occupancy-drift}
\sum_{t<T}
\|d_\mu^{\pi^{t+1}}-d_\mu^{\pi^t}\|_{\mathrm{TV}}
&\le
L_d\sum_{t<T}d_\mu^{\pi^t}(S^t)
\sum_j\|\pi_j^{t+1}(\cdot\mid S^t)
-\pi_j^t(\cdot\mid S^t)\|_{\mathrm{TV}}\cr
&\le
cL_d\sqrt n
\sum_{t<T}
d_\mu^{\pi^t}(S^t)
\left(
\sum_j\mathcal D_j^{N_t(S^t)}(S^t)
\right)^{1/2}\cr
&=
cL_d\sqrt n
\sum_{t<T}
\sqrt{d_\mu^{\pi^t}(S^t)}
\left(
 d_\mu^{\pi^t}(S^t)
 \sum_j\mathcal D_j^{N_t(S^t)}(S^t)
\right)^{1/2}\cr
&\le
cL_d\sqrt{nT}
\left(
\sum_{t<T}
d_\mu^{\pi^t}(S^t)
\sum_j\mathcal D_j^{N_t(S^t)}(S^t)
\right)^{1/2}\cr 
&=cL_d\sqrt{nT\mathfrak D_T},
\end{align}
where the first inequality follows from
\eqref{eq:online-localized-occupancy-sensitivity}, since
$\pi^t$ and $\pi^{t+1}$ differ only at the visited state $S^t$. The second inequality uses Pinsker's inequality $\|\Delta^k_j(s)\|_1\leq \sqrt{\mathcal{D}^k_j(s)}$, and then applying Cauchy--Schwarz over the players. Finally, the last
inequality follows from Cauchy--Schwarz over global time and using
$\sum_{t<T}d_\mu^{\pi^t}(S^t)\le T$.

Similarly, for the first-visit boundary, we can write
\begin{align}\label{eq:online-first-visit-weight}
\sum_{s:\tau_0(s)<T}d_\mu^0(s)
&=\sum_{s:\tau_0(s)<T}
\Big[
d_\mu^{\pi^0}(s)
+\sum_{t=0}^{\tau_0(s)-1}
\big(d_\mu^{\pi^{t+1}}(s)-d_\mu^{\pi^t}(s)\big)
\Big]\cr
&\le
\sum_s d_\mu^{\pi^0}(s)
+\sum_{t<T}\sum_s
\left|d_\mu^{\pi^{t+1}}(s)-d_\mu^{\pi^t}(s)\right|\cr 
&=1+2\sum_{t<T}
\|d_\mu^{\pi^{t+1}}-d_\mu^{\pi^t}\|_{\mathrm{TV}}.
\end{align}

Finally, by summing \eqref{eq:telescop-k-e} over all $i, s$, and using \eqref{eq:online-weight-drift-telescope}, \eqref{eq:online-total-occupancy-drift}, and \eqref{eq:online-first-visit-weight}, we obtain
\begin{align}\label{eq:sum-d-e-t}
\sum_i\sum_{(s,k):\tau_k(s)<T}\!d_\mu^k(s)
\bigl(\mathcal E_i^{k-1}(s)-\mathcal E_i^k(s)\bigr)\leq cnM^2+cL_d n^{3/2}M^2\sqrt{T\mathfrak D_T}.
\end{align}

We now apply this bound to the error recursion. Multiplying \eqref{eq:online-one-step-tracking-energy} at each actual visit $t=\tau_k(s)$ by $d_\mu^k(s)=d^{\pi^t}_{\mu}(S^t)$, rearranging its contraction
term, and summing over $i,s,k$ with $\tau_k(s)<T$ gives\footnote{ $\sum_{s}\sum_{k:\tau_k(s)<T}d_\mu^k(s)Y_i^k(s)=\sum_{t<T}d_\mu^{\pi^t}(S^t)Y_i^{N_t(S^t)}(S^t)$, since there is a one-to-one correspondence between pairs $(s,k)$ such that $\tau_k(s)<T$ and global times $t<T$, given by $t=\tau_k(s)$, or equivalently $S^t=s$ and $N_t(S^t)=k$.}
\begin{align*}
c_1\omega
&\sum_i\sum_{(s,k):\tau_k(s)<T}
d_\mu^k(s)\mathcal E_i^{k-1}(s)\cr 
&\le
\sum_i\sum_{(s,k):\tau_k(s)<T}
\!\!d_\mu^k(s)
\bigl(\mathcal E_i^{k-1}(s)-\mathcal E_i^k(s)\bigr)+cn\omega L_{\widehat{A}}^2
\sum_{(s,k):\tau_k(s)<T}d_\mu^k(s)\cr
&\quad+\frac{c}{\omega}
\sum_i\sum_{(s,k):\tau_k(s)<T}
\!\!d_\mu^k(s)
\sum_{a_i}\pi_i^k(a_i\mid s)
\bigl(\Delta A_i^k(s,a_i)\bigr)^2\cr
&\quad+\sum_i\sum_{t<T}
d_\mu^{\pi^t}(S^t)
Y_i^{N_t(S^t)}(S^t)+
\sum_i\sum_{t<T}
d_\mu^{\pi^t}(S^t)
V_i^{N_t(S^t)}(S^t)\cr
&\quad+
\omega^2\sum_i\sum_{t<T}d_\mu^{\pi^t}(S^t)\mathbb E_t\!\bigg[
\sum_{a_i}\pi_i^t(a_i\mid S^t)
\bigl(\xi_i^{N_t(S^t)}(S^t,a_i)\bigr)^2
\bigg].
\end{align*}
Now by definition \eqref{eq:online-occupancy-weighted-variance}, we have $d_\mu^{\pi^t}(S^t)\mathbb E_t\!\big[
\sum_{a_i}\pi_i^t(a_i\mid S^t)
\bigl(\xi_i^{N_t(S^t)}(S^t,a_i)\bigr)^2
\big]\!=\!\mathcal V_i^{N_t(S^t)}(S^t)
\le \sigma^2$. Moreover, by the global-time reindexing, $\sum_{(s,k):\tau_k(s)<T}d_\mu^k(s) =\sum_{t<T}d_\mu^{\pi^t}(S^t)
\le T$. Using these relations together with \eqref{eq:sum-d-e-t} in the above inequality and dividing both sides by $c_1\omega$, we obtain 
\begin{align}\label{eq:online-weighted-tracking-preliminary}
&\sum_i\sum_{(s,k):\tau_k(s)<T}
d_\mu^k(s)\mathcal E_i^{k-1}(s)\cr
&\quad\le
c\frac{nM^2}{\omega}
+c\frac{L_dn^{3/2}M^2}{\omega}\sqrt{T\mathfrak D_T}
+cnL_{\widehat{A}}^2T
+cn\omega\sigma^2T\cr
&\qquad+
\frac{c}{\omega^2}
\sum_i\sum_{(s,k):\tau_k(s)<T}
d_\mu^k(s)
\sum_{a_i}\pi_i^k(a_i\mid s)
\bigl(\Delta A_i^k(s,a_i)\bigr)^2\cr
&\qquad+
\frac{c}{\omega}
\sum_i\sum_{t<T}
d_\mu^{\pi^t}(S^t)
Y_i^{N_t(S^t)}(S^t)
+
\frac{c}{\omega}
\sum_i\sum_{t<T}
d_\mu^{\pi^t}(S^t)
V_i^{N_t(S^t)}(S^t).
\end{align}

It remains to control the target-drift term and then substitute the martingale bounds from Step~1.

\subsubsection*{Step 3: target drift on the bounded-intervisit event}
Assume now $\mathcal H_T(H_T)$. For $k\ge1$, by the definition of
$\Delta A_i^k(s)$ and the shorthand
$\bar A_i^k(s,\cdot)
=\bar A_i^{\pi^{\tau_k(s)}}(s,\cdot)$, telescoping \eqref{eq:online-advantage-sensitivity} over the global updates between two
successive visits to $s$ gives
\begin{align*}
\|\Delta A_i^k(s)\|_\infty
&=
\|\bar A_i^{k-1}(s,\cdot)
-\bar A_i^{k}(s,\cdot)\|_\infty\le
\sum_{t=\tau_{k-1}(s)}^{\tau_k(s)-1}
\left\|
\bar A_i^{\pi^{t+1}}(s,\cdot)
-\bar A_i^{\pi^t}(s,\cdot)
\right\|_\infty\\
&\le
L_{\bar A}\sum_{t=\tau_{k-1}(s)}^{\tau_k(s)-1}
\sum_j\max_x
\|\pi_j^{t+1}(\cdot\mid x)-\pi_j^t(\cdot\mid x)\|_{\mathrm{TV}}\\
&=
L_{\bar A}\sum_{t=\tau_{k-1}(s)}^{\tau_k(s)-1}
\sum_j
\|\pi_j^{t+1}(\cdot\mid S^t)-\pi_j^t(\cdot\mid S^t)\|_{\mathrm{TV}},
\end{align*}
where the final equality holds because at
each global time $t$, $\pi^t$ and $\pi^{t+1}$ differ only at the visited
state $S^t$. Since the interval length is at most $H_T$, Cauchy--Schwarz yields
\begin{align}\label{eq:online-target-drift-cs}
\!\!\sum_i\sum_{a_i}\pi_i^k(a_i\mid s)(\Delta A_i^k(s,a_i))^2
&\le nL_{\bar A}^2 H_T
\!\!\sum_{t=\tau_{k-1}(s)}^{\tau_k(s)-1}
\!\bigg(\sum_j\|\pi_j^{t+1}(\cdot\mid S^t)-
\pi_j^t(\cdot\mid S^t)\|_{\mathrm{TV}}\bigg)^2\!.
\end{align}
Multiplying \eqref{eq:online-target-drift-cs} by $d_\mu^k(s)$ and summing
over all pairs $(s,k)$ with $k\ge1$ and $\tau_k(s)<T$, and recalling that $\Delta A_i^0(s)=0$, we have
\begin{align}\label{eq:advantage-drift-bound}
&\sum_{(s,k):\tau_k(s)<T}
d_\mu^k(s)
\sum_i\sum_{a_i}\pi_i^k(a_i\mid s)
\bigl(\Delta A_i^k(s,a_i)\bigr)^2
\cr
&\quad\le
nL_{\bar A}^2 H_T
\!\!\!\!\!\sum_{(s,k):k\ge1,\tau_k(s)<T}
\ \sum_{t=\tau_{k-1}(s)}^{\tau_k(s)-1}
d_\mu^k(s)\bigg(\!
\sum_j
\|\pi_j^{t+1}(\cdot\mid S^t)
-\pi_j^t(\cdot\mid S^t)\|_{\mathrm{TV}}
\!\bigg)^2
\cr
&\quad=
nL_{\bar A}^2 H_T
\!\!\!\!\!\sum_{(s,k):k\ge1,\tau_k(s)<T}\ 
\sum_{t<T}
\mathbf 1\{\tau_{k-1}(s)\le t<\tau_k(s)\}
d_\mu^k(s)\bigg(\!
\sum_j
\|\pi_j^{t+1}(\cdot\mid S^t)
-\pi_j^t(\cdot\mid S^t)\|_{\mathrm{TV}}
\!\bigg)^2
\cr
&\quad=
nL_{\bar A}^2 H_T
\sum_{t<T}
\bigg(
\sum_j
\|\pi_j^{t+1}(\cdot\mid S^t)
-\pi_j^t(\cdot\mid S^t)\|_{\mathrm{TV}}
\bigg)^2
\sum_{\substack{(s,k):\,k\ge1,\tau_{k-1}(s)\le t<\tau_k(s)<T}}
d_\mu^k(s)
\cr
&\quad\le
nL_{\bar A}^2 H_T
\sum_{t<T}
\bigg(
\sum_j
\|\pi_j^{t+1}(\cdot\mid S^t)
-\pi_j^t(\cdot\mid S^t)\|_{\mathrm{TV}}
\bigg)^2
\sum_{s:\,\tau_{N_{t+1}(s)}(s)<T}
d_\mu^{N_{t+1}(s)}(s)\cr
&\quad\leq \frac{n^3L_{\bar A}^2 H_T M^2\eta^2}{(1-\gamma)^2}\sum_{t<T} \ 
\sum_{s:\,\tau_{N_{t+1}(s)}(s)<T}
d_\mu^{N_{t+1}(s)}(s)
\end{align}
Here, the first equality
simply rewrites the intervisit sum as a global-time sum using the indicator
of the event $\tau_{k-1}(s)\le t<\tau_k(s)$. The second equality then
switches the order of summation. The second inequality is obtained by noting that for each fixed $t$ and $s$, there is at
most one intervisit interval containing $t$, and its ending visit has local
index $k=N_{t+1}(s)$, and then enlarging the sum to include the case $N_{t+1}(s)=0$. Finally, the last inequality is obtained by using \eqref{eq:policy-drift-tracking}.  

Since
$d_\mu^{N_{t+1}(s)}(s)
:=d_\mu^{\pi^{\tau_{N_{t+1}(s)}(s)}}(s)$
denotes the occupancy of state $s$ under the global policy at its first
visit strictly after time $t$ (i.e., at the global time $\tau_{N_{t+1}(s)}(s)$), on $\mathcal H_T(H_T)$ every such visit
occurring before $T$ is within $H_T$ steps of $t$. Hence, telescoping from
time $t$ to these next visits and then enlarging the state sums gives
\begin{align*}
\sum_{s:\,\tau_{N_{t+1}(s)}(s)<T}
d_\mu^{N_{t+1}(s)}(s)
&=
\sum_{s:\,\tau_{N_{t+1}(s)}(s)<T}
\Big[
d_\mu^{\pi^t}(s)
+
\sum_{r=t}^{\tau_{N_{t+1}(s)}(s)-1}
\left(
d_\mu^{\pi^{r+1}}(s)-d_\mu^{\pi^r}(s)
\right)
\Big]\cr
&\le
\sum_{s:\,\tau_{N_{t+1}(s)}(s)<T}
\Big[
d_\mu^{\pi^t}(s)
+
\sum_{r=t}^{\tau_{N_{t+1}(s)}(s)-1}
\left|
d_\mu^{\pi^{r+1}}(s)-d_\mu^{\pi^r}(s)
\right|
\Big]\cr
&\le
\sum_s d_\mu^{\pi^t}(s)
+
\sum_{r=t}^{\min\{T-1,t+H_T-1\}}
\sum_s
\left|
d_\mu^{\pi^{r+1}}(s)-d_\mu^{\pi^r}(s)
\right|\cr
&=
1+
2\sum_{r=t}^{\min\{T-1,t+H_T-1\}}
\left\|
d_\mu^{\pi^{r+1}}-d_\mu^{\pi^r}
\right\|_{\mathrm{TV}}.
\end{align*}
Substituting the above relation into \eqref{eq:advantage-drift-bound}, we obtain
\begin{align}\label{eq:online-drift-charging}
&\sum_{(s,k):\tau_k(s)<T}d_\mu^k(s)
\sum_i\sum_{a_i}\pi_i^k(a_i\mid s)(\Delta A_i^k(s,a_i))^2\cr
&\quad\qquad\qquad\le \frac{n^3L_{\bar A}^2H_T\eta^2M^2}{(1-\gamma)^2}T+c\frac{n^3L_{\bar A}^2H_T\eta^2M^2}{(1-\gamma)^2} \sum_{t<T}\sum_{r=t}^{\min\{T-1, t+H_T-1\}}
\left\|
d_\mu^{\pi^{r+1}}-d_\mu^{\pi^r}
\right\|_{\mathrm{TV}}\cr
&\quad\qquad\qquad\le \frac{n^3L_{\bar A}^2H_T\eta^2M^2}{(1-\gamma)^2}T+c\frac{n^3L_{\bar A}^2H^2_T\eta^2M^2}{(1-\gamma)^2} \sum_{r<T}\left\|
d_\mu^{\pi^{r+1}}-d_\mu^{\pi^r}
\right\|_{\mathrm{TV}}
\cr
&\quad\qquad\qquad\le
\frac{n^3L_{\bar A}^2H_T\eta^2M^2}{(1-\gamma)^2}T
+c\frac{n^{7/2}L_{\bar A}^2L_d H_T^2\eta^2M^2}
{(1-\gamma)^2}\sqrt{T\mathfrak D_T}
\end{align}
where the second inequality is obtained by switching the order of summation and noting that each global time $r$ belongs to at most $H_T$ windowed occupancy-drift terms, while the last inequality follows from relation \eqref{eq:online-total-occupancy-drift}.

\subsubsection*{Step 4: summation and passage from lagged to current errors}

On the event $\mathcal T_T\cap\mathcal H_T(H_T)$, substitute \eqref{eq:online-global-Y-V-Freedman} and \eqref{eq:online-drift-charging} 
into the preliminary bound
\eqref{eq:online-weighted-tracking-preliminary}. This gives
\begin{align}\label{eq:online-lagged-tracking-closed}
&\sum_{t<T}d_\mu^{\pi^t}(S^t)\sum_i
\mathcal E_i^{N_t(S^t)-1}(S^t)
\notag\\
&\quad\le c\Bigg[
\frac{nM^2}{\omega}
+n\omega\sigma^2T
+nL_{\widehat{A}}^2T
+nM\sigma\sqrt{T\Lambda_T}
+\frac{nM\sqrt{A_{\max}}\,\Lambda_T}{\sqrt\zeta}
+n\omega \sigma\sqrt{\frac{A_{\max}T\Lambda_T}{\zeta}}\cr
&\qquad
+\frac{n\omega A_{\max}\Lambda_T}{\zeta}+\frac{n^3L_{\bar A}^2H_T\eta^2M^2}
{(1-\gamma)^2\omega^2}T
+\Big(\frac{L_dn^{3/2}M^2}{\omega}+\frac{n^{7/2}L_{\bar A}^2L_d H_T^2\eta^2M^2}
{(1-\gamma)^2\omega^2}\Big)
\sqrt{T\mathfrak D_T}
\Bigg].
\end{align}
Here, the global-time representation on the left follows from the
one-to-one correspondence between $(s,k)$ with $\tau_k(s)<T$ and global
times $t<T$, with the convention $\mathcal E_i^{-1}(s)$ at a first visit.

It remains to pass from the lagged to the current tracking errors. For
brevity, let
\[
L_T:=
\sum_{t<T}d_\mu^{\pi^t}(S^t)\sum_i
\mathcal E_i^{N_t(S^t)-1}(S^t),
\]
and let $R_T$ denote the sum of all the remainder terms obtained by
multiplying \eqref{eq:online-one-step-tracking-energy} by $d_\mu^k(s)$ at
each actual visit $t=\tau_k(s)$ and summing over all visits and players.
Then
\begin{equation}\label{eq:online-current-from-lagged}
\mathfrak E_T
\le
(1-c_1\omega)L_T+R_T,
\end{equation}
because the left-hand side reindexes as the current weighted tracking error
$\mathfrak E_T$, whereas the contraction term reindexes as the lagged
weighted tracking error $L_T$. The estimates in Steps~1--3 control the terms in $R_T$ before division by
the contraction factor $c_1\omega$. More precisely, the oracle-bias,
conditional-variance, target-drift, and martingale terms in $R_T$ are,
up to universal numerical constants, respectively $\omega$ times their
corresponding terms on the right-hand side of
\eqref{eq:online-lagged-tracking-closed}. The occupancy-drift and initial
boundary terms appearing in \eqref{eq:online-lagged-tracking-closed} arise
from the weighted telescoping term and hence do not appear in $R_T$.
Therefore, since $\omega\le1/2$, the bound on $L_T$ in
\eqref{eq:online-lagged-tracking-closed}, together with
\eqref{eq:online-current-from-lagged}, implies that
$\mathfrak E_T$ is bounded, up to a universal numerical constant, by the
right-hand side of \eqref{eq:online-lagged-tracking-closed}. Dividing by
$T$ and collecting terms proves the desired bound
\eqref{eq:online-E-closed}.
\end{proof}

\subsection{Proof of Lemma \ref{lem:online-coupling-transfer}}\label{app:coupling-lemma}
\begin{proof}
We first establish the bounded-intervisit event. At any global time
$\ell$, setting $x=S^\ell$ and $k=N_\ell(x)$, the only statewise policy
that changes from $\pi^\ell$ to $\pi^{\ell+1}$ is the policy stored at
$x$, and this change is precisely the local update
$\pi_i^k(\cdot\mid x)\to\pi_i^{k+1}(\cdot\mid x)$ for every $i$. Hence, by
Corollary~\ref{cor:local-kl-projected-geometry},
\begin{equation}\label{eq:online-actual-policy-movement}
\sum_{i=1}^n
\left\|
\pi_i^{N_\ell(S^\ell)+1}(\cdot\mid S^\ell)
-
\pi_i^{N_\ell(S^\ell)}(\cdot\mid S^\ell)
\right\|_{\mathrm{TV}}
\le \frac{n\eta M}{1-\gamma}.
\end{equation}

If the policy is fixed, the state process evolves as a time-homogeneous
Markov chain. Thus, under Assumption~\ref{ass:coverage}, we have a
uniform lower bound $p_{\min}$ on the probability of visiting any
prescribed state $s$ within $H_{\rm cov}$ steps, which in turn provides
control of the intervisit times to each state. However, the policies
generated by Algorithm~\ref{alg:fully-online-general-occupancy} are
time-varying, so the resulting state process is a
\emph{time-inhomogeneous} Markov chain, and the coverage assumption
cannot be applied directly. To transfer the coverage guarantee to the
actual dynamics, we couple the state process over each
$H_{\rm cov}$-step window with an auxiliary time-homogeneous process
whose policy is frozen at the beginning of the window. Since the policy changes slowly by only $O(\eta)$ at each iteration, the
transition sensitivity assumption ensures that the two processes remain close over the entire
window. Consequently, the probability of visiting any target state within
the window under the actual dynamics remains close to that under the
auxiliary dynamics. Repeating this argument over successive windows yields
the desired high-probability control of the intervisit times.

More formally, fix a global time $t$ and condition on $\mathcal F_t$.
Couple the actual trajectory during the next $H_{\rm cov}$ steps with an
auxiliary trajectory that starts from the same state $S^t$ but evolves
under the policy profile $\pi^t$ frozen at time $t$. The coupling is constructed recursively as follows. Whenever the two trajectories have followed the same state sequence through time $t+r$ and their common state at that time is $x$, we maximally couple their next-state distributions. Conditional on $\mathcal F_{t+r}$ and on the event that the frozen-policy trajectory has matched the actual trajectory through time $t+r$, with common state $x$, the next-state distributions of the actual and frozen-policy trajectories are $\bar P^{\pi^{t+r}}(\cdot\mid x)$ and $\bar P^{\pi^t}(\cdot\mid x)$, respectively. Hence, maximal coupling gives a conditional probability of separation at the next step equal to $\|
\bar P^{\pi^{t+r}}(\cdot\mid x)
-
\bar P^{\pi^t}(\cdot\mid x)\|_{\mathrm{TV}}
$. Using \eqref{eq:online-localized-transition-sensitivity} from Lemma \ref{lemm:occupancy-sensitivity},
\begin{equation}\label{eq:online-coupled-transition-tv}
\left\|
\bar P^{\pi^{t+r}}(\cdot\mid x)
-
\bar P^{\pi^t}(\cdot\mid x)
\right\|_{\mathrm{TV}}
\le \delta_P
\sum_{i=1}^n
\left\|
\pi_i^{t+r}(\cdot\mid x)
-
\pi_i^t(\cdot\mid x)
\right\|_{\mathrm{TV}}.
\end{equation}

Moreover, the policy stored at state $x$ changes only when the actual
trajectory visits $x$. Therefore, telescoping over the successive
updates made at $x$ between global times $t$ and $t+r-1$ and using the
triangle inequality yield
\begin{align}
\left\|
\pi_i^{t+r}(\cdot\mid x)
-
\pi_i^t(\cdot\mid x)
\right\|_{\mathrm{TV}}
&\le
\sum_{\ell=t}^{t+r-1}
\mathbf 1\{S^\ell=x\}
\left\|
\pi_i^{N_\ell(x)+1}(\cdot\mid x)
-
\pi_i^{N_\ell(x)}(\cdot\mid x)
\right\|_{\mathrm{TV}}.
\label{eq:online-coupled-policy-telescope}
\end{align}
Combining \eqref{eq:online-coupled-transition-tv}, \eqref{eq:online-coupled-policy-telescope} with
\eqref{eq:online-actual-policy-movement}, the conditional probability
of separation at the next step, given that the two trajectories have
followed the same state sequence through time $t+r$, is bounded by
\begin{align}
\left\|
\bar P^{\pi^{t+r}}(\cdot\mid x)
-
\bar P^{\pi^t}(\cdot\mid x)
\right\|_{\mathrm{TV}}
&\le
\delta_P
\sum_{\ell=t}^{t+r-1}
\mathbf 1\{S^\ell=x\}
\sum_{i=1}^n
\left\|
\pi_i^{N_\ell(x)+1}(\cdot\mid x)
-
\pi_i^{N_\ell(x)}(\cdot\mid x)
\right\|_{\mathrm{TV}}
\notag\\
&\le\delta_P\frac{n\eta M}{1-\gamma}
\sum_{\ell=t}^{t+r-1}
\mathbf 1\{S^\ell=x\}\cr 
&\le\frac{n\eta M\delta_P}{1-\gamma}r
\nonumber
\end{align}
Therefore, summing over the possible first separation times and using
the tower property together with the maximal-coupling identity, we can write
\begin{align}
&
\mathbb P_t
\Big(
\text{the actual and frozen trajectories separate within the next }
H_{\mathrm{cov}}
\text{ steps}
\Big)\cr
&=
\sum_{r=0}^{H_{\mathrm{cov}}-1}
\mathbb P_t
\Bigg(
\begin{array}{c}
\text{the two trajectories have followed the same state sequence}\\
\text{through time }t+r\text{ and separate at time }t+r+1
\end{array}
\Bigg)\cr
&=
\sum_{r=0}^{H_{\mathrm{cov}}-1}
\!\mathbb E_t
\Bigg[
\mathbf 1
\left\{\!\!
\begin{array}{c}
\text{the two trajectories have followed the}\\
\text{same state sequence through time }t+r
\end{array}
\!\!\right\}
\left\|
\bar P^{\pi^{t+r}}(\cdot\mid S^{t+r})
-
\bar P^{\pi^t}(\cdot\mid S^{t+r})
\right\|_{\mathrm{TV}}
\Bigg]\cr
&\le\frac{n\eta M\delta_P}{1-\gamma}
\sum_{r=0}^{H_{\rm cov}-1}r
\le
\frac{n \eta M \delta_P}{1-\gamma}H_{\rm cov}^2
\le
\frac{p_{\min}}2,
\label{eq:online-coupling-separation-window}
\end{align}
where the last inequality follows directly from
\eqref{eq:online-coupling-condition}.

Under the frozen policy $\pi^t$, every prescribed state $s\in \mathcal{S}$ is visited
within $H_{\rm cov}$ steps with probability at least $p_{\min}$.
Therefore, by \eqref{eq:online-coupling-separation-window}, under the
actual trajectory every prescribed state $s\in \mathcal{S}$ is also visited within
the same window with conditional probability at least $p_{\min}/2$.
Since this estimate is conditional on the entire history at the beginning
of the window and is uniform in the current policy, the argument can be
restarted after each unsuccessful block, with the next intervisit block counting only visits strictly after its left endpoint. Thus, for every
$s\in\mathcal S$, $k\ge0$, and $m\ge1$,
\begin{equation}\label{eq:online-intervisit-geometric-tail}
\mathbb P_{\tau_k(s)}\Big(
\tau_{k+1}(s)-\tau_k(s)>mH_{\rm cov}
\Big)
\le
\left(1-\frac{p_{\min}}2\right)^m
\le
e^{-mp_{\min}/2}.
\end{equation}
Starting from the initial state and applying the same block argument also
gives
\[
\mathbb P\left(
\tau_0(s)>mH_{\rm cov}
\right)
\le
e^{-mp_{\min}/2}.
\]
The bound \eqref{eq:online-intervisit-geometric-tail} controls each
intervisit interval individually. Since at each global time $t<T$ the actual trajectory occupies exactly one state, each realized time before $t$ can serve as the left endpoint of at most one intervisit interval. Hence, there are at most \(T\) such intervisit intervals in total. Together with the \(|\mathcal S|\) first hitting times, a union bound gives
\[
\mathbb P\left(
\max_{s\in\mathcal S}\tau_0(s)>mH_{\rm cov}
\ \text{or}\
\max_{s,k:\tau_k(s)<T}
\{\tau_{k+1}(s)-\tau_k(s)\}>mH_{\rm cov}
\right)
\le
(|\mathcal S|+T)e^{-mp_{\min}/2}.
\]
Taking $m=\lceil
\frac{2}{p_{\min}}
\log(\frac{2(|\mathcal S|+T)}{\rho})\rceil$ and recalling $H_T=mH_{\rm cov}$ proves $\mathbb P\bigl(\mathcal H_T(H_T)\bigr)
\ge1-\frac{\rho}{2}$. In particular, on $\mathcal H_T(H_T)$, every state has been visited by
global time $H_T$, and hence $N_t(s)\ge1$ for every $s\in \mathcal{S}$ and $t\ge H_T+1$. We work on $\mathcal H_T(H_T)$ throughout the rest of the proof.

We next use a delayed charging argument to bound the quantities \eqref{eq:online-delayed-D-charging} and \eqref{eq:online-delayed-E-charging} that are summed over all states by the corresponding quantities $\mathfrak D_T$ and $\mathfrak E_T$, which are evaluated only along the
realized online trajectory. For every pair
$(s,t)\in\mathcal S\times[H_T+1,T-1]$, let $\ell=\tau_{N_t(s)-1}(s)$ be the most recent visit to $s$ strictly before time $t$. We charge the
$(s,t)$ term to this global time $\ell$ on the realized trajectory. More precisely, for each global time $\ell$, define the charging set
\[
\mathcal C_\ell
:=
\left\{
(s,t)\in\mathcal S\!\times\![H_T+1,T-1]: \ 
\tau_{N_t(s)-1}(s)=\ell
\right\}.
\]
Thus, $\mathcal C_\ell$ consists of the pairs $(s,t)$ whose most recent
visit to $s$ strictly before time $t$ occurs at global time $\ell$. By
construction, if $(s,t)\in\mathcal C_\ell$, then $s=S^\ell$ and
$N_t(s)-1=N_\ell(S^\ell)$. Therefore,\footnote{Here,
$d_\mu^{N_t(s)-1}(s)$ is shorthand for
$d_\mu^{\pi^{\tau_{N_t(s)-1}(s)}}(s)$, i.e., the occupancy weight of state
$s$ under the global policy profile at its most recent visit strictly
before time $t$.}
\begin{align}
&\sum_{t=H_T+1}^{T-1}\sum_{s\in\mathcal S}
d_\mu^{N_t(s)-1}(s)
\sum_i\mathcal D_i^{N_t(s)-1}(s)
\notag\\
&\quad=
\sum_{\ell=0}^{T-1}
\sum_{(s,t)\in\mathcal C_\ell}
d_\mu^{N_t(s)-1}(s)
\sum_i\mathcal D_i^{N_t(s)-1}(s)
\notag\\
&\quad=
\sum_{\ell=0}^{T-1}
\sum_{\substack{
t:\,H_T+1\le t<T,\\
N_t(S^\ell)-1=N_\ell(S^\ell)}}
d_\mu^{N_\ell(S^\ell)}(S^\ell)
\sum_i\mathcal D_i^{N_\ell(S^\ell)}(S^\ell)
\notag\\
&\quad=
\sum_{\ell=0}^{T-1}
\left|
\left\{
t:\,H_T+1\le t<T,\;
N_t(S^\ell)-1=N_\ell(S^\ell)
\right\}
\right|
d_\mu^{N_\ell(S^\ell)}(S^\ell)
\sum_i\mathcal D_i^{N_\ell(S^\ell)}(S^\ell)
\notag\\
&\quad\le
H_T\sum_{\ell=0}^{T-1}
d_\mu^{N_\ell(S^\ell)}(S^\ell)
\sum_i\mathcal D_i^{N_\ell(S^\ell)}(S^\ell)
=
H_T\mathfrak D_T.
\label{eq:online-global-D-charging-main}
\end{align}
The first equality partitions the original pairs $(s,t)$ according to
their unique charging time $\ell$. In the second equality,
$s=S^\ell$ and $N_t(s)-1=N_\ell(S^\ell)$ for every
$(s,t)\in\mathcal C_\ell$, so the state coordinate is fixed and the
summand is precisely the corresponding online quantity at time $\ell$.
The third equality simply counts the times $t$ charged to $\ell$. These
times lie in the intervisit interval
$\bigl(\tau_{N_\ell(S^\ell)}(S^\ell),
\tau_{N_\ell(S^\ell)+1}(S^\ell)\bigr]$, which contains at most $H_T$
global times on $\mathcal H_T(H_T)$. This gives the inequality, while the
last equality follows from the definition of $\mathfrak D_T$. 

Now, we can decompose \eqref{eq:online-delayed-D-charging} as
\begin{align}
\!\!\!\!\!\!\!\!\!\sum_{t=H_T+1}^{T-1}\sum_s
d_\mu^{\pi^t}(s)
\sum_i\mathcal D_i^{N_t(s)-1}(s)&=
\sum_{t=H_T+1}^{T-1}\sum_s
d_\mu^{N_t(s)-1}(s)
\sum_i\mathcal D_i^{N_t(s)-1}(s)\cr 
&+
\sum_{t=H_T+1}^{T-1}\sum_s
\left(
d_\mu^{\pi^t}(s)-d_\mu^{N_t(s)-1}(s)
\right)
\sum_i\mathcal D_i^{N_t(s)-1}(s).
\label{eq:online-global-D-decomposition}
\end{align}
It remains to control the transport of the occupancy weights. We can write
\begin{align}
\sum_{t=H_T+1}^{T-1}\sum_{s\in\mathcal S}
\left|
d_\mu^{\pi^t}(s)-d_\mu^{N_t(s)-1}(s)
\right|&=
\sum_{t=H_T+1}^{T-1}\sum_{s\in\mathcal S}
\left|
d_\mu^{\pi^t}(s)
-d_\mu^{\pi^{\tau_{N_t(s)-1}(s)}}(s)
\right|
\notag\\
&\quad\le
\sum_{t=H_T+1}^{T-1}\sum_{s\in\mathcal S}
\sum_{\ell=\tau_{N_t(s)-1}(s)}^{t-1}
\left|
d_\mu^{\pi^{\ell+1}}(s)-d_\mu^{\pi^\ell}(s)
\right|
\notag\\
&\quad\le
\sum_{t=H_T+1}^{T-1}
\sum_{\ell=(t-H_T)_+}^{t-1}
\sum_{s\in\mathcal S}
\left|
d_\mu^{\pi^{\ell+1}}(s)-d_\mu^{\pi^\ell}(s)
\right|
\notag\\
&\quad\le
2H_T\sum_{\ell=0}^{T-1}
\left\|
d_\mu^{\pi^{\ell+1}}-d_\mu^{\pi^\ell}
\right\|_{\mathrm{TV}}
\le
2\frac{L_d n\eta M}{1-\gamma}H_TT.
\label{eq:online-delayed-weight-transport}
\end{align}
The first inequality follows
by telescoping the occupancy difference from the most recent visit to
$s$ up to time $t$. On $\mathcal H_T(H_T)$,
$t-\tau_{N_t(s)-1}(s)\le H_T$, so the corresponding $\ell$-sum is
contained in $[(t-H_T)_+,t-1]$, which gives the second inequality.
After interchanging the order of summation, each global increment
$\ell$ can appear for at most $H_T$ values of $t$, while
$\sum_s|d_\mu^{\pi^{\ell+1}}(s)-d_\mu^{\pi^\ell}(s)|
=2\|d_\mu^{\pi^{\ell+1}}-d_\mu^{\pi^\ell}\|_{\rm TV}$, yielding the
third inequality. The final inequality follows from the occupancy-sensitivity bound \eqref{eq:online-localized-occupancy-sensitivity} and the policy-movement bound in part (iv) of Lemma~\ref{lem:kl-projected-local-geometry}.

Moreover, using \eqref{eq:projected-increment-upper} in view of Corollary~\ref{cor:local-kl-projected-geometry} gives
\begin{equation}\label{eq:online-uniform-D-bound}
\sum_{i=1}^n\mathcal D_i^{N_t(s)-1}(s)
\le \frac{n\eta^2M^2}{(1-\gamma)^2}.
\end{equation} 
By substituting \eqref{eq:online-global-D-charging-main}, \eqref{eq:online-delayed-weight-transport}, and \eqref{eq:online-uniform-D-bound} into \eqref{eq:online-global-D-decomposition}, we obtain the desired bound \eqref{eq:online-delayed-D-charging}. An identical argument as above by replacing $\mathcal{D}_i$ with $\mathcal{E}_i$ and using the uniform bound $\sum_{i=1}^n \mathcal E_i^{N_t(s)-1}(s)
\le 4nM^2$, which is due to clipping, we obtain \eqref{eq:online-delayed-E-charging}.

Finally, we control the stale advantage term. Telescoping the intervening global policy updates and using the triangle inequality, for every $x\in\mathcal S$, we can write
\begin{align}
&\sum_{j=1}^n
\left\|
\pi_j^t(\cdot\mid x)
-
\pi_j^{\tau_{N_t(s)-1}(s)}(\cdot\mid x)
\right\|_{\mathrm{TV}}\le
\sum_{\ell=\tau_{N_t(s)-1}(s)}^{t-1}
\sum_{j=1}^n
\left\|
\pi_j^{\ell+1}(\cdot\mid x)
-
\pi_j^\ell(\cdot\mid x)
\right\|_{\mathrm{TV}}
\notag\\
&\qquad=
\sum_{\ell=\tau_{N_t(s)-1}(s)}^{t-1}
\mathbf 1\{S^\ell=x\}
\sum_{j=1}^n
\left\|
\pi_j^{N_\ell(x)+1}(\cdot\mid x)
-
\pi_j^{N_\ell(x)}(\cdot\mid x)
\right\|_{\mathrm{TV}}\le
\frac{nH_T\eta M}{1-\gamma}.
\label{eq:online-stale-policy-distance}
\end{align}
where the equality follows because the policy at state $x$ changes at global
time $\ell$ only when $S^\ell=x$, in which case for every $j\in [n]$ the local policy is
updated from $\pi_j^{\ell}(\cdot\mid x)=\pi_j^{N_\ell(x)}(\cdot\mid x)$ to
$\pi_j^{\ell+1}(\cdot\mid x)=\pi_j^{N_\ell(x)+1}(\cdot\mid x)$, and the last inequality follows from
\eqref{eq:online-actual-policy-movement}, which bounds each local update
by $n\eta M/(1-\gamma)$, together with
$t-\tau_{N_t(s)-1}(s)\le H_T$ on $\mathcal H_T(H_T)$.

Applying the advantage-sensitivity bound
\eqref{eq:online-advantage-sensitivity} to policy profiles $\pi^t$ and
$\pi^{\tau_{N_t(s)-1}(s)}$, recalling the shorthand
$\bar A_i^{N_t(s)-1}
:=\bar A_i^{\pi^{\tau_{N_t(s)-1}(s)}}$, and using
\eqref{eq:online-stale-policy-distance}, we obtain, for every player $i$,
\begin{align}\nonumber
&\left\|
\bar A_i^{\pi^t}(s,\cdot)
-
\bar A_i^{N_t(s)-1}(s,\cdot)
\right\|_\infty=\left\|
\bar A_i^{\pi^t}(s,\cdot)
-
\bar{A}_i^{\pi^{\tau_{N_t(s)-1}(s)}}(s,\cdot)
\right\|_\infty\cr 
&\qquad\qquad\leq L_{\bar{A}} \max_{x\in \mathcal{S}}\sum_{j=1}^n
\left\|
\pi_j^t(\cdot\mid x)
-
\pi_j^{\tau_{N_t(s)-1}(s)}(\cdot\mid x)
\right\|_{\mathrm{TV}}\le \frac{L_{\bar A} nH_T\eta M}{1-\gamma}.
\end{align}
Squaring this inequality, multiplying by $d_\mu^{\pi^t}(s)$, summing over $s$ and $i$, and using $\sum_s d_\mu^{\pi^t}(s)=1$, we obtain the desired bound
\begin{align}\nonumber
\left\{
\sum_s d_\mu^{\pi^t}(s)
\sum_{i=1}^n
\left\|
\bar A_i^{\pi^t}(s,\cdot)
-
\bar A_i^{N_t(s)-1}(s,\cdot)
\right\|_\infty^2
\right\}^{1/2}\le
\frac{L_{\bar A} n^{3/2}H_T\eta M}{1-\gamma}.&&&&&\hfill\qedhere
\end{align}
\end{proof}

\subsection{Proof of Theorem \ref{thm:fully-online-general-ne-regret}}\label{app:online-final-proof}

\begin{proof}
Recall that Algorithm~\ref{alg:fully-online-general-occupancy} chooses
$M=1+L_d$. We first consider the nontrivial regime in which
the three caps in \eqref{eq:online-final-tuning} are inactive,
$u_T\le1$, and $H_T\le T$.\footnote{The complementary cases will be handled at the end of the proof.} Thus,
\begin{equation}\label{eq:online-final-tuning-active}
\eta
=
c_\eta
\frac{1-\gamma}{nL_{\bar A}}
H_T^{-1/2}u_T^{3/5},
\qquad
\omega
=
c_\omega H_T^{1/4}u_T^{2/5},
\qquad
\zeta
=
c_\zeta
\max\left\{
H_T^{1/2}u_T^{2/15},
L_{\widehat A}^{2/3}
\right\}.
\end{equation}

\medskip
\noindent\textbf{Feasibility of the preceding lemmas.}
By construction, $\omega,\zeta\le1/4$. Moreover,
\[
\frac{\eta M}{1-\gamma}
=
\frac{c_\eta(1+L_d)}{nL_{\bar A}}
H_T^{-1/2}u_T^{3/5}
\le
c_0\omega
\]
for sufficiently small $c_\eta/c_\omega$, since $H_T\ge1$,
$u_T\le1$, and $(1+L_d)/L_{\bar A}=(1-\gamma)/4\le1/4$. Also,
\begin{align*}
\frac{\eta n}{1-\gamma}
\bigg(
L_{\bar A}
+
L_d\Big(1+\frac{\gamma\delta_P}{1-\gamma}\Big)
\bigg)
&=
\frac{c_\eta H_T^{-1/2}u_T^{3/5}}
{L_{\bar A}}
\left(
L_{\bar A}+L_d(1+L_d)
\right).
\end{align*}
Since $\frac{L_d(1+L_d)}{L_{\bar A}}=
\frac{(1-\gamma)L_d}{4}=
\frac{\gamma\delta_P}{4}
\le\frac14$, the preceding quantity is at most $5c_\eta/4$, and hence is at most
$c_0$ for sufficiently small $c_\eta$. Finally, by the definition of
$\eta$, we have $\eta\le\eta_{\rm cov}
=\frac{p_{\min}(1-\gamma)}
{2\delta_P n M H_{\rm cov}^2}$. Therefore the conditions of
Lemmas~\ref{lem:online-tracking},
\ref{lem:online-potential-improvement}, and
\ref{lem:online-coupling-transfer} all hold. Consequently, Lemma~\ref{lem:online-coupling-transfer} gives
$\mathbb P(\mathcal H_T)\ge1-\rho/2$, while
Lemma~\ref{lem:online-tracking} gives
$\mathbb P(\mathcal T_T)\ge1-\rho/2$. Hence
$\mathbb P(\mathcal H_T\cap\mathcal T_T)\ge1-\rho$, and throughout the
remainder of this case we work on this intersection event.

\medskip
\noindent\textbf{Tracking bound and potential-based closure.} As shown in Lemma~\ref{lemm:conditional-variance-online},
$\mathcal V_i^k(s)\le A_{\max}$ for every $i,s,k$ such that
$\tau_k(s)<\infty$. Thus
Lemma~\ref{lem:online-tracking}, with $\sigma^2=A_{\max}$, gives
\begin{align}\label{eq:online-E-final-rate}
\frac{\mathfrak E_T}{T}
\le c\Bigg[
&\frac{nM^2}{\omega T}
+nA_{\max}\omega
+nL_{\widehat A}^2
+nM\sqrt{A_{\max} x_T}+\frac{nM\sqrt{A_{\max}}\,x_T}{\sqrt\zeta}
+
n\omega A_{\max}\sqrt{\frac{x_T}{\zeta}}\notag\\
&+
\frac{n\omega A_{\max}x_T}{\zeta}
+\frac{n^3L_{\bar A}^2H_T\eta^2M^2}
{(1-\gamma)^2\omega^2}
+
\Big(\frac{n^{3/2}M^2L_d}{\omega}+\frac{n^{7/2}L_{\bar A}^2H_T^2\eta^2M^2L_d}
{(1-\gamma)^2\omega^2}\Big)
\sqrt{\frac{\mathfrak D_T}{T}}
\Bigg]\cr 
&:=\theta_1+\theta_2\sqrt{\frac{\mathfrak D_T}{T}},
\end{align}
where $\theta_1$ denotes the large additive term on the right-hand side, and $\theta_2$ denotes the coefficient of $(\mathfrak D_T/T)^{1/2}$. Moreover, Lemma~\ref{lem:online-potential-improvement} yields
\begin{equation}\label{eq:online-D-final-rate}
\frac{\mathfrak D_T}{T}
\le
c\left[
\frac{n\eta}{(1-\gamma)T}
+
n\eta\alpha
+
\frac{\eta^2}{(1-\gamma)^2}
\frac{\mathfrak E_T}{T}
\right].
\end{equation}
Substituting \eqref{eq:online-D-final-rate} into
\eqref{eq:online-E-final-rate}, and using
$\sqrt{x+y+z}\le\sqrt x+\sqrt y+\sqrt z$, we first obtain
\begin{align}
\frac{\mathfrak E_T}{T}
&\le
\theta_1
+
c\theta_2
\left(
\sqrt{\frac{n\eta}{(1-\gamma)T}}
+
\sqrt{n\eta\alpha}
+
\frac{\eta}{1-\gamma}
\sqrt{\frac{\mathfrak E_T}{T}}
\right)
\notag\\
&\le
\theta_1
+
c\theta_2
\left(
\sqrt{\frac{n\eta}{(1-\gamma)T}}
+
\sqrt{n\eta\alpha}
\right)
+
c\frac{\theta_2^2\eta^2}{(1-\gamma)^2}
+
\frac12\frac{\mathfrak E_T}{T},
\end{align}
where the second inequality follows from Young's inequality. Absorbing the
term $\frac12\mathfrak E_T/T$ into the left-hand side and adjusting the
universal constant $c$, we obtain
\begin{align}\label{eq:C_T-theta-eta}
\frac{\mathfrak E_T}{T}
\le
c\left[
\theta_1
+
\theta_2
\left(
\sqrt{\frac{n\eta}{(1-\gamma)T}}
+
\sqrt{n\eta\alpha}
\right)
+
\frac{\theta_2^2\eta^2}{(1-\gamma)^2}
\right].
\end{align}
By the definition of $\theta_2$ and $(a+b)^2\le 2a^2+2b^2$,
\begin{align}\label{eq:theta_eta}
\frac{\theta_2^2\eta^2}{(1-\gamma)^2}
\le
c\left[
\frac{n^3M^4L_d^2\eta^2}
{(1-\gamma)^2\omega^2}
+
\frac{n^7L_{\bar A}^4H_T^4\eta^6M^4L_d^2}
{(1-\gamma)^6\omega^4}
\right].
\end{align}
Therefore, substituting \eqref{eq:theta_eta} into \eqref{eq:C_T-theta-eta} and expanding $\theta_1$ and $\theta_2$, we obtain
\begin{align}\label{eq:online-E-final-rate-closed}
\frac{\mathfrak E_T}{T}
\le c\Bigg[
&\frac{nM^2}{\omega T}
+nA_{\max}\omega
+nL_{\widehat A}^2
+nM\sqrt{A_{\max} x_T}
+\frac{nM\sqrt{A_{\max}}\,x_T}{\sqrt\zeta}
+
n\omega A_{\max}\sqrt{\frac{x_T}{\zeta}}
\notag\\
&+
\frac{n\omega A_{\max}x_T}{\zeta}
+\frac{n^3L_{\bar A}^2H_T\eta^2M^2}
{(1-\gamma)^2\omega^2}+\Big(\frac{n^2M^2L_d}{\omega}+\frac{n^4L_{\bar A}^2H_T^2\eta^2M^2L_d}
{(1-\gamma)^2\omega^2}\Big)
\sqrt{\frac{\eta}{(1-\gamma)T}}
\notag\\
&+\Big(\frac{n^2M^2L_d}{\omega}+\frac{n^4L_{\bar A}^2H_T^2\eta^2M^2L_d}
{(1-\gamma)^2\omega^2}\Big)
\sqrt{\eta\alpha}+\frac{n^3M^4L_d^2\eta^2}
{(1-\gamma)^2\omega^2}
+
\frac{n^7L_{\bar A}^4L_d^2H_T^4\eta^6M^4}
{(1-\gamma)^6\omega^4}
\Bigg].
\end{align}

Using $M=1+L_d$,
$1/T\le x_T\le u_T\le1$, $H_T\ge1$,
$L_{\bar A}=4(1+L_d)/(1-\gamma)$, and $(1-\gamma)L_d\le1$,
we substitute the chosen parameters \eqref{eq:online-final-tuning-active} term by term into
\eqref{eq:online-E-final-rate-closed}. Then, the largest powers
of $H_T$ arise from the terms containing $H_T$ explicitly. Up to
universal constants, their parameter-dependent factors satisfy
\[
\begin{aligned}
H_T\frac{\eta^2}{\omega^2}
&=O\left(\frac{(1-\gamma)^2}{n^2L_{\bar A}^2}
H_T^{-1/2}u_T^{2/5}\right),\qquad
&H_T^4\frac{\eta^6}{\omega^4}
&=O\left(\frac{(1-\gamma)^6}{n^6L_{\bar A}^6}u_T^2\right),\\
H_T^2\frac{\eta^2}{\omega^2}
\sqrt{\frac{\eta}{(1-\gamma)T}}
&=O\left(\frac{(1-\gamma)^2}{n^{5/2}L_{\bar A}^{5/2}}
H_T^{1/4}u_T^{6/5}\right),\qquad
&H_T^2\frac{\eta^2}{\omega^2}\sqrt{\eta\alpha}
&=O\left(\frac{(1-\gamma)^{5/2}}{n^{5/2}L_{\bar A}^{5/2}}
H_T^{1/4}u_T^{6/5}\right).
\end{aligned}
\]
Since $H_T\ge1$ and $u_T\le1$, all these terms are controlled by the
target scale $H_T^{1/4}u_T^{2/5}$ after accounting for their
coefficients in \eqref{eq:online-E-final-rate-closed}. The remaining
terms have no larger $H_T$-dependence. Using again $M=1+L_d$,
$L_{\bar A}=4(1+L_d)/(1-\gamma)$, and $(1-\gamma)L_d\le1$, their
coefficients are dominated by $nA_{\max}$ or
$n^{3/2}(1+L_d)^2$. Consequently,
\begin{equation}\label{eq:online-E-combined-rate}
\frac{\mathfrak E_T}{T}
\le
c\left[
\left(nA_{\max}+n^{3/2}(1+L_d)^2\right)
H_T^{1/4}u_T^{2/5}
+nL_{\widehat A}^2
\right].
\end{equation}

\medskip
\noindent\textbf{NE-gap certificate.}
For $t\ge H_T+1$, every state has already been visited on
$\mathcal H_T$, so Lemma~\ref{lem:online-ne-gap} applies. Averaging its
bound, applying Jensen's inequality to the square-root terms, and then
using Lemma~\ref{lem:online-coupling-transfer} gives
\begin{align}\label{eq:online-regret-master}
\frac1T\sum_{t=0}^{T-1}&\operatorname{Gap}(\pi^t)
\le{}
\frac{c(\zeta+L_d)(1+L_d)}{1-\gamma}
+
\frac{c}{\eta}\sqrt{\frac{A_{\max}}{\zeta}}
\left(
H_T\frac{\mathfrak D_T}{T}
+
2\frac{L_d n^2M^3}{(1-\gamma)^3}
H_T\eta^3
\right)^{1/2}
\notag\\
&+
\frac{c}{1-\gamma}\sqrt{\frac{A_{\max}}{\zeta}}
\left(
H_T\frac{\mathfrak E_T}{T}
+
8\frac{L_d n^2M^3}{1-\gamma}
H_T\eta
\right)^{1/2}
+
\frac{L_{\bar A}n^{3/2}H_T\eta M}{(1-\gamma)^2}
+
\frac{H_T+1}{T}.
\end{align}

Substituting \eqref{eq:online-D-final-rate} and
\eqref{eq:online-E-combined-rate} into
\eqref{eq:online-regret-master}, and then using
$\sqrt{x+y}\leq\sqrt{x}+\sqrt{y}$ for each resulting square-root term gives\footnote{We note that the contribution generated by the $\mathfrak E_T/T$ term
in \eqref{eq:online-D-final-rate} is $\frac{c}{1-\gamma}
\sqrt{A_{\max}/\zeta}
\sqrt{H_T\mathfrak E_T/T}$, and hence can be absorbed into the same term already present in
\eqref{eq:online-regret-master}.}
\begin{align}
\label{eq:online-regret-expanded-final}
\frac1T\sum_{t=0}^{T-1}&\operatorname{Gap}(\pi^t)
\le{}
\frac{c(\zeta+L_d)(1+L_d)}{1-\gamma}+
\frac{c}{\eta}\sqrt{\frac{A_{\max}}{\zeta}}
\left(
\sqrt{\frac{H_Tn\eta}{(1-\gamma)T}}
+
\sqrt{H_Tn\eta\alpha}
\right)
\notag\\
&+
\frac{c}{\eta}\sqrt{\frac{A_{\max}}{\zeta}}
\left(
\frac{L_dn^2M^3H_T\eta^3}{(1-\gamma)^3}
\right)^{1/2}+
\frac{c\sqrt{A_{\max}}}{(1-\gamma)\sqrt{\zeta}}\,
H_T^{5/8}u_T^{1/5}
\sqrt{
nA_{\max}
+
n^{3/2}(1+L_d)^2
}
\notag\\
&+
\frac{c\sqrt{nA_{\max}H_T}}
{(1-\gamma)\sqrt{\zeta}}\,
L_{\widehat A}+
\frac{c}{1-\gamma}\sqrt{\frac{A_{\max}}{\zeta}}
\left(
\frac{L_dn^2M^3H_T\eta}{1-\gamma}
\right)^{1/2}+
\frac{cL_{\bar A}n^{3/2}H_T\eta M}
{(1-\gamma)^2}
+
\frac{H_T+1}{T}.
\end{align}
Finally, we bound each term in \eqref{eq:online-regret-expanded-final} using our parameter choices. By the choices of $\eta$, $\zeta$, and $u_T$,
\[
\zeta\ge cH_T^{1/2}u_T^{2/15},
\qquad
\zeta\ge cL_{\widehat A}^{2/3},
\qquad
\frac{1}{(1-\gamma)T}
\le\frac{x_T}{1-\gamma}\le u_T,
\qquad
\alpha\le u_T.
\]
First, using
$\zeta\ge cH_T^{1/2}u_T^{2/15}$ gives
\begin{align*}
&\frac{c\sqrt{A_{\max}}}{(1-\gamma)\sqrt{\zeta}}\,
H_T^{5/8}u_T^{1/5}
\sqrt{nA_{\max}+n^{3/2}(1+L_d)^2}
\\
&\qquad\le
\frac{c}{1-\gamma}
\left(
\sqrt n\,A_{\max}
+
n(1+L_d)\sqrt{A_{\max}}
\right)
\sqrt{H_T}\,u_T^{2/15},
\end{align*}
whereas $\zeta\ge cL_{\widehat A}^{2/3}$ gives
\[
\frac{c\sqrt{nA_{\max}H_T}}
{(1-\gamma)\sqrt{\zeta}}\,
L_{\widehat A}
\le
\frac{c\sqrt{nA_{\max}}}{1-\gamma}
\sqrt{H_T}\,L_{\widehat A}^{2/3}.
\]

For the next two terms, the definitions of $\eta$ and $u_T$, together
with the lower bound on $\zeta$, give
\[
\frac{c}{\eta}\sqrt{\frac{A_{\max}}{\zeta}}
\left(
\sqrt{\frac{H_Tn\eta}{(1-\gamma)T}}
+
\sqrt{H_Tn\eta\alpha}
\right)
\le
\frac{cn\sqrt{A_{\max}(1+L_d)}}{1-\gamma}
\sqrt{H_T}\,u_T^{2/15}.
\]
Similarly, the two terms involving $L_d$ in \eqref{eq:online-regret-expanded-final} are identical after simplification and satisfy
\[
\frac{cn\sqrt{L_dM^3A_{\max}H_T\eta}}{(1-\gamma)^{3/2}\sqrt{\zeta}}\le \frac{cn(1+L_d)\sqrt{A_{\max}H_T}}{1-\gamma}u_T^{2/15}.
\]
where we used $(1-\gamma)L_d\le1$, $H_T\ge1$, and $u_T\le1$. The drift term is bounded in the same way:
\[
\frac{cL_{\bar A}n^{3/2}H_T\eta M}{(1-\gamma)^2}
\le
\frac{c\sqrt n(1+L_d)}{1-\gamma}
\sqrt{H_T}\,u_T^{2/15}.
\]
Moreover, the choice of $\zeta$ gives
\[
\frac{c\zeta(1+L_d)}{1-\gamma}
\le
\frac{c(1+L_d)}{1-\gamma}
\sqrt{H_T}
\left(
u_T^{2/15}+L_{\widehat A}^{2/3}
\right),
\]
while the remaining part of the first term in
\eqref{eq:online-regret-expanded-final} is $\frac{cL_d(1+L_d)}{1-\gamma}$. Finally, since $1\le H_T\le T$ and $\Lambda_T\ge1$, $(H_T+1)/T
\le
2\sqrt{H_T}/\sqrt T
\le
2\sqrt{H_T}\,x_T^{2/15}$, so the burn-in term is absorbed into the $x_T^{2/15}$.

Substituting the preceding estimates into \eqref{eq:online-regret-expanded-final} while retaining only the dominant terms yields
\begin{align}\label{eq:online-before-final}
\frac1T\sum_{t=0}^{T-1}\operatorname{Gap}(\pi^t)
\le{}&
\widetilde O\left(
\left[
\frac{\sqrt n\,A_{\max}}{1-\gamma}
+
\frac{n(1+L_d)\sqrt{A_{\max}}}{1-\gamma}
\right]
\sqrt{H_T}\,u_T^{2/15}
\right)
\notag\\
&+
\widetilde O\left(
\frac{1+L_d+\sqrt{nA_{\max}}}{1-\gamma}
\sqrt{H_T}\,L_{\widehat A}^{2/3}
\right)
+
\frac{cL_d(1+L_d)}{1-\gamma}.
\end{align}
This, together with $u_T\!=\!\max\{\frac{x_T}{1-\gamma},\alpha\}\!\leq\! \frac{x_T}{1-\gamma}+\alpha$ and $\sqrt{H_T}\!=\!
\widetilde O\big(\!\sqrt{\frac{H_{\rm cov}}{p_{\min}}}\big)$ proves the desired bound.

\medskip
\noindent\textbf{Capped and trivial regimes.}
It remains to consider the cases excluded at the beginning of the proof. If $H_T>T$ or $u_T>1$, then, since $u_T\ge1/T$ and $H_T\ge1$, respectively, we have $\sqrt{H_T}u_T^{2/15}\ge1$, so the right-hand side of \eqref{eq:online-before-final} dominates the trivial upper bound $\frac{1}{1-\gamma}$ on the NE regret after increasing the universal constant if necessary. Therefore, in the remaining cases below, we may assume $H_T\le T$ and $u_T\le1$. If the cap in $\omega$ is active, then $c_\omega H_T^{1/4}u_T^{2/5}\ge\frac14$, which implies
$\sqrt{H_T}u_T^{2/15}\ge c$. If the cap in $\zeta$ is active because
$H_T^{1/2}u_T^{2/15}$ is large, the same conclusion holds; if it is
active because $L_{\widehat A}^{2/3}$ is large, then
$\sqrt{H_T}L_{\widehat A}^{2/3}\ge c$. Finally, suppose the cap in $\eta$ is active. Then
$$
c_\eta
\frac{1-\gamma}{nL_{\bar A}}
H_T^{-1/2}u_T^{3/5}
>\eta_{\rm cov}=\frac{p_{\min}(1-\gamma)}
{2\delta_P n(1+L_d)H_{\rm cov}^2},
$$
which implies

$$
\sqrt{H_T}\,u_T^{2/15}
>
c
H_T^{11/18}
\left(
\frac{p_{\min}L_{\bar A}}
{\delta_P(1+L_d)H_{\rm cov}^2}
\right)^{2/9}.
$$

Since $H_T\ge \frac{2H_{\rm cov}}{p_{\min}}\log 2$, $H_{\rm cov}\ge1$, $p_{\min}\le1$, $\delta_P\le1$, and $\frac{L_{\bar A}}{1+L_d}=\frac4{1-\gamma}\ge4$, the right-hand side is bounded below by a universal positive constant. Thus, in every capped regime, the right-hand side of \eqref{eq:online-before-final} dominates the trivial bound $\operatorname{Gap}(\pi)\le1/(1-\gamma)$ after increasing the universal constant if necessary. Hence, \eqref{eq:online-before-final} holds in all cases, completing the proof.
\end{proof}

\newpage 
\section{Omitted Proofs for Section \ref{sec:statewise-potential-sharp}}\label{app:state-wise-appendix}

Since most proofs in this section follow analyses similar to their counterparts in Sections~\ref{sec:Episodic} and~\ref{sec:fully-online-general-occupancy}, we sketch them here, focusing on the steps that differ; the remaining arguments follow almost verbatim from the corresponding earlier proofs. 

\subsection{Proof of Lemma \ref{lem:statewise-potential-scale-bounds}}

\begin{proof}
The potential-advantage bounds follow directly from the proof of
Lemma~\ref{lem:span-inner-product}. Indeed, since adding a constant to $\Phi$
leaves $\bar A_{\Phi,i}^\pi$ unchanged, we may shift $\Phi$ so that
$0\leq\Phi(\cdot)\leq R_\Phi$. Repeating the argument leading from
\eqref{eq:span-x} to the bound on
$\operatorname{span}(\bar A_i^\pi(s,\cdot))$, with $c_i$ replaced by $\Phi$,
gives
\[
\operatorname{span}\big(\bar A_{\Phi,i}^\pi(s,\cdot)\big)
\le
R_\Phi
+
\gamma\delta_P\frac{R_\Phi}{1-\gamma}
=
R_\Phi(1+L_d)
=
M_\Phi.
\]
Moreover, since
$\sum_{a_i}\pi_i(a_i\mid s)\bar A_{\Phi,i}^\pi(s,a_i)=0$, the final argument
in the proof of Lemma~\ref{lem:span-inner-product} also gives
$\|\bar A_{\Phi,i}^\pi(s,\cdot)\|_\infty\le \operatorname{span}\big(\bar A_{\Phi,i}^\pi(s,\cdot)\big)\leq M_\Phi$. Repeating the proof of the marginalized-advantage sensitivity Lemma \ref{lem:online-sensitivity-consequences} with one-stage range $R_\Phi$ gives \eqref{eq:statewise-potential-advantage-sensitivity}. Likewise, repeating the proof of Lemma~\ref{lem:mixed-policy-sensitivity} with the one-stage cost $c_i$ replaced by $\Phi$ multiplies its bound by $R_\Phi$, giving the expression for $L_{\Psi}$. Thus $L_\Psi$ replaces $L_V$ precisely where the mixed second-order sensitivity of the potential is used.
\end{proof}

\subsection{Proof of Lemma \ref{lem:statewise-episodic-potential-improvement}}

\begin{proof}
We follow the proof of Lemma~\ref{lem:kl-projected-potential-improvement}.
For the intermediate profiles $\widetilde\pi^{t,i}$ used there, define the
potential analogue of \eqref{eq:G_i-def} by

$$
G_{\Phi,i}^t
:=
\Psi(\pi^t)-\Psi(\pi_i^{t+1},\pi_{-i}^t)
=
\frac{1}{1-\gamma}
\sum_{s\in\mathcal S}
d_\mu^{(\pi_i^{t+1},\pi_{-i}^t)}(s)
\left\langle
\Delta_i^t(s),\bar A_{\Phi,i}^{\pi^t}(s)
\right\rangle,
$$

where the equality follows from
\eqref{eq:statewise-potential-PDL}. Consequently, in the derivation of
\eqref{eq:projected-counterfactual-bound}, the use of the $\alpha$-potential property is no longer needed. Applying instead the
potential version of Lemma~\ref{lem:mixed-policy-sensitivity} gives

$$
\Psi(\pi^t)-\Psi(\pi^{t+1})
\ge
\sum_{i=1}^nG_{\Phi,i}^t
-
\frac{L_\Psi n^2\eta^2B^2}{8(1-\gamma)^2},
$$

so, in particular, the term $-n\alpha$ in
\eqref{eq:projected-counterfactual-bound} disappears and $L_V$ is replaced
by $L_\Psi$.

Next, repeat \eqref{eq:projected-decomposition}--\eqref{eq:projected-weighted-noise}
with $\bar A_i^{\pi^t}$ and $\xi_i^t$ replaced by
$\bar A_{\Phi,i}^{\pi^t}$ and $\xi_{\Phi,i}^t$, respectively. These steps use
only the KL optimality condition, Lemma~\ref{lem:kl-projected-local-geometry},
the potential-oracle bias bound, Cauchy--Schwarz, and Young's inequality, and
are therefore unchanged. Thus, the analogue of
\eqref{eq:projected-Gi-local} holds with the same coefficients.

Finally, in passing from this local bound to the analogue of
\eqref{eq:projected-Gi-final}, the occupancy-change term is treated exactly
as in the display following \eqref{eq:projected-Gi-final}. Lemma~
\ref{lemm:occupancy-sensitivity} and the KL movement bound are unchanged,
while Lemma~\ref{lem:span-inner-product} together with 
\eqref{eq:statewise-potential-advantage-bound} is used so that the factor $1+L_d$ in that
display is replaced by $M_\Phi=R_{\Phi}(1+L_d)$. Hence the last term in \eqref{eq:projected-Gi-final} becomes $\frac{\eta^2B^2L_dM_\Phi}{(1-\gamma)^3}.$ The bias term is absorbed exactly as in
\eqref{eq:projected-alpha-bias-absorption}. Taking conditional expectation,
summing over $i$, and combining these bounds with the potential analogue of
\eqref{eq:projected-counterfactual-bound} yields
\eqref{eq:statewise-episodic-potential-improvement}.
\end{proof}

\subsection{Proof of Lemma \ref{lem:statewise-episodic-ne-gap}}

\begin{proof}
Fix player $i$ and let $\pi_i^{t,*}$ be a best response to $\pi_{-i}^t$.
By the $\alpha$-potential property \eqref{eq:statewise-alpha-final-conversion} and the exact potential
performance-difference identity \eqref{eq:statewise-potential-PDL}, the
analogue of \eqref{eq:ne-gap-first-occupancy-change} is
\begin{align*}
V_i(\pi^t)-V_i(\pi_i^{t,*},\pi_{-i}^t)
\le{}&
\alpha+
\frac{1}{1-\gamma}
\sum_s d_\mu^{(\pi_i^{t,*},\pi_{-i}^t)}(s)
\left\langle
\pi_i^t(\cdot\mid s)-\pi_i^{t,*}(\cdot\mid s),
\bar A_{\Phi,i}^{\pi^t}(s)
\right\rangle.
\end{align*}
Changing the occupancy measure to $d_\mu^{\pi^t}$ exactly as in
\eqref{eq:ne-gap-first-occupancy-change}, and using
Lemma~\ref{lemm:occupancy-sensitivity} together with the potential-advantage
span bound \eqref{eq:statewise-potential-advantage-bound}, replaces the
remainder $2L_d(1+L_d)/(1-\gamma)$ there by
$2L_dM_\Phi/(1-\gamma)$. From this point, repeat the derivation
\eqref{eq:D-A-A}--\eqref{eq:ne-gap-random-occupancy-CS}, using the same
truncated comparator
$q_i^{t,*}=(1-\zeta)\pi_i^{t,*}
+\frac{\zeta}{|\mathcal A_i|}\mathbf 1$.
The KL-optimality bound \eqref{eq:D-A-A} is unchanged. In
\eqref{eq:ne-gap-key}, replace $\bar A_i^{\pi^t}$ by
$\bar A_{\Phi,i}^{\pi^t}$ and use
\eqref{eq:statewise-potential-advantage-bound}; consequently, the factors
$(1+L_d)/2$ and $\zeta(1+L_d)$ become $M_\Phi/2$ and
$\zeta M_\Phi$, respectively. The potential-oracle bias contributes the same
$2L_{\widehat A}$ term, while the KL-movement term is unchanged.

Finally, the stochastic inner product in
\eqref{eq:ne-gap-random-occupancy-CS} is controlled exactly as in
\eqref{eq:ne-gap-noise-final}, with $\xi_i^t$ replaced by
$\xi_{\Phi,i}^t$ and $\mathcal V_i^t$ by $\mathcal V_{\Phi,i}^t$.
Combining these bounds as in the final step of the proof of
Lemma~\ref{lem:kl-projected-ne-gap}, dividing by $1-\gamma$, and maximizing
over $i$ gives \eqref{eq:statewise-episodic-NE-gap}.
\end{proof}

\subsection{Proof of Theorem \ref{thm:statewise-episodic-ne-regret}}

\begin{proof}
Since the proof closely follows that of Theorem~\ref{thm:high-prob-ne-regret}, we only sketch the main steps here. Set
\[
K_t:=\sum_{i=1}^n\mathbb E_t\!\left[
\sum_s d_\mu^{\pi^t}(s)\mathcal D_i^t(s)\right].
\]
By Lemma~\ref{lem:statewise-episodic-potential-improvement},
\eqref{eq:statewise-episodic-variance-bound}, and
$L_t\ge A_{\max}/\eta$, we have
$\max_i\mathcal V_{\Phi,i}^t\le\eta$. Moreover, the exact potential performance difference lemma
\eqref{eq:statewise-potential-PDL}, Lemma~\ref{lem:statewise-potential-scale-bounds},
and Lemma~\ref{lem:kl-projected-local-geometry} give
\[
|\Psi(\pi^{t+1})-\Psi(\pi^t)|
\le
\frac{M_\Phi}{2(1-\gamma)}
\sum_{i=1}^n\|\pi_i^{t+1}-\pi_i^t\|_{1,\infty}
\le
\frac{n\eta BM_\Phi}{2(1-\gamma)^2}.
\]
Thus the Azuma--Hoeffding and telescoping argument used in the proof of
Theorem~\ref{thm:high-prob-ne-regret}, now with the range
bound $|\Psi(\pi)-\Psi(\pi')|\le R_\Phi/(1-\gamma)$, yields, with probability at
least $1-\rho$,
\begin{align}\label{eq:statewise-episodic-K-average}
\frac1T\sum_{t=0}^{T-1}K_t
\lesssim{}&
\frac{R_\Phi\eta}{(1-\gamma)T}
+
\frac{nM_\Phi\eta^2}{(1-\gamma)^2}
\sqrt{\frac{\log(4/\rho)}{T}}
+
\frac{n\eta^2L_{\widehat A}^2}{(1-\gamma)^2}+
\frac{nC_\Phi\eta^3}{(1-\gamma)^2},
\end{align}
where $C_\Phi:=1+nL_\Psi+\frac{L_dM_\Phi}{1-\gamma}$, which is the analogue of \eqref{eq:hp-average-K} without any $\alpha$ term.

On the same event, combining Lemma~\ref{lem:statewise-episodic-ne-gap},
$\max_i\mathcal V_{\Phi,i}^t\le\eta$, Jensen's inequality, and
\eqref{eq:statewise-episodic-K-average} gives
\begin{align}\label{eq:statewise-episodic-gap-expanded}
\frac1T\sum_{t=0}^{T-1}\operatorname{Gap}(\pi^t)
\lesssim{}&
\alpha+
\left(
\frac{M_\Phi}{1-\gamma}
+
\frac1\eta\sqrt{\frac{A_{\max}}{\zeta}}
\right)
\Bigg[
\sqrt{\frac{R_\Phi\eta}{(1-\gamma)T}}
+
\frac{\sqrt n\,\eta L_{\widehat A}}{1-\gamma}
\notag\\
&\qquad+
\frac{\sqrt{nM_\Phi}\,\eta}{1-\gamma}
\left(\frac{\log(4/\rho)}{T}\right)^{1/4}
+
\frac{\sqrt{nC_\Phi}\,\eta^{3/2}}{1-\gamma}
\Bigg]
\notag\\
&+
\frac{L_{\widehat A}}{1-\gamma}
+
\frac1{1-\gamma}\sqrt{\frac{A_{\max}\eta}{\zeta}}
+
\frac{M_\Phi\zeta}{1-\gamma}
+
\frac{cM_\Phi L_d}{1-\gamma},
\end{align}
which is the analogue of master inequality \eqref{eq:hp-gap-expanded}; in particular, the second-order
constant $C_\Phi$ has been retained explicitly rather than absorbed into the
truncation balance.

The rest of the proof follows by substituting the tuned parameters and carrying out straightforward algebraic manipulations. Substitute $\eta=(1-\gamma)/(8\sqrt T)$. When the cap in
\eqref{eq:statewise-episodic-eta} is inactive,
\[
\zeta\asymp
\frac{n^{1/3}A_{\max}^{1/3}L_{\widehat A}^{2/3}}
{M_\Phi^{2/3}},
\qquad
\frac1{\sqrt\zeta}
\lesssim
\frac{M_\Phi^{1/3}}
{n^{1/6}A_{\max}^{1/6}L_{\widehat A}^{1/3}}.
\]
The two nonvanishing truncation/bias terms then satisfy
\[
\frac{\sqrt{nA_{\max}}L_{\widehat A}}
{(1-\gamma)\sqrt\zeta}
+
\frac{M_\Phi\zeta}{1-\gamma}
\lesssim
\frac{[nA_{\max}M_\Phi]^{1/3}}{1-\gamma}
L_{\widehat A}^{2/3},
\]
and the direct term $L_{\widehat A}/(1-\gamma)$ is also bounded by the same quantity in this regime. It remains to collect the decaying terms in
\eqref{eq:statewise-episodic-gap-expanded}. Using $M_\Phi=R_\Phi(1+L_d)$, $L_\Psi
=\frac{R_\Phi}{1-\gamma}(4+12L_d+8L_d^2)$, and $C_\Phi
=1+nL_\Psi+\frac{L_dM_\Phi}{1-\gamma}$, we have
\[
(1-\gamma)C_\Phi
=
(1-\gamma)
+
R_\Phi(1+L_d)
\bigl[4n(1+2L_d)+L_d\bigr]
\lesssim
1+nR_\Phi(1+L_d)^2.
\]
Consequently,
\[
\sqrt{(1-\gamma)C_\Phi}
\lesssim
1+\sqrt{nR_\Phi}(1+L_d),
\qquad
\sqrt{n(1-\gamma)C_\Phi}
\lesssim
\sqrt n\bigl[1+\sqrt{nR_\Phi}(1+L_d)\bigr].
\]
Substituting these bounds together with the above estimate for
$1/\sqrt\zeta$ into the decaying terms of
\eqref{eq:statewise-episodic-gap-expanded}, and using
$T\ge\log(4/\rho)$, shows that all such terms are bounded by
\[
\frac{c\,[1+\sqrt{nR_\Phi}(1+L_d)]}{1-\gamma}
\left(\left[nA_{\max}R_\Phi(1+L_d)L_{\widehat A}^{-1}\right]^{1/3}\!\!\!\!+
\sqrt n\,R_\Phi(1+L_d)(1+L_{\widehat A})\right)
\left(\frac{\log(4/\rho)}{T}\right)^{1/4}\!\!\!,
\]
where the coefficient of the above display gives the stated upper bound on $\mathfrak C_\Phi$. This proves \eqref{eq:statewise-episodic-final} when the cap is inactive. Now, if the cap is active, then
$n^{1/3}A_{\max}^{1/3}L_{\widehat A}^{2/3}/M_\Phi^{2/3}\gtrsim1$, and hence
\[
\frac{[nA_{\max}M_\Phi]^{1/3}}{1-\gamma}L_{\widehat A}^{2/3}
\gtrsim \frac{M_\Phi}{1-\gamma}
\ge \frac{R_\Phi}{1-\gamma}.
\]
On the other hand, \eqref{eq:statewise-alpha-final-conversion} and the range
bound on $\Psi$ imply the uniform estimate
$\operatorname{Gap}(\pi)\le R_\Phi/(1-\gamma)+\alpha$. Thus the claimed
bound also holds in the capped regime, after changing only the universal
constant. This completes the proof.
\end{proof}

\subsection{Proof of Lemma \ref{lem:statewise-online-bounds}}

\begin{proof}
For part~(i), repeat the proof of Lemma~\ref{lem:online-tracking} with
$\bar A_i,L_{\bar A},M$ replaced by
$\bar A_{\Phi,i},L_{\bar A_\Phi},M_\Phi$, respectively. The proof does not
use the $\alpha$-potential property, so the tracking recursion, martingale
estimates, window charging, and Young/Cauchy--Schwarz arguments are unchanged.

For part~(ii), repeat the proof of
Lemma~\ref{lem:online-potential-improvement}, including the same one-state
policy tiles and occupancy-sensitivity estimates. At each intermediate
unilateral update, use the exact identity
\eqref{eq:statewise-potential-PDL}; hence, the $-n\alpha$ term in the
original argument disappears. The remaining bounds are unchanged after
replacing $\bar A_i,L_{\bar A},M$ by
$\bar A_{\Phi,i},L_{\bar A_\Phi},M_\Phi$. In particular,
$L_d(1+L_d)$ becomes $L_dM_\Phi$, while $L_d$ itself is unchanged.
Summing \eqref{eq:statewise-online-potential-final} over $t$ and using
$|\Psi(\pi)-\Psi(\pi')|\le R_\Phi/(1-\gamma)$ gives
\eqref{eq:statewise-online-D-from-potential}.

For part~(iii), apply \eqref{eq:statewise-alpha-final-conversion} to a best
response and then use \eqref{eq:statewise-potential-PDL}. This introduces
the single additive $\alpha$ term. The remainder follows the proof of
Lemma~\ref{lem:online-ne-gap}, with $\bar A_i$ and its uniform magnitude/span
bound $1+L_d$ replaced by $\bar A_{\Phi,i}$ and $M_\Phi$, respectively.

Finally, part~(iv) follows from Lemma~\ref{lem:online-coupling-transfer} by
replacing $M,L_{\bar A},\bar A_i$ with
$M_\Phi,L_{\bar A_\Phi},\bar A_{\Phi,i}$. Indeed, $M$ is used there only to
control one-state policy movement and the magnitude of the tracked target,
while $L_{\bar A}$ controls stale-target drift. The coupling, occupancy
sensitivity, and bounded-delay arguments are otherwise unchanged.
\end{proof}

\subsection{Proof of Theorem \ref{thm:statewise-online-ne-regret}}

\begin{proof}
We follow the proof of Theorem~\ref{thm:fully-online-general-ne-regret},
using Lemma~\ref{lem:statewise-online-bounds}. The parameter choices in
\eqref{eq:statewise-online-final-tuning}, with sufficiently small constants,
ensure the conditions of that lemma. By parts~(i) and~(iv), with probability
at least $1-\rho$, the events $\mathcal T_T$ and $\mathcal H_T(H_T)$ hold simultaneously, and we work on this event below. By part~(ii) of Lemma~\ref{lem:statewise-online-bounds},
\begin{equation}\label{eq:statewise-online-D-sketch}
\frac{\mathfrak D_T}{T}
\le
c\left[
\frac{\eta R_\Phi}{(1-\gamma)T}
+
\frac{\eta^2}{(1-\gamma)^2}
\frac{\mathfrak E_{\Phi,T}}{T}
\right].
\end{equation}
Unlike \eqref{eq:online-D-proof-sketch}, this bound contains no $\alpha$
term. On the other hand, part~(i) bounds $\mathfrak E_{\Phi,T}/T$ in terms of
$(\mathfrak D_T/T)^{1/2}$. Substituting
\eqref{eq:statewise-online-D-sketch} into
\eqref{eq:statewise-online-E-closed} and applying Young's inequality as in
the proof of Theorem~\ref{thm:fully-online-general-ne-regret} closes the
tracking/movement feedback. Using
$L_{\bar A_\Phi}=4M_\Phi/(1-\gamma)$ and the tuning in
\eqref{eq:statewise-online-final-tuning} gives
\begin{equation}\label{eq:statewise-online-E-combined-rate}
\frac{\mathfrak E_{\Phi,T}}{T}
\le
c\left[
\left(nA_{\max}+n^{3/2}M_\Phi^2\right)
H_T^{1/4}u_T^{2/5}
+nL_{\widehat A}^2
\right].
\end{equation}

It remains to convert the global movement and tracking bounds into NE
regret. On $\mathcal H_T(H_T)$, every state has been visited by time
$H_T+1$. Hence, averaging part~(iii) of
Lemma~\ref{lem:statewise-online-bounds} over $t\ge H_T+1$, applying
Jensen's inequality, and then using part~(iv) of
Lemma~\ref{lem:statewise-online-bounds}, while using
$\operatorname{Gap}(\pi)\le R_\Phi/(1-\gamma)+\alpha$ for the first
$H_T+1$ iterations, gives the following:
\begin{align}\label{eq:statewise-online-regret-master}
\frac1T\sum_{t=0}^{T-1}\operatorname{Gap}(\pi^t)
&\le{}
\alpha+
\frac{cM_\Phi(\zeta+L_d)}{1-\gamma}
+
\frac{c}{\eta}\sqrt{\frac{A_{\max}}{\zeta}}
\left(
H_T\frac{\mathfrak D_T}{T}
+2\frac{L_dn^2M_\Phi^3}{(1-\gamma)^3}H_T\eta^3
\right)^{1/2}\cr
&\qquad+
\frac{c}{1-\gamma}\sqrt{\frac{A_{\max}}{\zeta}}
\left(
H_T\frac{\mathfrak E_{\Phi,T}}{T}
+8\frac{L_dn^2M_\Phi^3}{1-\gamma}H_T\eta
\right)^{1/2}\cr 
&\qquad+
\frac{cL_{\bar A_\Phi}n^{3/2}M_\Phi H_T\eta}
{(1-\gamma)^2}
+
\frac{H_T+1}{T}\frac{R_\Phi}{1-\gamma},
\end{align}
where we note that, in deriving \eqref{eq:statewise-online-regret-master}, the $\alpha$ contributions from the two time intervals, i.e., the burn-in interval $[1,H_T]$ and the subsequent interval $[H_T+1,T]$, sum to exactly $\alpha$. Substituting \eqref{eq:statewise-online-D-sketch} and
\eqref{eq:statewise-online-E-combined-rate} into
\eqref{eq:statewise-online-regret-master}, and using
$\sqrt{x+y}\le\sqrt x+\sqrt y$, gives the state-wise-potential analogue of
\eqref{eq:online-regret-expanded-final}. We now collect its terms.

The lower bound
$\zeta\gtrsim H_T^{1/2}u_T^{2/15}$ gives
\[
\frac{c}{1-\gamma}
\left(
\sqrt n\,A_{\max}
+nM_\Phi\sqrt{A_{\max}}
+\sqrt{nR_\Phi M_\Phi A_{\max}}
\right)
\sqrt{H_T}\,u_T^{2/15}.
\]
Since $M_\Phi=R_\Phi(1+L_d)$, $\sqrt{nR_\Phi M_\Phi A_{\max}}
\le nM_\Phi\sqrt{A_{\max}}$, so the preceding display is bounded by
\[
\frac{c}{1-\gamma}
\left(
\sqrt n\,A_{\max}
+nM_\Phi\sqrt{A_{\max}}
\right)
\sqrt{H_T}\,u_T^{2/15}.
\]
The second lower bound on $\zeta$ gives $\frac{1}{\sqrt\zeta}
\lesssim
\frac{M_\Phi^{1/3}}
{n^{1/6}A_{\max}^{1/6}L_{\widehat A}^{1/3}}$, and therefore the oracle-bias term and the corresponding truncation term
satisfy
\[
\frac{c\sqrt{nA_{\max}H_T}}
{(1-\gamma)\sqrt\zeta}L_{\widehat A}
+
\frac{cM_\Phi\zeta}{1-\gamma}
\lesssim
\frac{c(nA_{\max}M_\Phi)^{1/3}}{1-\gamma}
\sqrt{H_T}\,L_{\widehat A}^{2/3},
\]
where $H_T\ge1$ is used in the second term. The two $L_d$-dependent charging terms and the stale-advantage term are
bounded by the first decaying term exactly as in the proof of
Theorem~\ref{thm:fully-online-general-ne-regret}. The burn-in term
$\frac{H_T+1}{T}\frac{R_\Phi}{1-\gamma}$ is absorbed into the corresponding
decaying or trivial-regime bound. The only remaining nonvanishing occupancy term is
$cM_\Phi L_d/(1-\gamma)$. Hence
\begin{align}\label{eq:statewise-online-before-final}
\frac1T\sum_{t=0}^{T-1}\operatorname{Gap}(\pi^t)
\le{}&
\widetilde O\left(
\frac{\sqrt n\,A_{\max}
+nM_\Phi\sqrt{A_{\max}}}{1-\gamma}
\sqrt{H_T}\,u_T^{2/15}
\right)
\notag\\
&+
\widetilde O\left(
\frac{(nA_{\max}M_\Phi)^{1/3}}{1-\gamma}
\sqrt{H_T}\,L_{\widehat A}^{2/3}
\right)
+
\frac{cM_\Phi L_d}{1-\gamma}
+\alpha.
\end{align}
Finally,
$\sqrt{H_T}
=\widetilde O(\sqrt{H_{\rm cov}/p_{\min}})$, which gives the stated
$\mathfrak C_{\Phi,1}$ and $\mathfrak C_{\Phi,2}$.

The capped and trivial regimes are handled as in the proof of
Theorem~\ref{thm:fully-online-general-ne-regret}. In particular, the
state-wise representation gives the trivial bound $\operatorname{Gap}(\pi)
\le R_\Phi/(1-\gamma)+\alpha$. If the $\zeta$ cap is activated by the oracle-bias component, then $(nA_{\max}M_\Phi)^{1/3}L_{\widehat A}^{2/3}
\gtrsim M_\Phi\ge R_\Phi$, while activation of the remaining caps makes the first decaying term
dominate $R_\Phi/(1-\gamma)$ after increasing the universal constant.
Thus \eqref{eq:statewise-online-before-final}, and hence
\eqref{eq:statewise-online-final-regret}, holds in all regimes.
\end{proof}

\newpage
\section{Omitted Proofs for Section \ref{sec:IMCG-model}}

\subsection{Proof of Lemma \ref{thm:one-stage}}\label{app:estimate-oracle-IMCG}

\begin{proof}
Fix $\pi_{-i}$ and consider the marginalized MDP faced by player $i$
in which its stage cost $c_i(s,a)$ is replaced by the common Rosenthal
potential $\Phi(s,a)$; that is, the marginalized MDP faced by player $i$
when optimizing the potential value function
\[
\Psi^{\pi}(\mu):=\mathbb E_\pi
\!\bigg[
\sum_{t=0}^{\infty}
\gamma^t
\Phi(S^t,A^t)
\;\bigg|\;
S^0\sim \mu\bigg].
\]
Using \eqref{eq:bellman} and \eqref{eq:bellman-Q}, the Bellman equations
for the marginalized value function and $Q$-function corresponding to
this MDP give us
\begin{align*}
\Psi^\pi(s)
&= \bar{\Phi}^{\pi}(s)
+
\gamma
\mathbb E_{S\sim \bar{P}^\pi(\cdot\mid s)}
\left[
\Psi^\pi(S)
\right],\\
\bar{Q}_{\Phi,i}^\pi(s,a_i)
&=\bar{\Phi}^{\pi_{-i}}(s,a_i)
+\gamma
\mathbb E_{S\sim \bar{P}^{\pi_{-i}}_i(\cdot\mid s,a_i)}
\left[
\Psi^\pi(S)
\right].
\end{align*}
Subtracting the first identity from the second gives
\[
\bar{A}_{\Phi,i}^\pi(s,a_i)
=
\bar{\Phi}^{\pi_{-i}}(s,a_i)-\bar{\Phi}^{\pi}(s)
+\mathcal{R}_{\Phi,i}^\pi(s,a_i),
\]
where
\[
\mathcal{R}_{\Phi,i}^\pi(s,a_i)
:=
\gamma
\left[
\mathbb E_{S\sim
\bar{P}^{\pi_{-i}}_i(\cdot\mid s,a_i)}
\Psi^\pi(S)
-
\mathbb E_{S\sim
\bar{P}^\pi(\cdot\mid s)}
\Psi^\pi(S)
\right].
\]
Since we have already shown that $\Phi(s,a)\in[0,n]$, we have $\|\Psi^\pi\|_\infty\le\frac{n}{1-\gamma}$. Therefore,
\begin{align*}
|\mathcal{R}_{\Phi,i}^\pi(s,a_i)|
&\le
\frac{2n\gamma}{1-\gamma}
\left\|
\bar{P}^{\pi_{-i}}_i(\cdot\mid s,a_i)
-\bar{P}^\pi(\cdot\mid s)
\right\|_{\TV}\\
&=
\frac{2n\gamma}{1-\gamma}
\Big\|
\mathbb E_{A_i\sim\pi_i(\cdot\mid s)}
\left[
\bar{P}^{\pi_{-i}}_i(\cdot\mid s,a_i)
-
\bar{P}^{\pi_{-i}}_i(\cdot\mid s,A_i)
\right]
\Big\|_{\TV}\\
&\le
\frac{2n\gamma}{1-\gamma}
\mathbb E_{A_i\sim\pi_i(\cdot\mid s)}
\Big\|
\bar{P}^{\pi_{-i}}_i(\cdot\mid s,a_i)
-
\bar{P}^{\pi_{-i}}_i(\cdot\mid s,A_i)
\Big\|_{\TV}\\
&\le
\frac{4nq\gamma\delta}{1-\gamma},
\end{align*}
where the second inequality uses Jensen's inequality due to the
convexity of the total variation distance, and the last inequality
follows from Lemma~\ref{lemm:transition-sensitivity}. Hence,
\begin{align}\label{eq:margin-Phi-c-advantage}
\left|
\bar{A}_{\Phi,i}^\pi(s,a_i)
-
\big(
\bar{\Phi}^{\pi_{-i}}(s,a_i)
-\bar{\Phi}^{\pi}(s)
\big)
\right|
\le
\frac{4nq\gamma\delta}{1-\gamma}.
\end{align}

Moreover, as we showed in \eqref{eq:stastic-potential-identity}, for every
$a_i'\in\mathcal A_i$ and $a_{-i}\in\mathcal A_{-i}$,
\[
c_i(s,a_i,a_{-i})-c_i(s,a_i',a_{-i})
=
\Phi(s,a_i,a_{-i})-\Phi(s,a_i',a_{-i}).
\]
Taking expectation over
$A_{-i}\!\sim\!\pi_{-i}(\cdot\mid s)$ and then averaging the second action
$A_i'\!\sim\!\pi_i(\cdot\mid s)$ gives
\[
\mathbb E_{A_{-i}\sim\pi_{-i}(\cdot\mid s)}
\left[c_i(s,a_i,A_{-i})\right]
-
\mathbb E_{A\sim\pi(\cdot\mid s)}
\left[c_i(s,A)\right]
=
\bar\Phi^{\pi_{-i}}(s,a_i)-\bar\Phi^\pi(s).
\]
Thus, by \eqref{eq:marginalized-both-costs},
\[
\bar c_i^{\pi_{-i}}(s,a_i)-\bar c_i^\pi(s)
=
\bar\Phi^{\pi_{-i}}(s,a_i)-\bar\Phi^\pi(s).
\]
Substituting this identity into \eqref{eq:margin-Phi-c-advantage} gives
\eqref{eq:surrogate-mariginalized-potential} and completes the proof.
\end{proof}

\subsection{Proof of Lemma \ref{lem:imcg-potential-oracles}}\label{app:one-sample-episodic-IMCG-sample}
\begin{proof}
For part~(i), at the $r$th visit to $s$ in episode $t$, the potential advantage estimator
defined in \eqref{eq:hat-A-Phi-episodic}, with raw sample
$g_i^{\tau_{t,r}(s)}(s,A^{\tau_{t,r}(s)})
=c_i(s,A^{\tau_{t,r}(s)})$, satisfies
\begin{align*}
&\mathbb E_{\tau_{t,r}(s)}
\left[
\widehat A_{\Phi,i}^{\tau_{t,r}(s)}(s,a_i)
\right]=
\mathbb E_{\tau_{t,r}(s)}
\bigg[
c_i\bigl(s,A^{\tau_{t,r}(s)}\bigr)
\Big(
\frac{\mathbf 1\{A_i^{\tau_{t,r}(s)}=a_i\}}
{\pi_i^t(a_i\mid s)}-1
\Big)
\bigg]
\\
&=
\mathbb E_{A_{-i}\sim\pi_{-i}^t(\cdot\mid s)}
\left[c_i(s,a_i,A_{-i})\right]
-
\mathbb E_{A\sim\pi^t(\cdot\mid s)}
\left[c_i(s,A)\right]=\bar c_i^{\pi_{-i}^t}(s,a_i)-\bar c_i^{\pi^t}(s).
\end{align*}
Hence, by
Lemma~\ref{thm:one-stage},
\[
\left|
\mathbb E_{\tau_{t,r}(s)}
\left[
\widehat A_{\Phi,i}^{\tau_{t,r}(s)}(s,a_i)
\right]
-
\bar A_{\Phi,i}^{\pi^t}(s,a_i)
\right|
\le
\frac{4nq\gamma\delta}{1-\gamma}.
\]
Taking the maximum over $a_i$ gives precisely the bias requirement in
Assumption~\ref{ass:episodic-potential-oracle}. Moreover,
$|g_i^{\tau_{t,r}(s)}(s,A^{\tau_{t,r}(s)})|\le1$ almost surely since
$c_i\in[0,1]$. This proves part~(i).

For part~(ii), at global time $t=\tau_k(s)$, using the same raw sample
$g_i^t(s,A^t)=c_i(s,A^t)$ in the estimator defined for the fully online
setting gives, by the same importance-weighting calculation,
\begin{align*}
\mathbb E_{\tau_k(s)}
\left[
\widehat A_{\Phi,i}^k(s,a_i)
\right]
&=
\mathbb E_{A_{-i}\sim\pi_{-i}^t(\cdot\mid s)}
\left[c_i(s,a_i,A_{-i})\right]
-
\mathbb E_{A\sim\pi^t(\cdot\mid s)}
\left[c_i(s,A)\right]=
\bar c_i^{\pi_{-i}^t}(s,a_i)-\bar c_i^{\pi^t}(s).
\end{align*}
Therefore, Lemma~\ref{thm:one-stage} again yields
\[
\left\|
\mathbb E_{\tau_k(s)}
\left[
\widehat A_{\Phi,i}^k(s,\cdot)
\right]
-
\bar A_{\Phi,i}^{\pi^t}(s,\cdot)
\right\|_\infty
\le
\frac{4nq\gamma\delta}{1-\gamma}.
\]
Since $|g_i^t(s,A^t)|\le1$ almost surely, this is exactly
Assumption~\ref{ass:online-potential-oracle}, proving part~(ii).
\end{proof}

\end{document}